\documentclass{article}

\usepackage{microtype}
\usepackage{graphicx}
\usepackage{subfigure}
\usepackage{booktabs} 

\usepackage{aliascnt}
\usepackage{cite}
\usepackage{wrapfig}
\usepackage{gensymb}

\usepackage[table]{xcolor}

\usepackage{multirow}
\usepackage{enumitem}
\usepackage{pifont}

\usepackage{hyperref}

\usepackage{amsmath,amssymb,amsfonts,amsthm}
\usepackage{algorithmic}
\usepackage{graphicx}
\usepackage{textcomp}
\usepackage{caption}
\usepackage{float}
\usepackage[utf8]{inputenc}
\usepackage{tabularx}
\usepackage{xcolor}
\usepackage{booktabs}
\usepackage{cases}
\usepackage{subfigure}
\usepackage{bm}
\usepackage{graphicx}
\usepackage{mathrsfs}
\usepackage{enumitem}

\usepackage[accepted]{icml2026}

\usepackage{mathtools}

\usepackage[capitalize,noabbrev]{cleveref}

\theoremstyle{plain}
\newtheorem{theorem}{Theorem}

\newtheorem{lemma}{Lemma}
\newtheorem{corollary}{Corollary}
\newtheorem{assumption}{Assumption}

\theoremstyle{definition}

\theoremstyle{remark}
\newtheorem{remark}{Remark}

\definecolor{mygray}{gray}{0.92} 

\usepackage[textsize=tiny]{todonotes}

\icmltitlerunning{Federated Bilevel Performative Prediction}

\begin{document}

\twocolumn[
\icmltitle{Federated Bilevel Performative Prediction}




\begin{icmlauthorlist}
\icmlauthor{Liangxin Qian}{ntu}
\icmlauthor{Chang Liu}{zju}
\icmlauthor{Xuanyu Cao}{wsu}
\icmlauthor{Jun Zhao}{ntu}
\icmlauthor{Kwok-Yan Lam}{ntu}
\end{icmlauthorlist}

\icmlaffiliation{ntu}{College of Computing and Data Science, Nanyang Technological University, 639798, Singapore}
\icmlaffiliation{zju}{Innovation and Management Center, School of Software Technology (Ningbo), Zhejiang University, 315100, China}
\icmlaffiliation{wsu}{School of Electrical Engineering \& Computer Science, Washington State University, Pullman, WA 99164, USA}



\icmlcorrespondingauthor{Chang Liu}{chang.liu@zju.edu.cn}

\icmlkeywords{Machine Learning, ICML}

\vskip 0.3in
]



\printAffiliationsAndNotice{}  

\begin{abstract}
Federated bilevel optimization is widely used for nested learning problems across distributed clients, such as federated hyperparameter tuning and meta-learning under privacy and communication constraints. Most existing formulations assume fixed client data distributions, which can be violated by performativity, where deployed decisions reshape client behavior and data collection, inducing client-specific, decision-dependent distribution shift. We study federated bilevel performative prediction, where both upper-level (UL) and lower-level (LL) objectives are evaluated under client-dependent, decision-dependent distributions. We formalize the federated bilevel performatively stable (FBPS) point under a decoupled-risk perspective and provide sufficient conditions for its existence and uniqueness. We then develop two federated methods to compute the FBPS solution: FBi-RRM, which converges linearly under a contraction condition, and FBi-SGD, a communication-efficient stochastic method based on federated hypergradient estimation with convergence guarantees under diminishing step sizes when sensitivities are sufficiently small. Experiments on strategic regression and meta strategic classification validate the predicted stability thresholds and demonstrate improved meta-generalization over non-performative baselines, and CNN-based classification further demonstrates the practical effectiveness of the proposed methods in nonconvex neural network settings.

\end{abstract}

\section{Introduction}
\label{sec_intro}

Bilevel optimization provides a framework for learning problems with nested objectives, where an upper-level (UL) decision is optimized subject to a lower-level (LL) learning process. This structure underlies many applications, e.g., hyperparameter tuning \cite{bardenet2013collaborative, a2024hyperparameter, li2020system}, meta-learning \cite{franceschi2018bilevel, qin2023bi, kim2025stochastic}, and representation learning \cite{arora2020provable,huang2024multiplex}. Extending bilevel optimization to federated systems is substantially more challenging due to distributed and non-i.i.d.\ data, limited communication, partial participation, and client heterogeneity. Recent federated bilevel optimization (FBO) methods therefore emphasize communication-efficient hypergradient estimation and scalable updates with provable convergence, such as FedNest~\cite{tarzanagh2022fednest}, FedBiOAcc~\cite{li2023communication}, SimFBO/ShroFBO~\cite{yang2023simfbo}, FedMBO~\cite{huang2023achieving}, MemFBO~\cite{yang2025first}, and fully first-order federated stochastic bilevel schemes~\cite{zhang2025federated}.

Despite this progress, most existing FBO formulations implicitly assume that each client’s data distribution is fixed during training and evaluation. This assumption can be violated by performativity, which captures a feedback loop where model decisions influence the environment and reshape future data distributions. Perdomo \textit{et al}.~\cite{perdomo2020performative} formalized performative risk minimization and the equilibrium notion of performative stability, motivating repeated retraining as a mechanism to reach stable solutions. Subsequent work studied stochastic optimization dynamics under deployment-induced shifts~\cite{mendler2020stochastic}, extensions to stateful populations~\cite{brown2022performative}, and causal perspectives for anticipating performativity~\cite{mendler2022anticipating}. These results indicate that treating data as static can yield persistent distribution shift and unstable learning dynamics.

Performativity is also natural in federated settings: a shared model can affect user behavior, data collection, and downstream outcomes at each client, inducing decision-dependent client distributions. While performative effects have been explored for single-level federated objectives, such as performative FedAvg~\cite{jin2024performative}, the federated bilevel literature largely assumes fixed client distributions. Conversely, existing bilevel performative prediction studies are primarily centralized and do not address the communication-limited and heterogeneous federated setting. This gap motivates a unified formulation that combines bilevel structure with client-specific decision-dependent distribution shifts.

\begin{table*}[t]
\centering
\caption{Comparison of related works on performative prediction, bilevel optimization, and federated learning. $\varepsilon_{i,c}$ and $\varepsilon_{i,d}$ denote the Wasserstein-1 sensitivities of the UL and LL distribution maps $(C_i,D_i)$ for client $i$; $\bar{\varepsilon}$ is the averaged (single-level) sensitivity, and $\bar{\varepsilon}_c,\bar{\varepsilon}_d$ are the averaged UL/LL sensitivities. Condition: order of a sufficient bound that must be below a problem-dependent threshold to ensure the existence of performatively stable points. Gap: distance between performatively stable and optimal points. $r$ is the iteration index.}
\label{tab:pp_bilevel_comparison}
\small
\begin{tabular}{lccccc}
\toprule
Algorithms & Setting & Framework & Condition & Gap & Rate \\
\midrule
RRM and GD \cite{perdomo2020performative} & centralized & single-level & $\mathcal{O}(\varepsilon)$ & $\mathcal{O}(\varepsilon)$  & $\mathcal{O}(\log(1/r))$ \\
Bi-RRM \cite{lu2023bilevel}& centralized & bilevel & $\mathcal{O}(\varepsilon_c\varepsilon_d+\varepsilon_d)$ & $\mathcal{O}(\varepsilon_c(1+\varepsilon_d))$  & $\mathcal{O}(\log(1/r))$ \\
Bi-SGD \cite{lu2023bilevel}& centralized & bilevel & $\mathcal{O}(\varepsilon_c+\varepsilon_c\varepsilon_d+\varepsilon_d)$ & $\mathcal{O}(\varepsilon_c(1+\varepsilon_d))$  & $\mathcal{O}(1/r)$ \\
P-FedAvg \cite{jin2024performative} & federated & single-level & $\mathcal{O}(\bar{\varepsilon})$ & $\mathcal{O}(\bar{\varepsilon})$  & $\mathcal{O}(1/r)$ \\
\rowcolor{mygray}
FBi-RRM (This paper)& federated & bilevel & $\mathcal{O}(\bar{\varepsilon}_c\bar{\varepsilon}_d+\bar{\varepsilon}_d)$ & $\mathcal{O}(\bar{\varepsilon}_c(1+\bar{\varepsilon}_d))$  & $\mathcal{O}(\log(1/r))$ \\
\rowcolor{mygray}
FBi-SGD (This paper)& federated & bilevel & $\mathcal{O}(\bar{\varepsilon}_c+\bar{\varepsilon}_c\bar{\varepsilon}_d+\bar{\varepsilon}_d)$ & $\mathcal{O}(\bar{\varepsilon}_c(1+\bar{\varepsilon}_d))$  & $\mathcal{O}(1/r)$ \\
\bottomrule
\end{tabular}
\end{table*}

Federated bilevel performative prediction introduces coupled challenges beyond those in FBO or performative prediction alone. First, decision-dependent shifts can bias local gradients and hypergradients if treated as static, complicating UL updates. Second, heterogeneity arises in two intertwined forms, namely heterogeneous data and heterogeneous performative responses, which jointly affect stability under partial participation. Third, aggregation and redeployment are coupled because redeployment changes the distributions observed by clients, making the communication schedule part of the learning dynamics. These challenges call for new formulations and analyses that connect federated hypergradient estimation with performative stability. Table 1 presents representative works related to performative prediction, biblevel optimization, and federated learning.

\subsection{Main Contributions}
This paper addresses the key challenges of federated bilevel learning under decision-dependent distribution shifts by developing a unified framework and provably efficient algorithms for federated bilevel performative prediction. By explicitly incorporating client-specific performativity into both UL and LL objectives, our approach connects performative stability with communication-efficient hypergradient computation in heterogeneous federated systems. To the best of our knowledge, this is the first work to combine decision-dependent distribution shifts with federated bilevel optimization and provide convergence guarantees under coupled performativity. Our main contributions are:
\begin{itemize}[leftmargin=10pt]
    \item \textbf{Federated bilevel performative formulation.} We formulate federated bilevel optimization under client-specific decision-dependent distributions, capturing coupled performative shifts across UL and LL objectives. We introduce the federated bilevel performatively stable (FBPS) point under a decoupled-risk perspective and characterize its existence and uniqueness via a contraction condition.
    \item \textbf{Theoretical analysis.} We develop two federated algorithms tailored to performative shifts and provide their convergence guarantees. Specifically, we show that when the client-wise performative sensitivities are below an explicit threshold, the proposed update operator is contractive, implying the existence and uniqueness of a federated bilevel performatively stable (FBPS) point and linear convergence of FBi-RRM to this point. Building on this stability result, we further propose a communication-efficient stochastic method, FBi-SGD, based on federated hypergradient estimation, and prove that with diminishing step sizes and sufficiently small sensitivities, the iterates converge to the FBPS point with a rate of $\mathcal{O}(1/r)$.
    \item \textbf{Application to strategic learning.} We evaluate the proposed framework on quadratically regularized bilevel strategic regression and meta strategic classification, where the results validate the predicted stability thresholds and show that performativity-aware training improves meta-generalization over non-performative baselines. We further conduct CNN-based simulations on federated MNIST to empirically demonstrate the effectiveness of the proposed methods in nonconvex neural network settings.
\end{itemize}

\section{Federated Bilevel Performative Prediction}
We consider a federated learning system with a central server and $M$ clients indexed by $\mathcal{M} = \{1,2,\cdots,M\}$. Each client $i\in\mathcal{M}$ holds a private local dataset and is assigned a priority weight $p_i$ that reflects its importance (e.g., higher $p_i$ means greater influence). We assume $p_i \in [p_{\text{min}}, p_{\text{max}}]$ and $\sum_{i=1}^M p_i = 1$. This work focuses on a performative setting in which the deployed model influences LL behaviors, subsequently changing the data distribution observed at the UL. Consequently, both UL and LL learning can be decision-dependent through model-induced distribution shifts.

\subsection{Problem Formulation}
Motivated by applications, e.g., federated hyperparameter tuning\cite{khodak2021federated} and meta-learning\cite{liu2023federated},
we formulate the federated bilevel performative prediction problem as
\begin{subequations}\label{prob1}
\begin{align}
&\!\!\!\min\limits_{\bm{x}\in\mathbb{R}^{d_1}}  \mathcal{F}(\bm{x})=\sum_{i=1}^{M}p_i\underset{\xi\sim\mathcal{C}_i(\bm{y}^*(\bm{x}))}{\mathbb{E}} f_i (\bm{x},\bm{y}^*(\bm{x});\xi)
\label{prob1a}\\
&\text{s.t.} \quad 
\bm{y}^*(\bm{x})\in \arg\min_{\bm{y}\in\mathbb{R}^{d_2}} \sum_{i=1}^M p_i\underset{\zeta\sim\mathcal{D}_i(\bm{x})}{\mathbb{E}}
g_i(\bm{x},\bm{y};\zeta),\label{prob1b}
\end{align}
\end{subequations}
where $\bm{x}\in\mathbb{R}^{d_1}$ denotes the UL model parameters shared across clients, and $\bm{y}\in\mathbb{R}^{d_2}$ represents the LL parameters optimized to adapt to client-specific data. For a given $\bm{x}$, the LL problem returns the optimal response $\bm{y}^*(\bm{x})$.
$\mathcal{C}_i$ and $\mathcal{D}_i$ are UL and LL data distributions of the $i$-th client, respectively. $f_i (\bm{x},\bm{y}^*(\bm{x});\xi)$ and $g_i(\bm{x},\bm{y};\zeta)$ are the UL and LL loss functions of the $i$-th client, respectively.

\subsubsection{Federated Bilevel Performative Stable Point}
We next define a stability notion under a decoupled-risk perspective, extending the bilevel performative stability definition in prior centralized work \cite{lu2023bilevel} to the federated setting. The difference is that, in our model, the induced distributions and losses are client dependent and need to be aggregated through the federation weights $\{p_i\}_{i=1}^M$. A point $\bm{x}_s$ is said to be federated bilevel performatively stable (FBPS) if it satisfies:
\begin{equation}
    \bm{x}_s = \arg \min_{\bm{x}} \sum_{i=1}^{M}p_i \underset{\xi\sim\mathcal{C}_i(\bm{y}^*(\bm{x},\bm{x}_s))}{\mathbb{E}}f_i (\bm{x},\bm{y}^*(\bm{x},\bm{x}_s);\xi),
\end{equation}
where the inner solution $\bm{y}^*(\bm{x},\bm{x}_s)$ is given by
\begin{equation}
    \bm{y}^*(\bm{x},\bm{x}_s) = \arg \min_{\bm{y}} \sum_{i=1}^M p_i\underset{\zeta\sim\mathcal{D}_i(\bm{x}_s)}{\mathbb{E}}
g_i(\bm{x},\bm{y};\zeta).
\end{equation}
Intuitively, FBPS characterizes a fixed-point type stability: when the data distribution is associated with $\bm{x}_s$, the resulting UL performative risk is minimized at $\bm{x}=\bm{x}_s$. This extension is essential in federated learning because distribution shift is observed locally at each client and cannot be evaluated centrally due to privacy and communication constraints.

\subsubsection{Federated Bilevel Performative Optimal Point}
For completeness, we also define an optimality notion that directly minimizes the performative UL risk induced by the current decision. A point $\bm{x}_o$ is said to be federated bilevel performatively optimal (FBPO) if it the optimal solution to Problem (\ref{prob1}).
FBPO evaluates each candidate $\bm{x}$ under the distributions induced by deploying that same $\bm{x}$. As a result, the objective is fully coupled, since $\bm{x}$ affects the UL loss directly and also affects the UL sampling distributions through the LL response $\bm{y}^*(\bm{x})$. While FBPO can be viewed as the ideal performative optimum, computing it is typically more challenging in federated environments because it requires estimating the coupled impact of $\bm{x}$ on the LL response and the resulting UL risk on each client. Motivated by these considerations, this paper focuses on finding FBPS and developing federated algorithms and theory tailored to the stability formulation.

\subsection{Theoretical Assumptions}
Before presenting our main results, we introduce several standard regularity conditions that characterize the distribution shifts and the smoothness of the UL and LL objectives. These assumptions are commonly adopted in federated bilevel optimization and performative prediction \cite{lu2023bilevel, yang2023simfbo, perdomo2020performative}.



\begin{assumption}\label{assump_1}
    Assume that the distribution maps of the $i$-th client $\mathcal{C}_i$ and $\mathcal{D}_i$ are $\varepsilon_c$- and $\varepsilon_d$-sensitive, respectively, and denoted as $\mathcal{W}_1(\mathcal{C}_i(\bm{x}), \mathcal{C}_i(\bm{x}^\prime))\leq \varepsilon_c \|\bm{x}-\bm{x}^\prime\|_2$ and $\mathcal{W}_1(\mathcal{D}_i(\bm{y}), \mathcal{D}_i(\bm{y}^\prime))\leq \varepsilon_d \|\bm{y}-\bm{y}^\prime\|_2$, respectively. $\mathcal{W}_1(\cdot,\cdot)$ is the Wasserstein-1 distance between two distributions.
\end{assumption}
\begin{remark}
In general, the distribution maps may have client-specific sensitivities, i.e., $\mathcal{C}_i$ and $\mathcal{D}_i$ are $\varepsilon_{i,c}$- and $\varepsilon_{i,d}$-sensitive, respectively. Throughout the paper, we adopt the uniform bounds
$\varepsilon_c := \max_{i\in\mathcal{M}} \varepsilon_{i,c}$ and $\varepsilon_d := \max_{i\in\mathcal{M}} \varepsilon_{i,d}$, so that Assumption~\ref{assump_1} holds for all clients. This choice simplifies notation and yields worst-case sufficient conditions; the analysis can be tightened by tracking client-wise sensitivities and their weighted aggregates.
\end{remark}

\begin{assumption}\label{assump_2}
    Assume that the performative risk $\mathcal{F}_i(\bm{x})$ and the loss function $g_i(\bm{x})$ of each client are $\gamma_f$ and $\gamma_g$ strongly convex, respectively.
\end{assumption}

\begin{assumption}\label{assump_3}
Assume that
    $\nabla_{\xi} f_i$ is $L_f^{\xi}$-Lipschitz in $\xi$, $\forall \bm{x},\bm{y} $, and $\nabla_{\zeta} g_i$ is $L_g^{\zeta}$-Lipschitz in $\zeta$, $\forall \bm{x},\bm{y}$.
    $\nabla_x f_i$ is $L_f^{x}$-Lipschitz in $\bm{x}$ and $\bar L_f^{x}$-Lipschitz in $\bm{y}$; $\nabla_y f_i$ is $L_f^{y}$-Lipschitz in $\bm{x}$ and $\bar L_f^{y}$-Lipschitz in $\bm{y}$.
    $g_i$ is $L_g^{x}$-smooth in $\bm{x}$ and $L_g^{y}$-smooth in $\bm{y}$.
    $f_i$ and $g_i$ are twice continuously differentiable, and $\nabla_{xy}^2 f_i,\nabla_{yy}^2 f_i$ are Lipschitz with constants $L_{fxy}^{\xi},L_{fyy}^{\xi}$ in $\xi$ and $L_{fxy}^{x},L_{fyy}^{x}$ in $\bm{x}$ and $L_{fxy}^{y},L_{fyy}^{y}$ in $\bm{y}$. Similarly, $\nabla_{xy}^2 g_i,\nabla_{yy}^2 g_i$ are Lipschitz with constants $L_{gxy}^{\zeta},L_{gyy}^{\zeta}$ in $\zeta$ and $L_{gxy}^{x},L_{gyy}^{x}$ in $\bm{x}$ and $L_{gxy}^{y},L_{gyy}^{y}$ in $\bm{y}$.
    $g_i$ is three times continuously differentiable, and $\nabla_{yyx}^3 g_i$ is Lipschitz with constants $L_{gyyx}^{\zeta}$ in $\zeta$, $L_{gyyx}^{x}$ in $\bm{x}$, and $L_{gyyx}^{y}$ in $\bm{y}$.
\end{assumption}
 
\begin{assumption}\label{assump_4}
    Assume that $\forall \bm{x},\bm{y}$, $\|\nabla_x f_i(\bm{x},\bm{y})\|\leq C_{f}^x$, $\|\nabla_y f_i(\bm{x},\bm{y})\|\leq C_{f}^y$, $\|\nabla_{xy}^2 g_i(\bm{x},\bm{y})\| \leq C_{g}^{xy}$, $\|\nabla_{yy}^2 g_i(\bm{x},\bm{y})\| \leq C_{g}^{yy}$, and $\|\nabla_{yyx}^3 g_i(\bm{x},\bm{y})\| \leq C_{g}^{yyx}$.
\end{assumption}

\section{Existence and Uniqueness of FBPS Point}
In this section, we study the existence and uniqueness of an FBPS point for problem (\ref{prob1}). Our approach is to construct a federated repeated risk minimization procedure, referred to as FBi-RRM, whose fixed point coincides with an FBPS solution. We then show that the induced update operator is a contraction under mild conditions, which implies both the existence and uniqueness of the FBPS point and the linear convergence of the proposed method.
\subsection{Federated Bilevel Repeated Risk Minimization (FBi-RRM)}
Use $r$ to denote the index of global iterations, where $r = 0,1,2,\cdots,T$. In the FBi-RRM method, the server directly aggregates the client side solutions to update the global variables. For each client $i$, the local UL variable $\bm{x}_{i,r}$ is computed by solving a local bilevel problem:
\begin{align}
    \bm{x}_{i,r+1}= \arg \min_{\bm{\varphi}} \underset{\xi\sim\mathcal{C}_i(\bm{y}^*(\bm{\varphi},\bm{x}_{i,r}))}{\mathbb{E}}f_i (\bm{\varphi},\bm{y}^*(\bm{\varphi},\bm{x}_{i,r});\xi),
\end{align}
where the corresponding LL response mapping is 
\begin{align}
    \bm{y}^*(\bm{\varphi},\bm{x}_{i,r}) = \arg \min_{\bm{y}} \underset{\zeta\sim\mathcal{D}_i(\bm{x}_{i,r})}{\mathbb{E}}
g_i(\bm{\varphi},\bm{y};\zeta).
\end{align}

Specifically, we define the repeated risk minimization update at round $r+1$ as 
\begin{align}
    &\bm{x}_{r+1}=R(\bm{x}_r)=\sum_{i\in \mathcal{M}}p_i \bm{x}_{i,r+1},
\end{align}
where the operator $R(\cdot)$ maps the current iterate $\bm{x}_r$ to the aggregated minimizers obtained under the induced risks at the clients. Therefore, any fixed point $\bm{x} = R(\bm{x})$ corresponds to a stable point under the decoupled risk definition, which aligns with the FBPS notion introduced earlier.

\subsection{Convergence of FBi-RRM}
We next establish that the operator $R(\cdot)$ is contractive under Assumptions \ref{assump_1} to \ref{assump_4}. This property directly implies the existence of a unique fixed point and yields linear convergence of the iterates generated by FBi-RRM.

\begin{theorem}\label{theorem_1}
Under Assumptions \ref{assump_1}-\ref{assump_4}, the update operator $R(\cdot)$ satisfies
\begin{align}
    \|R(\bm{x})-R(\bm{x}^\prime)\|\leq \frac{\bar{L}_{F_c}^x(\varepsilon_c,\varepsilon_d)}{\gamma_f}\|\bm{x}-\bm{x}^\prime\|,\forall \bm{x},\bm{x}^\prime,
\end{align}    
where $\bar{L}_{F_c}^x(\varepsilon_c,\varepsilon_d)$ is a function of $(\varepsilon_c,\varepsilon_d)$, and its expression is 
\begin{align}
    &\bar{L}_{F_c}^x(\varepsilon_c,\varepsilon_d) = C_d\varepsilon_d + C_{cd}\varepsilon_c\varepsilon_d, 
\end{align}
where we define
\begin{align}
    C_d = &\frac{L_g^\zeta}{\gamma_g}\Bigg(
\bar{L}_f^x
+
C_g^{xy}\Bigg(
\frac{L_f^y}{\gamma_g}+\frac{C_f^y}{\gamma_g^2}L_{gyy}^y
\Bigg)
+\frac{C_f^y}{\gamma_g}L_{gxy}^y
\Bigg)\nonumber \\
&+\frac{C_f^y}{\gamma_g}L_{gxy}^y+\frac{C_f^y}{\gamma_g^2}L_{gyy}^\zeta,\\
C_{cd}=&
\frac{L_g^\zeta}{\gamma_g}\,L_f^\xi\Bigg(1+\frac{C_g^{xy}}{\gamma_g}\Bigg).
\end{align}

Under the condition $\frac{\bar{L}_{F_c}^x(\varepsilon_c,\varepsilon_d)}{\gamma_f}\leq 1$, the solution obtained by the FBi-RRM algorithm converges to a unique FBPS point $\bm{x}_s$ at a linear rate, i.e., 
\begin{align}
    \|\bm{x}_r-\bm{x}_s\|\leq \left(
    \frac{\bar{L}_{F_c}^x(\varepsilon_c,\varepsilon_d)}{\gamma_f}
    \right)^r\|\bm{x}_0-\bm{x}_s\|,
\end{align}
where $\bm{x}_0$ is a starting point.
\end{theorem}
This result shows that under Assumptions \ref{assump_1}-\ref{assump_4}, and the condition $\frac{\bar{L}_{F_c}^x(\varepsilon_c,\varepsilon_d)}{\gamma_f}\leq 1$, the FBPS point exists and is unique, since it coincides with the unique fixed point of $R(\cdot)$. In addition, the same condition guarantees that the proposed FBi-RRM method converges to this FBPS point at a linear rate.
\begin{remark}
    Our sufficient condition has the same linear-plus-bilinear structure as the Bi-RRM result in \cite{lu2023bilevel}, i.e., a linear term in the distribution-shift parameter and a coupled cross term. In our setting, the corresponding coefficients $C_d$ and $C_{cd}$ explicitly capture the impact of decision-dependent distributions through the sensitivity constants (e.g., $L_g^\zeta$, $L_f^\xi$) and the LL, UL conditioning (via $\gamma_g$, $\gamma_f$).
\end{remark} 
\begin{corollary}[Sufficient condition under heterogeneous performative sensitives]
    Under Assumptions \ref{assump_1}-\ref{assump_4}, if the client-wise sensitivities satisfy
    \begin{align}
        C_d\bar{\varepsilon}_d + C_{cd} \bar{\varepsilon}_c\bar{\varepsilon}_d<\gamma_f,
    \end{align}
    where $\bar{\varepsilon}_c:=\sum_{i\in \mathcal{M}}p_i\,\varepsilon_{i,c}$ and 
$\bar{\varepsilon}_d:=\sum_{i\in \mathcal{M}}p_i\,\varepsilon_{i,d}$, $\varepsilon_{i,c}$ and $\varepsilon_{i,d}$ are sensitive parameters of the client $i$'s distribution mappings $\mathcal{C}_i$ and $\mathcal{D}_i$, then the update operator $R(\cdot)$ is contractive, and the FBi-RRM iterates converge linearly to the unique FBPS point.
\end{corollary}
\begin{remark}
    The condition depends on heterogeneous performative effects only through the weighted aggregates $\bar{\varepsilon}_c$ and $\bar{\varepsilon}_d$. In particular, larger weights $p_i$ and larger sensitivities increase $\bar{\varepsilon}_c$ and $\bar{\varepsilon}_d$, thereby tightening the contraction requirement. When $\varepsilon_{i,c}= \varepsilon_c$ and $\varepsilon_{i,d}= \varepsilon_d$ for all client $i$, the condition reduces to Theorem~\ref{theorem_1}.
\end{remark}

\section{Federated Bilevel Performative Prediction Algorithm}
In FBi-RRM, the optimal solutions to the LL and UL problems are computed, which is computationally demanding in general. In this section, we propose an efficient gradient-based stochastic bilevel algorithm to find the FBPS point.

\subsection{Preliminaries on Federated Bilevel Optimization}
The update rule of FBi-RRM does not require explicitly computing the global hypergradient $\nabla\mathcal{F}(\bm{x})$. In contrast, gradient-based bilevel methods update the UL variable $\bm{x}$ via $\nabla\mathcal{F}(\bm{x})$, which typically contains an inverse-Hessian-vector-product (IHVP) term induced by the LL problem, making hypergradient estimation nontrivial in federated systems. In particular, while the global LL Hessian aggregates client-side Hessians, its inverse is not decomposable, and thus averaging client-wise IHVPs generally does not yield the global IHVP.\footnote{$\big(\sum_{i=1}^M p_i \nabla_{yy}^2 g_i(\bm{x},\bm{y})\big)^{-1}\neq \sum_{i=1}^M p_i\big(\nabla_{yy}^2 g_i(\bm{x},\bm{y})\big)^{-1}$.} Following \cite{yang2023simfbo, li2023communication}, we introduce an auxiliary variable $\bm{v}$ to represent the IHVP by solving an equivalent strongly convex quadratic subproblem, which replaces the non-decomposable inverse with a decomposable objective amenable to local updates and server aggregation. Specifically, we solve the following federated quadratic optimization problem:
\begin{align}
&\min_{\bm{v}\in\mathbb{R}^{d_2}}l(\bm{x},\bm{y},\bm{v})\nonumber \\
    &=\sum_{i=1}^M p_i\left(\frac{1}{2}\bm{v}^\intercal\nabla_{yy}^2g_i(\bm{x},\bm{y})\bm{v} - \bm{v}^\intercal\nabla_y f_i(\bm{x},\bm{y})\right),
\end{align}
and denote its minimizer by $\bm{v}^*$. With $\bm{v}^*$, the hyper-gradient $\nabla\mathcal{F}(\bm{x})$ admits a federated-friendly linear aggregation form:
\begin{align}
    \nabla\mathcal{F}(\bm{x})=\sum_{i=1}^Mp_i\left(
    \nabla_xf_i(\bm{x},\bm{y})-\nabla_{xy}^2g_i(\bm{x},\bm{y})\bm{v}^*\right)
    .
\end{align}
Here, we define the part of the UL objective function coming from client $i$ as 
\begin{align}
\mathcal{F}_i(\bm{x})=\underset{\xi\sim\mathcal{C}_i(\bm{y}^*(\bm{x}))}{\mathbb{E}} f_i (\bm{x},\bm{y}^*(\bm{x});\xi),
\end{align}
so that $\mathcal{F}(\bm{x})=\sum_{i=1}^M p_i\mathcal{F}_i(\bm{x})$. We then form the client $i$'s local contribution to the hyper gradient as
\begin{align}
    \nabla\mathcal{F}_i(\bm{x}) := \nabla_xf_i(\bm{x},\bm{y})-\nabla_{xy}^2g_i(\bm{x},\bm{y})\bm{v}_i^*,
\end{align}
where $\bm{v}_i^*=\left[\nabla_{yy}^2 g_i(\bm{x},\bm{y}^*(\bm{x}))\right]^{-1}\nabla_y f_i (\bm{x},\bm{y}^*(\bm{x}))$ and is obtained by solving the local quadratic optimization problem:
\begin{align}
    \min_{\bm{v}_i}l_i(\bm{x},\bm{y},\bm{v}_i)=\frac{1}{2}\bm{v}_i^\intercal\nabla_{yy}^2g_i(\bm{x},\bm{y})\bm{v}_i - \bm{v}_i^\intercal\nabla_y f_i(\bm{x},\bm{y}).
\end{align}
Based on the hypergradient computation in federated bilevel optimization, we then apply it in the performative prediction model.

\subsection{Federated Bilevel Stochastic Gradient Descent (FBi-SGD)}
We next introduce a communication-efficient stochastic bilevel algorithm. 
At the $r$-th communication round, the server first samples a subset $C_r \subseteq \mathcal{M}$ of participating clients without replacement, and then active clients run several local SGD steps to update $(\bm{x}_i, \bm{y}_i,\bm{v}_i)$, and the server aggregates the accumulated gradients to update the global variables. We use $\tilde{p}_i:=\frac{M}{|C_r|}p_i$ to denote the effective aggregation weight of client $i$ such that $\mathbb{E}(\sum_{i\in C_r}\tilde{p}_i)=1$. \newline
\textbf{Initialization.}
The server initializes the global variables $\bm{x}_{0}$, $\bm{y}_{0}$, $\bm{v}_{0}$, and broadcasts them to all clients. At the beginning of round $r$, each participating client $i\in C_r$ sets
$\bm{x}_{i,r,0}=\bm{x}_r$, $\bm{y}_{i,r,0}=\bm{y}_r$, and $\bm{v}_{i,r,0}=\bm{v}_r$.

\textbf{Local update.}
At the $r$-th communication round, each active client $i\in C_r$ performs $K_r$ local iterations to update its local variables $\bm{x}_i, \bm{y}_i$, and $\bm{v}_i$. At the $k$-th local step, the updates are
\begin{align}
    &\bm{y}_{i,r,k+1}\! =\! \bm{y}_{i,r,k}-\eta_y^{(c)}\widehat{\nabla}_yg_i(\bm{x}_{i,r,k},\bm{y}_{i,r,k};\zeta), \\
    &\bm{v}_{i,r,k+1} \!=\! \bm{v}_{i,r,k}-\eta_v^{(c)}\widehat{\nabla}_vl_i(\bm{x}_{i,r,k},\bm{y}_{i,r,k},\bm{v}_{i,r,k};\psi), \\
    &\bm{x}_{i,r,k+1} \!=\! \bm{x}_{i,r,k}-\eta_x^{(c)}\widehat{\nabla}_x \mathcal{F}_i(\bm{x}_{i,r,k},\bm{y}_{i,r,k},\bm{v}_{i,r,k};\psi),
\end{align}
where $\psi=(\xi,\zeta)$ collects the data randomness in the UL and LL (i.e., $\xi$ for $f_i$ and $\zeta$ for $g_i$). The terms
$\widehat{\nabla}_yg_i(\cdot)$, $\widehat{\nabla}_vl_i(\cdot)$, $\widehat{\nabla}_x\mathcal{F}_i(\cdot)$ are stochastic gradient estimators of $\nabla_yg_i(\cdot)$, $\nabla_vl_i(\cdot)$, and $\nabla_x\mathcal{F}_i(\cdot)$, respectively, computed from a minibatch of data samples, respectively.
$\eta_y^{(c)},\eta_v^{(c)}$, and $\eta_x^{(c)}$ are local stepsizes for updating $\bm{y}_i,\bm{v}_i,\bm{x}_i$, respectively.
The superscript ``$(c)$'' means the client-side.
\newline
\textbf{Client and server-side aggregation.}
Once local updates are finished, the active $i$-th client aggregates all the local gradients, and then communicates the aggregated $\nabla^{(c)}_{y,i,r}$, $\nabla^{(c)}_{v,i,r}$, and $\nabla^{(c)}_{x,i,r}$ to the server. Then, the server further aggregates them to obtain $\nabla^{(s)}_{y,r}$, $\nabla^{(s)}_{v,r}$, and $\nabla^{(s)}_{x,r}$, where the superscript ``$(s)$'' means the server-side. The updating process at this phase is presented as 
\begin{align}
&\nabla^{(s)}_{y,r}=
\sum_{i\in C_r}\tilde{p}_i
\underbrace{\sum_{k=0}^{K_r-1}\widehat{\nabla}_yg_i(\bm{x}_{i,r,k},\bm{y}_{i,r,k};\zeta)}_{:=\nabla^{(c)}_{y,i,r}},
\end{align}
\begin{align}
&\nabla^{(s)}_{v,r}=
\sum_{i\in C_r}\tilde{p}_i
\underbrace{\sum_{k=0}^{K_r-1}\widehat{\nabla}_vl_i(\bm{x}_{i,r,k},\bm{y}_{i,r,k},\bm{v}_{i,r,k};\psi)}_{:=\nabla^{(c)}_{v,i,r}},
\end{align}
\begin{align}
&\nabla^{(s)}_{x,r}=
\sum_{i\in C_r}\tilde{p}_i
\underbrace{\sum_{k=0}^{K_r-1}\widehat{\nabla}_x \mathcal{F}_i(\bm{x}_{i,r,k},\bm{y}_{i,r,k},\bm{v}_{i,r,k};\psi)}_{:=\nabla^{(c)}_{x,i,r}}.
\end{align}

\textbf{Server-side updates.} Given the aggregated gradients, the server updates $(\bm{y}_{r+1},\bm{v}_{r+1},\bm{x}_{r+1})$ as
\begin{align}
    &\bm{y}_{r+1}=\bm{y}_{r}-\eta^{(s)}_y\nabla^{(s)}_{y,r},\\
    &\bm{v}_{r+1}=\mathcal{P}_\iota(\bm{v}_r-\eta^{(s)}_v\nabla^{(s)}_{v,r}),\\
    &\bm{x}_{r+1}=\bm{x}_{r}-\eta^{(s)}_x\nabla^{(s)}_{x,r},
\end{align}
where $\eta^{(s)}_y, \eta^{(s)}_v, \eta^{(s)}_x$ are server-side stepsizes for $\bm{y}, \bm{v}, \bm{x}$, respectively, and $\mathcal{P}_\iota(\bm{v}):=\min \{1,\frac{\iota}{\|\bm{v}\|}\}\bm{v}$ denotes the projection onto a bounded ball with a radius of $\iota$. Before giving the convergence analysis of the FBi-SGD algorithm, we present some theoretical assumptions to facilitate the subsequent statistical analysis.
\begin{assumption}\label{assump_5}
    Assume that the LL gradient estimate is unbiased and with bounded variance and the UL gradient estimate is biased and with bounded variance.
\end{assumption}
\begin{remark}
    Assumption \ref{assump_5} specifies standard stochastic approximation conditions for the gradient and Hessian estimators used in FBi-SGD.
\end{remark}

\begin{figure*}[t]
\subfigure[FBi-RRM ($\varepsilon_c=\varepsilon_d$).]{\includegraphics[width=.33\textwidth]{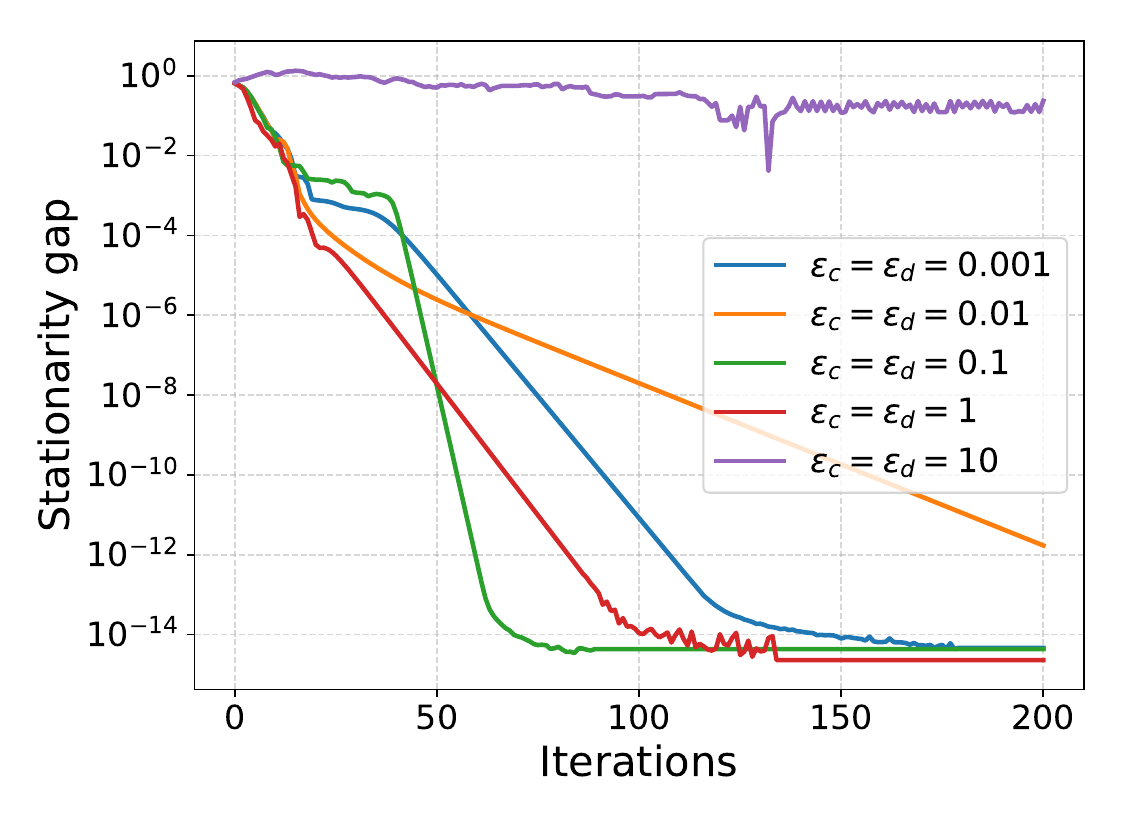}\label{fig.gap_toy_rrm_eps}}
\subfigure[FBi-RRM ($\varepsilon_c\neq\varepsilon_d$).]{\includegraphics[width=.33\textwidth]{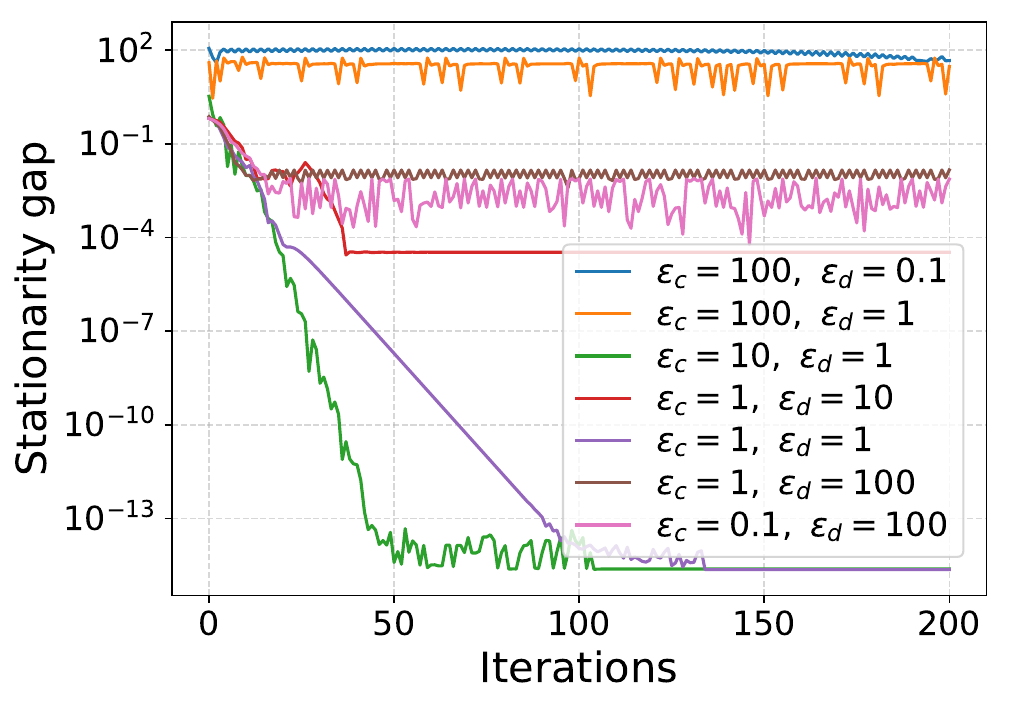}\label{fig.gap_toy_rrm_epspair}}
\subfigure[FBi-SGD ($\varepsilon_c=\varepsilon_d$).]{\includegraphics[width=.33\textwidth]{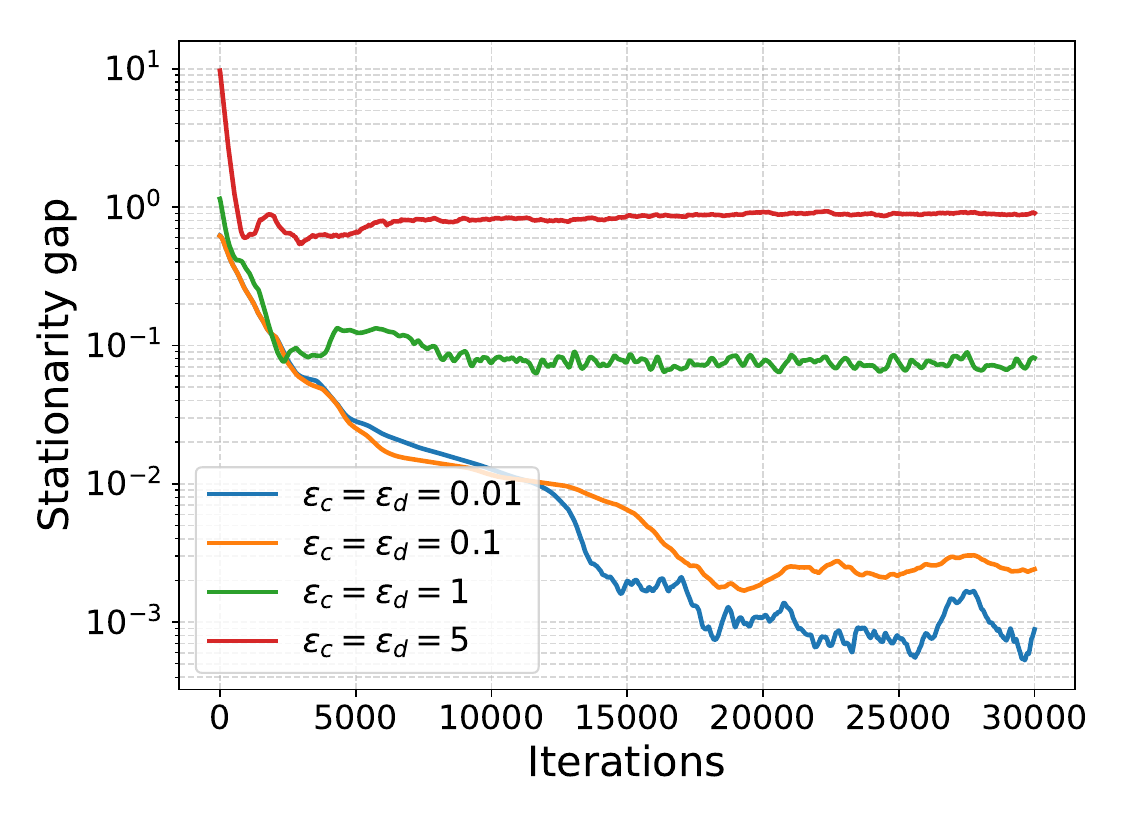}\label{fig.gap_toy_sgd}}
\caption{Performance comparisons of FBi-RRM and FBi-SGD methods under different $\varepsilon$ and ($\varepsilon_c,\varepsilon_d$) pairs.} 
\end{figure*}

\subsection{Convergence of FBi-SGD}
We then present the convergence guarantee of FBi-SGD. The main difficulty is that the UL update depends on approximate LL solutions and the auxiliary variable $\bm{v}$, both computed from noisy local information. Assumption \ref{assump_5} ensures that the resulting stochastic errors are well controlled. With diminishing stepsizes and sufficiently small $\varepsilon_c$ and $\varepsilon_d$, the iterates converge in mean square and almost surely.
\begin{theorem}\label{theorem_2}
Under Assumptions \ref{assump_1}-\ref{assump_5}, if the step sizes are chosen as 
$\eta_x^{(s)} = \eta_y^{(s)} = \eta_v^{(s)} = \mathcal{O}\left(\frac{1}{r}\right)$, $\eta_x^{(c)} = \eta_y^{(c)} = \eta_v^{(c)} = \mathcal{O}\left(\frac{1}{r}\right)$,
and the values of $\varepsilon_c$ and $\varepsilon_d$ satisfy $A_c \varepsilon_c + A_d \varepsilon_d + A_{cd}\varepsilon_c \varepsilon_d\leq 1$, where $A_c, A_d, A_{cd}$ are positive constants,
we have
\begin{align}
    &\mathbb{E}\|\bm{x}_r-\bm{x}_{s}\|^2 + \mathbb{E}\|\bm{y}_r-\bm{y}^*(\bm{x}_s)\|^2\nonumber \\
    &+ \mathbb{E}\|\bm{v}_r-\bm{v}^*(\bm{x}_s,\bm{y}^*(\bm{x}_s))\|^2 
     = \mathcal{O}\left(
    \frac{1}{r}
    \right),
\end{align}
and 
\begin{align}
    &\lim_{r\rightarrow \infty}\|\bm{x}_r - \bm{x}_s\|\rightarrow 0, 
    \lim_{r\rightarrow \infty}\|\bm{y}_r - \bm{y}^*({\bm{x}_s})\|\rightarrow 0,\nonumber \\
    &
    \lim_{r\rightarrow \infty}\|\bm{v}_r - \bm{v}^*({\bm{x}_s})\|\rightarrow 0 
    \quad \text{almost surely},
\end{align}
Detailed conditions of $\varepsilon_c,\varepsilon_d$ can be found in Appendix \ref{sec.sgd_whole_convergence}.
\end{theorem}
\begin{remark}
    Compared with the Bi-SGD analysis in \cite{lu2023bilevel}, the resulting conditions are more involved because FBi-SGD additionally needs to control stochastic local estimation errors and the auxiliary variable $\bm{v}$ across communication rounds. Nevertheless, the admissible requirements on $\varepsilon_c$ and $\varepsilon_d$ remain of the same order as those in Bi-SGD.
\end{remark}
\begin{remark}
    Under heterogeneous performative sensitives, the sufficient condition in Theorem \ref{theorem_2} becomes $A_c \bar{\varepsilon}_c + A_d \bar{\varepsilon}_d + A_{cd}\bar{\varepsilon}_c \bar{\varepsilon}_d\leq 1$.
\end{remark}
We next present a global convergence result for the deterministic variant of our algorithm (FBi-GD), which can be viewed as the full-gradient counterpart of FBi-SGD. 
\begin{corollary}[Convergence of FBi-GD]
Under Assumptions \ref{assump_1}-\ref{assump_5}, if the step sizes are chosen as $\eta_x^{(s)} = \eta_y^{(s)} = \eta_v^{(s)} = \mathcal{O}\left(1\right)$, $\eta_x^{(c)} = \eta_y^{(c)} = \eta_v^{(c)} = \mathcal{O}\left(1\right)$, and the values of $\varepsilon_c$ and $\varepsilon_d$ are small enough, we have
\begin{align}
    \mathbb{E}[\mathcal{P}_{r+1}]&\leq \left(
    1- 
    \vartheta
    \right)\mathbb{E}[\mathcal{P}_r],\forall r,
\end{align}    
where $\mathcal{P}_r$ is the Lyapunov function and defined as $\mathcal{P}_r=\|\bm{x}_r-\bm{x}_{s}\|^2 + \|\bm{y}_r-\bm{y}^*(\bm{x}_s)\|^2 + \|\bm{v}_r-\bm{v}^*(\bm{x}_s,\bm{y}^*(\bm{x}_s))\|^2$, and the detailed expression of $\vartheta$ can be found in Appendix \ref{sec.sgd_whole_convergence}.
\end{corollary}
\begin{remark}
The above recursion indicates a linear convergence rate whenever $\vartheta\in(0,1)$. Intuitively, smaller $(\varepsilon_c,\varepsilon_d)$ reduce the mismatch induced by decision-dependent distributions and enlarge the admissible stepsize region, thereby strengthening the contraction factor. The same Lyapunov function $\mathcal{P}_r$ will also be used to characterize the convergence of the stochastic case, where additional variance terms appear.
\end{remark}

\subsection{Distance between FBPO and FBPS Points under Heterogeneous Sensitivities}



To quantify how client-wise heterogeneity in decision-dependent distributions affects the distance between the FBPO and FBPS points, we generalize the FBPO–FBPS distance bound to the case where each participating client $i\in C_r$ admits its own sensitivity parameters $(\varepsilon_{i,c},\varepsilon_{i,d})$. The following result shows that the distance between $\bm{x}_o$ and $\bm{x}_s$ can still be controlled, and the bound depends on the weighted aggregate of client sensitivities.

\begin{theorem}\label{thm:relation_hetero_eps}
Suppose Assumptions~\ref{assump_1}--\ref{assump_5} hold, and each active client $i\in C_r$ has its own performative sensitivities $(\varepsilon_{i,c},\varepsilon_{i,d})$, then the FBPO point $\bm{x}_o$ and the FBPS point $\bm{x}_s$ satisfy
\begin{equation}\label{eq:relation_hetero_bound}
\|\bm{x}_o-\bm{x}_s\|
\le
\frac{2L_f^\xi\,\bar{\varepsilon}_c\left(C_g^{xy}+L_g^\zeta\,\bar{\varepsilon}_d\right)}{\gamma_f\gamma_g},
\end{equation}
where
$\bar{\varepsilon}_c:=\sum_{i\in C_r}\tilde{p}_i\,\varepsilon_{i,c}$ and 
$\bar{\varepsilon}_d:=\sum_{i\in C_r}\tilde{p}_i\,\varepsilon_{i,d}$.
\end{theorem}

\begin{remark}
Theorem~\ref{thm:relation_hetero_eps} implies that heterogeneity affects the gap between the FBPO and FBPS points through the weighted averages $(\bar{\varepsilon}_c,\bar{\varepsilon}_d)$. Consequently, clients with larger participation weights $\tilde{p}_i$ and larger sensitivities dominate the bound, while low-weight or low-sensitivity clients have a limited impact. When $\varepsilon_{i,c}=\varepsilon_c$ and $\varepsilon_{i,d}=\varepsilon_d$ for all $i$, \eqref{eq:relation_hetero_bound} reduces to the homogeneous bound in \cite{lu2023bilevel}.
\end{remark}

\section{Numerical Experiments}
We conduct numerical experiments to evaluate the proposed FBi-RRM and FBi-SGD methods with the impact of decision-dependent distribution shifts. Following \cite{perdomo2020performative, lu2023bilevel}, we model the distribution shift as $\varepsilon_c \bm{y}^*(\bm{x}_r)$ at UL optimization and $\varepsilon_d \bm{x}_r$ at LL optimization in FBi-RRM, and for FBi-SGD, the UL's distribution shift is changed to $\varepsilon_c \bm{y}_r$.

\subsection{Quadratically Regularized Bilevel Strategic Regression}
We study a federated extension of the bilevel strategic regression toy example in \cite{lu2023bilevel}:
\begin{subequations}\label{prob2}
\begin{align}
&\min\limits_{\bm{x}\in\mathbb{R}^{d}}  \sum_{i=1}^{M}p_i\underset{(\bm{a}_{ij},\bm{b}_{ij})\sim\mathcal{C}_i(\bm{y}^*(\bm{x}))}{\mathbb{E}} f_i(\bm{x},\bm{y}^*(\bm{x}))
\\
& \text{s.t.} \,
\bm{y}^*(\bm{x})= \arg\min_{\bm{y}\in\mathbb{R}^{d}} \sum_{i=1}^M p_i\underset{(\bm{c}_{ij},\bm{d}_{ij})\sim\mathcal{D}_i(\bm{x})}{\mathbb{E}}
g_i(\bm{x},\bm{y}),
\end{align}
\end{subequations}
where $f_i(\bm{x},\bm{y}^*(\bm{x})) = \sum_{j=1}^{N}q_j \|\bm{b}_{ij} - \bm{a}_{ij}^\intercal \bm{x}\|^2 + \frac{\lambda_x}{2}\|\bm{x}-\bm{y}^*(\bm{x})\|^2$ and $g_i(\bm{x},\bm{y}) = \sum_{j=1}^{N}q_j\|\bm{d}_{ij}-\bm{c}_{ij}^\intercal\bm{y}\|^2 + \frac{\lambda_y}{2}\|\bm{y}-\bm{x}\|^2$.
We follow the same synthetic data generation protocol and parameter choices as~\cite{lu2023bilevel}, except that we distribute the data across $M=10$ clients and uniform aggregation weights $p_i=0.1$. 
Client $i$ holds local UL and LL datasets $\{(\bm{a}_{ij},\bm{b}_{ij})\}_{j=1}^{N}$ and $\{(\bm{c}_{ij},\bm{d}_{ij})\}_{j=1}^{N}$, respectively, generated from client-specific linear models with Gaussian features and noise. We introduce parameter heterogeneity by perturbing the underlying regression parameters across clients. 
Unless stated otherwise, we use $d=5$, $N=50$, $\lambda_x=\lambda_y=10^{-3}$. 
Complete simulation details can be found in Appendix~\ref{append_setting1}.

For the FBi-RRM method in Fig.~\ref{fig.gap_toy_rrm_eps}, when $\varepsilon_c=\varepsilon_d$ is small (up to $1$), the stationarity gap rapidly vanishes to numerical precision, whereas for large sensitivity (e.g., $10$) it stops decreasing and oscillates at a nontrivial level, showing loss of contraction. 
Fig.~\ref{fig.gap_toy_rrm_epspair} highlights the coupled role of ($\varepsilon_c,\varepsilon_d$). A large $\varepsilon_d$ can prevent further gap reduction even if $\varepsilon_c$ is very small, whereas overly large $\varepsilon_c$ leads to a persistent gap. This reflects that controlling only $\varepsilon_c$ or $\varepsilon_d$ is insufficient for convergence.
Finally, Fig.~\ref{fig.gap_toy_sgd} shows FBi-SGD is more sensitive: for small $(\varepsilon_c,\varepsilon_d)$ it converges only to a noise-dominated neighborhood, and this residual floor grows with sensitivity until the method effectively stagnates under strong performative feedback.

\begin{figure*}[t]
\subfigure[Stationarity gap.]{\includegraphics[width=.33\textwidth]{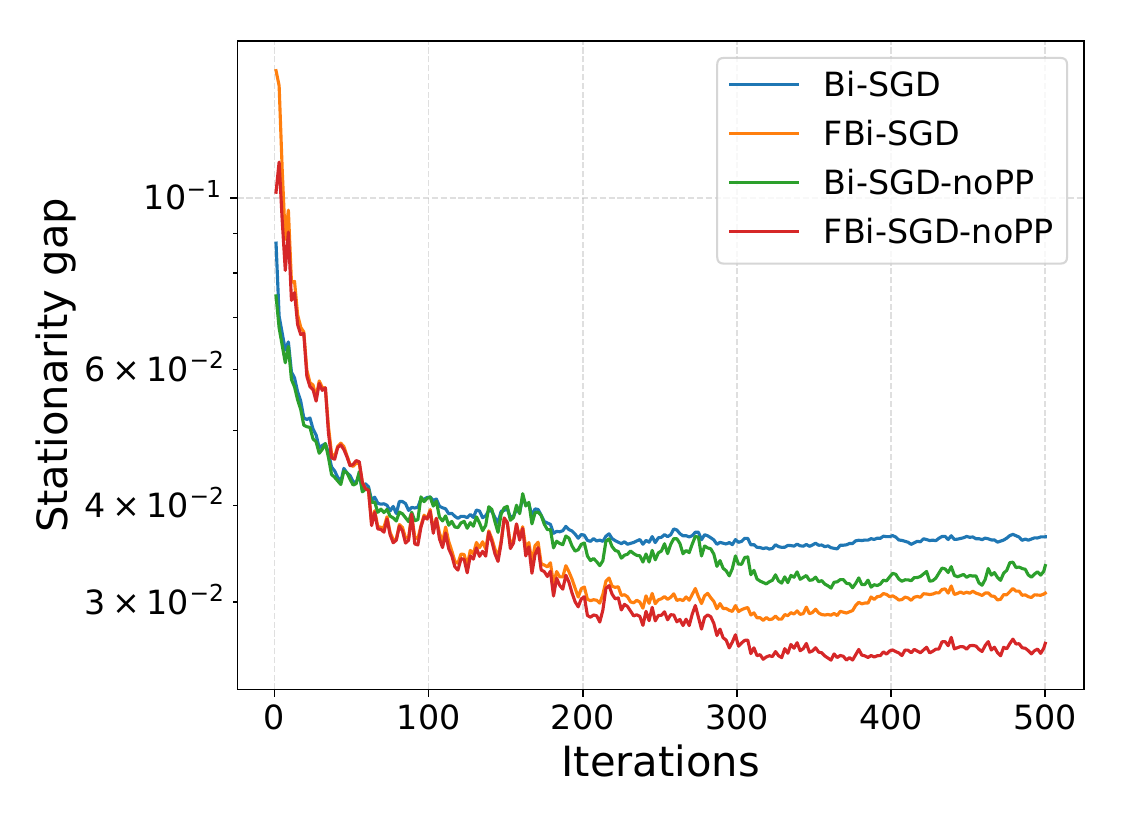}\label{fig.gap_comp_baseline}}
\subfigure[Train accuracy.]{\includegraphics[width=.33\textwidth]{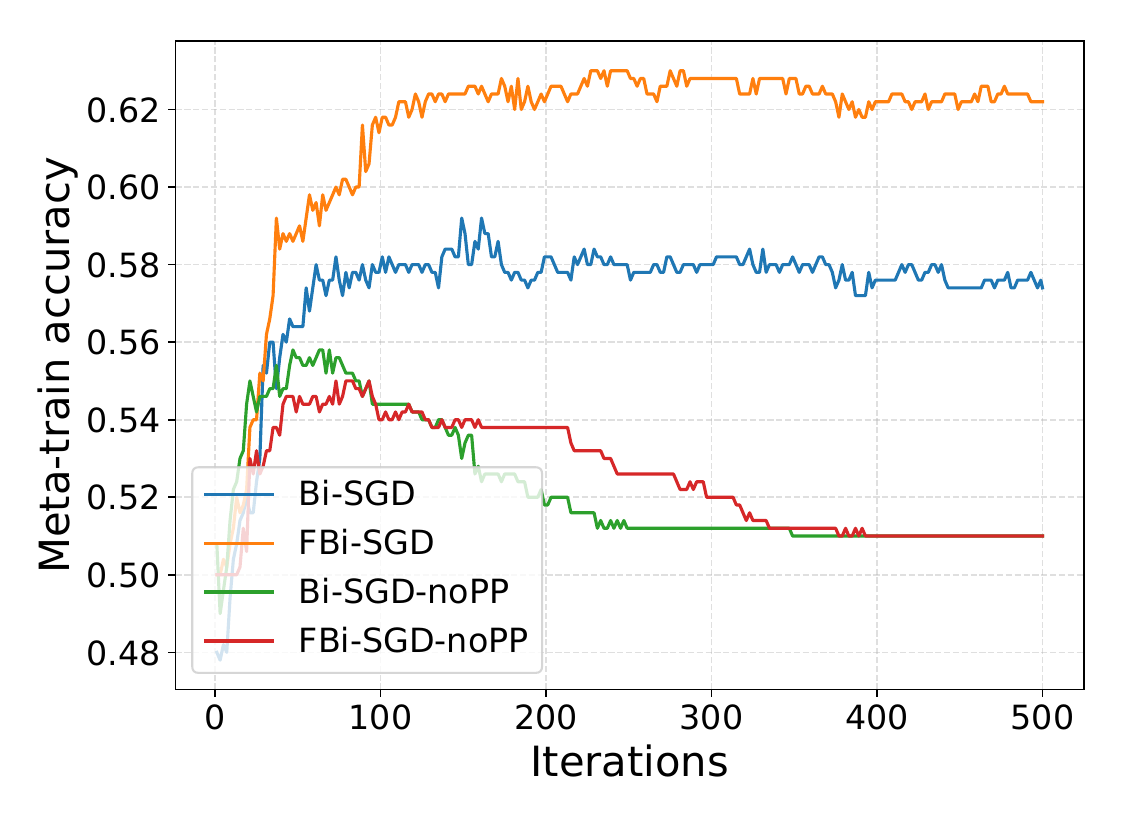}\label{fig.train_comp_baseline}}
\subfigure[Test accuracy.]{\includegraphics[width=.33\textwidth]{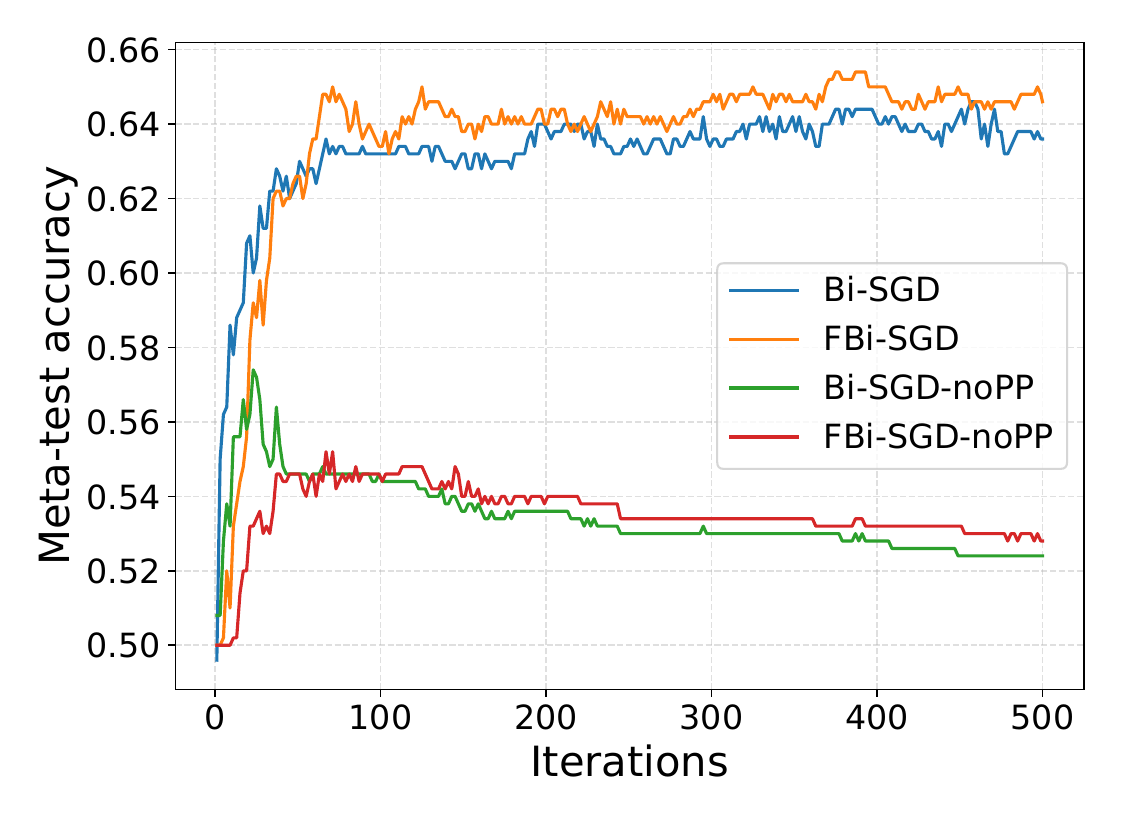}\label{fig.test_comp_baseline}}
\caption{Performance comparisons of FBi-SGD and Bi-SGD methods with and without performative prediction training.} 
\end{figure*}
\begin{figure*}[t]
\subfigure[Stationarity gap.]{\includegraphics[width=.33\textwidth]{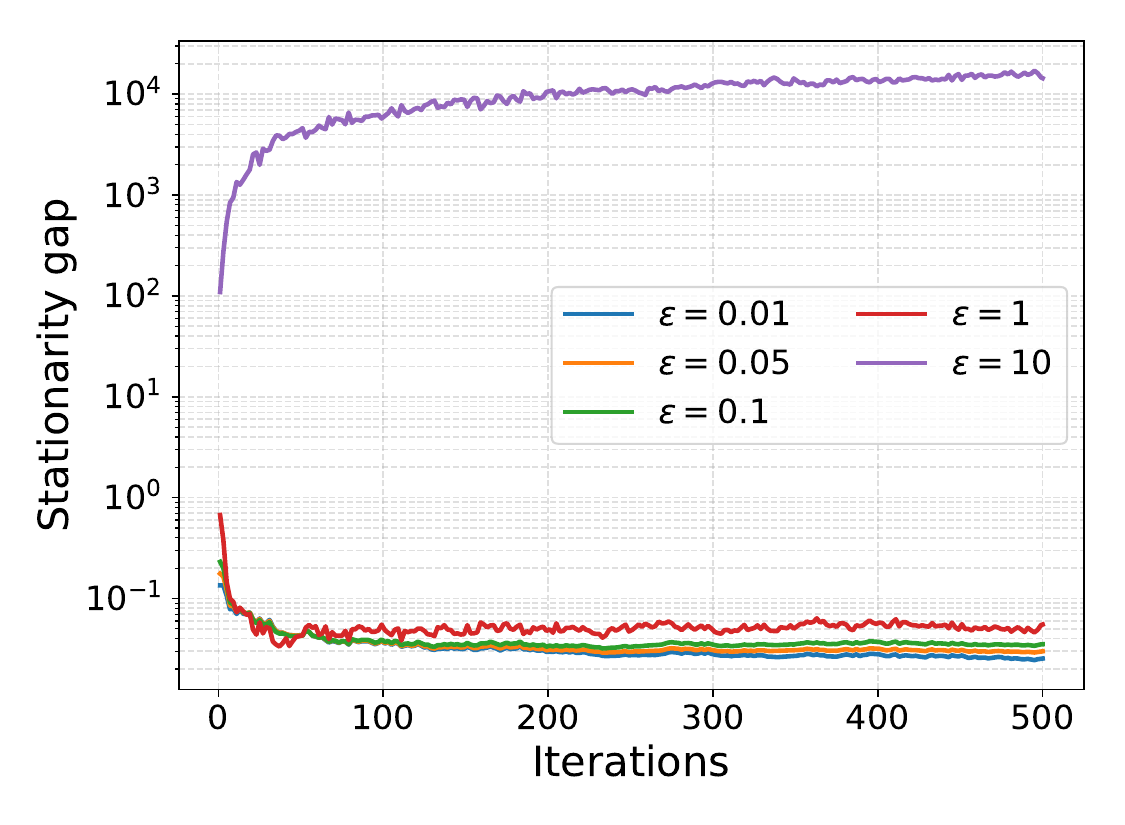}\label{fig.gap_comp_eps}}
\subfigure[Train accuracy.]{\includegraphics[width=.33\textwidth]{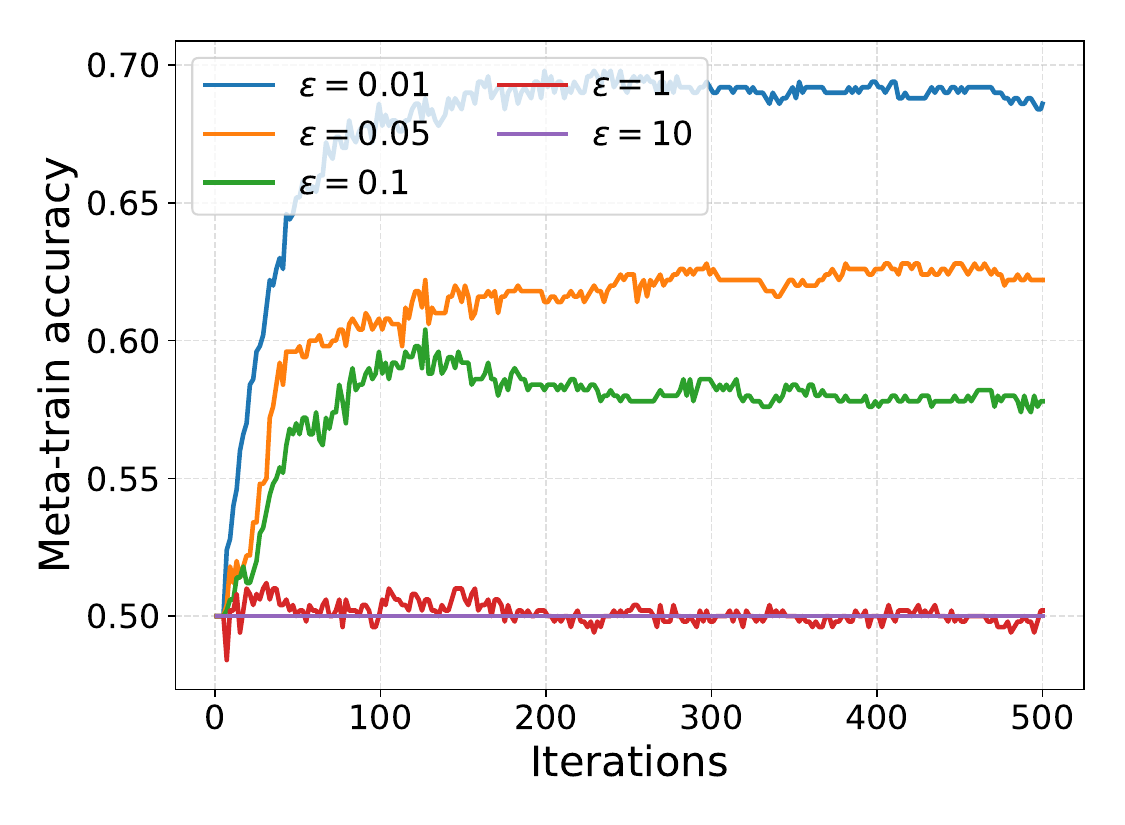}\label{fig.train_comp_eps}}
\subfigure[Test accuracy.]{\includegraphics[width=.33\textwidth]{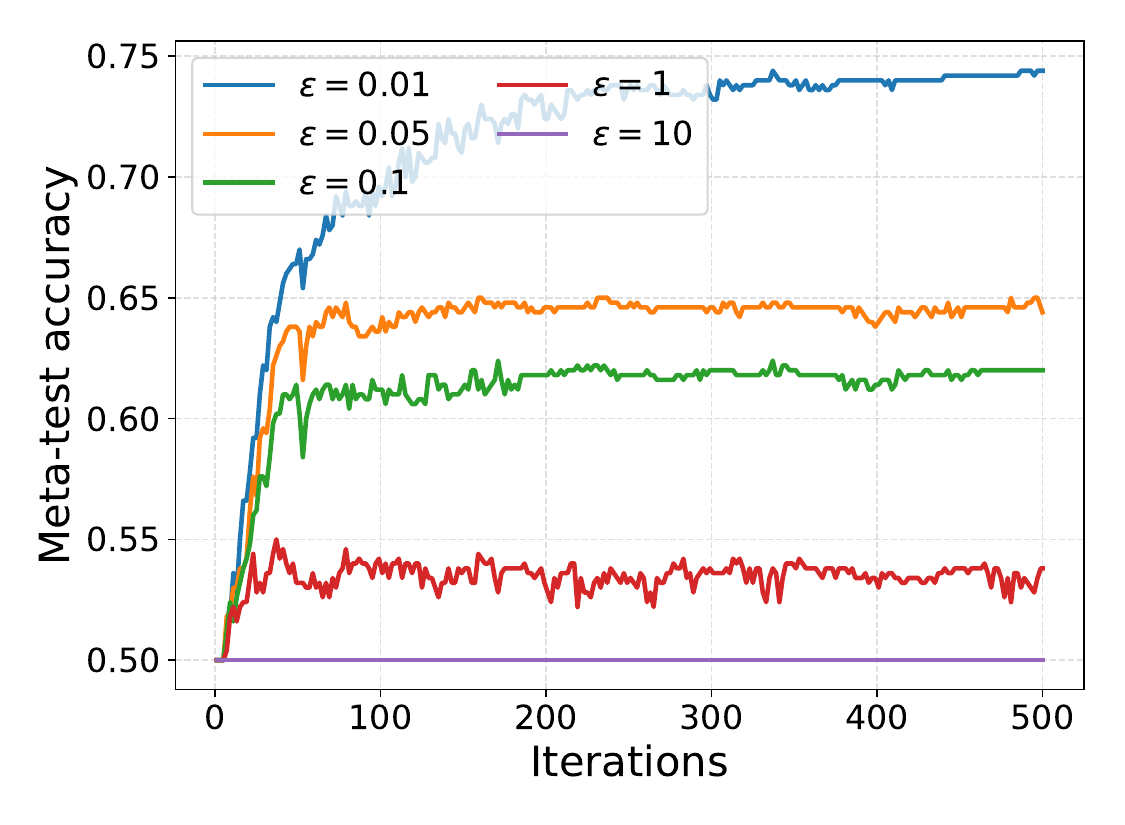}\label{fig.test_comp_eps}} 
\caption{Performance comparisons of the FBi-SGD method under different $\varepsilon$ ($\varepsilon_c=\varepsilon_d=\varepsilon$).} \label{fig.comp_eps}
\end{figure*}

\subsection{Meta Strategic Classification}
We further evaluate our methods on a federated variant of the meta strategic classification problem studied in \cite{lu2023bilevel}. The problem formulation is given as
\begin{subequations}\label{prob3}
\begin{align}
&\min\limits_{\bm{x}\in\mathbb{R}^{d}}  \sum_{i=1}^{M}p_i\sum_{j=1}^{N}q_j\underset{\xi_{i,j}\sim\mathcal{C}_{i,j}(\bm{y}^*(\bm{x}))}{\mathbb{E}} \!\!\!\!\!f_{i,j} (\bm{y}^*(\bm{x});\xi_{i,j}) \!+\! \frac{\lambda_x}{2}\|\bm{x}\|^2
\\
& \text{s.t.}\,\,
\bm{y}^*(\bm{x})\!\in\! \arg\min_{\bm{y}\in\mathbb{R}^{d}} \sum_{i=1}^{M}p_i\sum_{j=1}^{N}q_j
\underset{\zeta_{i,j}\sim\mathcal{D}_{i,j}(\bm{x})}{\mathbb{E}}\!\!\!\!\!
g_{i,j}(\bm{x},\bm{y};\zeta_{i,j}),
\end{align}
\end{subequations}
where $f_{i,j}(\cdot)$ is the logistic regression loss
and $g_{i,j}(\bm{x},\bm{y};\zeta_{i,j}) = \langle\bm{y},\nabla f_{i,j}(\bm{x};\zeta_{i,j})\rangle+\frac{\lambda_y}{2} \|\bm{y}-\bm{x}\|^2$.
We adopt the same task construction and data processing as~\cite{lu2023bilevel} based on the Amazon Reviews subset of the UCI Sentiment dataset, and distribute the data over $M=20$ clients with uniform weights $p_i=0.05$.
We choose $N=5$ tasks and sample $160$ data points per task from the dataset. Each task is constructed by masking a task-specific feature block: for the $i$-th task, we set the entries in the $i$-th feature partition to zero. In addition, we generate two meta datasets with $100$ samples each for meta-training and meta-testing, respectively. The meta-testing set partially overlaps with the meta-training set (overlap ratio: $20\%$) to ensure the transferability of latent features across tasks. We set the regularization parameters as $\lambda_x=0.001$ and $\lambda_y=1$.
Client data are non-IID across clients via Dirichlet partitioning.
For comparison, we also include the centralized Bi-SGD in~\cite{lu2023bilevel} as a baseline.
Full hyperparameter settings are provided in Appendix~\ref{append_setting2}. More simulation results can be found in Appendix \ref{append_moreresults}.

Fig.~\ref{fig.test_comp_baseline} highlights the importance of performativity-aware training: both Bi-SGD and FBi-SGD rapidly improve meta-test accuracy and converge to the best final performance, whereas all noPP (abbreviation of without performative prediction) variants stay at noticeably lower accuracy. The finding is that ignoring decision-dependent distribution shift leads to a mismatched objective and poor generalization. Notably, Fig.~\ref{fig.gap_comp_baseline} shows that a smaller stationarity gap is not sufficient for better deployment performance: some noPP methods attain comparable gaps but still underperform on meta-test due to the PP mismatch. Comparing centralized and federated implementations, 
FBi-SGD is competitive with (and slightly better than) Bi-SGD in both stationarity gap and meta-train/test accuracy (Fig.~\ref{fig.gap_comp_baseline}--\ref{fig.test_comp_baseline}).
This advantage is mainly due to the fact that, in the federated implementation, each client performs multiple local SGD steps per iteration, whereas the centralized baseline performs only one SGD step per iteration.


Fig.~\ref{fig.gap_comp_eps} presents the stability transition as $\varepsilon$ increases. For small $\varepsilon$, the stationarity gap decreases and stabilizes at a low level, while for $\varepsilon=10$ it grows rapidly to extremely large magnitudes, indicating divergence. This trend is mirrored by generalization: in Fig.~\ref{fig.train_comp_eps} and Fig.~\ref{fig.test_comp_eps}, smaller $\varepsilon$ yields higher and more stable meta-train/test accuracies, whereas increasing $\varepsilon$ degrades both. In particular, the $\varepsilon=10$ curve stays near chance level, consistent with the instability observed in Fig.~\ref{fig.gap_comp_eps}. These results further indicate that stronger performative feedback makes the induced bilevel objective harder to optimize and weakens the reliability of stochastic hypergradient estimates. Thus, FBi-SGD can reliably track the bilevel-performative solution only when the induced distribution shift remains moderate, aligning with the theoretical stability condition.

\begin{figure}[t]
\centering
\subfigure[Train accuracy.]{
\includegraphics[width=.28\textwidth]{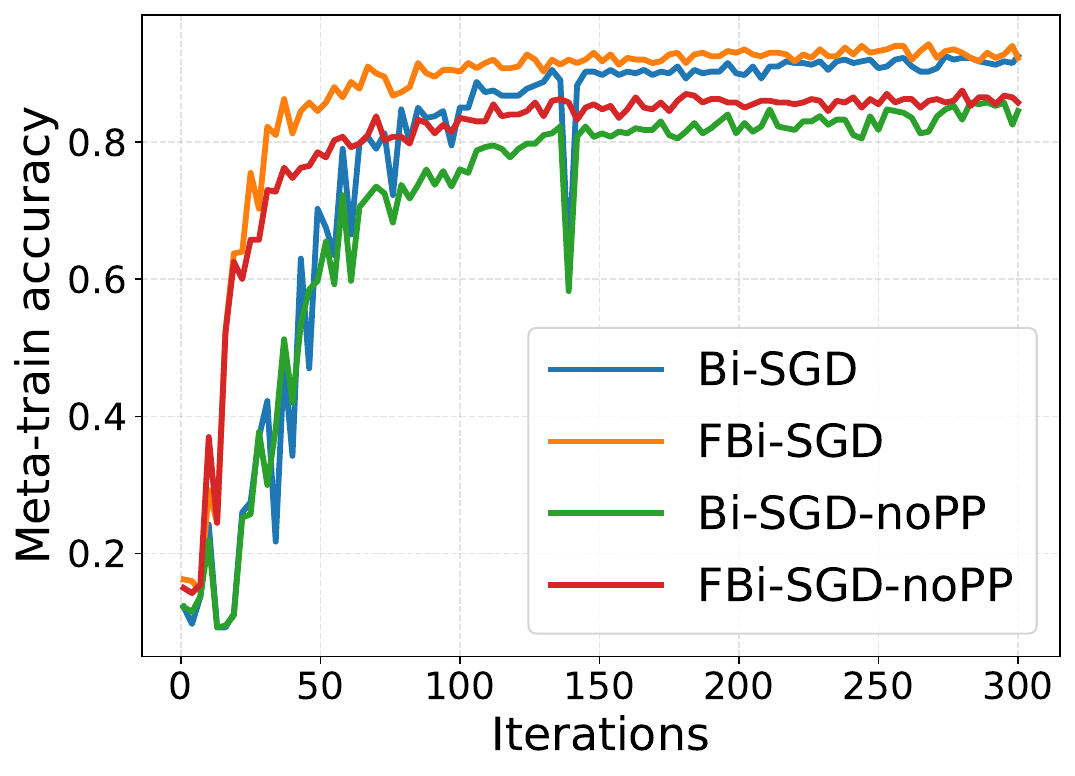}
}
\subfigure[Test accuracy.]{
\includegraphics[width=.28\textwidth]{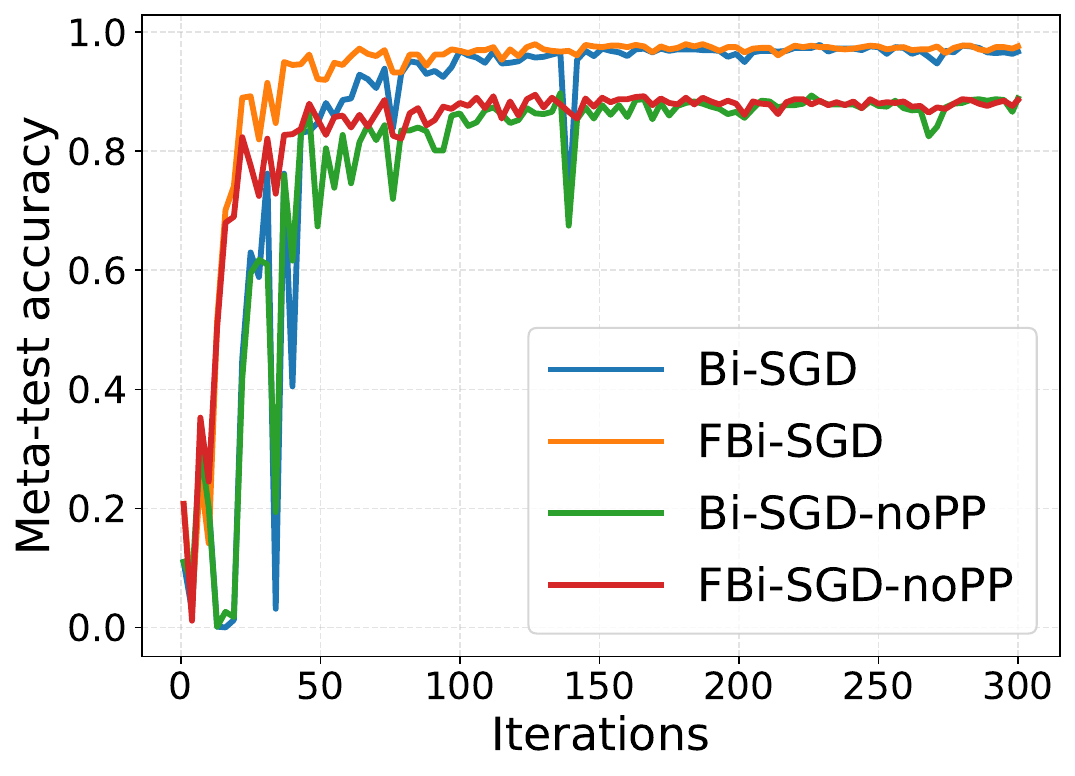}
}
\caption{Performance comparisons of FBi-SGD and Bi-SGD methods with and without performative prediction training on the CNN-MNIST simulation.}
\label{fig:cnn_mnist_acc}
\end{figure}

\subsection{CNN Classification on Federated MNIST Dataset}
\label{subsec:nonconvex_cnn_mnist}

We further evaluate the proposed methods on a federated nonconvex classification problem with neural networks, using MNIST as the benchmark dataset. The purpose of this experiment is to verify whether the proposed bilevel performative prediction methods remain effective when the LL model is parameterized by a nonconvex neural network rather than a convex model. The problem formulation is given as
\begin{subequations}\label{prob_cnn}
\begin{align}
&\min\limits_{\bm{x}\in\mathbb{R}^{d}}  
\sum_{i=1}^{M}p_i
\underset{\xi_i\sim\mathcal{C}_{i}(\bm{y}^*(\bm{x}))}{\mathbb{E}}
f_i(\bm{y}^*(\bm{x});\xi_i)
+
\frac{\lambda_x}{2}\|\bm{x}\|^2
\\
& \text{s.t.}\,\,
\bm{y}^*(\bm{x})\in
\arg\min_{\bm{y}}
\sum_{i=1}^{M}p_i
\underset{\zeta_i\sim\mathcal{D}_{i}(\bm{x})}{\mathbb{E}}
g_i(\bm{x},\bm{y};\zeta_i),
\end{align}
\end{subequations}
where $f_i(\cdot)$ is the cross-entropy loss and $g_i(\bm{x},\bm{y};\zeta_i)$ is the regularized cross-entropy loss. 
We distribute the MNIST data over $M=10$ clients with uniform client weights. Client data are non-IID across clients via Dirichlet partitioning with concentration parameter $0.3$. Each client's local data are divided into an LL set and a UL set, where the UL split ratio is $20\%$. The LL set is used to train the CNN response, while the UL set is used to evaluate the bilevel objective. We set the LL and UL minibatch sizes as $20$, the number of local steps as $K=3$, and the client participation ratio as $0.5$. The performative sensitivities are set as $\varepsilon_c=0.1$ and $\epsilon_d=0.05$.

The considered CNN model contains two convolutional layers and two fully connected layers. The first convolutional layer maps the single-channel MNIST input to $16$ feature maps, and the second convolutional layer maps $16$ feature maps to $32$ feature maps. Each convolutional layer is followed by ReLU activation and max-pooling. The resulting feature representation is passed through a fully connected hidden layer with $64$ neurons and a final output layer with $10$ neurons. Since the model contains nonlinear activations and multiple trainable layers, the resulting LL problem is nonconvex.

Fig.~\ref{fig:cnn_mnist_acc} shows the meta-train and meta-test accuracy of FBi-SGD and Bi-SGD methods with and without performative prediction training on the CNN-MNIST simulation. Both Bi-SGD and FBi-SGD consistently outperform their noPP counterparts, demonstrating that performativity-aware training remains effective in the nonconvex neural network setting. This result indicates that ignoring decision-dependent distribution shift leads to a mismatched training objective even when the prediction model is parameterized by a CNN. Comparing the centralized and federated implementations, FBi-SGD achieves the best final performance and slightly outperforms Bi-SGD. This improvement is mainly due to the fact that FBi-SGD performs multiple local stochastic updates on selected clients before server aggregation. As a result, FBi-SGD can exploit client-side local computation and exhibits the expected federated speedup effect in terms of effective stochastic updates per communication round. Overall, these results demonstrate that the proposed FBi-SGD method is effective not only for convex settings, but also for federated bilevel performative learning with nonconvex neural networks.

\section{Conclusion}
We formulated federated bilevel performative prediction by modeling client-specific decision-dependent distributions and introduced the FBPS notion to capture fixed-point stability in federated systems. We developed FBi-RRM and proved the existence, uniqueness, and linear convergence to the FBPS point via a contraction analysis. We further proposed FBi-SGD as a practical communication-efficient stochastic method with convergence guarantees when performative sensitivities are sufficiently small. Experiments on strategic regression and meta strategic classification verify the predicted stability transitions under increasing performative sensitivity and demonstrate improved meta-test accuracy over non-performative baselines. Additional CNN-based simulations on federated MNIST further show the practical effectiveness of the proposed methods in nonconvex neural network settings. Future work will extend the proposed framework to fully distributed bilevel performative prediction, where clients collaboratively track performative equilibria without relying on a central server.

\clearpage
\section*{Acknowledgement}
This research is supported by the National Research Foundation, Singapore and Infocomm Media Development Authority under its Trust Tech Funding Initiative. Any opinions, findings and conclusions or recommendations expressed in this material are those of the author(s) and do not reflect the views of National Research Foundation, Singapore and Infocomm Media Development Authority.

\section*{Impact Statement}
This work advances machine learning by studying federated bilevel learning under decision-dependent (performative) distribution shifts, with the goal of improving stability and meta-generalization in deployment.

Ethically, our experiments use standard benchmark data and synthetic simulations, and we do not collect or process sensitive personal data. To minimize any risk of unintentional disclosure in our research workflow, we follow strict access and data-handling controls: only authorized project members can access the data and experiment artifacts; all processing is performed on institution-managed secure servers; and we avoid using external third-party services for storing or handling any non-public materials. We also adopt conservative logging practices (e.g., no raw text release beyond what is already public) and continuously review outputs for potential leakage. When applying our methods to real-world federated deployments, we recommend complementing them with established privacy protections and governance mechanisms, and auditing for disparate impacts across heterogeneous clients.

\bibliography{ref}
\bibliographystyle{icml2026}

\newpage
\appendix
\onecolumn
\section{Related Works}
In this section, we introduce the related works from three aspects: performative prediction, federated learning, strategic classification.
\subsection{Performative Prediction}
Performative prediction studies learning settings where deploying a predictor changes the future data distribution it is trained on. The first work to formalize this phenomenon is Perdomo \textit{et al}.~\cite{perdomo2020performative}, who introduced performative risk minimization and the equilibrium notion of performative stability, and established conditions under which repeated retraining converges to a nearly optimal stable point. Mendler-D{\"u}nner \textit{et al}.~\cite{mendler2020stochastic} then analyzed stochastic optimization under performativity by distinguishing parameter updates from deployment events, providing convergence rates for greedy versus lazy deployment schedules and clarifying when each is preferable. Beyond the original model, Brown \textit{et al}.~\cite{brown2022performative} generalized performative prediction to stateful populations whose response depends on both the deployed classifier and an evolving population state, deriving sharp convergence conditions for retraining dynamics. From a causal perspective, Mendler-D{\"u}nner \textit{et al}.~\cite{mendler2022anticipating} studied when the causal effect of predictions on outcomes is identifiable from observational data and emphasized the importance of logging predictions during data collection. Recent work further relaxes assumptions and broadens the scope: Mofakhami \textit{et al}.~\cite{mofakhami2023performative} replace parameter-Lipschitzness by prediction-Lipschitzness to enable learning performatively stable predictors with neural networks; multi-agent extensions model strategic interactions and network effects~\cite{narang2023multiplayer,wang2023network}; and robustness to distribution-map misspecification is addressed via distributionally robust performative prediction~\cite{xue2024distributionally}.

\subsection{Federated Learning}
Federated bilevel optimization (FBO) has gained momentum as a framework for nested objectives in federated settings, e.g., hyperparameter tuning and hyper-representation learning. FedNest~\cite{tarzanagh2022fednest} proposed a federated alternating stochastic-gradient method for general nested problems with federated hypergradient computation and variance reduction. Subsequent work improved efficiency and robustness under heterogeneity and partial participation, including communication-efficient acceleration in FedBiOAcc~\cite{li2023communication}, simple sub-loop-free frameworks such as SimFBO~\cite{yang2023simfbo}, and linear speedup guarantees on non-i.i.d. data via tailored sampling in FedMBO~\cite{huang2023achieving}. More recent methods further reduce memory and second-order dependence by adopting fully single-loop, first-order updates, e.g., MemFBO~\cite{yang2025first} and related first-order federated stochastic bilevel schemes~\cite{zhang2025federated}.

In contrast, federated performative prediction remains relatively underexplored. Jin et al.~\cite{jin2024performative} formalized model-dependent distribution shifts in FL using performative distribution mappings, established conditions for a unique performatively stable solution, and proposed a performative FedAvg algorithm with $\mathcal{O}(1/T)$ convergence under full and partial participation.

\subsection{Strategic Classification}
Strategic classification studies settings where individuals strategically manipulate their features to obtain favorable predictions. Perdomo \textit{et al}.~\cite{perdomo2020performative} connected this line to performative prediction by showing that strategic classification is a special case of performativity, and provided conditions under which repeated retraining can overcome such feedback. From an optimization viewpoint, Mendler-D{\"u}nner \textit{et al}.~\cite{mendler2020stochastic} analyzed how deployment frequency (greedy vs.\ lazy) affects convergence under performative/strategic feedback.

Subsequent work refined the game-theoretic and behavioral modeling. Zrnic \textit{et al}.~\cite{zrnic2021leads} showed that the effective order of play depends on relative adaptation rates, potentially reversing Stackelberg roles, while Jagadeesan \textit{et al}.~\cite{jagadeesan2021alternative} argued that perfect-information microfoundations can be brittle and proposed noisy response models yielding more robust stability insights. Meanwhile, practical and realistic extensions addressed optimization and information constraints: Levanon and Rosenfeld~\cite{levanon2021strategic} directly minimized strategic empirical risk by differentiating through agents' responses; Ghalme \textit{et al}.~\cite{ghalme2021strategic} studied opaque settings where agents must learn the classifier and quantified the ``price of opacity''; and Cohen \textit{et al}.~\cite{cohen2024bayesian} modeled Bayesian agents and showed that partial information release can improve accuracy.

Finally, learning-theoretic and causal perspectives broadened guarantees and applicability. Lechner and Urner~\cite{lechner2022learning} analyzed sample complexity via a strategic manipulation loss under known/unknown manipulation graphs; Horowitz and Rosenfeld~\cite{horowitz2023causal} incorporated causal effects where feature changes alter true outcomes, producing two interacting distribution shifts; and Cohen \textit{et al}.~\cite{cohen2024learnability} provided near-tight learnability bounds under limited observability and manipulation-structure uncertainty.

\section{Detailed Simulation Settings}\label{append_simusetting}
Here we present detailed simulation settings.

\subsection{Quadratically Regularized Bilevel Strategic Regression}\label{append_setting1}
In the optimization problem (\ref{prob2}), $\bm{b}_i$ and $\bm{d}_i$ are generated from linear regression models, i.e., $\bm{b}_i = \bm{a}_i^\intercal \check{\bm{x}} + \bm{n}_x$ and $\bm{d}_i = \bm{c}_i^\intercal \check{\bm{y}}+\bm{n}_y$ for all $i$, where $\check{\bm{x}}$, $\check{\bm{y}}$, and the noise terms $\bm{n}_x$, $\bm{n}_y$ are i.i.d. Gaussian random variables with zero means and unit variance. The features $\bm{a}_i$ and $\bm{c}_i$ are i.i.d. Gaussian random variables with zero means and variances of one and two, respectively.
we consider $M=10$ clients with full participation and uniform aggregation weights $p_i=0.1$. Each client $i$ holds local datasets $\{(\bm{a}_{ij},\bm{b}_{ij})\}_{j=1}^N$ and $\{(\bm{c}_{ij},\bm{d}_{ij})\}_{j=1}^N$.
To model non-IID data across clients, we introduce parameter heterogeneity by assigning client-specific parameters $(\check{\bm{x}}_i,\check{\bm{y}}_i)$ obtained by perturbing shared base parameters $(\check{\bm{x}}_{\mathrm{base}},\check{\bm{y}}_{\mathrm{base}})$ with Gaussian offsets scaled by a heterogeneity level parameter $h$, i.e., $\check{\bm{x}}_i = \check{\bm{x}}_{\mathrm{base}} + h \delta_{x,i}$, $\check{\bm{y}}_i = \check{\bm{y}}_{\mathrm{base}} + h \delta_{y,i}$, where $\delta_{x,i}$ and $\delta_{y,i}$ are also i.i.d. Gaussian random variables with zero means and unit variance. $h$ is set as one. $N=50$ and $d=5$. $\lambda_x$ and $\lambda_y$ are set as $0.001$.
For FBi-RRM, we approximate the exact UL/LL best responses by running 200 steps of gradient descent for the corresponding subproblems at each round. For FBi-SGD, we use $K_r=5$ local steps per round with UL and LL minibatch sizes 20, a constant local stepsize 0.02 for updating $(\bm{x},\bm{y},\bm{v})$, and a server-side decay schedule $1/\sqrt{r}$, where $r$ is the communication round index.

\subsection{Meta Strategic Classification}\label{append_setting2}
Following \cite{lu2023bilevel}, we also consider the Amazon Reviews subset of the UCI Sentiment Labeled Sentences dataset, and choose $N=5$ tasks and sample $160$ data points per task. Each task is constructed by masking a task-specific feature block: for the $i$-th task, we set the entries in the $i$-th feature partition to zero. In addition, we generate two meta datasets with $100$ samples each for meta-training and meta-testing, respectively. The meta-testing set partially overlaps with the meta-training set (overlap ratio: $20\%$) to ensure the transferability of latent features across tasks. We set the regularization parameters as $\lambda_x=0.001$ and $\lambda_y=1$. The default performative prediction sensitivity is set to $0.05$ for both training and testing. We use $M=20$ clients with uniform weights $p_i=0.05$. 
We generate Dirichlet-style non-IID client data by partitioning each task’s samples across clients according to proportions drawn from a Dirichlet distribution with concentration parameter $\alpha$, where smaller $\alpha$ implies higher heterogeneity. $\alpha$ is set as 0.3. Each communication round uses $K_r=3$ local steps to update $(\bm{x},\bm{y},\bm{v})$, with UL and LL minibatch sizes both set to $20$. The learning rate for the auxiliary variable $\bm{v}$ is set to $0.05$. The stepsize is chosen as $\frac{10}{\sqrt{10+r}}$, where $r$ is the communication-round index. 

\section{More Simulation Results}\label{append_moreresults}
In this section, we present more simulation results to validate the performances of the FBi-SGD algorithm in problem (\ref{prob3}) under different settings.
\subsection{Impact of Minibatch Size}
\begin{figure*}[t]
\subfigure[Stationarity gap.]{\includegraphics[width=.33\textwidth]{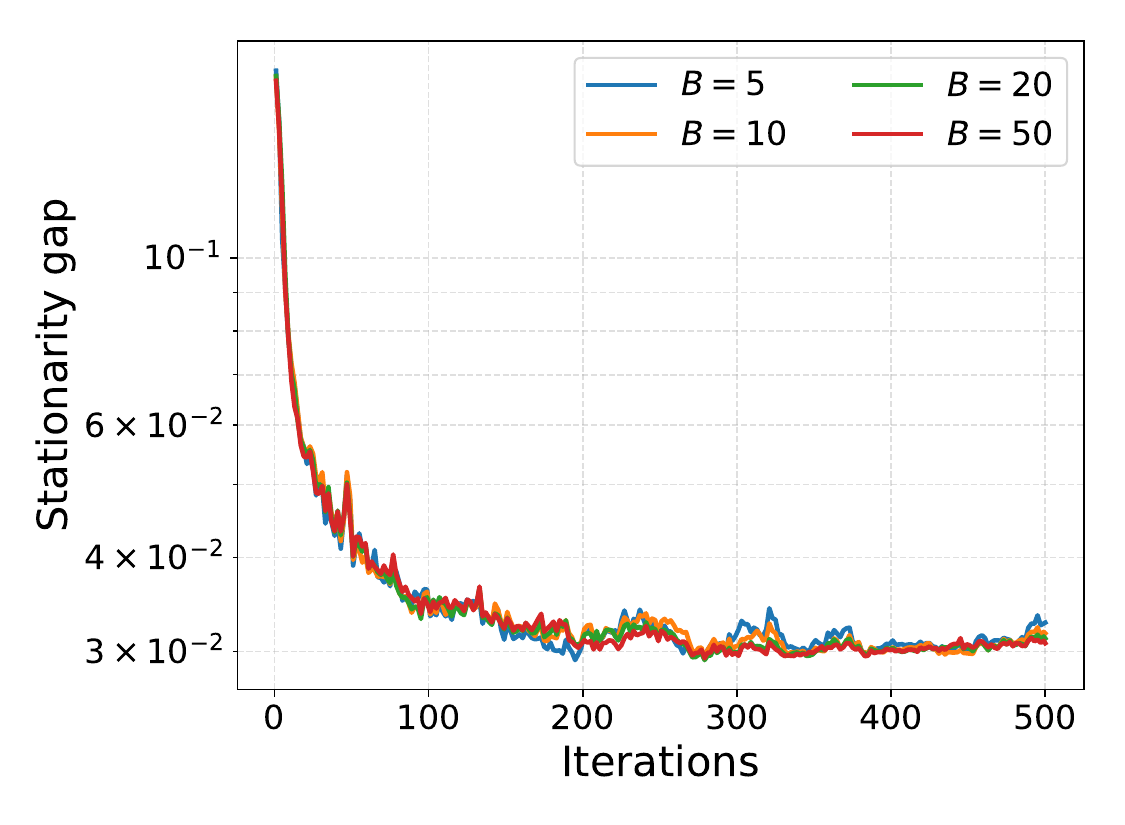}\label{fig.gap_comp_multibatch}}
\subfigure[Train accuracy.]{\includegraphics[width=.33\textwidth]{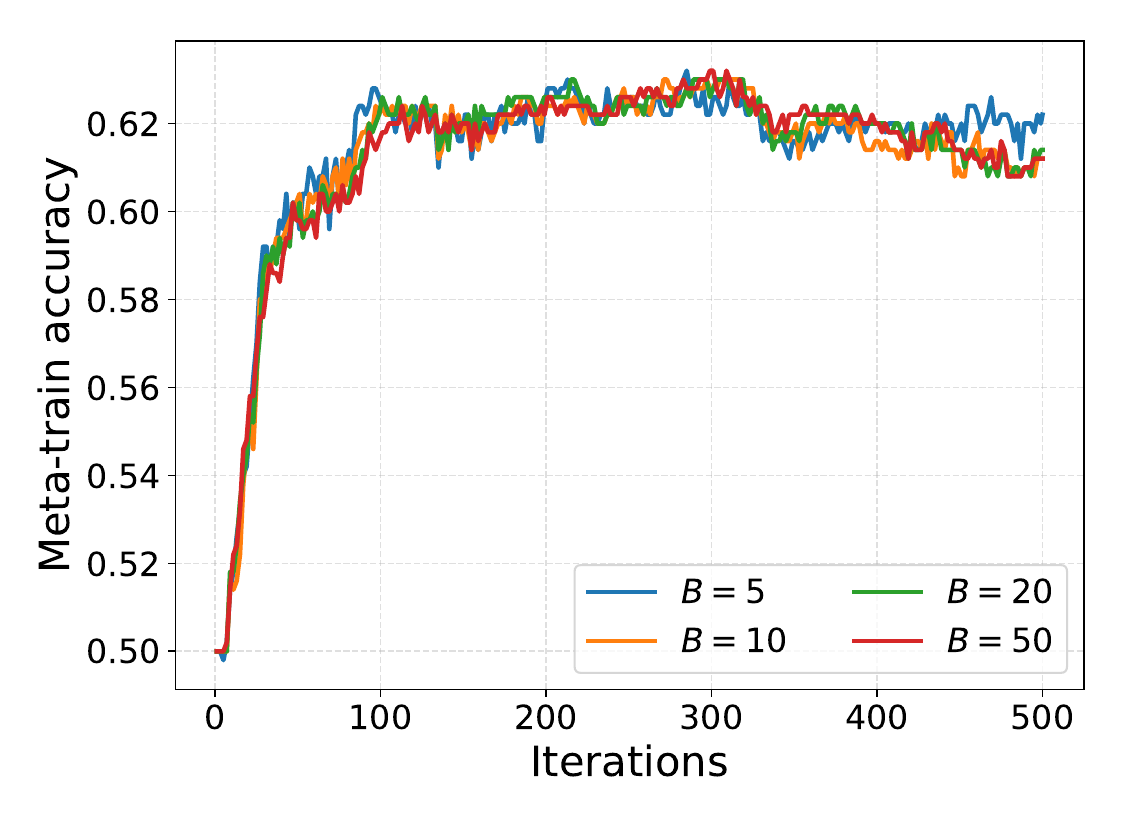}\label{fig.train_comp_multibatch}}
\subfigure[Test accuracy.]{\includegraphics[width=.33\textwidth]{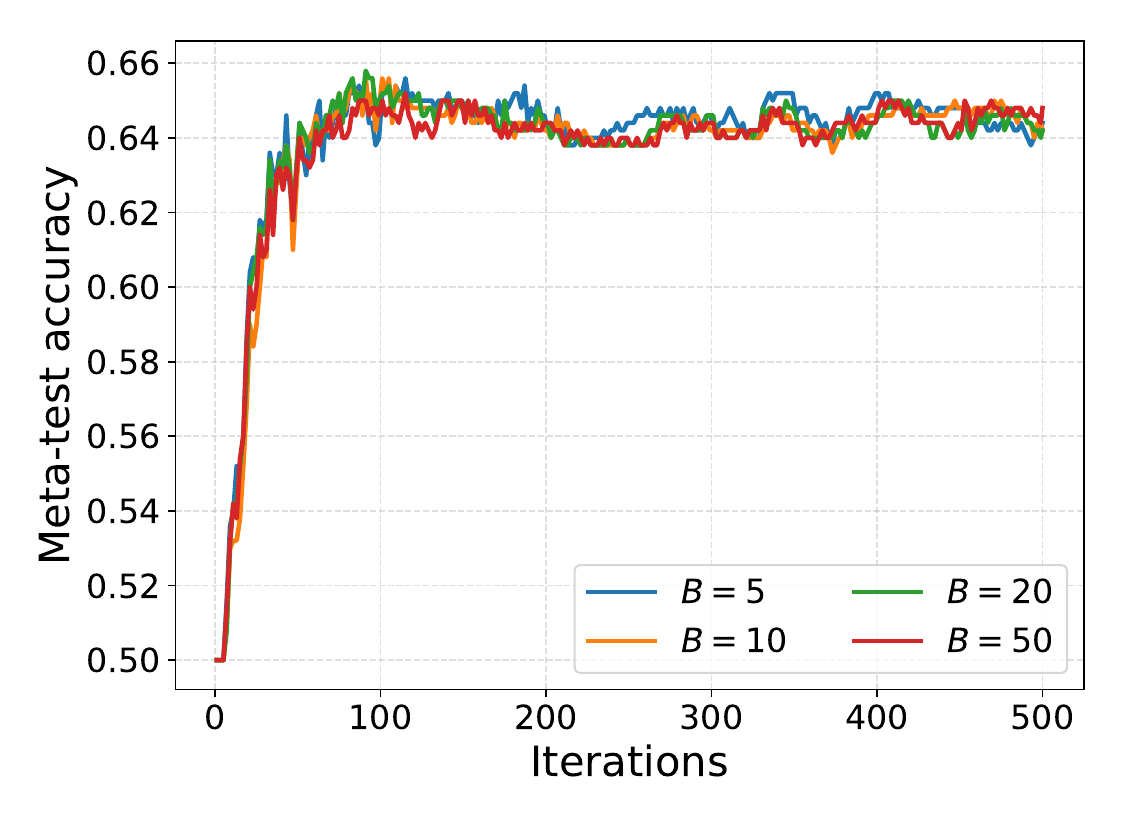}\label{fig.test_comp_multibatch}} 
\caption{Performance comparisons of the FBi-SGD method under different minibatches ($\varepsilon_c=\varepsilon_d=0.05$). $B$ denotes the minibatch size.} 
\end{figure*}
As shown in Fig.~\ref{fig.gap_comp_multibatch}, the stationarity gap drops sharply in the first few dozen iterations and then gradually stabilizes around a similar level for all minibatch sizes, which indicates that FBi-SGD maintains consistent convergence behavior under different $B$, where $B$ is the number of minibatches. Meanwhile, smaller minibatches (e.g., $B=5$) exhibit slightly larger oscillations due to higher stochastic gradient variance, whereas larger minibatches (e.g., $B=20$ and $B=50$) yield smoother trajectories, but without a clear improvement in the final gap.

Figs.~\ref{fig.train_comp_multibatch} and~\ref{fig.test_comp_multibatch} further demonstrate that both meta-train and meta-test accuracies increase rapidly in the early stage and then plateau, and the curves under different $B$ largely overlap throughout training. The sensitivity to minibatch size is marginal in this setting. A moderate minibatch (e.g., $B=10$ or $B=20$) can already achieve a favorable balance between stability and computational cost.

\subsection{Impact of ($\varepsilon_c$, $\varepsilon_d$)}
\begin{figure*}[t]
\subfigure[Stationarity gap.]{\includegraphics[width=.33\textwidth]{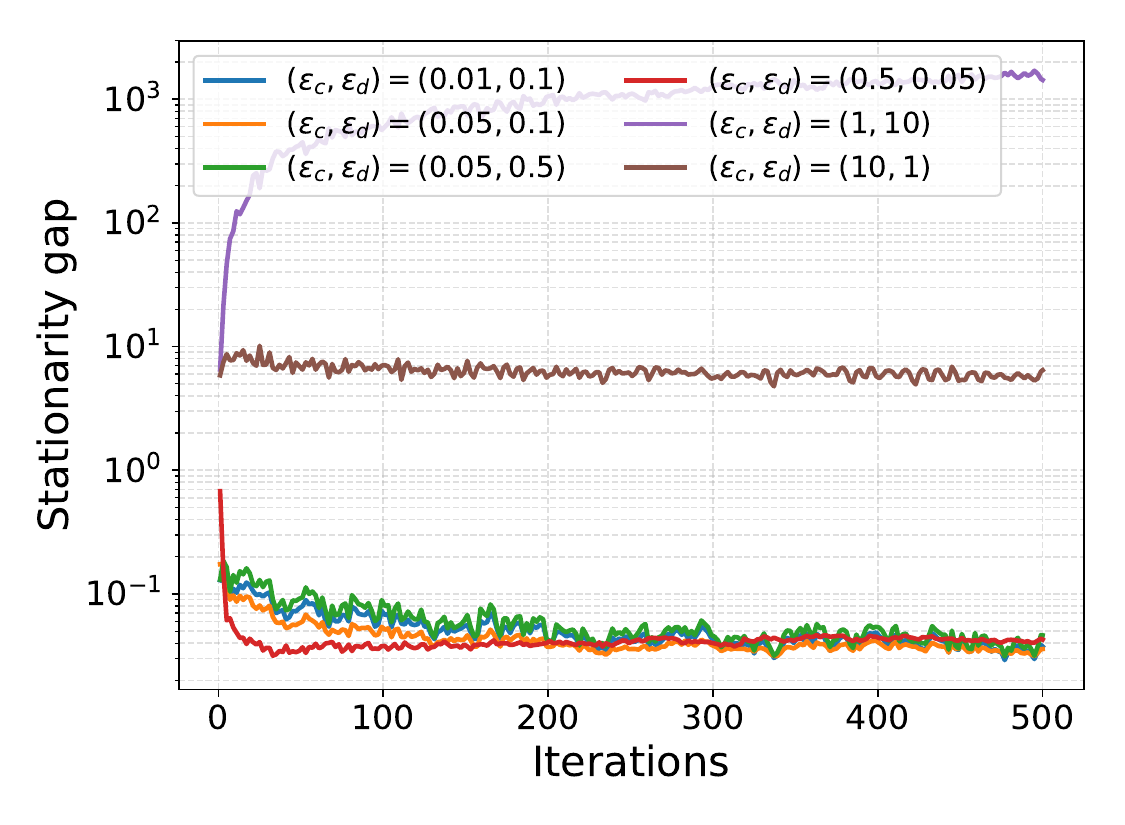}\label{fig.gap_comp_epspair}}
\subfigure[Train accuracy.]{\includegraphics[width=.33\textwidth]{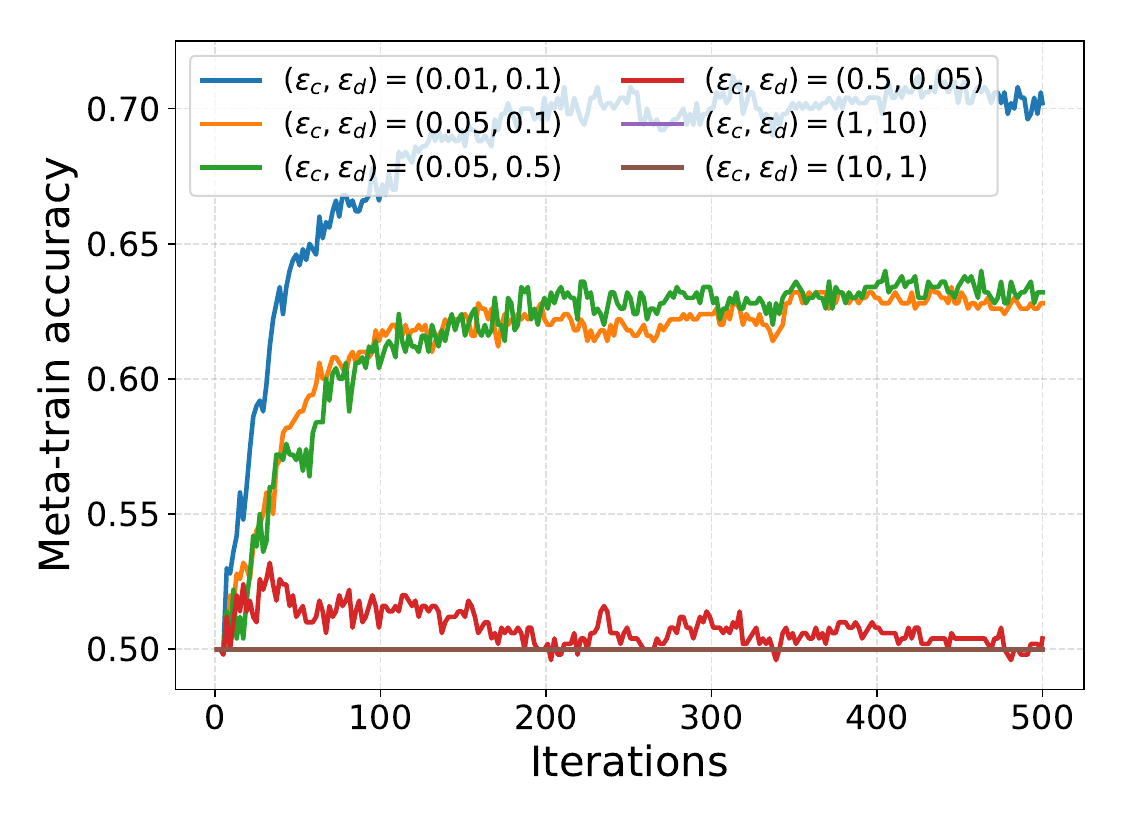}\label{fig.train_comp_epspair}}
\subfigure[Test accuracy.]{\includegraphics[width=.33\textwidth]{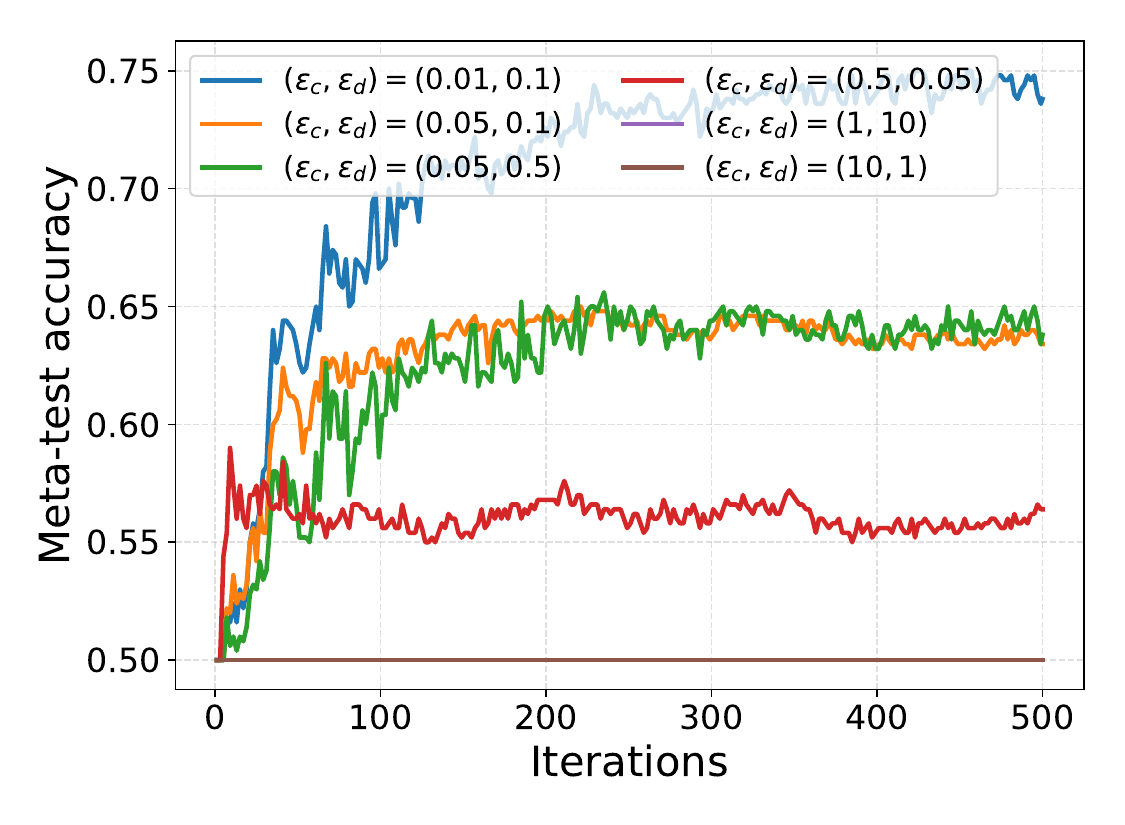}\label{fig.test_comp_epspair}} 
\caption{Performance comparisons of the FBi-SGD method under different ($\varepsilon_c$, $\varepsilon_d$).} \label{fig.comp_epspair}
\end{figure*}

As illustrated in Fig.~\ref{fig.gap_comp_epspair}, the choice of $(\varepsilon_c,\varepsilon_d)$ has a pronounced impact on the convergence behavior. For moderate pairs such as $(0.01,0.1)$, $(0.05,0.1)$, $(0.05,0.5)$, and $(0.5,0.05)$, the stationarity gap decreases quickly in the early stage and then stabilizes at a relatively small level, indicating stable training. In contrast, overly aggressive settings lead to instability: $(1,10)$ yields a rapidly growing gap, while $(10,1)$ stabilizes at a much larger gap than the moderate cases, suggesting that too large coupling can prevent the algorithm from approaching a stationary region.

The accuracy trends in Figs.~\ref{fig.train_comp_epspair} and~\ref{fig.test_comp_epspair} are consistent with the above observation. Among the tested pairs, $(0.01,0.1)$ achieves the best meta-train and meta-test accuracies and continues to improve over iterations. The moderate choices $(0.05,0.1)$ and $(0.05,0.5)$ reach similar but lower accuracy plateaus. Meanwhile, $(0.5,0.05)$ exhibits noticeably worse accuracy despite having a small stationarity gap, which indicates that a smaller gap does not necessarily imply better generalization under decision-dependent shifts. For extreme pairs, the performance degrades significantly: $(10,1)$ is essentially stuck around chance-level accuracy, aligning with its large stationarity gap, and $(1,10)$ becomes unstable.

\subsection{Impact of Shifting Heterogeneity}
\begin{figure*}[t]
\subfigure[Stationarity gap.]{\includegraphics[width=.33\textwidth]{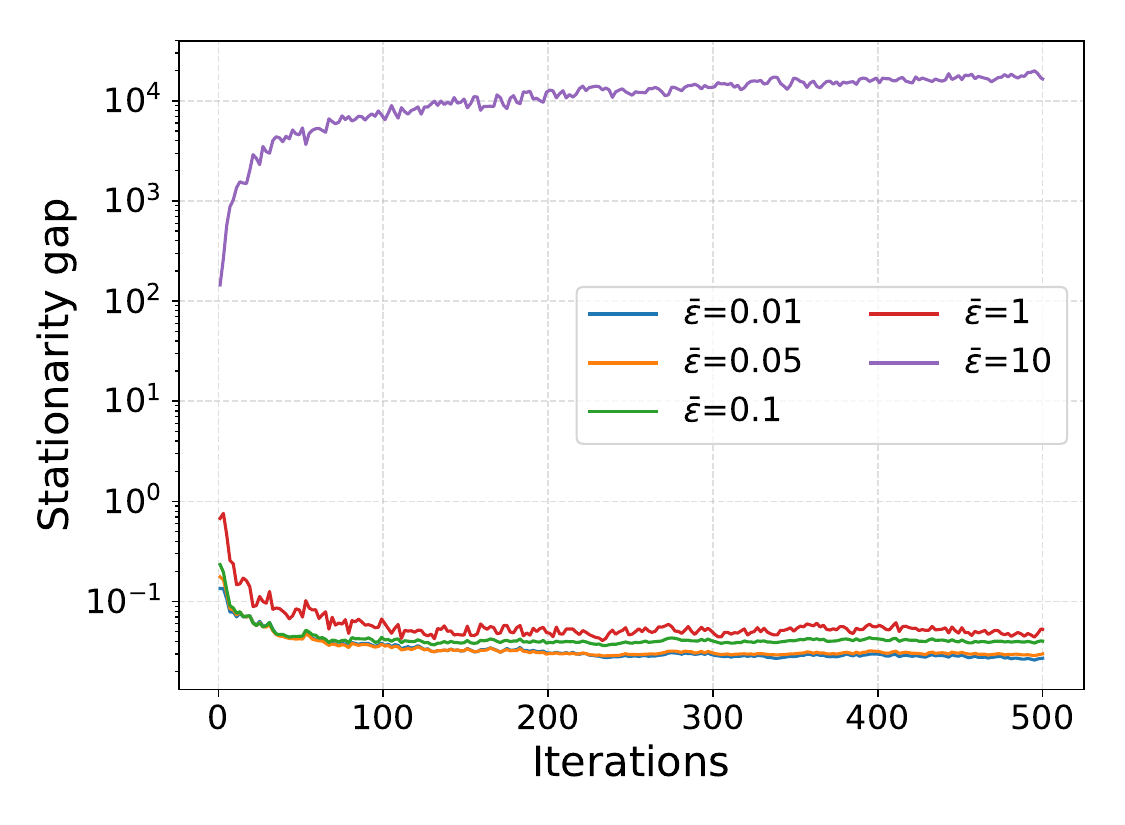}\label{fig.gap_comp_bareps}}
\subfigure[Train accuracy.]{\includegraphics[width=.33\textwidth]{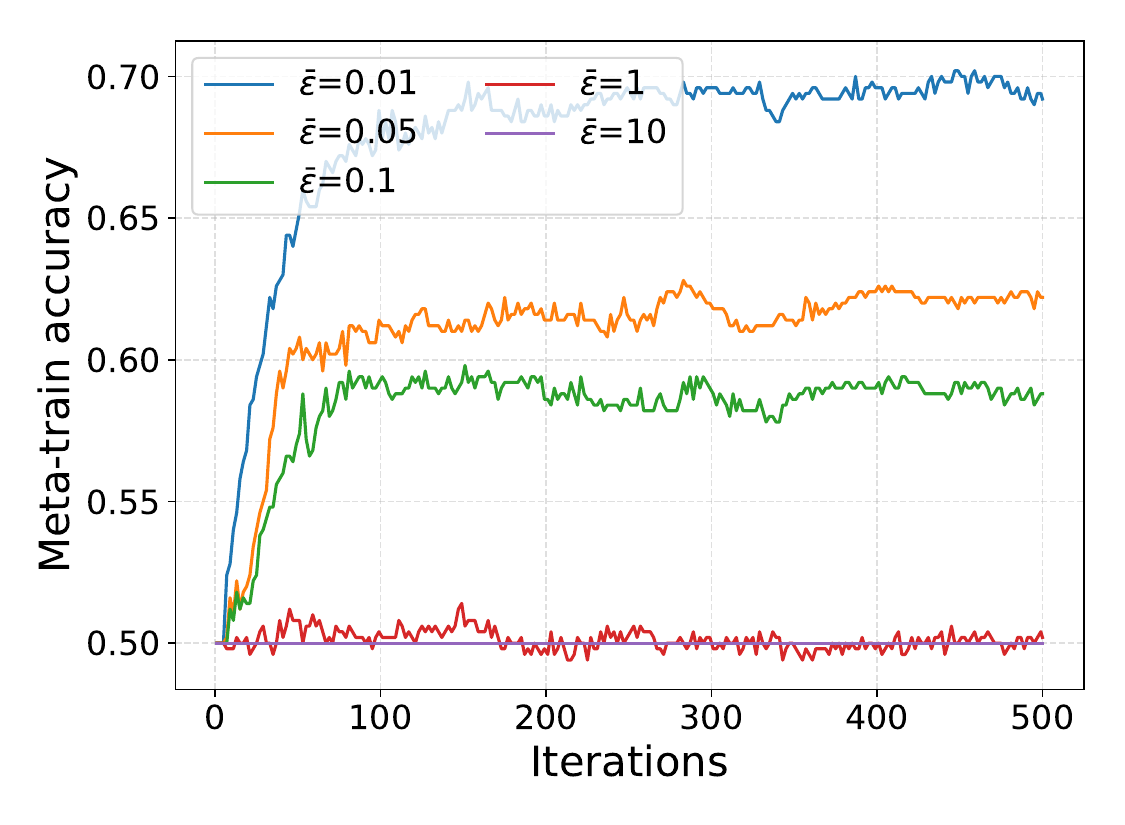}\label{fig.train_comp_bareps}}
\subfigure[Test accuracy.]{\includegraphics[width=.33\textwidth]{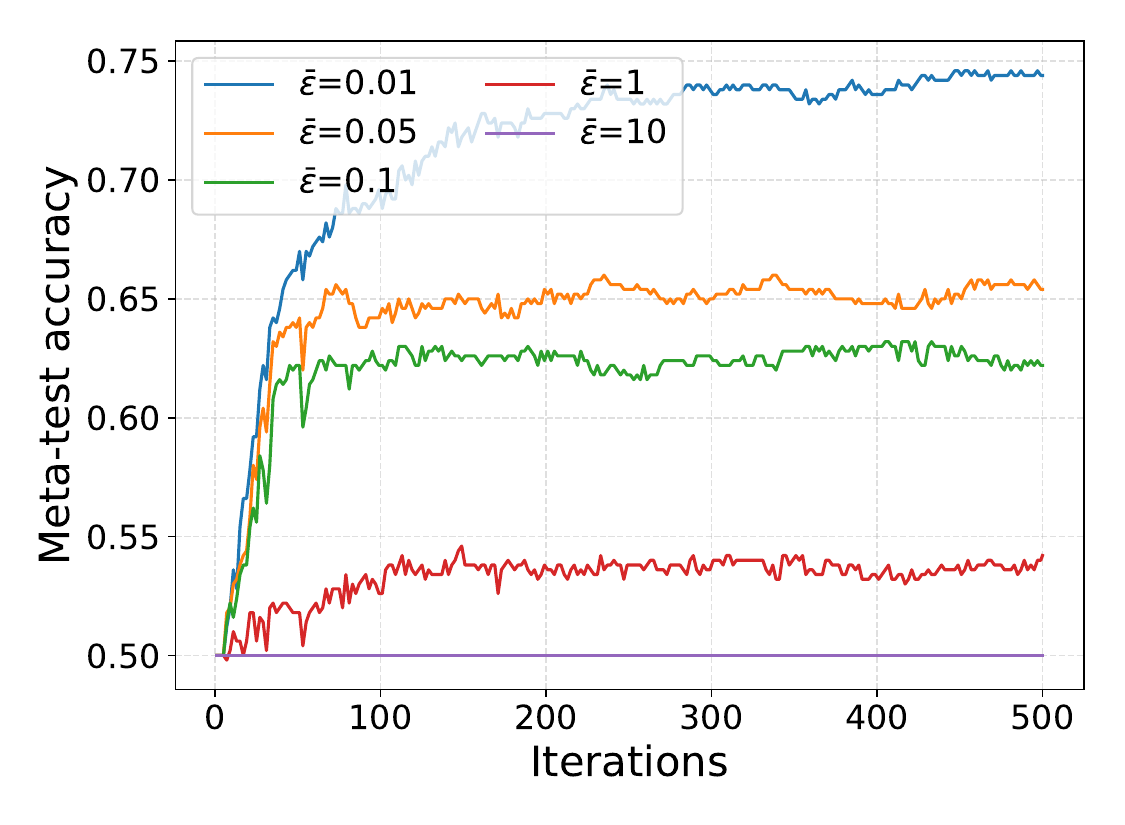}\label{fig.test_comp_bareps}} 
\caption{Performance comparisons of the FBi-SGD method under different distribution shifting heterogeneity ($\bar{\varepsilon}_c = \bar{\varepsilon}_d = \bar{\varepsilon}$).} \label{fig.comp_bareps}
\end{figure*}
\begin{figure*}[t]
\subfigure[Stationarity gap.]{\includegraphics[width=.33\textwidth]{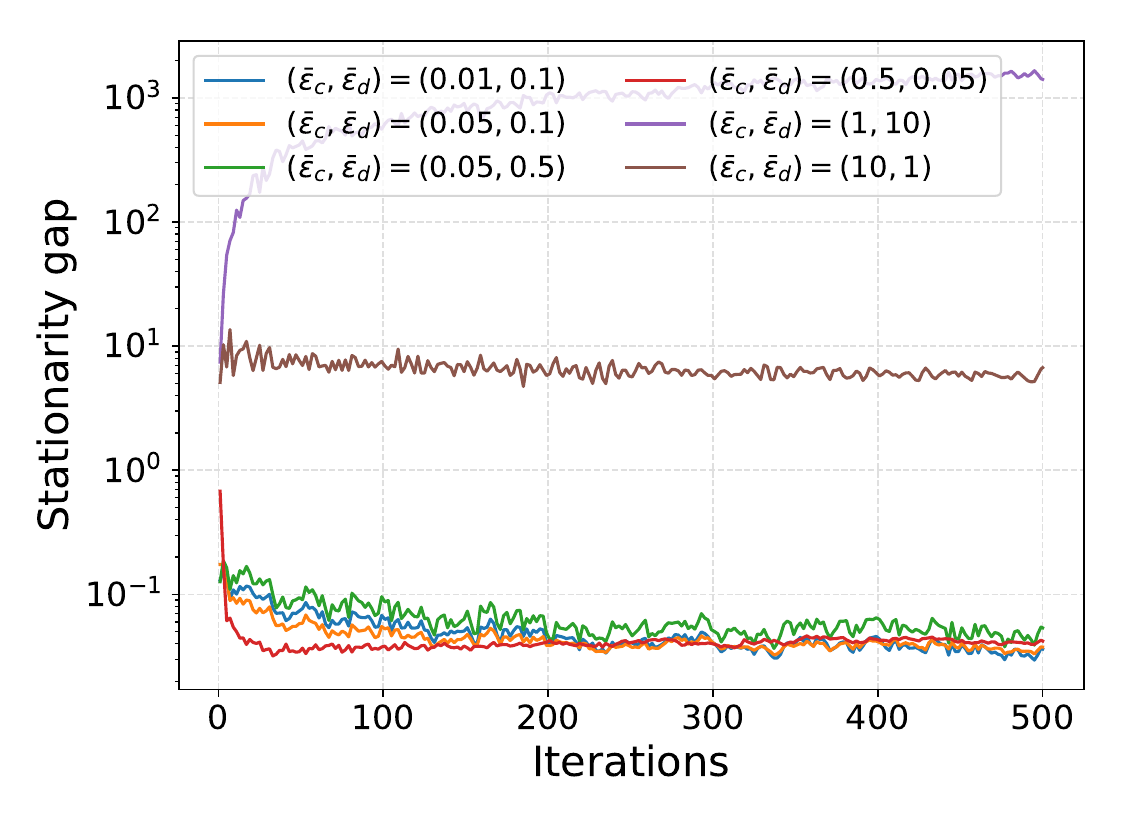}\label{fig.gap_comp_barepspair}}
\subfigure[Train accuracy.]{\includegraphics[width=.33\textwidth]{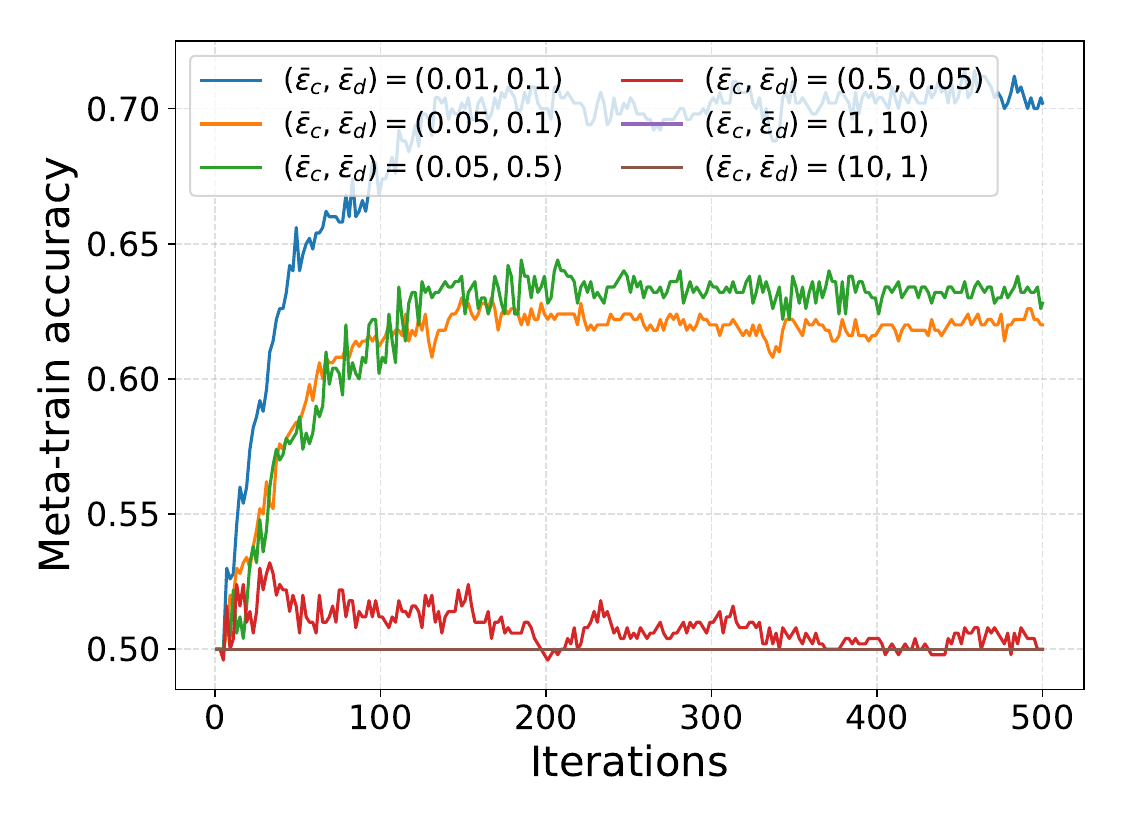}\label{fig.train_comp_barepspair}}
\subfigure[Test accuracy.]{\includegraphics[width=.33\textwidth]{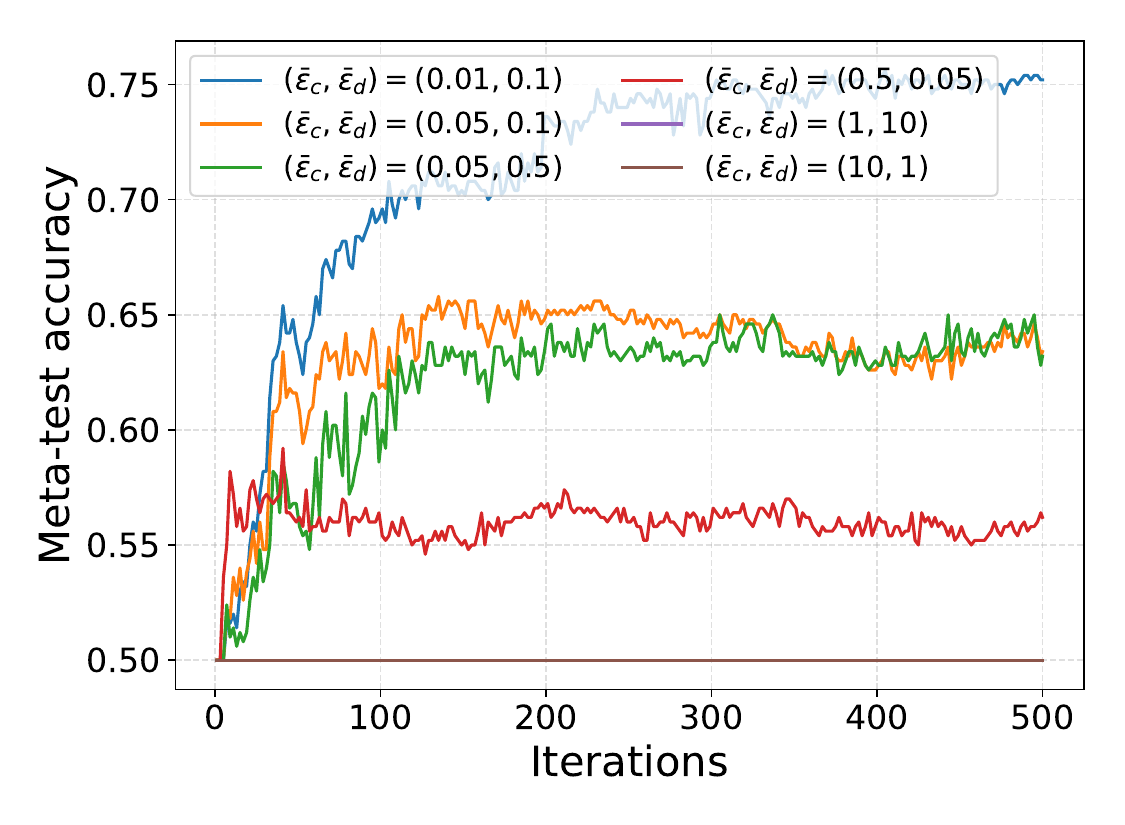}\label{fig.test_comp_barepspair}} 
\caption{Performance comparisons of the FBi-SGD method under different distribution shifting heterogeneity ($\bar{\varepsilon}_c \neq \bar{\varepsilon}_d$).} \label{fig.comp_barepspair}
\end{figure*}

To evaluate heterogeneous performative sensitivities, we parameterize the setting by a target averaged sensitivity $\bar{\varepsilon}$ and then generate client-wise sensitivities $\varepsilon_{i,c}$ and $\varepsilon_{i,d}$ such that their weighted averages match the prescribed level (e.g., $\sum_{i=1}^{C_r} p_i \varepsilon_{i,c}=\bar{\varepsilon}_c$ and $\sum_{i=1}^{C_r} p_i \varepsilon_{i,d}=\bar{\varepsilon}_d$). This construction preserves client heterogeneity while ensuring that the overall strength of the decision-dependent shift is comparable across different runs. As shown in Figs. \ref{fig.comp_bareps} and \ref{fig.comp_barepspair}, introducing heterogeneous shifts under this aggregated control has only a minor impact on both accuracy and stationary-gap trends. Consequently, the performance of FBi-SGD in these figures remains very close to the corresponding homogeneous settings in Figs. \ref{fig.comp_eps} and \ref{fig.comp_epspair}.

\subsection{Impact of Data Heterogeneity}
\begin{figure*}[t]
\subfigure[Stationarity gap.]{\includegraphics[width=.33\textwidth]{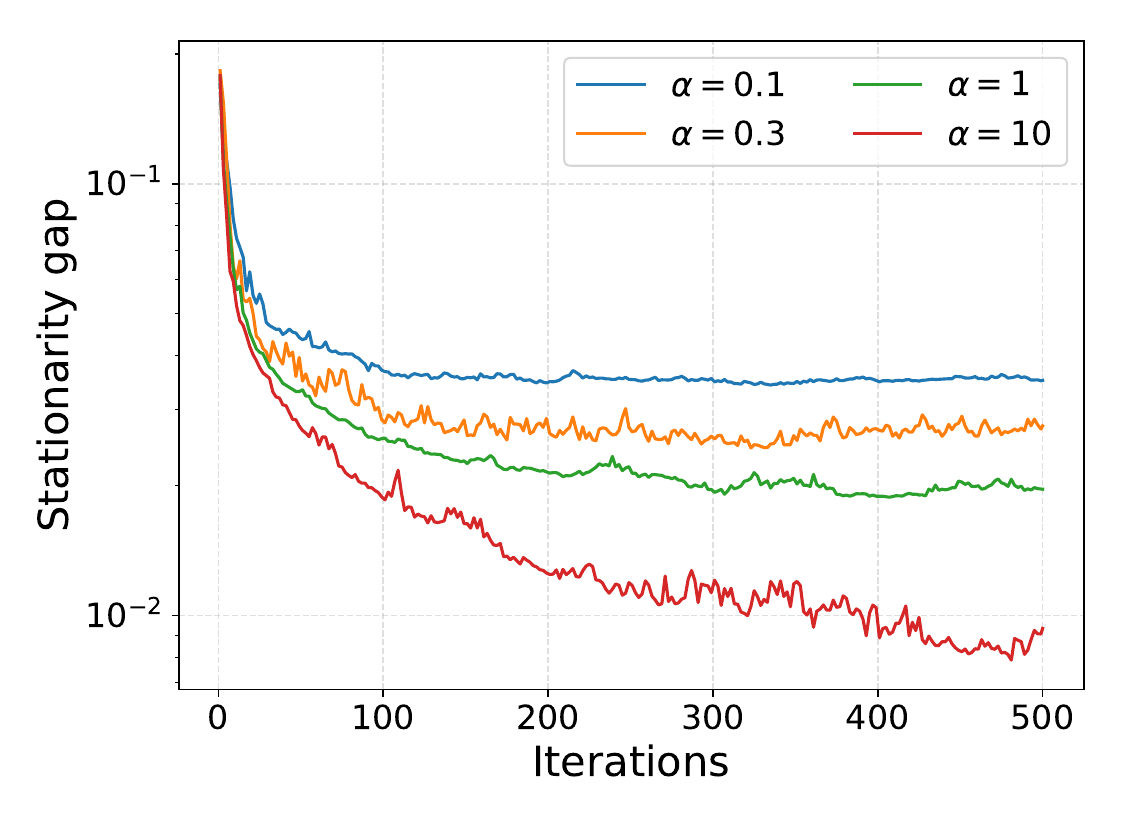}\label{fig.gap_comp_heterogeneity}}
\subfigure[Train accuracy.]{\includegraphics[width=.33\textwidth]{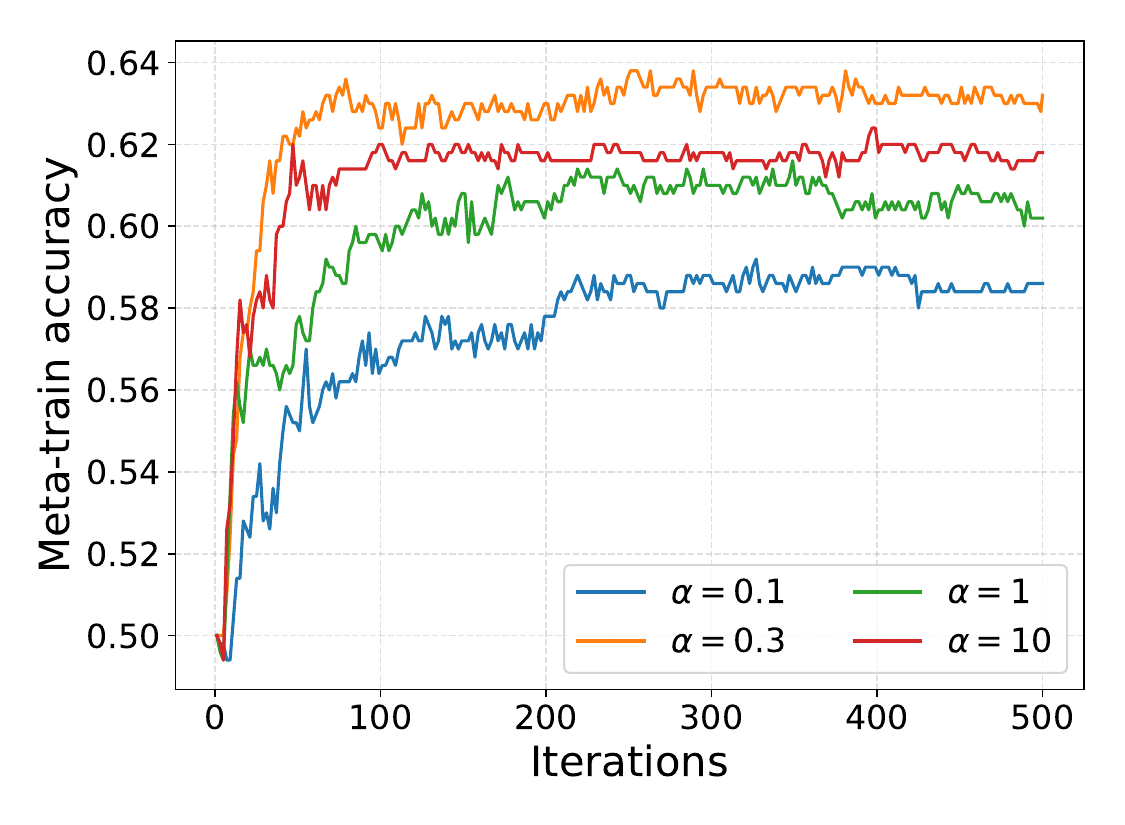}\label{fig.train_comp_heterogeneity}}
\subfigure[Test accuracy.]{\includegraphics[width=.33\textwidth]{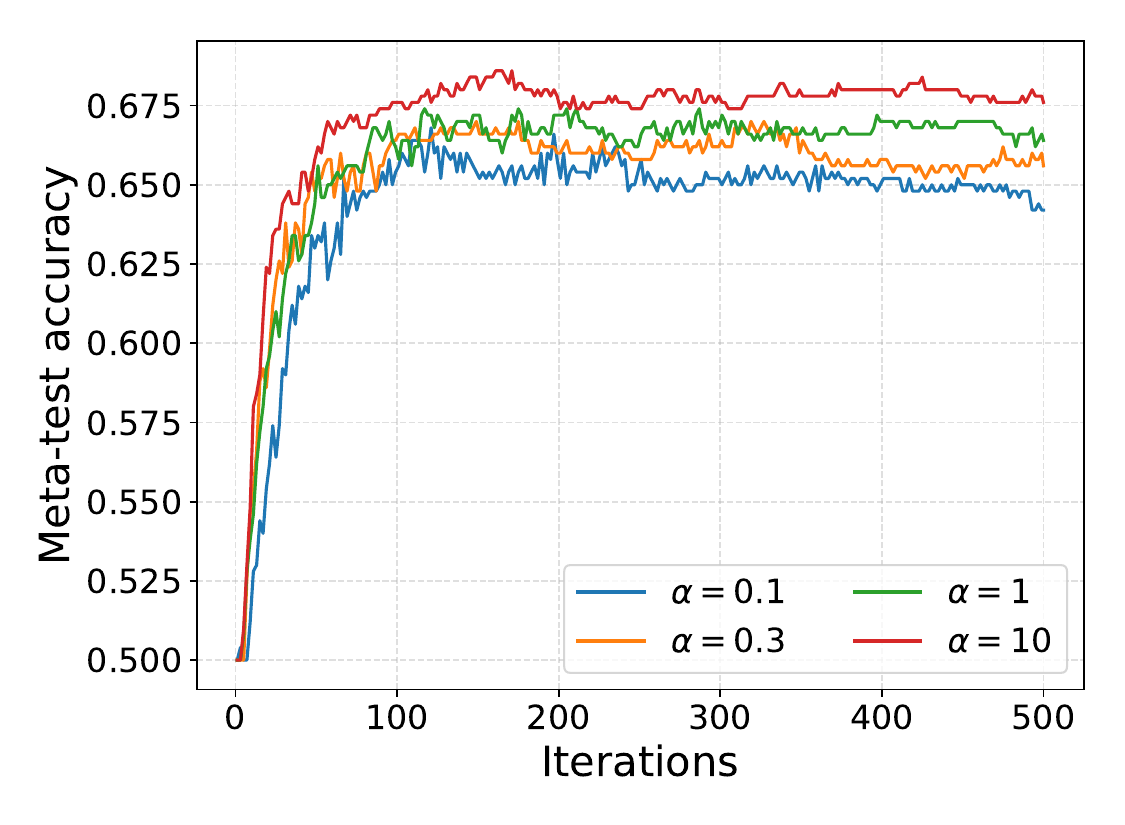}\label{fig.test_comp_heterogeneity}} 
\caption{Performance comparisons of the FBi-SGD method under different data heterogeneity.} 
\end{figure*}
We evaluate the impact of data heterogeneity by adopting a Dirichlet-style non-i.i.d.\ partition across clients, where a smaller concentration parameter $\alpha$ corresponds to more skewed local data distributions and thus higher heterogeneity. As shown in Fig.~\ref{fig.gap_comp_heterogeneity}, increasing $\alpha$ consistently accelerates convergence and reduces the final stationarity gap: the most heterogeneous setting ($\alpha=0.1$) converges to the largest gap, while the near-i.i.d.\ case ($\alpha=10$) achieves the smallest gap and keeps decreasing throughout training. This trend shows that stronger heterogeneity induces larger client drift and gradient variance, making it harder for the global iterate to approach a stationary region.

The accuracy curves in Figs.~\ref{fig.train_comp_heterogeneity} and~\ref{fig.test_comp_heterogeneity} further corroborate the above observation. In general, larger $\alpha$ yields better and more stable meta-test performance, with $\alpha=10$ achieving the highest test accuracy and $\alpha=0.1$ the lowest. Interestingly, the meta-train accuracy does not strictly improve with $\alpha$: the moderate heterogeneity case ($\alpha=0.3$) attains the highest train accuracy, while its test accuracy remains below the $\alpha=10$ case. This reveals that reducing heterogeneity primarily benefits generalization and optimization stability, whereas moderate heterogeneity can increase training fit but does not necessarily translate into better meta-test accuracy.

\subsection{Impact of Local Steps}
\begin{figure*}[t]
\subfigure[Stationarity gap.]{\includegraphics[width=.33\textwidth]{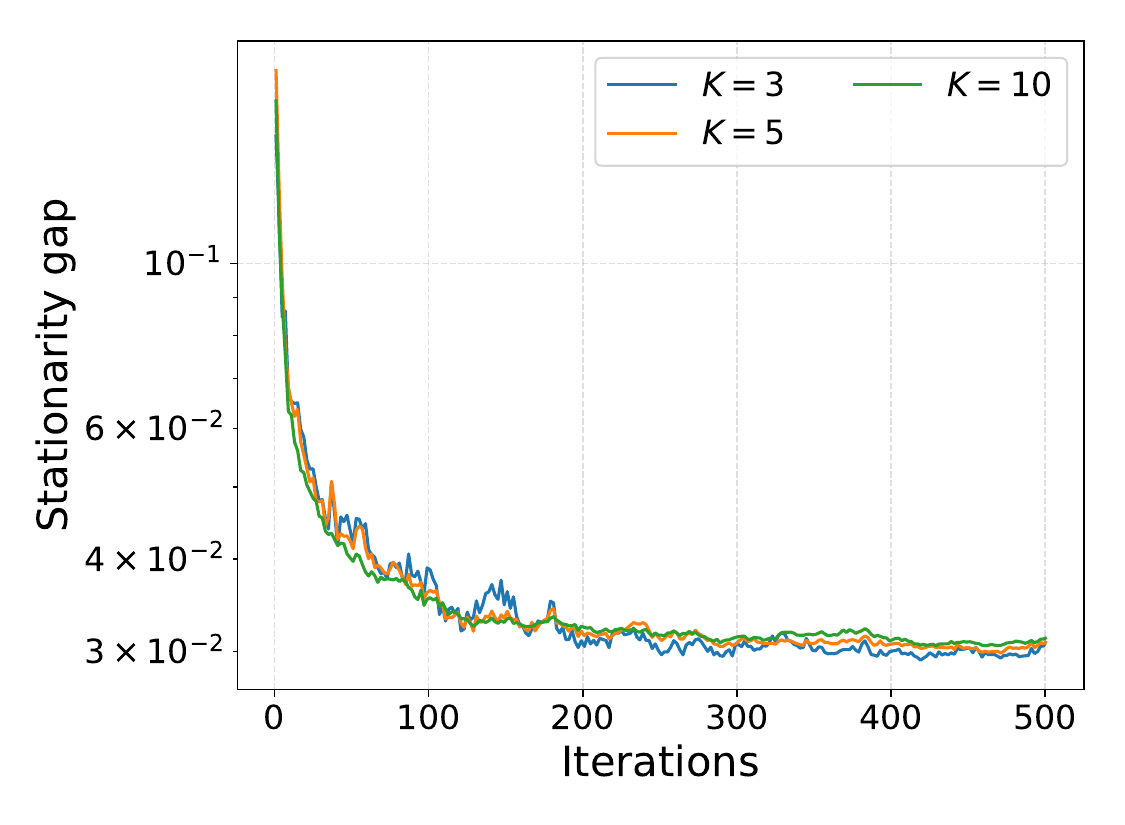}\label{fig.gap_comp_k}}
\subfigure[Train accuracy.]{\includegraphics[width=.33\textwidth]{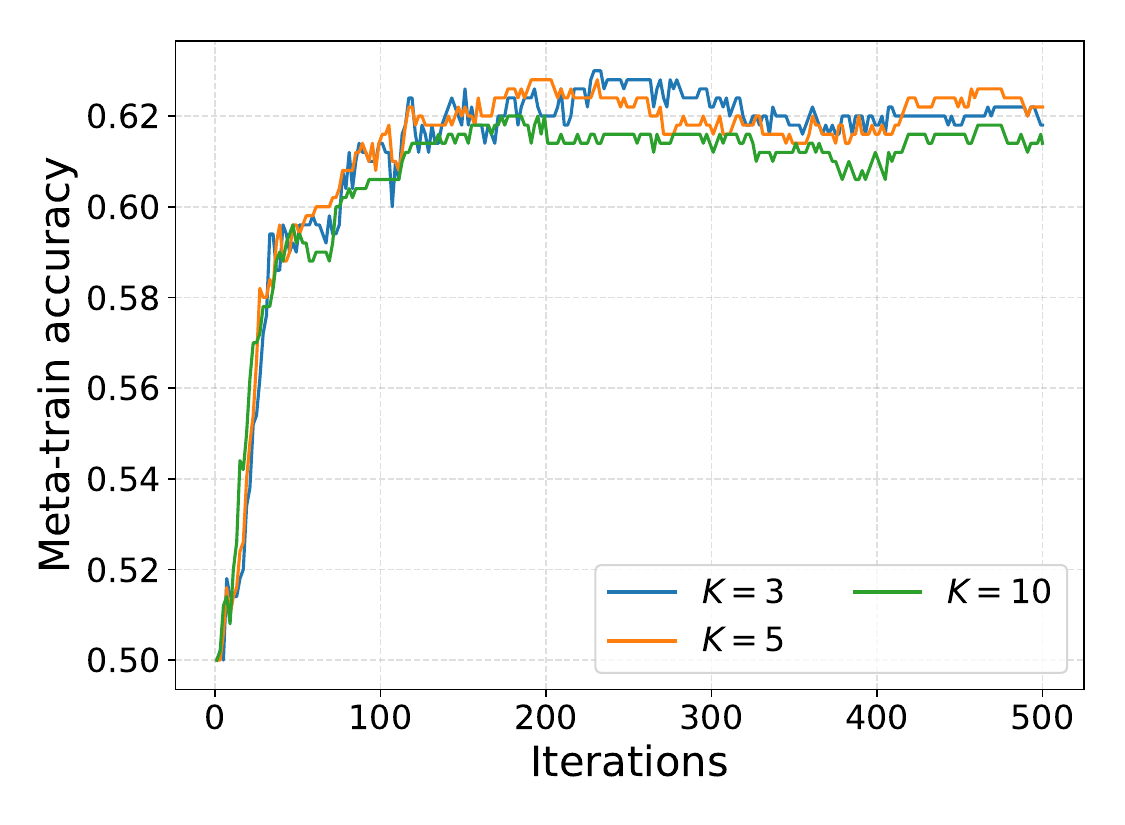}\label{fig.train_comp_k}}
\subfigure[Test accuracy.]{\includegraphics[width=.33\textwidth]{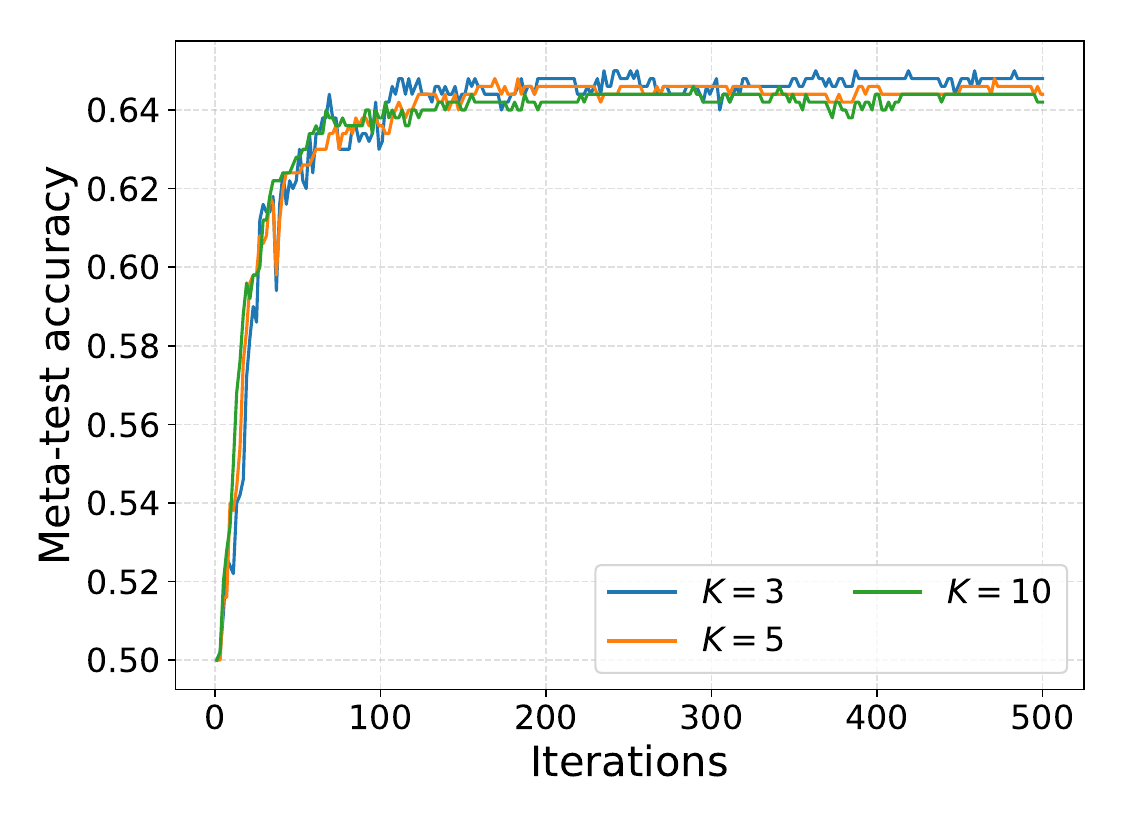}\label{fig.test_comp_k}} 
\caption{Performance comparisons of the FBi-SGD method under different local steps ($\varepsilon_c=\varepsilon_d=0.05$). $K$ denotes the number of local steps.} 
\end{figure*}
As shown in Fig.~\ref{fig.gap_comp_k}, varying the number of local steps $K$ mainly affects the transient behavior but has a limited impact on the final convergence level. All three settings ($K\in\{3,5,10\}$) lead to a fast decrease of the stationarity gap in the early stage, followed by a gradual stabilization around a similar magnitude (approximately $3\times 10^{-2}$). Compared with smaller $K$, using more local steps yields a slightly smoother trajectory, while the overall convergence speed and the final gap remain largely comparable.

Figs.~\ref{fig.train_comp_k} and~\ref{fig.test_comp_k} further present that both meta-train and meta-test accuracies improve rapidly within the first $\sim$100 iterations and then plateau, and the curves under different $K$ almost overlap throughout training. The differences in the final accuracies are marginal, with $K=3$ and $K=5$ being slightly better than $K=10$ in terms of the achieved plateau. Increasing $K$ does not noticeably improve the performance in this setting, and moderate local steps (e.g., $K=3$ or $K=5$) already achieve strong and stable performance.

\subsection{Impact of Participation Ratio}
\begin{figure*}[t]
\subfigure[Stationarity gap.]{\includegraphics[width=.33\textwidth]{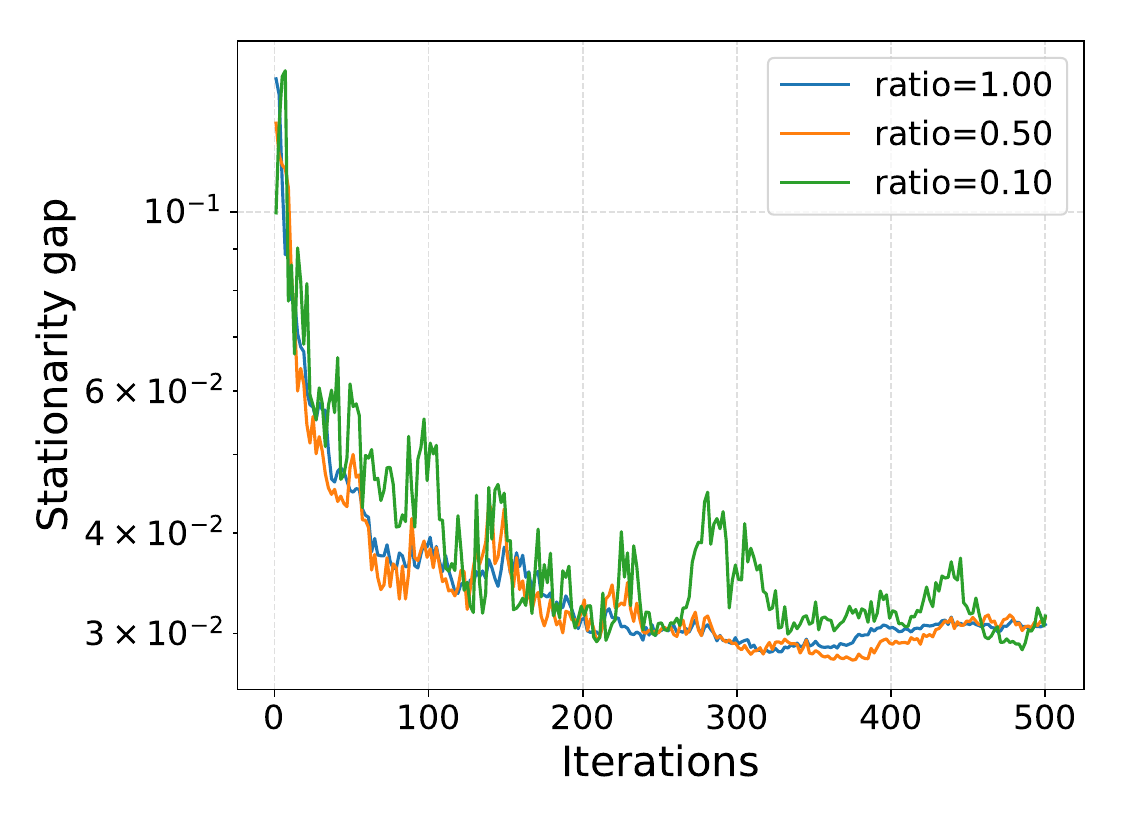}\label{fig.gap_comp_partial}}
\subfigure[Train accuracy.]{\includegraphics[width=.33\textwidth]{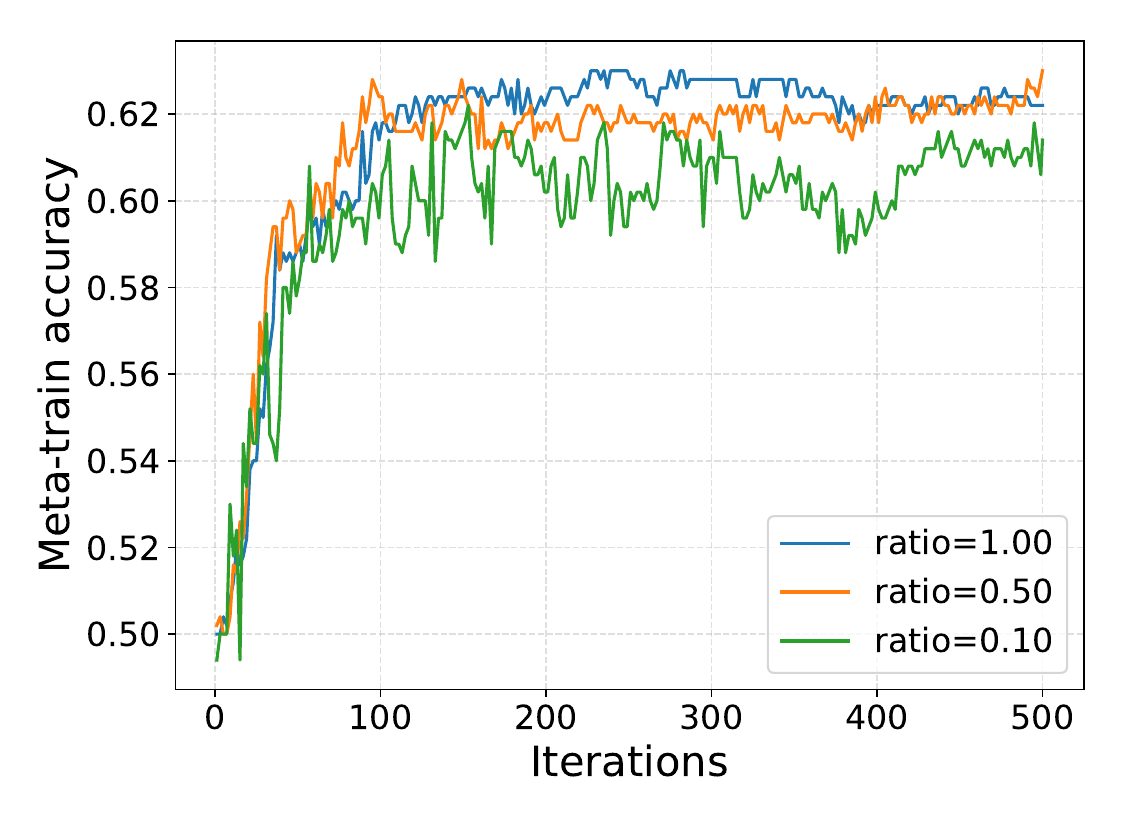}\label{fig.train_comp_partial}}
\subfigure[Test accuracy.]{\includegraphics[width=.33\textwidth]{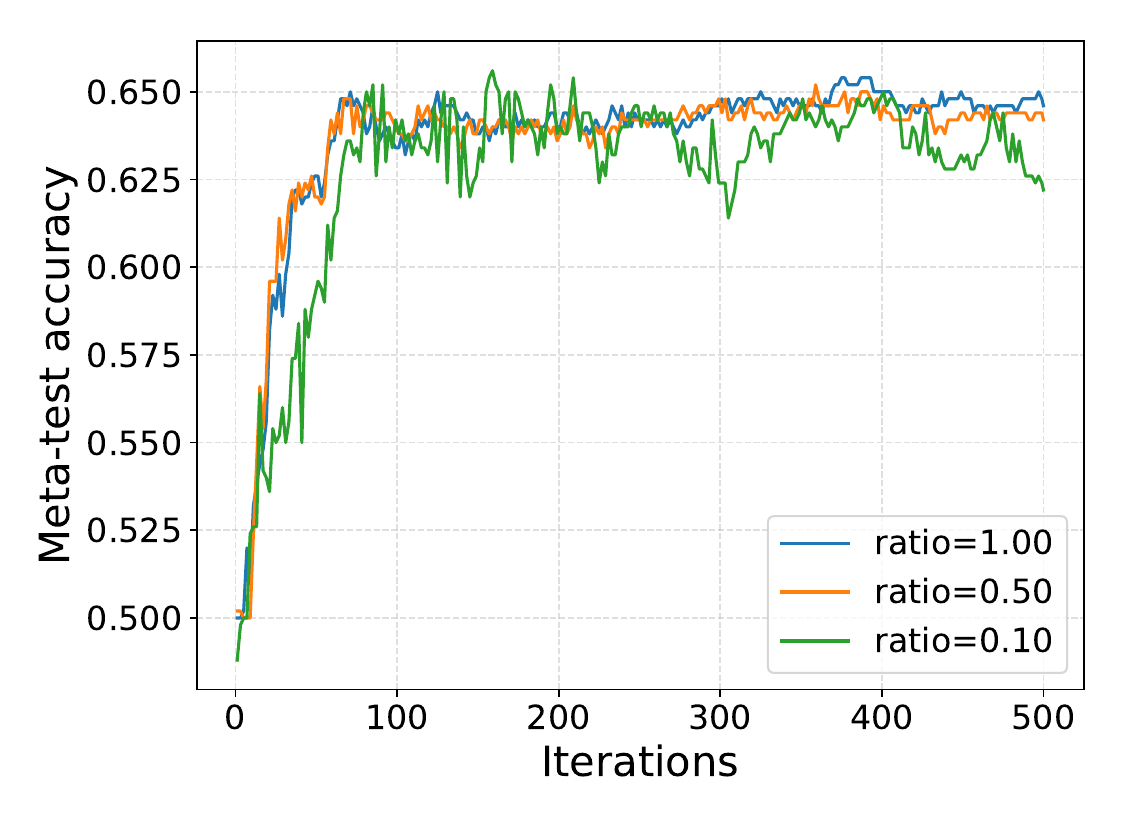}\label{fig.test_comp_partial}} 
\caption{Performance comparisons of the FBi-SGD method under different participation ratios ($\varepsilon_c=\varepsilon_d=0.05$).} 
\end{figure*}
As shown in Fig.~\ref{fig.gap_comp_partial}, reducing the participation ratio leads to a slower and less stable convergence behavior. With full participation (ratio$=1.00$), the stationarity gap decreases steadily and reaches the smallest value among all cases, while partial participation (ratio$=0.50$ and $0.10$) results in a noticeably larger gap throughout training, indicating that limited client availability introduces additional variance in the aggregated updates and weakens the descent toward a stationary point. In particular, the most restrictive setting (ratio$=0.10$) exhibits the largest gap and the slowest decay. 

The accuracy curves in Figs.~\ref{fig.train_comp_partial} and~\ref{fig.test_comp_partial} further corroborate this trend. Both meta-train and meta-test accuracies are the highest under full participation, and they degrade as the participation ratio decreases. Moreover, smaller ratios yield slower improvement and lower accuracy plateaus, especially on the meta-test set, suggesting that reduced participation exacerbates client drift under non-i.i.d.\ data and hurts generalization.

\section{Detailed Theoretical Assumptions}
We give the detailed expressions of theoretical assumptions as follows:

    \textbf{Strong convexity}: Assume that the loss functions $f_i(\bm{x})$ and $g_i(\bm{x})$ are $\gamma_f$ and $\gamma_g$ strongly convex, respectively, i.e., 
    \begin{equation}
        f_i(\bm{x})\geq f_i(\bm{x}^\prime)+\nabla_x f_i(\bm{x}^\prime)^\intercal (\bm{x}-\bm{x}^\prime)+\frac{\gamma_f}{2}\|\bm{x}-\bm{x}^\prime\|^2,
    \end{equation}
    \begin{equation}
        g_i(\bm{x},\bm{y})\geq g_i(\bm{x},\bm{y}^\prime)+\nabla_y g_i(\bm{x},\bm{y}^\prime)^\intercal (\bm{y}-\bm{y}^\prime)+\frac{\gamma_g}{2}\|\bm{y}-\bm{y}^\prime\|^2.
    \end{equation}

    \textbf{Smoothness of the UL loss function and distribution}: Assume that the loss function $f_i(\bm{x},\bm{y};\xi)$ w.r.t. $\bm{x}, \bm{y}$ is smooth and the gradient of $f_i(\bm{x},\bm{y};\xi)$ w.r.t. $\bm{\xi}$ is $L_f^\xi$-Lipschitz continuous $\forall\bm{x}, \bm{y}$, i.e.,
    \begin{align}
        \|\nabla_xf_i(\bm{x},\bm{y};\xi)-\nabla_xf_i(\bm{x},\bm{y};\xi^\prime)\|\leq L_f^\xi\|\xi-\xi^\prime\|,
        \qquad
        \|\nabla_yf_i(\bm{x},\bm{y};\xi)-\nabla_yf_i(\bm{x},\bm{y};\xi^\prime)\|\leq L_f^\xi\|\xi-\xi^\prime\|.
    \end{align}
    Assume that the gradient of the loss function $f_i(\bm{x},\bm{y};\xi)$ w.r.t. $\bm{x}$ is Lipschitz continuous $\forall\bm{x}, \bm{x}^\prime, \bm{y}, \bm{y}^\prime$, i.e.,
    \begin{align}
        \|\underset{\xi\sim\mathcal{C}_i(\cdot)}{\mathbb{E}} \nabla_xf_i (\bm{x},\bm{y};\xi) - \nabla_xf_i (\bm{x}^\prime,\bm{y};\xi)\|\leq L^x_{f}\|\bm{x}-\bm{x}^\prime\|,
        \|\underset{\xi\sim\mathcal{C}_i(\cdot)}{\mathbb{E}} \nabla_xf_i (\bm{x},\bm{y};\xi) - \nabla_xf_i (\bm{x},\bm{y}^\prime;\xi)\|\leq \bar{L}^x_{f}\|\bm{y}-\bm{y}^\prime\|.
    \end{align}
    Assume that the gradient of the loss function $f_i(\bm{x},\bm{y};\xi)$ w.r.t. $\bm{y}$ is Lipschitz continuous $\forall\bm{x}, \bm{x}^\prime, \bm{y}, \bm{y}^\prime$, i.e.,
    \begin{align}
        \|\underset{\xi\sim\mathcal{C}_i(\cdot)}{\mathbb{E}} \nabla_yf_i (\bm{x},\bm{y};\xi) - \nabla_yf_i (\bm{x}^\prime,\bm{y};\xi)\|\leq L^y_{f}\|\bm{x}-\bm{x}^\prime\|,
        \|\underset{\xi\sim\mathcal{C}_i(\cdot)}{\mathbb{E}} \nabla_yf_i (\bm{x},\bm{y};\xi) - \nabla_yf_i (\bm{x},\bm{y}^\prime;\xi)\|\leq \bar{L}^y_{f}\|\bm{y}-\bm{y}^\prime\|.
    \end{align}
    \textbf{Smoothness of the LL loss function and distribution}: Assume that the loss function $g_i(\bm{x},\bm{y};\zeta)$ is smooth and the gradient of $g_i(\bm{x},\bm{y};\zeta)$ is $L_g^\zeta$-Lipschitz continuous $\forall \bm{x},\bm{y}$, i.e.,
    \begin{equation}
        \|\nabla_yg_i(\bm{x},\bm{y};\zeta)-\nabla_yg_i(\bm{x},\bm{y};\zeta^\prime)\|\leq L_g^\zeta\|\zeta-\zeta^\prime\|.
    \end{equation}
    Assume the loss function $g_i(\bm{x},\bm{y};\zeta)$ is $L_g^y$-smooth, $\forall \bm{x},\bm{y},\bm{y}^\prime$, and $L_g^x$-smooth, $\forall \bm{x},\bm{x}^\prime,\bm{y}$, i.e.,
    \begin{align}
        \|\underset{\zeta\sim\mathcal{D}_i(\cdot)}{\mathbb{E}} \nabla_y g_i(\bm{x},\bm{y};\zeta) - \nabla_y g_i(\bm{x},\bm{y}^\prime;\zeta)\| \leq L^y_g\|\bm{y}-\bm{y}^\prime\|,
        \|\underset{\zeta\sim\mathcal{D}_i(\cdot)}{\mathbb{E}} \nabla_y g_i(\bm{x},\bm{y};\zeta) - \nabla_y g_i(\bm{x}^\prime,\bm{y};\zeta)\| \leq L^x_g\|\bm{x}-\bm{x}^\prime\|,
    \end{align}
    \textbf{Second-order smoothness of the LL loss function and distribution}: Assume that the loss function $g_i(\bm{x},\bm{y})$ is continuously twice differentiable and its Jacobian and Hessian matrices are $L_{gxy}^\zeta$- and $L_{gyy}^\zeta$-Lipschitz continuous $\forall \bm{x},\bm{y}$, respectively, i.e.,
    \begin{align}
        \|\nabla_{xy}^2 g_i(\bm{x},\bm{y};\zeta)-\nabla_{xy}^2 g_i(\bm{x},\bm{y};\zeta^\prime)\|\leq L_{gxy}^\zeta \|\zeta-\zeta^\prime\|,
        \|\nabla_{yy}^2 g_i(\bm{x},\bm{y};\zeta)-\nabla_{yy}^2 g_i(\bm{x},\bm{y};\zeta^\prime)\|\leq L_{gyy}^\zeta \|\zeta-\zeta^\prime\|.
    \end{align}
    Assume that the Jacobian and Hessian matrices of the loss function $g_i(\bm{x},\bm{y})$ are Lipschitz continuous $\forall \bm{x},\bm{x}^\prime,\bm{y},\bm{y}^\prime$, i.e.,
    \begin{align}
         \|\underset{\zeta\sim\mathcal{D}_i(\cdot)}{\mathbb{E}} \nabla_{xy}^2 g_i(\bm{x},\bm{y};\zeta)-\nabla_{xy}^2 g_i(\bm{x}^\prime,\bm{y};\zeta)\|\leq L_{gxy}^x \|\bm{x}-\bm{x}^\prime\|,
         \|\underset{\zeta\sim\mathcal{D}_i(\cdot)}{\mathbb{E}} \nabla_{xy}^2 g_i(\bm{x},\bm{y};\zeta)-\nabla_{xy}^2 g_i(\bm{x},\bm{y}^\prime;\zeta)\|\leq L_{gxy}^y \|\bm{y}-\bm{y}^\prime\|,
    \end{align}
    \begin{align}
         \|\underset{\zeta\sim\mathcal{D}_i(\cdot)}{\mathbb{E}} \nabla_{yy}^2 g_i(\bm{x},\bm{y};\zeta)-\nabla_{yy}^2 g_i(\bm{x}^\prime,\bm{y};\zeta)\|\leq L_{gyy}^x \|\bm{x}-\bm{x}^\prime\|,
         \|\underset{\zeta\sim\mathcal{D}_i(\cdot)}{\mathbb{E}} \nabla_{yy}^2 g_i(\bm{x},\bm{y};\zeta)-\nabla_{yy}^2 g_i(\bm{x},\bm{y}^\prime;\zeta)\|\leq L_{gyy}^y \|\bm{y}-\bm{y}^\prime\|.
    \end{align}
    \textbf{Second-order smoothness of the UL loss function and distribution}: Assume that the loss function $f_i(\bm{x},\bm{y})$ is continuously twice differentiable and its Jacobian and Hessian matrices are $L_{fxy}^\xi$- and $L_{fyy}^\xi$-Lipschitz continuous $\forall \bm{x},\bm{y}$, respectively, i.e.,
    \begin{align}
        \|\nabla_{xy}^2 f_i(\bm{x},\bm{y};\xi)-\nabla_{xy}^2 f_i(\bm{x},\bm{y};\xi^\prime)\|\leq L_{fxy}^\xi \|\xi-\xi^\prime\|,
        \|\nabla_{yy}^2 f_i(\bm{x},\bm{y};\xi)-\nabla_{yy}^2 f_i(\bm{x},\bm{y};\xi^\prime)\|\leq L_{fyy}^\xi \|\xi-\xi^\prime\|.
    \end{align}
    Assume that the Jacobian and Hessian matrices of the loss function $f_i(\bm{x},\bm{y})$ are Lipschitz continuous $\forall \bm{x},\bm{x}^\prime,\bm{y},\bm{y}^\prime$, i.e.,
    \begin{equation}
         \|\underset{\xi\sim\mathcal{C}_i(\cdot)}{\mathbb{E}} \nabla_{xy}^2 f_i(\bm{x},\bm{y};\xi)-\nabla_{xy}^2 f_i(\bm{x}^\prime,\bm{y};\xi)\|\leq L_{fxy}^x \|\bm{x}-\bm{x}^\prime\|,
    \end{equation}
    \begin{equation}
         \|\underset{\xi\sim\mathcal{C}_i(\cdot)}{\mathbb{E}} \nabla_{xy}^2 f_i(\bm{x},\bm{y};\xi)-\nabla_{xy}^2 f_i(\bm{x},\bm{y}^\prime;\xi)\|\leq L_{fxy}^y \|\bm{y}-\bm{y}^\prime\|,
    \end{equation}
    \begin{equation}
         \|\underset{\xi\sim\mathcal{C}_i(\cdot)}{\mathbb{E}} \nabla_{yy}^2 f_i(\bm{x},\bm{y};\xi)-\nabla_{yy}^2 f_i(\bm{x}^\prime,\bm{y};\xi)\|\leq L_{fyy}^x \|\bm{x}-\bm{x}^\prime\|,
    \end{equation}
    \begin{equation}
         \|\underset{\xi\sim\mathcal{C}_i(\cdot)}{\mathbb{E}} \nabla_{yy}^2 f_i(\bm{x},\bm{y};\xi)-\nabla_{yy}^2 f_i(\bm{x},\bm{y}^\prime;\xi)\|\leq L_{fyy}^y \|\bm{y}-\bm{y}^\prime\|,
    \end{equation}
    \textbf{Third-order smoothness of the LL loss function and distribution}: Assume that the loss function $g_i(\bm{x},\bm{y})$ is continuously three times differentiable and its Jacobian and Hessian matrices are $L_{gyyx}^\zeta$-Lipschitz continuous $\forall \bm{x},\bm{y}$, respectively, i.e.,
    \begin{equation}
        \|\nabla_{yyx}^3 g_i(\bm{x},\bm{y};\zeta)-\nabla_{yyx}^3 g_i(\bm{x},\bm{y};\zeta^\prime)\|\leq L_{gyyx}^\zeta \|\zeta-\zeta^\prime\|,
    \end{equation}
    Assume that the Jacobian and Hessian matrices of the loss function $g_i(\bm{x},\bm{y})$ are Lipschitz continuous $\forall \bm{x},\bm{x}^\prime,\bm{y},\bm{y}^\prime$, i.e.,
    \begin{align}
         &\|\underset{\zeta\sim\mathcal{D}_i(\cdot)}{\mathbb{E}} \nabla_{yyx}^3 g_i(\bm{x},\bm{y};\zeta)-\nabla_{yyx}^3 g_i(\bm{x}^\prime,\bm{y};\zeta)\| \leq L_{gyyx}^x \|\bm{x}-\bm{x}^\prime\|,\\
         &\|\underset{\zeta\sim\mathcal{D}_i(\cdot)}{\mathbb{E}} \nabla_{yyx}^3 g_i(\bm{x},\bm{y};\zeta)-\nabla_{yyx}^3 g_i(\bm{x},\bm{y}^\prime;\zeta)\|\leq L_{gyyx}^y \|\bm{y}-\bm{y}^\prime\|.
    \end{align}
    \textbf{Boundedness of LL and UL gradient estimate}:
    Let $\chi_r=\sigma\{(\bm{y}_1,\bm{x}_1),\cdots,(\bm{y}_r,\bm{x}_r)\}$ represent the filtration of the random variables up to iteration $r$, where $\sigma\{\cdot\}$ is the $\sigma$-algebra generated by the random variables. Assume that the LL gradient estimate is unbiased and with bounded variance, i.e.,
    \begin{align}
        &\mathbb{E}[\widehat{\nabla}_yg_i(\bm{x},\bm{y};\zeta)|\chi_r]=\nabla_yg_i(\bm{x},\bm{y}),
        \qquad\mathbb{E}[\|\widehat{\nabla}_yg_i(\bm{x},\bm{y};\zeta)-\nabla_yg_i(\bm{x},\bm{y})\|^2|\chi_r]\leq \sigma_g^2,\\
        &\mathbb{E}[\|\widehat{\nabla}_{yy}^2g_i(\bm{x},\bm{y};\zeta)-\nabla_{yy}^2g_i(\bm{x},\bm{y})\|^2|\chi_r]\leq \sigma_{gg}^2,
        \qquad\mathbb{E}[\|\widehat{\nabla}_{xy}^2g_i(\bm{x},\bm{y};\zeta)-\nabla_{xy}^2g_i(\bm{x},\bm{y})\|^2|\chi_r]\leq \sigma_{gg}^2.
    \end{align}
    Assume that the UL gradient estimate is biased and with bounded variance, i.e.,
    \begin{align}
        &\mathbb{E}[\|\widehat{\nabla}_xf_i(\bm{x},\bm{y};\xi)|\chi_r]=\nabla_xf_i(\bm{x},\bm{y})+b_r,
        \qquad\mathbb{E}[\|\widehat{\nabla}_xf_i(\bm{x},\bm{y};\xi)-\nabla_xf_i(\bm{x},\bm{y})-b_r\|^2|\chi_r]\leq \sigma_f^2,
    \end{align}
    where $b_r$ is the bias term and $\|b_r\|\leq \sigma_r, \forall r$.

\section{Preliminaries for FBi-RRM Methods}
We present some helpful inequalities that will be used in future proofs.\\

\begin{lemma}[Boundedness of $\bm{v}^*$ and local $\bm{v}_i^*, \bm{v}_i$]\label{lemma_1}
    From Lemma 2 in \cite{yang2023simfbo}, with the fact that $\bm{v}^*=\arg\min_{\bm{v}}l(\bm{x},\bm{y}^*(\bm{x}),\bm{v})$, and under Assumptions \ref{assump_2}, \ref{assump_4}, we can easily obtain that $\|\bm{v}^*\|\leq\iota$, $\|\bm{v}_i^*\|\leq\iota$, and $\|\bm{v}_{i,r,k}\|\leq 2\iota$, where 
    $\iota=\frac{C_f^y}{\gamma_g}$.
\end{lemma}

\begin{lemma}\label{lemma_2}
    Under Assumptions \ref{assump_2} and \ref{assump_4}, we have $l_i(\bm{x},\bm{y},\bm{v})$ is $\gamma_g$-strongly convex w.r.t $\bm{v}$ and 
    \begin{align}
        &\|
            \nabla_vl_i(\bm{x},\bm{y},\bm{v})-\nabla_vl_i(\bm{x}^\prime,\bm{y}^\prime,\bm{v})
            \|\leq \left(
    2\iota(L_{gyy}^\zeta\varepsilon_d+L_{gyy}^x)+L_f^y
    \right)\|\bm{x}-\bm{x}^\prime\|+
            \left(
    2\iota L_{gyy}^y+L_f^\xi \varepsilon_c+\bar{L}_f^y
    \right)\|\bm{y}-\bm{y}^\prime\|.
    \end{align}
    Similarly, we have 
    \begin{align}
    &\|\nabla_vl_i(\bm{x},\bm{y}^*(\bm{x}),\bm{v})-\nabla_vl_i(\bm{x}^\prime,\bm{y}^*(\bm{x}^\prime),\bm{v})
            \|\leq \left(
    \frac{C_f^y}{\gamma_g}(L_{gyy}^\zeta\varepsilon_d+L_{gyy}^x)+L_f^y
    + \left(
    \frac{C_f^y}{\gamma_g}L_{gyy}^y+L_f^\xi \varepsilon_c+\bar{L}_f^y
    \right)\bar{L}_y
    \right)\|\bm{x}-\bm{x}^\prime\|,
\end{align}
and 
\begin{align}
    \|\nabla_vl_i(\bm{x},\bm{y},\bm{v})-\nabla_vl_i(\bm{x},\bm{y}^*(\bm{x}),\bm{v})
            \| \leq  \left(
    2\iota L_{gyy}^y+L_f^\xi \varepsilon_c+\bar{L}_f^y
    \right)\|\bm{y}-\bm{y}^*(\bm{x})\|.
\end{align}
    \begin{proof}
    Strong convexity: It's obvious that
        \begin{align}
            \nabla_{vv}^2 l_i(\bm{x},\bm{y})=\nabla_{yy}^2 g_i(\bm{x},\bm{y})\succeq\gamma_g \bm{I}.
        \end{align}
    Note that
        \begin{align}
            \|
            \nabla_vl_i(\bm{x},\bm{y},\bm{v})-\nabla_vl_i(\bm{x}^\prime,\bm{y}^\prime,\bm{v})
            \| &\leq \left\|
                \left(
                    \nabla_{yy}^2g_i(\bm{x},\bm{y})-\nabla_{yy}^2g_i(\bm{x}^\prime,\bm{y}^\prime)
                \right)\bm{v}
            \right\|+
            \left\|
                \nabla_y f_i(\bm{x},\bm{y})-\nabla_y f_i(\bm{x}^\prime,\bm{y}^\prime)
            \right\|\nonumber \\
            &\leq 2\iota\|\nabla_{yy}^2g_i(\bm{x},\bm{y})-\nabla_{yy}^2g_i(\bm{x}^\prime,\bm{y}^\prime)\| + \left\|
                \nabla_y f_i(\bm{x},\bm{y})-\nabla_y f_i(\bm{x}^\prime,\bm{y}^\prime)
            \right\|.\label{lemma_2_eq1}
        \end{align}
For the first term on the right side of Inequality (\ref{lemma_2_eq1}), we have 
\begin{align}
    &\|\nabla_{yy}^2g_i(\bm{x},\bm{y})-\nabla_{yy}^2g_i(\bm{x}^\prime,\bm{y}^\prime)\|\nonumber \\
    &\leq\|\underset{\zeta\sim\mathcal{D}_i(\bm{x})}{\mathbb{E}}\nabla_{yy}^2g_i(\bm{x},\bm{y};\zeta)-\underset{\zeta\sim\mathcal{D}_i(\bm{x}^\prime)}{\mathbb{E}}\nabla_{yy}^2g_i(\bm{x}^\prime,\bm{y}^\prime;\zeta)
    \|\nonumber \\
    &\leq \|\underset{\zeta\sim\mathcal{D}_i(\bm{x})}{\mathbb{E}}\nabla_{yy}^2g_i(\bm{x},\bm{y};\zeta)-\underset{\zeta\sim\mathcal{D}_i(\bm{x}^\prime)}{\mathbb{E}}\nabla_{yy}^2g_i(\bm{x},\bm{y};\zeta)
    \|+\|\underset{\zeta\sim\mathcal{D}_i(\bm{x}^\prime)}{\mathbb{E}}\nabla_{yy}^2g_i(\bm{x},\bm{y};\zeta)-\underset{\zeta\sim\mathcal{D}_i(\bm{x}^\prime)}{\mathbb{E}}\nabla_{yy}^2g_i(\bm{x}^\prime,\bm{y}^\prime;\zeta)
    \|\nonumber \\
    &\leq L_{gyy}^\zeta\varepsilon_d\|\bm{x}-\bm{x}^\prime\| + L_{gyy}^x\|\bm{x}-\bm{x}^\prime\| + L_{gyy}^y\|\bm{y}-\bm{y}^\prime\|\nonumber \\
    &=(L_{gyy}^\zeta\varepsilon_d+L_{gyy}^x)\|\bm{x}-\bm{x}^\prime\| + L_{gyy}^y\|\bm{y}-\bm{y}^\prime\|.
\end{align}
Similarly, for the second term on the right side of Inequality (\ref{lemma_2_eq1}), we also obtain that
\begin{align}
    &\left\|\nabla_y f_i(\bm{x},\bm{y})-\nabla_y f_i(\bm{x}^\prime,\bm{y}^\prime)
            \right\|\nonumber \\
    &\leq \|\underset{\xi\sim\mathcal{C}_i(\bm{y})}{\mathbb{E}} \nabla_y f_i(\bm{x},\bm{y};\xi)-\underset{\xi\sim\mathcal{C}_i(\bm{y}^\prime)}{\mathbb{E}} \nabla_y f_i(\bm{x}^\prime,\bm{y}^\prime;\xi)
    \|\nonumber \\
    &\leq \|\underset{\xi\sim\mathcal{C}_i(\bm{y})}{\mathbb{E}} \nabla_y f_i(\bm{x},\bm{y};\xi)-\underset{\xi\sim\mathcal{C}_i(\bm{y}^\prime)}{\mathbb{E}} \nabla_y f_i(\bm{x},\bm{y};\xi)
    \|+
    \|\underset{\xi\sim\mathcal{C}_i(\bm{y}^\prime)}{\mathbb{E}} \nabla_y f_i(\bm{x},\bm{y};\xi)-\underset{\xi\sim\mathcal{C}_i(\bm{y}^\prime)}{\mathbb{E}} \nabla_y f_i(\bm{x}^\prime,\bm{y}^\prime;\xi)
    \|\nonumber \\
    &\leq L_f^\xi \varepsilon_c\|\bm{y}-\bm{y}^\prime\|+ L_f^y\|\bm{x}-\bm{x}^\prime\| + \bar{L}_f^y\|\bm{y}-\bm{y}^\prime\|\nonumber \\
    &=L_f^y\|\bm{x}-\bm{x}^\prime\| + (L_f^\xi \varepsilon_c+\bar{L}_f^y)\|\bm{y}-\bm{y}^\prime\|.
\end{align}
Therefore, it yields that
\begin{align}
    &\|\nabla_vl_i(\bm{x},\bm{y},\bm{v})-\nabla_vl_i(\bm{x}^\prime,\bm{y}^\prime,\bm{v})
            \|\nonumber \\
    &\leq 2\iota((L_{gyy}^\zeta\varepsilon_d+L_{gyy}^x)\|\bm{x}-\bm{x}^\prime\| + L_{gyy}^y\|\bm{y}-\bm{y}^\prime\|) + L_f^y\|\bm{x}-\bm{x}^\prime\| + (L_f^\xi \varepsilon_c+\bar{L}_f^y)\|\bm{y}-\bm{y}^\prime\|\nonumber \\
    &=\left(
    2\iota(L_{gyy}^\zeta\varepsilon_d+L_{gyy}^x)+L_f^y
    \right)\|\bm{x}-\bm{x}^\prime\| + \left(
    2\iota L_{gyy}^y+L_f^\xi \varepsilon_c+\bar{L}_f^y
    \right)\|\bm{y}-\bm{y}^\prime\|.
\end{align}
Similarly, it's easy to know that
\begin{align}
    &\|\nabla_vl_i(\bm{x},\bm{y}^*(\bm{x}),\bm{v})-\nabla_vl_i(\bm{x}^\prime,\bm{y}^*(\bm{x}^\prime),\bm{v})
            \|\nonumber \\
    &\leq \left(
    2\iota(L_{gyy}^\zeta\varepsilon_d+L_{gyy}^x)+L_f^y
    \right)\|\bm{x}-\bm{x}^\prime\| + \left(
    2\iota L_{gyy}^y+L_f^\xi \varepsilon_c+\bar{L}_f^y
    \right)\|\bm{y}^*(\bm{x})-\bm{y}^*(\bm{x}^\prime)\|\nonumber \\
    &\leq \left(
    2\iota(L_{gyy}^\zeta\varepsilon_d+L_{gyy}^x)+L_f^y
    + \left(
    2\iota L_{gyy}^y+L_f^\xi \varepsilon_c+\bar{L}_f^y
    \right)\bar{L}_y
    \right)\|\bm{x}-\bm{x}^\prime\|.
\end{align}
The last inequality is based on Inequality (\ref{eq.norm_ystar_x_and_xstar}). When $\bm{x}^\prime = \bm{x}$, it further yields that
\begin{align}
    \|\nabla_vl_i(\bm{x},\bm{y},\bm{v})-\nabla_vl_i(\bm{x},\bm{y}^*(\bm{x}),\bm{v})
            \| \leq  \left(
    2\iota L_{gyy}^y+L_f^\xi \varepsilon_c+\bar{L}_f^y
    \right)\|\bm{y}-\bm{y}^*(\bm{x})\|.
\end{align}

\end{proof}
\end{lemma}

\begin{lemma}\label{lemma_3}
Under Assumptions \ref{assump_1}, \ref{assump_2}, \ref{assump_3}, and \ref{assump_4}, we have
    \begin{align}
         1.\,\, \|\nabla f_i(\bm{x},\bm{y}) - \nabla f_i(\bm{x},\bm{y}^*(\bm{x}))\|
        \leq \hat{L}_f\|\bm{y}-\bm{y}^*(\bm{x})\|,
    \end{align}
     \begin{align}
         2.\,\,\|\nabla f_i(\bm{x},\bm{y}^*(\bm{x})) - \nabla f_i(\bm{x}^\prime,\bm{y}^*(\bm{x}^\prime))\|
        \leq \tilde{L}_f\|\bm{x}-\bm{x}^\prime\|,
    \end{align}
    \begin{align}\label{eq.norm_ystar_x_and_xstar}
         3.\,\,\|\bm{y}^*(\bm{x})-\bm{y}^*(\bm{x}^\prime)\|\leq \bar{L}_y \|\bm{x}-\bm{x}^\prime\|,
    \end{align}
    \begin{align}\label{eq.norm_grad_ystar_x_and_xprime}
         4.\,\,\|\nabla\bm{y}^*(\bm{x})-\nabla\bm{y}^*(\bm{x}^\prime)\|\leq \tilde{L}_y\|\bm{x}-\bm{x}^\prime\|,
    \end{align}
    \begin{align}
         5.\,\,\|\bm{v}^*(\bm{x})-\bm{v}^*(\bm{x}^\prime)\|\leq \bar{L}_v\|\bm{x}-\bm{x}^\prime\|,
    \end{align}
    \begin{align}
         6.\,\,\|\nabla\bm{v}^*(\bm{x})-\nabla\bm{v}^*(\bm{x}^\prime)\|\leq \tilde{L}_v\|\bm{x}-\bm{x}^\prime\|,
    \end{align}
    \begin{align}
         7.\,\,\|\bm{v}^*(\bm{x},\bm{y})-\bm{v}^*(\bm{x},\bm{y}^*(\bm{x}))\|\leq \hat{L}_v\|\bm{y}-\bm{y}^*(\bm{x})\|,
    \end{align}
    where 
    \begin{align}
        \hat{L}_f = L^\xi_f \varepsilon_c+\bar{L}_f^x + \frac{L_{gxy}^y C_f^y }{\gamma_g} + \frac{C_g^{xy} C_f^y L_{gyy}^y}{\gamma_g^2} + \frac{C_{g}^{xy}(L_f^\xi\varepsilon_c+L_f^y)}{\gamma_g},
    \end{align}
    \begin{align}
        \tilde{L}_f
        &=\Bigg(
    \frac{C_g^{xy}+L_g^\zeta\varepsilon_d}{\gamma_g}(L_f^\xi\varepsilon_c+\bar{L}_f^x)+ L_f^x
    + \frac{C_f^y}{\gamma_g} \left(\frac{C_g^{xy}+L_g^\zeta\varepsilon_d}{\gamma_g}L_{gxy}^y+L_{gxy}^\zeta\varepsilon_d+L_{gxy}^x\right)\nonumber \\
    &\quad \quad+ \frac{C_g^{xy} C_f^y}{\gamma_g^2} \left(L_{gyy}^\zeta\varepsilon_d+L_{gyy}^x+\frac{C_g^{xy}+L_g^\zeta\varepsilon_d}{\gamma_g}L_{gyy}^y
    \right)
    + \frac{C_{g}^{xy}}{\gamma_g} \left(\frac{C_g^{xy}+L_g^\zeta\varepsilon_d}{\gamma_g}(L_f^\xi\varepsilon_c+L_f^y)+\bar{L}_f^y
    \right)
    \Bigg),
    \end{align}
    \begin{align}
        \bar{L}_y = \frac{C_g^{xy}+L_g^\zeta\varepsilon_d}{\gamma_g},
    \end{align}
    \begin{align}
        \tilde{L}_y = \frac{C_g^{xy}}{\gamma_g^2}(
    L_{gyy}^\zeta\varepsilon_d + L_{gyy}^x +L_{gyy}^y\bar{L}_y
    )
    +\frac{1}{\gamma_g}(
    L_{gxy}^\zeta\varepsilon_d + L_{gxy}^x +L_{gxy}^y\bar{L}_y
    ),
    \end{align}
    \begin{align}
        \bar{L}_v=\frac{1}{\gamma_g}\left(\frac{C_g^{xy}+L_g^\zeta\varepsilon_d}{\gamma_g}(L_f^\xi\varepsilon_c+L_f^y)+\bar{L}_f^y
    \right)+\frac{C_f^y}{\gamma_g^2} \left(L_{gyy}^\zeta\varepsilon_d+L_{gyy}^x+\frac{C_g^{xy}+L_g^\zeta\varepsilon_d}{\gamma_g}L_{gyy}^y
    \right),
    \end{align}
    \begin{align}
        &\tilde{L}_v\nonumber \\
        &=\frac{1}{\gamma_g}\left(
    \bar{L}_y \bigg(
    L_{fxy}^\xi\varepsilon_c+ L_{fxy}^y + \frac{C_f^y}{\gamma_g}L_{gyyx}^y + \bar{L}_vL_{gyy}^y
    \bigg)+
    L_{fxy}^x + C_{g}^{yyx}\bar{L}_v + \frac{C_f^y}{\gamma_g} L_{gyyx}^\zeta\varepsilon_d+L_{gyyx}^x + \bar{L}_v(L_{gyy}^\zeta\varepsilon_d+L_{gyy}^x
    )
    \right),
    \end{align}
    \begin{align}
        &\hat{L}_v=\left(\frac{1}{\gamma_g}\left(L_f^\xi\varepsilon_c+\bar{L}_f^y
    \right)+\frac{C_f^y}{\gamma_g^2}L_{gyy}
    \right).
    \end{align}
\end{lemma}
\begin{proof}
The first three inequalities are Lemma 1, 3, and 2 in \cite{lu2023bilevel}, respectively.

For the fourth inequality, we have
\begin{align}
    &\|\nabla\bm{y}^*(\bm{x})-\nabla\bm{y}^*(\bm{x}^\prime)\|\nonumber \\
    &=\Bigg\|- \sum_{i\in C_r}\tilde{p}_i \nabla_{xy}^2 g_i(\bm{x},\bm{y}^*(\bm{x})) \cdot \left[\sum_{i\in C_r}\tilde{p}_i\nabla_{yy}^2 g_i(\bm{x},\bm{y}^*(\bm{x}))\right]^{-1}+ \sum_{i\in C_r}\tilde{p}_i \nabla_{xy}^2 g_i(\bm{x}^\prime,\bm{y}^*(\bm{x}^\prime)) \left[\sum_{i\in C_r}\tilde{p}_i\nabla_{yy}^2 g_i(\bm{x}^\prime,\bm{y}^*(\bm{x}^\prime))\right]^{-1}\Bigg\|\nonumber\\
    &\leq \Bigg\| \sum_{i\in C_r}\tilde{p}_i \underset{\zeta\sim\mathcal{D}_i(\bm{x})}{\mathbb{E}}\nabla_{xy}^2 g_i(\bm{x},\bm{y}^*(\bm{x});\zeta) \cdot \left[\sum_{i\in C_r}\tilde{p}_i\underset{\zeta\sim\mathcal{D}_i(\bm{x})}{\mathbb{E}}\nabla_{yy}^2 g_i(\bm{x},\bm{y}^*(\bm{x});\zeta)\right]^{-1} \nonumber\\
    &\quad- \sum_{i\in C_r}\tilde{p}_i\underset{\zeta\sim\mathcal{D}_i(\bm{x}^\prime)}{\mathbb{E}} \nabla_{xy}^2 g_i(\bm{x}^\prime,\bm{y}^*(\bm{x}^\prime);\zeta)\cdot \left[\sum_{i\in C_r}\tilde{p}_i\underset{\zeta\sim\mathcal{D}_i(\bm{x}^\prime)}{\mathbb{E}}\nabla_{yy}^2 g_i(\bm{x}^\prime,\bm{y}^*(\bm{x}^\prime);\zeta)\right]^{-1}\Bigg\|
    \nonumber \\
    &\leq \Bigg\| \sum_{i\in C_r}\tilde{p}_i \underset{\zeta\sim\mathcal{D}_i(\bm{x})}{\mathbb{E}}\nabla_{xy}^2 g_i(\bm{x},\bm{y}^*(\bm{x});\zeta) \cdot \left[\sum_{i\in C_r}\tilde{p}_i\underset{\zeta\sim\mathcal{D}_i(\bm{x})}{\mathbb{E}}\nabla_{yy}^2 g_i(\bm{x},\bm{y}^*(\bm{x});\zeta)\right]^{-1} \nonumber\\
    &\quad - \sum_{i\in C_r}\tilde{p}_i \underset{\zeta\sim\mathcal{D}_i(\bm{x})}{\mathbb{E}}\nabla_{xy}^2 g_i(\bm{x},\bm{y}^*(\bm{x});\zeta) \cdot \left[\sum_{i\in C_r}\tilde{p}_i\underset{\zeta\sim\mathcal{D}_i(\bm{x}^\prime)}{\mathbb{E}}\nabla_{yy}^2 g_i(\bm{x}^\prime,\bm{y}^*(\bm{x}^\prime);\zeta)\right]^{-1} \Bigg\|\nonumber\\
    &\quad +\Bigg\|\sum_{i\in C_r}\tilde{p}_i \underset{\zeta\sim\mathcal{D}_i(\bm{x})}{\mathbb{E}}\nabla_{xy}^2 g_i(\bm{x},\bm{y}^*(\bm{x});\zeta) \cdot \left[\sum_{i\in C_r}\tilde{p}_i\underset{\zeta\sim\mathcal{D}_i(\bm{x}^\prime)}{\mathbb{E}}\nabla_{yy}^2 g_i(\bm{x}^\prime,\bm{y}^*(\bm{x}^\prime);\zeta)\right]^{-1}\nonumber \\
    &\quad-\sum_{i\in C_r}\tilde{p}_i\underset{\zeta\sim\mathcal{D}_i(\bm{x}^\prime)}{\mathbb{E}} \nabla_{xy}^2 g_i(\bm{x}^\prime,\bm{y}^*(\bm{x}^\prime);\zeta)\cdot \left[\sum_{i\in C_r}\tilde{p}_i\underset{\zeta\sim\mathcal{D}_i(\bm{x}^\prime)}{\mathbb{E}}\nabla_{yy}^2 g_i(\bm{x}^\prime,\bm{y}^*(\bm{x}^\prime);\zeta)\right]^{-1}
    \Bigg\|\nonumber \\
    &\leq \frac{C_g^{xy}}{\gamma_g^2}\Bigg\|
    \sum_{i\in C_r}\tilde{p}_i\underset{\zeta\sim\mathcal{D}_i(\bm{x})}{\mathbb{E}}\nabla_{yy}^2 g_i(\bm{x},\bm{y}^*(\bm{x});\zeta)-\sum_{i\in C_r}\tilde{p}_i\underset{\zeta\sim\mathcal{D}_i(\bm{x}^\prime)}{\mathbb{E}}\nabla_{yy}^2 g_i(\bm{x}^\prime,\bm{y}^*(\bm{x}^\prime);\zeta)
    \Bigg\|\nonumber \\
    &\quad + \frac{1}{\gamma_g}\Bigg\|
    \sum_{i\in C_r}\tilde{p}_i \underset{\zeta\sim\mathcal{D}_i(\bm{x})}{\mathbb{E}}\nabla_{xy}^2 g_i(\bm{x},\bm{y}^*(\bm{x});\zeta)-\sum_{i\in C_r}\tilde{p}_i\underset{\zeta\sim\mathcal{D}_i(\bm{x}^\prime)}{\mathbb{E}} \nabla_{xy}^2 g_i(\bm{x}^\prime,\bm{y}^*(\bm{x}^\prime);\zeta)
    \Bigg\|\nonumber \\
    &\leq \frac{C_g^{xy}}{\gamma_g^2} \Bigg(\Bigg\|
    \sum_{i\in C_r}\tilde{p}_i\underset{\zeta\sim\mathcal{D}_i(\bm{x})}{\mathbb{E}}\nabla_{yy}^2 g_i(\bm{x},\bm{y}^*(\bm{x});\zeta)-\sum_{i\in C_r}\tilde{p}_i\underset{\zeta\sim\mathcal{D}_i(\bm{x}^\prime)}{\mathbb{E}}\nabla_{yy}^2 g_i(\bm{x},\bm{y}^*(\bm{x});\zeta)
    \Bigg\|\nonumber \\
    &\quad +\Bigg\|
    \sum_{i\in C_r}\tilde{p}_i\underset{\zeta\sim\mathcal{D}_i(\bm{x}^\prime)}{\mathbb{E}}\nabla_{yy}^2 g_i(\bm{x},\bm{y}^*(\bm{x});\zeta) - \sum_{i\in C_r}\tilde{p}_i\underset{\zeta\sim\mathcal{D}_i(\bm{x}^\prime)}{\mathbb{E}}\nabla_{yy}^2 g_i(\bm{x}^\prime,\bm{y}^*(\bm{x}^\prime);\zeta)
    \Bigg\|\Bigg)\nonumber \\
    &\quad +\frac{1}{\gamma_g}\Bigg(\Bigg\|
    \sum_{i\in C_r}\tilde{p}_i \underset{\zeta\sim\mathcal{D}_i(\bm{x})}{\mathbb{E}}\nabla_{xy}^2 g_i(\bm{x},\bm{y}^*(\bm{x});\zeta)-\sum_{i\in C_r}\tilde{p}_i\underset{\zeta\sim\mathcal{D}_i(\bm{x}^\prime)}{\mathbb{E}} \nabla_{xy}^2 g_i(\bm{x},\bm{y}^*(\bm{x});\zeta)
    \Bigg\|\nonumber \\
    &\quad +\Bigg\|
    \sum_{i\in C_r}\tilde{p}_i\underset{\zeta\sim\mathcal{D}_i(\bm{x}^\prime)}{\mathbb{E}} \nabla_{xy}^2 g_i(\bm{x},\bm{y}^*(\bm{x});\zeta) - \sum_{i\in C_r}\tilde{p}_i\underset{\zeta\sim\mathcal{D}_i(\bm{x}^\prime)}{\mathbb{E}} \nabla_{xy}^2 g_i(\bm{x}^\prime,\bm{y}^*(\bm{x}^\prime);\zeta)
    \Bigg\|\Bigg)\nonumber \\
    &\leq  \Bigg(\frac{C_g^{xy}}{\gamma_g^2}(
    L_{gyy}^\zeta\varepsilon_d + L_{gyy}^x +L_{gyy}^y\bar{L}_y
    )
    +\frac{1}{\gamma_g}(
    L_{gxy}^\zeta\varepsilon_d + L_{gxy}^x +L_{gxy}^y\bar{L}_y
    )
    \Bigg)
    \|\bm{x}-\bm{x}^\prime\|,
\end{align}
where the third inequality is based on the fact that 
\begin{align}
    \|H_2^{-1}-H_1^{-1}\|=\|H_1^{-1}(H_1-H_2)H_2^{-1}\|\leq \|H_1^{-1}\|\|H_2^{-1}\|\|(H_1-H_2)\|,
\end{align}
for any invertible matrices $H_1$ and $H_2$, and the last inequality follows Corollary 3.1 in \cite{drusvyatskiy2023stochastic} and the third conclusion in Lemma \ref{lemma_3}.

For the fifth inequality, we know that 
\begin{align}
     &\|\bm{v}^*(\bm{x})-\bm{v}^*(\bm{x}^\prime)\| \nonumber \\
     &= \Bigg\| \left[\sum_{i\in C_r}\tilde{p}_i\nabla_{yy}^2 g_i(\bm{x},\bm{y}^*(\bm{x}))\right]^{-1}\sum_{i\in C_r}\tilde{p}_i \nabla_y f_i (\bm{x},\bm{y}^*(\bm{x}))- \left[\sum_{i\in C_r}\tilde{p}_i\nabla_{yy}^2 g_i(\bm{x}^\prime,\bm{y}^*(\bm{x}^\prime))\right]^{-1}\sum_{i\in C_r}\tilde{p}_i \nabla_y f_i (\bm{x}^\prime,\bm{y}^*(\bm{x}^\prime))
     \Bigg\|\nonumber \\
     &\leq \Bigg\| \left[\sum_{i\in C_r}\tilde{p}_i\underset{\zeta\sim\mathcal{D}_i(\bm{x})}{\mathbb{E}}\nabla_{yy}^2 g_i(\bm{x},\bm{y}^*(\bm{x});\zeta)\right]^{-1}\Bigg(\sum_{i\in C_r}\tilde{p}_i \underset{\xi\sim\mathcal{C}_i(\bm{y}^*(\bm{x}))}{\mathbb{E}}\nabla_y f_i (\bm{x},\bm{y}^*(\bm{x});\xi)\nonumber \\
     &\quad -\sum_{i\in C_r}\tilde{p}_i \underset{\xi\sim\mathcal{C}_i(\bm{y}^*(\bm{x}^\prime))}{\mathbb{E}}\nabla_y f_i (\bm{x}^\prime,\bm{y}^*(\bm{x}^\prime);\xi)\Bigg)\Bigg\|+\Bigg\|\Bigg(\left[\sum_{i\in C_r}\tilde{p}_i\underset{\zeta\sim\mathcal{D}_i(\bm{x})}{\mathbb{E}}\nabla_{yy}^2 g_i(\bm{x},\bm{y}^*(\bm{x});\zeta)\right]^{-1}\nonumber \\
     &\quad -\left[\sum_{i\in C_r}\tilde{p}_i\underset{\zeta\sim\mathcal{D}_i(\bm{x}^\prime)}{\mathbb{E}}\nabla_{yy}^2 g_i(\bm{x}^\prime,\bm{y}^*(\bm{x}^\prime);\zeta)\right]^{-1}\Bigg)\sum_{i\in C_r}\tilde{p}_i \underset{\xi\sim\mathcal{C}_i(\bm{y}^*(\bm{x}^\prime))}{\mathbb{E}}\nabla_y f_i (\bm{x}^\prime,\bm{y}^*(\bm{x}^\prime);\xi)\Bigg\|\nonumber \\
     &\leq \frac{1}{\gamma_g}\left(\frac{C_g^{xy}+L_g^\zeta\varepsilon_d}{\gamma_g}(L_f^\xi\varepsilon_c+L_f^y)+\bar{L}_f^y
    \right)\|\bm{x}-\bm{x}^\prime\|\nonumber \\
    & \quad + \frac{C_f^y}{\gamma_g^2} \Bigg\|\sum_{i\in C_r}\tilde{p}_i\underset{\zeta\sim\mathcal{D}_i(\bm{x})}{\mathbb{E}}\nabla_{yy}^2 g_i(\bm{x},\bm{y}^*(\bm{x});\zeta)-\sum_{i\in C_r}\tilde{p}_i\underset{\zeta\sim\mathcal{D}_i(\bm{x}^\prime)}{\mathbb{E}}\nabla_{yy}^2 g_i(\bm{x}^\prime,\bm{y}^*(\bm{x}^\prime);\zeta)\Bigg\|\nonumber \\
    &\leq\left(\frac{1}{\gamma_g}\left(\frac{C_g^{xy}+L_g^\zeta\varepsilon_d}{\gamma_g}(L_f^\xi\varepsilon_c+L_f^y)+\bar{L}_f^y
    \right)+\frac{C_f^y}{\gamma_g^2} \left(L_{gyy}^\zeta\varepsilon_d+L_{gyy}^x+\frac{C_g^{xy}+L_g^\zeta\varepsilon_d}{\gamma_g}L_{gyy}^y
    \right)\right)\|\bm{x}-\bm{x}^\prime\|.
\end{align}

For the sixth inequality, we have
\begin{align}
    \nabla_{v}l(\bm{x},\bm{y}^*(\bm{x}),\bm{v}^*(\bm{x})) = \sum_{i\in C_r}\tilde{p}_i \nabla_{yy}^2g_i(\bm{x},\bm{y}^*(\bm{x}))\bm{v}^*(\bm{x}) - \nabla_y f_i(\bm{x},\bm{y}^*(\bm{x}))=0,
\end{align}
from which, we further obtain that
\begin{align}
&\left\{
\begin{aligned}
    & \nabla_{vx}^2 l(\bm{x},\bm{y}^*(\bm{x}),\bm{v}^*(\bm{x})) = 
    \sum_{i\in C_r}\tilde{p}_i(\bm{v}^*(\bm{x}))^\intercal (\nabla_{yyx}^3g_i(\bm{x},\bm{y}^*(\bm{x}))+ \nabla_{yyy}^3g_i(\bm{x},\bm{y}^*(\bm{x}))\nabla\bm{y}^*(\bm{x}) )+\nabla_{yy}^2g_i(\bm{x},\bm{y}^*(\bm{x}))\nabla\bm{v}^*(\bm{x})\nonumber \\
    &\quad \quad \quad \quad \quad \quad \quad \quad \quad \quad \quad - \nabla_{yx}^2f_i(\bm{x},\bm{y}^*(\bm{x}))-\nabla_{yy}^2f_i(\bm{x},\bm{y}^*(\bm{x}))\nabla\bm{y}^*(\bm{x})= 0,\\
    &\nabla_{vy}^2 l(\bm{x},\bm{y}^*(\bm{x}),\bm{v}^*(\bm{x})) = v(\bm{v}^*(\bm{x}))^\intercal \nabla_{yyy}^3g_i(\bm{x},\bm{y}^*(\bm{x})) - \nabla_{yy}^2g_i(\bm{x},\bm{y}^*(\bm{x}))=0,\\
\end{aligned}
\right.\\
    &\Rightarrow \sum_{i\in C_r}\tilde{p}_i\nabla_{yy}^2g_i(\bm{x},\bm{y}^*(\bm{x}))\nabla\bm{v}^*(\bm{x})=\sum_{i\in C_r}\tilde{p}_i\nabla_{yx}^2f_i(\bm{x},\bm{y}^*(\bm{x}))-(\bm{v}^*(\bm{x}))^\intercal \nabla_{yyx}^3g_i(\bm{x},\bm{y}^*(\bm{x}))
    .
\end{align}
It's easy to know that
\begin{align}
    &\sum_{i\in C_r}\tilde{p}_i\nabla_{yy}^2g_i(\bm{x},\bm{y}^*(\bm{x}))\nabla\bm{v}^*(\bm{x})-\nabla_{yy}^2g_i(\bm{x}^\prime,\bm{y}^*(\bm{x}^\prime))\nabla\bm{v}^*(\bm{x}^\prime)\nonumber \\
    &=\sum_{i\in C_r}\tilde{p}_i\nabla_{yy}^2g_i(\bm{x},\bm{y}^*(\bm{x}))\nabla\bm{v}^*(\bm{x})-\nabla_{yy}^2g_i(\bm{x}^\prime,\bm{y}^*(\bm{x}^\prime))\nabla\bm{v}^*(\bm{x})\nonumber \\
    &\quad + \nabla_{yy}^2g_i(\bm{x}^\prime,\bm{y}^*(\bm{x}^\prime))\nabla\bm{v}^*(\bm{x}) - \nabla_{yy}^2g_i(\bm{x}^\prime,\bm{y}^*(\bm{x}^\prime))\nabla\bm{v}^*(\bm{x}^\prime)\nonumber \\
    &=\sum_{i\in C_r}\tilde{p}_i\nabla_{yx}^2f_i(\bm{x},\bm{y}^*(\bm{x}))-\nabla_{yx}^2f_i(\bm{x}^\prime,\bm{y}^*(\bm{x}^\prime))-(\bm{v}^*(\bm{x}))^\intercal \nabla_{yyx}^3g_i(\bm{x},\bm{y}^*(\bm{x}))+(\bm{v}^*(\bm{x}^\prime))^\intercal \nabla_{yyx}^3g_i(\bm{x}^\prime,\bm{y}^*(\bm{x}^\prime))
\end{align}
\begin{align}
    &\sum_{i\in C_r}\tilde{p}_i\nabla_{yy}^2g_i(\bm{x}^\prime,\bm{y}^*(\bm{x}^\prime))(\nabla\bm{v}^*(\bm{x}) - \nabla\bm{v}^*(\bm{x}^\prime))\nonumber\\
    &=\sum_{i\in C_r}\tilde{p}_i\left(
    -\nabla_{yy}^2g_i(\bm{x},\bm{y}^*(\bm{x}))+\nabla_{yy}^2g_i(\bm{x}^\prime,\bm{y}^*(\bm{x}^\prime))
    \right)\nabla\bm{v}^*(\bm{x})
    +\nabla_{yx}^2f_i(\bm{x},\bm{y}^*(\bm{x}))-\nabla_{yx}^2f_i(\bm{x}^\prime,\bm{y}^*(\bm{x}^\prime))\nonumber\\
    &\quad-(\bm{v}^*(\bm{x}))^\intercal \nabla_{yyx}^3g_i(\bm{x},\bm{y}^*(\bm{x}))+(\bm{v}^*(\bm{x}^\prime))^\intercal \nabla_{yyx}^3g_i(\bm{x}^\prime,\bm{y}^*(\bm{x}^\prime)).
\end{align}
By taking the norm, we obtain that
\begin{align}
    &\Bigg\|\sum_{i\in C_r}\tilde{p}_i\nabla_{yy}^2g_i(\bm{x}^\prime,\bm{y}^*(\bm{x}^\prime))\Bigg\|\|\nabla\bm{v}^*(\bm{x}) - \nabla\bm{v}^*(\bm{x}^\prime)\|\nonumber\\
    &\leq \Bigg\|
    \sum_{i\in C_r}\tilde{p}_i\nabla_{yx}^2f_i(\bm{x},\bm{y}^*(\bm{x}))-\nabla_{yx}^2f_i(\bm{x}^\prime,\bm{y}^*(\bm{x}^\prime))
    \Bigg\|+
    \Bigg\|
    (\bm{v}^*(\bm{x})-\bm{v}^*(\bm{x}^\prime))^\intercal \sum_{i\in C_r}\tilde{p}_i\nabla_{yyx}^3g_i(\bm{x}^\prime,\bm{y}^*(\bm{x}^\prime))
    \Bigg\|\nonumber \\
    &\quad +\Bigg\|
    (\bm{v}^*(\bm{x}))^\intercal \left(\sum_{i\in C_r}\tilde{p}_i\nabla_{yyx}^3g_i(\bm{x},\bm{y}^*(\bm{x}))-\nabla_{yyx}^3g_i(\bm{x}^\prime,\bm{y}^*(\bm{x}^\prime))\right)
    \Bigg\|\nonumber \\
    &\quad+\Bigg\|
    \left(\sum_{i\in C_r}\tilde{p}_i\nabla_{yy}^2g_i(\bm{x},\bm{y}^*(\bm{x}))-\nabla_{yy}^2g_i(\bm{x}^\prime,\bm{y}^*(\bm{x}^\prime))\right)\nabla\bm{v}^*(\bm{x})
    \Bigg\|\nonumber\\
    &\leq \Bigg\|
    \sum_{i\in C_r}\tilde{p}_i\underset{\xi\sim\mathcal{C}_i(\bm{y}^*(\bm{x}))}{\mathbb{E}}\nabla_{yx}^2f_i(\bm{x},\bm{y}^*(\bm{x});\xi)-\underset{\xi\sim\mathcal{C}_i(\bm{y}^*(\bm{x}^\prime))}{\mathbb{E}}\nabla_{yx}^2f_i(\bm{x}^\prime,\bm{y}^*(\bm{x}^\prime);\xi)
    \Bigg\|\nonumber\\
    &\quad + \Bigg\|
    (\bm{v}^*(\bm{x})-\bm{v}^*(\bm{x}^\prime))^\intercal \sum_{i\in C_r}\tilde{p}_i\underset{\zeta\sim\mathcal{D}_i(\bm{x}^\prime)}{\mathbb{E}}\nabla_{yyx}^3g_i(\bm{x}^\prime,\bm{y}^*(\bm{x}^\prime);\zeta)
    \Bigg\|\nonumber \\
    &\quad + \Bigg\|
    (\bm{v}^*(\bm{x}))^\intercal \left(\sum_{i\in C_r}\tilde{p}_i\underset{\zeta\sim\mathcal{D}_i(\bm{x})}{\mathbb{E}}\nabla_{yyx}^3g_i(\bm{x},\bm{y}^*(\bm{x});\zeta)-\underset{\zeta\sim\mathcal{D}_i(\bm{x}^\prime)}{\mathbb{E}}\nabla_{yyx}^3g_i(\bm{x}^\prime,\bm{y}^*(\bm{x}^\prime);\zeta)\right)
    \Bigg\|\nonumber \\
    &\quad + \Bigg\|
    \left(\sum_{i\in C_r}\tilde{p}_i\underset{\zeta\sim\mathcal{D}_i(\bm{x})}{\mathbb{E}}\nabla_{yy}^2g_i(\bm{x},\bm{y}^*(\bm{x}))-\underset{\zeta\sim\mathcal{D}_i(\bm{x}^\prime)}{\mathbb{E}}\nabla_{yy}^2g_i(\bm{x}^\prime,\bm{y}^*(\bm{x}^\prime))\right)\nabla\bm{v}^*(\bm{x})
    \Bigg\|\nonumber\\
    &\leq \left(
    L_{fxy}^\xi\varepsilon_c+ L_{fxy}^y + \frac{C_f^y}{\gamma_g}L_{gyyx}^y + \bar{L}_vL_{gyy}^y
    \right)\|\bm{y}^*(\bm{x})-\bm{y}^*(\bm{x}^\prime)\| \nonumber \\
    &\quad + \left(
    L_{fxy}^x + C_{g}^{yyx}\bar{L}_v + \frac{C_f^y}{\gamma_g} L_{gyyx}^\zeta\varepsilon_d+L_{gyyx}^x + \bar{L}_v(L_{gyy}^\zeta\varepsilon_d+L_{gyy}^x
    )
    \right)\|\bm{x}-\bm{x}^\prime\|\nonumber \\
    &\leq \left(
    \bar{L}_y \bigg(
    L_{fxy}^\xi\varepsilon_c+ L_{fxy}^y + \frac{C_f^y}{\gamma_g}L_{gyyx}^y + \bar{L}_vL_{gyy}^y
    \bigg)+
    L_{fxy}^x + C_{g}^{yyx}\bar{L}_v + \frac{C_f^y}{\gamma_g} L_{gyyx}^\zeta\varepsilon_d+L_{gyyx}^x + \bar{L}_v(L_{gyy}^\zeta\varepsilon_d+L_{gyy}^x
    )
    \right)\nonumber \\
    &\quad \cdot
    \|\bm{x}-\bm{x}^\prime\|.
\end{align}
Thus, it's easy to know that
\begin{align}
    &\|\nabla\bm{v}^*(\bm{x}) - \nabla\bm{v}^*(\bm{x}^\prime)\|\nonumber\\
    &\leq \frac{1}{\gamma_g}\left(
    \bar{L}_y \bigg(
    L_{fxy}^\xi\varepsilon_c+ L_{fxy}^y + \frac{C_f^y}{\gamma_g}L_{gyyx}^y + \bar{L}_vL_{gyy}^y
    \bigg)+
    L_{fxy}^x + C_{g}^{yyx}\bar{L}_v + \frac{C_f^y}{\gamma_g} L_{gyyx}^\zeta\varepsilon_d+L_{gyyx}^x + \bar{L}_v(L_{gyy}^\zeta\varepsilon_d+L_{gyy}^x
    )
    \right)\nonumber \\
    &\quad \cdot
    \|\bm{x}-\bm{x}^\prime\|.
\end{align}
For the seventh inequality, we have
\begin{align}
     &\|\bm{v}^*(\bm{x},\bm{y})-\bm{v}^*(\bm{x},\bm{y}^*(\bm{x}))\| \nonumber \\
     &= \Bigg\| \left[\sum_{i\in C_r}\tilde{p}_i\nabla_{yy}^2 g_i(\bm{x},\bm{y})\right]^{-1}\sum_{i\in C_r}\tilde{p}_i \nabla_y f_i (\bm{x},\bm{y})- \left[\sum_{i\in C_r}\tilde{p}_i\nabla_{yy}^2 g_i(\bm{x},\bm{y}^*(\bm{x}))\right]^{-1}\sum_{i\in C_r}\tilde{p}_i \nabla_y f_i (\bm{x},\bm{y}^*(\bm{x}))
     \Bigg\|\nonumber \\
     &\leq \Bigg\| \left[\sum_{i\in C_r}\tilde{p}_i\underset{\zeta\sim\mathcal{D}_i(\bm{x})}{\mathbb{E}}\nabla_{yy}^2 g_i(\bm{x},\bm{y};\zeta)\right]^{-1}\Bigg(\sum_{i\in C_r}\tilde{p}_i \underset{\xi\sim\mathcal{C}_i(\bm{y})}{\mathbb{E}}\nabla_y f_i (\bm{x},\bm{y};\xi)\nonumber \\
     &\quad -\sum_{i\in C_r}\tilde{p}_i \underset{\xi\sim\mathcal{C}_i(\bm{y}^*(\bm{x}))}{\mathbb{E}}\nabla_y f_i (\bm{x},\bm{y}^*(\bm{x});\xi)\Bigg)\Bigg\|+\Bigg\|\Bigg(\left[\sum_{i\in C_r}\tilde{p}_i\underset{\zeta\sim\mathcal{D}_i(\bm{x})}{\mathbb{E}}\nabla_{yy}^2 g_i(\bm{x},\bm{y};\zeta)\right]^{-1}\nonumber \\
     &\quad -\left[\sum_{i\in C_r}\tilde{p}_i\underset{\zeta\sim\mathcal{D}_i(\bm{x})}{\mathbb{E}}\nabla_{yy}^2 g_i(\bm{x},\bm{y}^*(\bm{x});\zeta)\right]^{-1}\Bigg)\sum_{i\in C_r}\tilde{p}_i \underset{\xi\sim\mathcal{C}_i(\bm{y}^*(\bm{x}))}{\mathbb{E}}\nabla_y f_i (\bm{x},\bm{y}^*(\bm{x});\xi)\Bigg\|\nonumber \\
     &\leq \frac{1}{\gamma_g}\left(L_f^\xi\varepsilon_c+\bar{L}_f^y
    \right)\|\bm{y}-\bm{y}^*(\bm{x})\|+ \frac{C_f^y}{\gamma_g^2} \Bigg\|\sum_{i\in C_r}\tilde{p}_i\underset{\zeta\sim\mathcal{D}_i(\bm{x})}{\mathbb{E}}\nabla_{yy}^2 g_i(\bm{x},\bm{y}^*(\bm{x});\zeta)-\sum_{i\in C_r}\tilde{p}_i\underset{\zeta\sim\mathcal{D}_i(\bm{x})}{\mathbb{E}}\nabla_{yy}^2 g_i(\bm{x},\bm{y};\zeta)\Bigg\|\nonumber \\
    &\leq\left(\frac{1}{\gamma_g}\left(L_f^\xi\varepsilon_c+\bar{L}_f^y
    \right)+\frac{C_f^y}{\gamma_g^2}L_{gyy}
    \right)\|\bm{y}-\bm{y}^*(\bm{x})\|.
\end{align}

\end{proof}
\begin{lemma}\label{lemma_4}
Under Assumptions \ref{assump_1}, \ref{assump_2}, \ref{assump_3}, and \ref{assump_4}, we have
\begin{align}
    \sum_{i\in C_r}\tilde{p}_i\|\nabla_yg_i(\bm{x},\bm{y})\|\leq L_g^y\|\bm{y}-\bm{y}^*(\bm{x})\|,\\
    \sum_{i\in C_r}\tilde{p}_i\|\nabla_vl_i(\bm{x},\bm{y},\bm{v}^*)\|\leq \frac{C_f^y}{\gamma_g}C_g^{yy}+C_f^y,\\
    \sum_{i\in C_r}\tilde{p}_i\|\nabla_vl_i(\bm{x},\bm{y},\bm{v})\|\leq 2\iota C_g^{yy}+C_f^y,\\
    \sum_{i\in C_r}\tilde{p}_i\|\nabla \mathcal{F}_i(\bm{x},\bm{y},\bm{v}^*)\|\leq \frac{C_f^y}{\gamma_g}C_g^{xy}+C_f^x,\\
    \sum_{i\in C_r}\tilde{p}_i\|\nabla \mathcal{F}_i(\bm{x},\bm{y},\bm{v})\|\leq 2\iota C_g^{xy}+C_f^x.
\end{align}

    \begin{proof}
    For the first inequality, we have
        \begin{align}
            \sum_{i\in C_r}\tilde{p}_i\|\nabla_yg_i(\bm{x},\bm{y})\|&\leq \Bigg\|
            \sum_{i\in C_r}\tilde{p}_i \nabla_yg_i(\bm{x},\bm{y})-\nabla_yg_i(\bm{x},\bm{y}^*(\bm{x}))
            \Bigg\|\nonumber\\
            &\leq \Bigg\|
            \sum_{i\in C_r}\tilde{p}_i \underset{\zeta\sim\mathcal{D}_i(\bm{x})}{\mathbb{E}}\nabla_yg_i(\bm{x},\bm{y};\zeta)-\underset{\zeta\sim\mathcal{D}_i(\bm{x})}{\mathbb{E}}\nabla_yg_i(\bm{x},\bm{y}^*(\bm{x});\zeta)
            \Bigg\|\nonumber\\
            &\leq L_g^y\|\bm{y}-\bm{y}^*(\bm{x})\|.
        \end{align}
    For the second inequality, it holds that
        \begin{align}
            \sum_{i\in C_r}\tilde{p}_i\|\nabla_vl_i(\bm{x},\bm{y},\bm{v}^*)\|&\leq \Bigg\|\sum_{i\in C_r}\tilde{p}_i\nabla_{yy}g_i(\bm{x},\bm{y})\bm{v}^*\Bigg\|+\Bigg\|\sum_{i\in C_r}\tilde{p}_i\nabla_yf_i(\bm{x},\bm{y})\Bigg\|\nonumber \\
            &\leq \frac{C_f^y}{\gamma_g}C_g^{yy}+C_f^y.
        \end{align}
        Similarly, we obtain that 
        \begin{align}
            \sum_{i\in C_r}\tilde{p}_i\|\nabla_vl_i(\bm{x},\bm{y},\bm{v})\|\leq 2\iota C_g^{yy}+C_f^y,
        \end{align}
        \begin{align}
            \sum_{i\in C_r}\tilde{p}_i\|\nabla \mathcal{F}_i(\bm{x},\bm{y},\bm{v}^*)\|&\leq \Bigg\|
             \sum_{i\in C_r}\tilde{p}_i\nabla_{xy}g_i(\bm{x},\bm{y})\bm{v}^*
            \Bigg\|+\Bigg\|\sum_{i\in C_r}\tilde{p}_i\nabla_xf_i(\bm{x},\bm{y})
            \Bigg\|\nonumber \\
            &\leq \frac{C_f^y}{\gamma_g}C_g^{xy}+C_f^x.
        \end{align}
        Similarly, we obtain that 
        \begin{align}
            \sum_{i\in C_r}\tilde{p}_i\|\nabla \mathcal{F}_i(\bm{x},\bm{y},\bm{v})\|\leq 2\iota C_g^{xy}+C_f^x.
        \end{align}
    \end{proof}
\end{lemma}

\begin{lemma}\label{lemma_5}
The gradient of the $\mathcal{F}_i(\bm{x},\bm{y})$ is $L_{F_c}^y$-Lipschitz continuous, namely,
\begin{align}
    \|\nabla\mathcal{F}_i(\bm{x}, \bm{y})-\nabla\mathcal{F}_i(\bm{x}, \bm{y}^*(\bm{x}))\|\leq L_{F_c}^y \|\bm{y}-\bm{y}^*(\bm{x})\|,
\end{align}
where 
\begin{align}
    L_{F_c}^y = \frac{C_f^y}{\gamma_g}L_{gxy}^y+\frac{\hat{L}_v}{\gamma_g}+L_f^\xi\varepsilon_c+ 
            \bar{L}_f^x.
\end{align}
Similarly, the gradient of the $l_i(\bm{x},\bm{y})$ is $L_{lc}^y$-Lipschitz continuous, namely,
\begin{align}
    \|\nabla l_i(\bm{x}, \bm{y})-\nabla l_i(\bm{x}, \bm{y}^*(\bm{x}))\|\leq L_{l_c}^y \|\bm{y}-\bm{y}^*(\bm{x})\|,
\end{align}
where 
\begin{align}
    L_{l_c}^y = \frac{C_f^y}{\gamma_g}L_{gyy}^y+\frac{\hat{L}_v}{\gamma_g}+L_f^\xi\varepsilon_c+ 
            \bar{L}_f^y.
\end{align}

   \begin{proof}
According to the definitions of $\nabla\mathcal{F}_i(\bm{x}, \bm{y})$ and $\nabla\mathcal{F}_i(\bm{x}, \bm{y}^*(\bm{x}))$, we have
    \begin{align}
        &\|\nabla\mathcal{F}_i(\bm{x}, \bm{y})-\nabla\mathcal{F}_i(\bm{x}, \bm{y}^*(\bm{x}))\|\nonumber \\
        &=\|
        \nabla_xf_i(\bm{x},\bm{y})-\nabla_{xy}^2g_i(\bm{x},\bm{y})\bm{v}^*(\bm{x},\bm{y})-
        \nabla_xf_i(\bm{x},\bm{y}^*(\bm{x}))+\nabla_{xy}^2g_i(\bm{x},\bm{y}^*(\bm{x}))\bm{v}^*(\bm{x},\bm{y}^*(\bm{x}))
        \|\nonumber \\
        &=\|
        \nabla_xf_i(\bm{x},\bm{y})-\nabla_xf_i(\bm{x},\bm{y}^*(\bm{x}))
        -\nabla_{xy}^2g_i(\bm{x},\bm{y})\bm{v}^*(\bm{x},\bm{y})
        +\nabla_{xy}^2g_i(\bm{x},\bm{y}^*(\bm{x}))\bm{v}^*(\bm{x},\bm{y})\nonumber \\
        &\quad
        -\nabla_{xy}^2g_i(\bm{x},\bm{y}^*(\bm{x}))\bm{v}^*(\bm{x},\bm{y})
        +\nabla_{xy}^2g_i(\bm{x},\bm{y}^*(\bm{x}))\bm{v}^*(\bm{x},\bm{y}^*(\bm{x}))
        \|
        \nonumber \\
        &\leq \left\|
                \left(
                    \nabla_{xy}^2g_i(\bm{x},\bm{y})-\nabla_{xy}^2g_i(\bm{x},\bm{y}^*(\bm{x}))
                \right)\bm{v}^*(\bm{x},\bm{y})
            \right\|+
            \left\|
                \nabla_x f_i(\bm{x},\bm{y})-\nabla_x f_i(\bm{x},\bm{y}^*(\bm{x}))
            \right\|\nonumber \\
            &\quad +\|
            \nabla_{xy}^2 g_i(\bm{x},\bm{y}^*(\bm{x}))(\bm{v}^*(\bm{x},\bm{y})-\bm{v}^*(\bm{x},\bm{y}^*(\bm{x})))
            \|\nonumber \\
            &\leq \frac{C_f^y}{\gamma_g}\|\nabla_{xy}^2g_i(\bm{x},\bm{y})-\nabla_{xy}^2g_i(\bm{x},\bm{y}^*(\bm{x}))\| + \left\|
                \nabla_x f_i(\bm{x},\bm{y})-\nabla_x f_i(\bm{x},\bm{y}^*(\bm{x}))
            \right\| 
            +\frac{\hat{L}_v}{\gamma_g}\|\bm{y}-\bm{y}^*(\bm{x})\|
            \nonumber \\
            &\leq \frac{C_f^y}{\gamma_g}\|\underset{\zeta\sim\mathcal{D}_i(\bm{x})}{\mathbb{E}}\nabla_{xy}^2g_i(\bm{x},\bm{y};\zeta)-\underset{\zeta\sim\mathcal{D}_i(\bm{x})}{\mathbb{E}}\nabla_{xy}^2g_i(\bm{x},\bm{y}^*(\bm{x});\zeta)\| + \left\|
                \underset{\xi\sim\mathcal{C}_i(\bm{y})}{\mathbb{E}}\nabla_x f_i(\bm{x},\bm{y};\xi)-\underset{\xi\sim\mathcal{C}_i(\bm{y}^*(\bm{x}))}{\mathbb{E}}\nabla_x f_i(\bm{x},\bm{y}^*(\bm{x});\xi)
            \right\| \nonumber \\
            &\quad +\frac{\hat{L}_v}{\gamma_g}\|\bm{y}-\bm{y}^*(\bm{x})\|\nonumber \\
            &\leq \left(\frac{C_f^y}{\gamma_g}L_{gxy}^y+\frac{\hat{L}_v}{\gamma_g}\right)\|\bm{y}-\bm{y}^*(\bm{x})\|+\left\|
                \underset{\xi\sim\mathcal{C}_i(\bm{y})}{\mathbb{E}}\nabla_x f_i(\bm{x},\bm{y})-\underset{\xi\sim\mathcal{C}_i(\bm{y}^*(\bm{x}))}{\mathbb{E}}\nabla_x f_i(\bm{x},\bm{y})
            \right\|\nonumber \\
            &\quad 
            +\left\|
                \underset{\xi\sim\mathcal{C}_i(\bm{y}^*(\bm{x}))}{\mathbb{E}}\nabla_x f_i(\bm{x},\bm{y})-\underset{\xi\sim\mathcal{C}_i(\bm{y}^*(\bm{x}))}{\mathbb{E}}\nabla_x f_i(\bm{x},\bm{y}^*(\bm{x}))
            \right\|
             \nonumber \\
            &\leq \left(\frac{C_f^y}{\gamma_g}L_{gxy}^y+\frac{\hat{L}_v}{\gamma_g}+L_f^\xi\varepsilon_c+ 
            \bar{L}_f^x\right)\|\bm{y}-\bm{y}^*(\bm{x})\|.
        \end{align}
Similarly, we obtain that
\begin{align}
    \|\nabla l_i(\bm{x}, \bm{y})-\nabla l_i(\bm{x}, \bm{y}^*(\bm{x}))\|
    \leq
    \left(\frac{C_f^y}{\gamma_g}L_{gyy}^y+\frac{\hat{L}_v}{\gamma_g}+L_f^\xi\varepsilon_c+ 
            \bar{L}_f^y\right)\|\bm{y}-\bm{y}^*(\bm{x})\|.
\end{align}
        
    \end{proof}
   
\end{lemma}

\begin{lemma}\label{lemma_6}
    The gradient of the $\mathcal{F}_i(\bm{x},\bm{y})$ is $L_{F_c}^x$-Lipschitz continuous, w.r.t. $\bm{x}$, namely,
    \begin{align}
        \|\nabla\mathcal{F}_i(\bm{x})-\nabla\mathcal{F}_i(\bm{x}^\prime)\|\leq L_{F_c}^x\|\bm{x}-\bm{x}^\prime\|,
    \end{align}
    where
    \begin{align}
        L_{F_c}^x=(L_f^\xi\varepsilon_c+\bar{L}_f^x)\bar{L}_y+L_f^x+C_g^{xy}\bar{L}_v+\frac{C_f^y}{\gamma_g}(L_{gxy}^y\varepsilon_d+L_{gxy}^x+L_{gxy}^y\bar{L}_y).
    \end{align}
    Similarly, we also have
    \begin{align}
        \|\nabla l_i(\bm{x},\bm{y}^*(\bm{x})) - \nabla l_i(\bm{x}^\prime, \bm{y}^*(\bm{x}^\prime))\| \leq L_{lc}^x\|\bm{x}-\bm{x}^\prime\|,
    \end{align}
    where 
    \begin{align}
        L_{lc}^x = (L_f^\xi\varepsilon_c+\bar{L}_f^y)\bar{L}_y+L_f^y+C_g^{yy}\bar{L}_v+\frac{C_f^y}{\gamma_g}(L_{gyy}^y\varepsilon_d+L_{gyy}^x+L_{gyy}^y\bar{L}_y).
    \end{align}
\begin{proof}
Based on the definitions of $\nabla\mathcal{F}_i(\bm{x})$ and $\nabla\mathcal{F}_i(\bm{x}^\prime)$, we know
    \begin{align}
        &\|\nabla\mathcal{F}_i(\bm{x})-\nabla\mathcal{F}_i(\bm{x}^\prime)\|\nonumber\\
        &=\|
        \nabla_xf_i(\bm{x},\bm{y}^*(\bm{x}))-\nabla_{xy}^2g_i(\bm{x},\bm{y}^*(\bm{x}))\bm{v}^*(\bm{x})-
        \nabla_xf_i(\bm{x}^\prime,\bm{y}^*(\bm{x}^\prime))+\nabla_{xy}^2g_i(\bm{x}^\prime,\bm{y}^*(\bm{x}^\prime))\bm{v}^*(\bm{x}^\prime)
        \|\nonumber \\
        &\leq \|
        \nabla_xf_i(\bm{x},\bm{y}^*(\bm{x}))-\nabla_xf_i(\bm{x}^\prime,\bm{y}^*(\bm{x}^\prime))\|+
        \|
        \nabla_{xy}^2g_i(\bm{x},\bm{y}^*(\bm{x}))\bm{v}^*(\bm{x})-\nabla_{xy}^2g_i(\bm{x}^\prime,\bm{y}^*(\bm{x}^\prime))\bm{v}^*(\bm{x}^\prime)
        \|\nonumber \\
        &\leq \|
        \underset{\xi\sim\mathcal{C}_i(\bm{y}^*(\bm{x}))}{\mathbb{E}}\nabla_xf_i(\bm{x},\bm{y}^*(\bm{x});\xi)-\underset{\xi\sim\mathcal{C}_i(\bm{y}^*(\bm{x}^\prime))}{\mathbb{E}}\nabla_xf_i(\bm{x}^\prime,\bm{y}^*(\bm{x}^\prime);\xi)\|\nonumber \\
        &\quad+\|
        \underset{\zeta\sim\mathcal{D}_i(\bm{x})}{\mathbb{E}}\nabla_{xy}^2g_i(\bm{x},\bm{y}^*(\bm{x});\zeta)\bm{v}^*(\bm{x})-\underset{\zeta\sim\mathcal{D}_i(\bm{x}^\prime)}{\mathbb{E}}\nabla_{xy}^2g_i(\bm{x}^\prime,\bm{y}^*(\bm{x}^\prime);\zeta)\bm{v}^*(\bm{x}^\prime)\|.\label{lemma6_eq1}
    \end{align}
    For the first term on the right side of Inequality (\ref{lemma6_eq1}), we have
    \begin{align}
        &\|
        \underset{\xi\sim\mathcal{C}_i(\bm{y}^*(\bm{x}))}{\mathbb{E}}\nabla_xf_i(\bm{x},\bm{y}^*(\bm{x});\xi)-\underset{\xi\sim\mathcal{C}_i(\bm{y}^*(\bm{x}^\prime))}{\mathbb{E}}\nabla_xf_i(\bm{x}^\prime,\bm{y}^*(\bm{x}^\prime);\xi)\|\nonumber \\
        &\leq \|
        \underset{\xi\sim\mathcal{C}_i(\bm{y}^*(\bm{x}))}{\mathbb{E}}\nabla_xf_i(\bm{x},\bm{y}^*(\bm{x});\xi)-\underset{\xi\sim\mathcal{C}_i(\bm{y}^*(\bm{x}^\prime))}{\mathbb{E}}\nabla_xf_i(\bm{x},\bm{y}^*(\bm{x});\xi)\|\nonumber \\
        &\quad+\|
        \underset{\xi\sim\mathcal{C}_i(\bm{y}^*(\bm{x}^\prime))}{\mathbb{E}}\nabla_xf_i(\bm{x},\bm{y}^*(\bm{x});\xi)-\underset{\xi\sim\mathcal{C}_i(\bm{y}^*(\bm{x}^\prime))}{\mathbb{E}}\nabla_xf_i(\bm{x}^\prime,\bm{y}^*(\bm{x}^\prime);\xi)\|\nonumber \\
        &\leq (L_f^\xi\varepsilon_c+\bar{L}_f^x)\|\bm{y}^*(\bm{x})-\bm{y}^*(\bm{x}^\prime)\|+L_f^x\|\bm{x}-\bm{x}^\prime\|\nonumber \\
        &\leq ((L_f^\xi\varepsilon_c+\bar{L}_f^x)\bar{L}_y+L_f^x)\|\bm{x}-\bm{x}^\prime\|.
    \end{align}
For the second term on the right side of Inequality (\ref{lemma6_eq1}), it yields that
\begin{align}
    &\|
    \underset{\zeta\sim\mathcal{D}_i(\bm{x})}{\mathbb{E}}\nabla_{xy}^2g_i(\bm{x},\bm{y}^*(\bm{x});\zeta)\bm{v}^*(\bm{x})-\underset{\zeta\sim\mathcal{D}_i(\bm{x}^\prime)}{\mathbb{E}}\nabla_{xy}^2g_i(\bm{x}^\prime,\bm{y}^*(\bm{x}^\prime);\zeta)\bm{v}^*(\bm{x}^\prime)\|\nonumber \\
    &\leq \|
    \underset{\zeta\sim\mathcal{D}_i(\bm{x})}{\mathbb{E}}\nabla_{xy}^2g_i(\bm{x},\bm{y}^*(\bm{x});\zeta)\|\|\bm{v}^*(\bm{x})-\bm{v}^*(\bm{x}^\prime)\|+\|\bm{v}^*(\bm{x}^\prime)\|\|
    \underset{\zeta\sim\mathcal{D}_i(\bm{x})}{\mathbb{E}}\nabla_{xy}^2g_i(\bm{x},\bm{y}^*(\bm{x});\zeta)-\underset{\zeta\sim\mathcal{D}_i(\bm{x}^\prime)}{\mathbb{E}}\nabla_{xy}^2g_i(\bm{x}^\prime,\bm{y}^*(\bm{x}^\prime);\zeta)\|\nonumber \\
    &\leq \left(C_g^{xy}\bar{L}_v+\frac{C_f^y}{\gamma_g}(L_{gxy}^y\varepsilon_d+L_{gxy}^x+L_{gxy}^y\bar{L}_y)\right)\|\bm{x}-\bm{x}^\prime\|.
\end{align}
Therefore, we further obtain that
\begin{align}
    &\|\nabla\mathcal{F}_i(\bm{x})-\nabla\mathcal{F}_i(\bm{x}^\prime)\|\leq \left((L_f^\xi\varepsilon_c+\bar{L}_f^x)\bar{L}_y+L_f^x+C_g^{xy}\bar{L}_v+\frac{C_f^y}{\gamma_g}(L_{gxy}^y\varepsilon_d+L_{gxy}^x+L_{gxy}^y\bar{L}_y)\right)\|\bm{x}-\bm{x}^\prime\|.
\end{align}

Similarly, we have
\begin{align}
    &\|\nabla l_i(\bm{x})-\nabla l_i(\bm{x}^\prime)\|\leq \left((L_f^\xi\varepsilon_c+\bar{L}_f^y)\bar{L}_y+L_f^y+C_g^{xy}\bar{L}_v+\frac{C_f^y}{\gamma_g}(L_{gyy}^y\varepsilon_d+L_{gyy}^x+L_{gyy}^y\bar{L}_y)\right)\|\bm{x}-\bm{x}^\prime\|.
\end{align}

\end{proof}

\end{lemma}

\begin{lemma}\label{lemma_7}
    The gradient of the $\mathcal{F}_i(\bm{x},\bm{y}^*(\bm{\varphi},\bm{x}))$ is $\bar{L}_{F_c}^x$-Lipschitz continuous, w.r.t. $\bm{x}$, $\forall \bm{\varphi}, \bm{x}, \bm{x}^\prime$, namely,
    \begin{align}
        \|\nabla \mathcal{F}_i(\bm{x},\bm{y}^*(\bm{\varphi}, \bm{x}))-
        \nabla \mathcal{F}_i(\bm{x},\bm{y}^*(\bm{\varphi}, \bm{x}^\prime))
        \|\leq \bar{L}_{F_c}^x\|\bm{x}-\bm{x}^\prime\|,
    \end{align}
where 
\begin{align}
    \bar{L}_{F_c}^x=\left(
    \bar{L}_f^x+L_f^\xi\varepsilon_c
    +C_g^{xy}\left(
    \frac{1}{\gamma_g}(
    L_f^\xi\varepsilon_c+L_f^y
    )+\frac{C_f^y}{\gamma_g^2}L_{gyy}^y
    \right)+
    \frac{C_f^y}{\gamma_g}L_{gxy}^y
    \right)\frac{L_g^\zeta\varepsilon_d}{\gamma_g}
    +
    \frac{C_f^y}{\gamma_g} L_{gxy}^y\varepsilon_d+
    \frac{C_f^y}{\gamma_g^2}L_{gyy}^\zeta \varepsilon_d
    .
\end{align}
The gradient of the $\mathcal{F}_i(\bm{x},\bm{y}^*(\bm{x}))$ is $\tilde{L}_{F_c}^x$-Lipschitz continuous, w.r.t. $\bm{x}$, $\forall \bm{x}, \bm{x}^\prime$, namely,
    \begin{align}
        \|\nabla \mathcal{F}_i(\bm{x},\bm{y}^*(\bm{x}))-
        \nabla \mathcal{F}_i(\bm{x},\bm{y}^*(\bm{x}, \bm{x}^\prime))
        \|\leq \tilde{L}_{F_c}^x\|\bm{x}-\bm{x}^\prime\|,
    \end{align}
where 
\begin{align}
    \tilde{L}_{F_c}^x=\left(
    \bar{L}_f^x+
    \frac{C_f^y}{\gamma_g}L_{gxy}^y+
    \frac{C_g^{xy}L_f^y}{\gamma_g}+\frac{C_f^y}{\gamma_g^2}L_{gyy}^y
    \right)\frac{L_g^\zeta\varepsilon_d}{\gamma_g}+
    \left(
    L_f^\xi\varepsilon_c+\frac{C_g^{xy} L_f^\xi\varepsilon_c}{\gamma_g}
    \right)\bar{L}_y+
    \frac{C_f^y}{\gamma_g} L_{gxy}^y\varepsilon_d+\frac{C_g^{xy}C_f^y}{\gamma_g^2}L_{gyy}^\zeta \varepsilon_d.
\end{align}
Similarly, the gradient of the $l_i(\bm{x},\bm{y}^*(\bm{x}))$ is $\tilde{L}_{l_c}^x$-Lipschitz continuous, w.r.t. $\bm{x}, \forall \bm{x}, \bm{x}^\prime$, namely,
\begin{align}
        \|\nabla l_i(\bm{x},\bm{y}^*(\bm{x}))-
        \nabla l_i(\bm{x},\bm{y}^*(\bm{x}, \bm{x}^\prime))
        \|\leq \tilde{L}_{l_c}^x\|\bm{x}-\bm{x}^\prime\|,
\end{align}
where 
\begin{align}
    \tilde{L}_{l_c}^x=\left(
    \bar{L}_f^y+
    \frac{C_f^y}{\gamma_g}L_{gyy}^y+
    \frac{C_g^{yy}L_f^y}{\gamma_g}+\frac{C_f^y}{\gamma_g^2}L_{gyy}^y
    \right)\frac{L_g^\zeta\varepsilon_d}{\gamma_g}+
    \left(
    L_f^\xi\varepsilon_c+\frac{C_g^{yy} L_f^\xi\varepsilon_c}{\gamma_g}
    \right)\bar{L}_y+
    \frac{C_f^y}{\gamma_g} L_{gyy}^y\varepsilon_d+\frac{C_g^{yy}C_f^y}{\gamma_g^2}L_{gyy}^\zeta \varepsilon_d.
\end{align}

\end{lemma}
\begin{proof}
Based on the definition of $\nabla \mathcal{F}_i(\bm{x},\bm{y}^*(\bm{\varphi}, \bm{x}))$, it yields that
    \begin{align}
        &\|\nabla \mathcal{F}_i(\bm{x},\bm{y}^*(\bm{\varphi}, \bm{x}))-
        \nabla \mathcal{F}_i(\bm{x},\bm{y}^*(\bm{\varphi}, \bm{x}^\prime))
        \|\nonumber \\
        &\leq \|
        \nabla_xf_i(\bm{x},\bm{y}^*(\bm{\varphi},\bm{x}))-\nabla_xf_i(\bm{x},\bm{y}^*(\bm{\varphi},\bm{x}^\prime))\|\nonumber \\
        &\quad+
        \|
        \nabla_{xy}^2g_i(\bm{x},\bm{y}^*(\bm{\varphi},\bm{x}))\bm{v}^*(\bm{x},\bm{y}^*(\bm{\varphi},\bm{x}))-\nabla_{xy}^2g_i(\bm{x},\bm{y}^*(\bm{\varphi},\bm{x}^\prime))\bm{v}^*(\bm{x},\bm{y}^*(\bm{\varphi},\bm{x}^\prime))
        \|
        \nonumber \\
        &\leq \|
        \nabla_xf_i(\bm{x},\bm{y}^*(\bm{\varphi},\bm{x}))-\nabla_xf_i(\bm{x},\bm{y}^*(\bm{\varphi},\bm{x}^\prime))\|\nonumber \\
        &\quad+
        \|
        \nabla_{xy}^2g_i(\bm{x},\bm{y}^*(\bm{\varphi},\bm{x}))\bm{v}^*(\bm{x},\bm{y}^*(\bm{\varphi},\bm{x}))-\nabla_{xy}^2g_i(\bm{x},\bm{y}^*(\bm{\varphi},\bm{x}))\bm{v}^*(\bm{x},\bm{y}^*(\bm{\varphi},\bm{x}^\prime))
        \|
        \nonumber \\
        &\quad+
        \|
        \nabla_{xy}^2g_i(\bm{x},\bm{y}^*(\bm{\varphi},\bm{x}))\bm{v}^*(\bm{x},\bm{y}^*(\bm{\varphi},\bm{x}^\prime))-\nabla_{xy}^2g_i(\bm{x},\bm{y}^*(\bm{\varphi},\bm{x}^\prime))\bm{v}^*(\bm{x},\bm{y}^*(\bm{\varphi},\bm{x}^\prime))
        \|\nonumber \\
        &\leq \|
        \underset{\xi\sim\mathcal{C}_i(\bm{y}^*(\bm{\varphi },\bm{x}))}{\mathbb{E}}\nabla_xf_i(\bm{x},\bm{y}^*(\bm{\varphi},\bm{x});\xi)-\underset{\xi\sim\mathcal{C}_i(\bm{y}^*(\bm{\varphi },\bm{x}^\prime))}{\mathbb{E}}\nabla_xf_i(\bm{x},\bm{y}^*(\bm{\varphi},\bm{x}^\prime);\xi)\|\nonumber \\
        &\quad+
        \|
        \underset{\zeta\sim\mathcal{D}_i(\bm{x})}{\mathbb{E}}\nabla_{xy}^2g_i(\bm{x},\bm{y}^*(\bm{\varphi},\bm{x});\zeta)\bm{v}^*(\bm{x},\bm{y}^*(\bm{\varphi},\bm{x}))-\underset{\zeta\sim\mathcal{D}_i(\bm{x})}{\mathbb{E}}\nabla_{xy}^2g_i(\bm{x},\bm{y}^*(\bm{\varphi},\bm{x});\zeta)\bm{v}^*(\bm{x},\bm{y}^*(\bm{\varphi},\bm{x}^\prime))
        \|
        \nonumber \\
        &\quad+
        \|
        \underset{\zeta\sim\mathcal{D}_i(\bm{x})}{\mathbb{E}}\nabla_{xy}^2g_i(\bm{x},\bm{y}^*(\bm{\varphi},\bm{x});\zeta)\bm{v}^*(\bm{x},\bm{y}^*(\bm{\varphi},\bm{x}^\prime))-\underset{\zeta\sim\mathcal{D}_i(\bm{x}^\prime)}{\mathbb{E}}\nabla_{xy}^2g_i(\bm{x},\bm{y}^*(\bm{\varphi},\bm{x}^\prime);\zeta)\bm{v}^*(\bm{x},\bm{y}^*(\bm{\varphi},\bm{x}^\prime))\|.
    \end{align}
Note that $\bm{v}^*$ is introduced only as an auxiliary quantity in the convergence analysis to express the inverse–Hessian–vector product arising from implicit differentiation; it is not computed or maintained by the FBi-RRM algorithm.
Similar to Lemma \ref{lemma_6}, we can then obtain
\begin{align}
    &\|
        \underset{\xi\sim\mathcal{C}_i(\bm{y}^*(\bm{\varphi },\bm{x}))}{\mathbb{E}}\nabla_xf_i(\bm{x},\bm{y}^*(\bm{\varphi},\bm{x});\xi)-\underset{\xi\sim\mathcal{C}_i(\bm{y}^*(\bm{\varphi },\bm{x}^\prime))}{\mathbb{E}}\nabla_xf_i(\bm{x},\bm{y}^*(\bm{\varphi},\bm{x}^\prime);\xi)\|\nonumber \\
        &\leq (\bar{L}_f^x+L_f^\xi\varepsilon_c)\|\bm{y}^*(\bm{\varphi},\bm{x})-\bm{y}^*(\bm{\varphi},\bm{x}^\prime)\|,
\end{align}
and 
\begin{align}
    &\|
        \underset{\zeta\sim\mathcal{D}_i(\bm{x})}{\mathbb{E}}\nabla_{xy}^2g_i(\bm{x},\bm{y}^*(\bm{\varphi},\bm{x});\zeta)\bm{v}^*(\bm{x},\bm{y}^*(\bm{\varphi},\bm{x}))-\underset{\zeta\sim\mathcal{D}_i(\bm{x})}{\mathbb{E}}\nabla_{xy}^2g_i(\bm{x},\bm{y}^*(\bm{\varphi},\bm{x});\zeta)\bm{v}^*(\bm{x},\bm{y}^*(\bm{\varphi},\bm{x}^\prime))
        \|
        \nonumber \\
        &\quad+
        \|
        \underset{\zeta\sim\mathcal{D}_i(\bm{x})}{\mathbb{E}}\nabla_{xy}^2g_i(\bm{x},\bm{y}^*(\bm{\varphi},\bm{x});\zeta)\bm{v}^*(\bm{x},\bm{y}^*(\bm{\varphi},\bm{x}^\prime))-\underset{\zeta\sim\mathcal{D}_i(\bm{x}^\prime)}{\mathbb{E}}\nabla_{xy}^2g_i(\bm{x},\bm{y}^*(\bm{\varphi},\bm{x}^\prime);\zeta)\bm{v}^*(\bm{x},\bm{y}^*(\bm{\varphi},\bm{x}^\prime))\|\nonumber \\
        &\leq C_g^{xy}\|\bm{v}^*(\bm{x},\bm{y}^*(\bm{\varphi},\bm{x}))-\bm{v}^*(\bm{x},\bm{y}^*(\bm{\varphi},\bm{x}^\prime))\|+\frac{C_f^y}{\gamma_g}L_{gxy}^y\|\bm{y}^*(\bm{\varphi},\bm{x})-\bm{y}^*(\bm{\varphi},\bm{x}^\prime)\|+ \frac{C_f^y}{\gamma_g} L_{gxy}^y\varepsilon_d\|\bm{x}-\bm{x}^\prime\|.
\end{align}
We analyze $\|\bm{v}^*(\bm{x},\bm{y}^*(\bm{\varphi},\bm{x}))-\bm{v}^*(\bm{x},\bm{y}^*(\bm{\varphi},\bm{x}^\prime))\|$ as follows:
\begin{align}
    &\|\bm{v}^*(\bm{x},\bm{y}^*(\bm{\varphi},\bm{x}))-\bm{v}^*(\bm{x},\bm{y}^*(\bm{\varphi},\bm{x}^\prime))\|\nonumber \\
    &=\| \left[\nabla_{yy}^2 g_i(\bm{x},\bm{y}^*(\bm{\varphi},\bm{x}))\right]^{-1} \nabla_y f_i (\bm{x},\bm{y}^*(\bm{\varphi},\bm{x}))- \left[\nabla_{yy}^2 g_i(\bm{x},\bm{y}^*(\bm{\varphi},\bm{x}^\prime))\right]^{-1} \nabla_y f_i (\bm{x},\bm{y}^*(\bm{\varphi},\bm{x}^\prime))
     \|\nonumber \\
    &\leq \Bigg\| \left[\underset{\zeta\sim\mathcal{D}_i(\bm{x})}{\mathbb{E}}\nabla_{yy}^2 g_i(\bm{x},\bm{y}^*(\bm{\varphi},\bm{x});\zeta)\right]^{-1} \underset{\xi\sim\mathcal{C}_i(\bm{y}^*(\bm{\varphi },\bm{x}))}{\mathbb{E}}\nabla_y f_i (\bm{x},\bm{y}^*(\bm{\varphi},\bm{x});\xi)\nonumber \\
    &\quad- \left[\underset{\zeta\sim\mathcal{D}_i(\bm{x}^\prime)}{\mathbb{E}}\nabla_{yy}^2 g_i(\bm{x},\bm{y}^*(\bm{\varphi},\bm{x}^\prime);\zeta)\right]^{-1} \underset{\xi\sim\mathcal{C}_i(\bm{y}^*(\bm{\varphi },\bm{x}^\prime))}{\mathbb{E}}\nabla_y f_i (\bm{x},\bm{y}^*(\bm{\varphi},\bm{x}^\prime);\xi)
    \Bigg\|\nonumber \\
    &\leq \Bigg\| \left[\underset{\zeta\sim\mathcal{D}_i(\bm{x})}{\mathbb{E}}\nabla_{yy}^2 g_i(\bm{x},\bm{y}^*(\bm{\varphi},\bm{x});\zeta)\right]^{-1} \Bigg(\underset{\xi\sim\mathcal{C}_i(\bm{y}^*(\bm{\varphi },\bm{x}))}{\mathbb{E}}\nabla_y f_i (\bm{x},\bm{y}^*(\bm{\varphi},\bm{x});\xi)-\underset{\xi\sim\mathcal{C}_i(\bm{y}^*(\bm{\varphi },\bm{x}^\prime))}{\mathbb{E}}\nabla_y f_i (\bm{x},\bm{y}^*(\bm{\varphi},\bm{x}^\prime);\xi)\Bigg)\Bigg\|\nonumber \\
    &\quad + \Bigg\|\Bigg(
    \left[\underset{\zeta\sim\mathcal{D}_i(\bm{x})}{\mathbb{E}}\nabla_{yy}^2 g_i(\bm{x},\bm{y}^*(\bm{\varphi},\bm{x});\zeta)\right]^{-1}-\left[\underset{\zeta\sim\mathcal{D}_i(\bm{x}^\prime)}{\mathbb{E}}\nabla_{yy}^2 g_i(\bm{x},\bm{y}^*(\bm{\varphi},\bm{x}^\prime);\zeta)\right]^{-1}
    \Bigg)
    \underset{\xi\sim\mathcal{C}_i(\bm{y}^*(\bm{\varphi },\bm{x}^\prime))}{\mathbb{E}}\nabla_y f_i (\bm{x},\bm{y}^*(\bm{\varphi},\bm{x}^\prime);\xi)
    \Bigg\|\nonumber \\
    &\leq \frac{1}{\gamma_g}\left(
    L_f^\xi\varepsilon_c+L_f^y
    \right)\|\bm{y}^*(\bm{\varphi},\bm{x})-\bm{y}^*(\bm{\varphi},\bm{x}^\prime)\|+\frac{C_f^y}{\gamma_g^2}\|
    \underset{\zeta\sim\mathcal{D}_i(\bm{x})}{\mathbb{E}}\nabla_{yy}^2 g_i(\bm{x},\bm{y}^*(\bm{\varphi},\bm{x}))
    -\underset{\zeta\sim\mathcal{D}_i(\bm{x}^\prime)}{\mathbb{E}}\nabla_{yy}^2 g_i(\bm{x},\bm{y}^*(\bm{\varphi},\bm{x}^\prime))
    \|\\
    &\leq \left(
    \frac{1}{\gamma_g}(
    L_f^\xi\varepsilon_c+L_f^y
    )+\frac{C_f^y}{\gamma_g^2}L_{gyy}^y
    \right)\|\bm{y}^*(\bm{\varphi},\bm{x})-\bm{y}^*(\bm{\varphi},\bm{x}^\prime)\|+\frac{C_f^y}{\gamma_g^2}L_{gyy}^\zeta \varepsilon_d\|\bm{x}-\bm{x}^\prime\|.
\end{align}
From Eq. (51) in \cite{lu2023bilevel}, we know that 
\begin{align}
    \|\bm{y}^*(\bm{\varphi},\bm{x})-\bm{y}^*(\bm{\varphi},\bm{x}^\prime)\| \leq \frac{L_g^\zeta\varepsilon_d}{\gamma_g} \|\bm{x}-\bm{x}^\prime\|.
\end{align}
Therefore, combining all steps, we can get that
\begin{align}
    &\|\nabla \mathcal{F}_i(\bm{x},\bm{y}^*(\bm{\varphi}, \bm{x}))-
        \nabla \mathcal{F}_i(\bm{x},\bm{y}^*(\bm{\varphi}, \bm{x}^\prime))
        \|\nonumber \\
    &\leq \left(\left(
    \bar{L}_f^x+L_f^\xi\varepsilon_c
    +C_g^{xy}\left(
    \frac{1}{\gamma_g}(
    L_f^\xi\varepsilon_c+L_f^y
    )+\frac{C_f^y}{\gamma_g^2}L_{gyy}^y
    \right)+
    \frac{C_f^y}{\gamma_g}L_{gxy}^y
    \right)\frac{L_g^\zeta\varepsilon_d}{\gamma_g}
    +
    \frac{C_f^y}{\gamma_g} L_{gxy}^y\varepsilon_d+
    \frac{C_f^y}{\gamma_g^2}L_{gyy}^\zeta \varepsilon_d
    \right)\|\bm{x}-\bm{x}^\prime\|.
\end{align}
Similarly, it's easy to know that
\begin{align}
    &\|\nabla \mathcal{F}_i(\bm{x},\bm{y}^*( \bm{x}))-
        \nabla \mathcal{F}_i(\bm{x},\bm{y}^*(\bm{x}, \bm{x}^\prime))
        \|\nonumber \\
    &\leq \left(
    \bar{L}_f^x+
    \frac{C_f^y}{\gamma_g}L_{gxy}^y+
    \frac{C_g^{xy}L_f^y}{\gamma_g}+\frac{C_f^y}{\gamma_g^2}L_{gyy}^y
    \right)\|\bm{y}^*(\bm{x})-\bm{y}^*(\bm{x},\bm{x}^\prime)\|+
    \left(
    L_f^\xi\varepsilon_c+\frac{C_g^{xy} L_f^\xi\varepsilon_c}{\gamma_g}
    \right)\|\bm{y}^*(\bm{x})-\bm{y}^*(\bm{x}^\prime)\|\nonumber \\
    &\quad +
    \left(
    \frac{C_f^y}{\gamma_g} L_{gxy}^y\varepsilon_d+\frac{C_g^{xy}C_f^y}{\gamma_g^2}L_{gyy}^\zeta \varepsilon_d
    \right)\|\bm{x}-\bm{x}^\prime\| \\
    &\leq \left(
    \left(
    \bar{L}_f^x+
    \frac{C_f^y}{\gamma_g}L_{gxy}^y+
    \frac{C_g^{xy}L_f^y}{\gamma_g}+\frac{C_f^y}{\gamma_g^2}L_{gyy}^y
    \right)\frac{L_g^\zeta\varepsilon_d}{\gamma_g}+
    \left(
    L_f^\xi\varepsilon_c+\frac{C_g^{xy} L_f^\xi\varepsilon_c}{\gamma_g}
    \right)\bar{L}_y+
    \frac{C_f^y}{\gamma_g} L_{gxy}^y\varepsilon_d+\frac{C_g^{xy}C_f^y}{\gamma_g^2}L_{gyy}^\zeta \varepsilon_d
    \right)\nonumber \\
    &\quad \cdot\|\bm{x}-\bm{x}^\prime\|,
\end{align}
and
\begin{align}
    &\|\nabla l_i(\bm{x},\bm{y}^*(\bm{x}))-
        \nabla l_i(\bm{x},\bm{y}^*(\bm{x}, \bm{x}^\prime))
        \|\nonumber \\
    &\leq \left(
    \left(
    \bar{L}_f^y+
    \frac{C_f^y}{\gamma_g}L_{gyy}^y+
    \frac{C_g^{yy}L_f^y}{\gamma_g}+\frac{C_f^y}{\gamma_g^2}L_{gyy}^y
    \right)\frac{L_g^\zeta\varepsilon_d}{\gamma_g}+
    \left(
    L_f^\xi\varepsilon_c+\frac{C_g^{yy} L_f^\xi\varepsilon_c}{\gamma_g}
    \right)\bar{L}_y+
    \frac{C_f^y}{\gamma_g} L_{gyy}^y\varepsilon_d+\frac{C_g^{yy}C_f^y}{\gamma_g^2}L_{gyy}^\zeta \varepsilon_d
    \right)\nonumber \\
    &\quad \cdot\|\bm{x}-\bm{x}^\prime\|.
\end{align}

\end{proof}

\section{Convergence of the FBi-RRM Method (Proof of Theorem 1)}
\paragraph{Proof roadmap.}
The proof of Theorem~1 proceeds in three steps. First, we compare the local bilevel response maps induced by two different deployment points and bound how the corresponding LL solutions change under the decision-dependent distribution maps. This step uses the sensitivity of $\mathcal{D}_i(\cdot)$ and $\mathcal{C}_i(\cdot)$ together with the strong convexity of the LL problem. Second, we translate these response bounds into a perturbation bound for the local repeated-risk minimizer $R_i(\cdot)$. This yields a Lipschitz bound for each client-side update map, where the dependence on the performative sensitivities appears through a linear term in $\epsilon_d$ and a coupled term in $\epsilon_c\epsilon_d$. Finally, we aggregate the client-side bounds with weights $p_i$ to obtain a contraction bound for the global update operator $R(\cdot)$. When the resulting contraction coefficient is below one, Banach's fixed-point argument gives the existence and uniqueness of the FBPS point, and repeated application of the contraction yields the stated linear convergence rate.
\begin{proof}
Let
\begin{align}
    &\mathcal{F}_i(\bm{\varphi}):= \underset{\xi\sim\mathcal{C}_i(\bm{y}_i^*(\bm{\varphi },\bm{x}_i))}{\mathbb{E}} [f_i (\bm{\varphi},\bm{y}_i^*(\bm{\varphi},\bm{x}_i);\xi)], \, \text{s.t.} \,\, \bm{y}_i^*(\bm{\varphi},\bm{x}_i) = \arg\min_{\bm{y}_i} \underset{\zeta\sim\mathcal{D}_i(\bm{x}_i)}{\mathbb{E}}
    [g_i(\bm{\varphi},\bm{y}_i;\zeta)],\\
    &\mathcal{F}_i^\prime(\bm{\varphi}):= \underset{\xi\sim\mathcal{C}_i(\bm{y}_i^*(\bm{\varphi },\bm{x}_i^\prime))}{\mathbb{E}} [f_i (\bm{\varphi},\bm{y}_i^*(\bm{\varphi},\bm{x}_i^\prime);\xi)], \, \text{s.t.} \,\, \bm{y}_i^*(\bm{\varphi},\bm{x}_i^\prime) = \arg\min_{\bm{y}_i} \underset{\zeta\sim\mathcal{D}_i(\bm{x}_i^\prime)}{\mathbb{E}}
    [g_i(\bm{\varphi},\bm{y}_i;\zeta)].
\end{align}
Since $\mathcal{F}_i(\bm{\varphi})$ is $\gamma_f$-strongly convex, it yields that
\begin{align}
    \bm{x}_{r+1} &= R(\bm{x}_r)=\sum_{i\in C_r}\tilde{p}_i \bm{x}_{i,r}=\sum_{i\in C_r}\tilde{p}_i R_i(\bm{x}_r)\nonumber \\ 
    &:= \sum_{i\in C_r}\tilde{p}_i \arg\min_{\bm{\varphi}} \underset{\xi\sim\mathcal{C}_i(\bm{y}_i^*(\bm{\varphi },\bm{x}_{i,r}))}{\mathbb{E}} [f_i (\bm{\varphi},\bm{y}_i^*(\bm{\varphi},\bm{x}_{i,r});\xi)], \, \text{s.t.} \,\, \bm{y}_i^*(\bm{\varphi},\bm{x}_{i,r}) = \arg\min_{\bm{y}_i}  \underset{\zeta\sim\mathcal{D}_i(\bm{x}_{i,r})}{\mathbb{E}}
    [g_i(\bm{\varphi},\bm{y}_i;\zeta)],
\end{align}
and we further obtain that
\begin{align}
&\left\{
\begin{aligned}
    &\mathcal{F}_i(R_i(\bm{x}))-\mathcal{F}_i(R_i(\bm{x}^\prime)) \geq \langle R_i(\bm{x}) - R_i(\bm{x}^\prime), \nabla \mathcal{F}_i(R_i(\bm{x}^\prime)) \rangle + \frac{\gamma_f}{2}\|R_i(\bm{x}) - R_i(\bm{x}^\prime)\|^2,\\
    &\mathcal{F}_i(R_i(\bm{x}^\prime))-\mathcal{F}_i(R_i(\bm{x})) \geq \frac{\gamma_f}{2}\|R_i(\bm{x}) - R_i(\bm{x}^\prime)\|^2,\\
\end{aligned}
\right.\\
    &\Rightarrow -\gamma_f \|R_i(\bm{x}) - R_i(\bm{x}^\prime)\|^2 \geq \langle R_i(\bm{x}) - R_i(\bm{x}^\prime), \nabla \mathcal{F}_i(R_i(\bm{x}^\prime)) \rangle.\label{eq.rrm_convergence_1}
\end{align}
We analyze $\|\nabla \mathcal{F}_i(R_i(\bm{x}^\prime))-\nabla \mathcal{F}_i^\prime(R_i(\bm{x}^\prime))\|$ as follows:
\begin{align}
    &\|\nabla \mathcal{F}_i(R_i(\bm{x}^\prime))-\nabla \mathcal{F}_i^\prime(R_i(\bm{x}^\prime))\|\nonumber \\
    &\leq \bigg\|
        \underset{\xi\sim\mathcal{C}_i(\bm{y}^*(R_i(\bm{x}^\prime),\bm{x}))}{\mathbb{E}}\nabla_xf_i(R_i(\bm{x}^\prime),\bm{y}^*(R_i(\bm{x}^\prime),\bm{x});\xi)-\underset{\xi\sim\mathcal{C}_i(\bm{y}^*(R_i(\bm{x}^\prime),\bm{x}^\prime))}{\mathbb{E}}\nabla_xf_i(R_i(\bm{x}^\prime),\bm{y}^*(R_i(\bm{x}^\prime),\bm{x}^\prime);\xi)\bigg\|\nonumber \\
        &\quad+
        \bigg\|
        \underset{\zeta\sim\mathcal{D}_i(R_i(\bm{x}^\prime),\bm{x})}{\mathbb{E}}\nabla_{xy}^2g_i(R_i(\bm{x}^\prime),\bm{y}^*(R_i(\bm{x}^\prime),\bm{x});\zeta)\bm{v}^*(R_i(\bm{x}^\prime),\bm{y}^*(R_i(\bm{x}^\prime),\bm{x}))\nonumber \\
        &\quad -\underset{\zeta\sim\mathcal{D}_i(R_i(\bm{x}^\prime),\bm{x})}{\mathbb{E}}\nabla_{xy}^2g_i(R_i(\bm{x}^\prime),\bm{y}^*(R_i(\bm{x}^\prime),\bm{x});\zeta)\bm{v}^*(R_i(\bm{x}^\prime),\bm{y}^*(R_i(\bm{x}^\prime),\bm{x}^\prime))
        \bigg\|
        \nonumber \\
        &\quad+
        \bigg\|
        \underset{\zeta\sim\mathcal{D}_i(R_i(\bm{x}^\prime),\bm{x})}{\mathbb{E}}\nabla_{xy}^2g_i(R_i(\bm{x}^\prime),\bm{y}^*(R_i(\bm{x}^\prime),\bm{x});\zeta)\bm{v}^*(R_i(\bm{x}^\prime),\bm{y}^*(R_i(\bm{x}^\prime),\bm{x}^\prime))\nonumber\\
        &\quad -\underset{\zeta\sim\mathcal{D}_i(R_i(\bm{x}^\prime),\bm{x}^\prime)}{\mathbb{E}}\nabla_{xy}^2g_i(R_i(\bm{x}^\prime),\bm{y}^*(R_i(\bm{x}^\prime),\bm{x}^\prime);\zeta)\bm{v}^*(R_i(\bm{x}^\prime),\bm{y}^*(R_i(\bm{x}^\prime),\bm{x}^\prime))\bigg\|\nonumber \\
        &\leq \bar{L}_{F_c}^x\|\bm{x}-\bm{x}^\prime\|,
\end{align}
which further yields that
\begin{align}
    \langle R_i(\bm{x})-R_i(\bm{x}^\prime), \nabla\mathcal{F}_i(R_i(\bm{x}^\prime))\rangle-\langle R_i(\bm{x})-R_i(\bm{x}^\prime), \nabla\mathcal{F}_i^\prime(R_i(\bm{x}^\prime))\rangle\geq -\bar{L}_{F_c}^x\|\bm{x}-\bm{x}^\prime\|\|R_i(\bm{x})-R_i(\bm{x}^\prime)\|, 
\end{align}
and based on Eq. (\ref{eq.rrm_convergence_1}), we obtain
\begin{align}
    \|R_i(\bm{x})-R_i(\bm{x}^\prime)\|\leq \frac{\bar{L}_{F_c}^x}{\gamma_f}\|\bm{x}-\bm{x}^\prime\|.
\end{align}
Note that
\begin{align}
    \|R(\bm{x})-R(\bm{x}^\prime)\| &= \Bigg\|
    \sum_{i\in C_r}\tilde{p}_iR_i(\bm{x})-R_i(\bm{x}^\prime)
    \Bigg\|\nonumber\\
    &\leq  \sum_{i\in C_r}\tilde{p}_i\|R(\bm{x})-R(\bm{x}^\prime)\|\nonumber\\
    &\leq \frac{\bar{L}_{F_c}^x}{\gamma_f}\|\bm{x}_i-\bm{x}^\prime_i\|.
\end{align}
If every client holds the same global model (i.e., $\bm{x}_i=\bm{x}$ and $\bm{x}_i^\prime=\bm{x}^\prime$), and we can have
\begin{align}
    \|R(\bm{x})-R(\bm{x}^\prime)\|\leq \frac{\bar{L}_{F_c}^x}{\gamma_f}\|\bm{x}-\bm{x}^\prime\|.
\end{align}
Based on the definition of the performative stable point solution and FBi-RRM, it yields that 
\begin{align}
    \|\bm{x}_r-\bm{x}_s\|&\leq\frac{\bar{L}_{F_c}^x}{\gamma_f}\|\bm{x}_{r-1}-\bm{x}_s\|\nonumber \\
    &\leq \left(
    \frac{\bar{L}_{F_c}^x}{\gamma_f}
    \right)^r\|\bm{x}_0-\bm{x}_s\|.
\end{align}
This completes the proof.
\end{proof}

\section{Preliminaries for Convergence of the FBi-SGD Method}
\begin{lemma}\label{lemma_8}
    Similar to Lemma \ref{lemma_4}, under Assumptions \ref{assump_1}, \ref{assump_2}, \ref{assump_3}, and \ref{assump_4}, we can easily obtain that
    \begin{align}
        &\sum_{i\in C_r}\tilde{p}_i\|\nabla_yg_i(\bm{x}_r,\bm{y}_r)\|^2\leq (L_g^y)^2\|\bm{y}_r-\bm{y}^*(\bm{x}_r)\|^2,\\
        &\sum_{i\in C_r}\tilde{p}_i\|\nabla_vl_i(\bm{x}_r,\bm{y}_r,\bm{v}_r)\|^2\leq 3(L_{l_c}^y)^2\|\bm{y}_r-\bm{y}^*(\bm{x}_r)\|^2+3(\tilde{L}_{l_c}^x)^2\|\bm{x}_r-\bm{x}^*\|^2+3\mathcal{V}_l,\\
        &\sum_{i\in C_r}\tilde{p}_i\|\nabla \mathcal{F}_i(\bm{x}_r,\bm{y}_r,\bm{v}_r)\|^2\leq 3(L_{Fc}^y)^2\|\bm{y}_r-\bm{y}^*(\bm{x}_r)\|^2+3(\tilde{L}_{Fc}^x)^2\|\bm{x}_r-\bm{x}^*\|^2+3\mathcal{V}_F,
    \end{align} 
where we define
\begin{align}
    &\mathcal{V}_F = \sum_{i\in C_r}\tilde{p}_i\|
    \nabla\mathcal{F}_i(\bm{x}_r,\bm{y}^*(\bm{x}_r,\bm{x}^*),\bm{v}_r^*(\bm{x}_r,\bm{y}^*(\bm{x}_r,\bm{x}^*)))-\nabla\mathcal{F}_i(\bm{x}^*,\bm{y}^*(\bm{x}^*),\bm{v}^*(\bm{x}^*,\bm{y}^*(\bm{x}^*)))
    \|^2,\\
    &\mathcal{V}_l = \sum_{i\in C_r}\tilde{p}_i\|
    \nabla l_i(\bm{x}_r,\bm{y}^*(\bm{x}_r,\bm{x}^*),\bm{v}_r^*(\bm{x}_r,\bm{y}^*(\bm{x}_r,\bm{x}^*)))-\nabla l_i(\bm{x}^*,\bm{y}^*(\bm{x}^*),\bm{v}^*(\bm{x}^*,\bm{y}^*(\bm{x}^*)))
    \|^2.
\end{align}
\end{lemma}
\begin{proof}
From the definition of $\nabla_yg_i(\bm{x}_r,\bm{y}_r)$, we know
\begin{align}
            \sum_{i\in C_r}\tilde{p}_i\|\nabla_yg_i(\bm{x}_r,\bm{y}_r)\|^2&\leq \Bigg\|
            \sum_{i\in C_r}\tilde{p}_i \nabla_yg_i(\bm{x}_r,\bm{y}_r)-\nabla_yg_i(\bm{x}_r,\bm{y}^*(\bm{x}_r))
            \Bigg\|^2\nonumber\\
            &\leq \Bigg\|
            \sum_{i\in C_r}\tilde{p}_i \underset{\zeta\sim\mathcal{D}_i(\bm{x}_r)}{\mathbb{E}}\nabla_yg_i(\bm{x}_r,\bm{y}_r;\zeta)-\underset{\zeta\sim\mathcal{D}_i(\bm{x}_r)}{\mathbb{E}}\nabla_yg_i(\bm{x}_r,\bm{y}^*(\bm{x}_r);\zeta)
            \Bigg\|^2\nonumber\\
            &\leq (L_g^y)^2\|\bm{y}_r-\bm{y}^*(\bm{x}_r)\|^2.
        \end{align}

According to the definition of $\nabla\mathcal{F}_i(\bm{x}_r,\bm{y}_r,\bm{v}_r(\bm{x}_r,\bm{y}_r))$, it yields that
\begin{align}
    &\sum_{i\in C_r}\tilde{p}_i\|\nabla\mathcal{F}_i(\bm{x}_r,\bm{y}_r,\bm{v}_r(\bm{x}_r,\bm{y}_r))\|^2\nonumber \\
    &=\sum_{i\in C_r}\tilde{p}_i
    \|
    \nabla\mathcal{F}_i(\bm{x}_r,\bm{y}_r,\bm{v}_r(\bm{x}_r,\bm{y}_r))
    -\nabla\mathcal{F}_i(\bm{x}_r,\bm{y}^*(\bm{x}_r),\bm{v}_r^*(\bm{x}_r,\bm{y}^*(\bm{x}_r)))
    +\nabla\mathcal{F}_i(\bm{x}_r,\bm{y}^*(\bm{x}_r),\bm{v}_r^*(\bm{x}_r,\bm{y}^*(\bm{x}_r)))\nonumber \\
    &\quad
    -\nabla\mathcal{F}_i(\bm{x}_r,\bm{y}^*(\bm{x}_r,\bm{x}^*),\bm{v}_r^*(\bm{x}_r,\bm{y}^*(\bm{x}_r,\bm{x}^*)))
    +\nabla\mathcal{F}_i(\bm{x}_r,\bm{y}^*(\bm{x}_r,\bm{x}^*),\bm{v}_r^*(\bm{x}_r,\bm{y}^*(\bm{x}_r,\bm{x}^*)))\nonumber \\
    &\quad
    -\nabla\mathcal{F}_i(\bm{x}^*,\bm{y}^*(\bm{x}^*),\bm{v}^*(\bm{x}^*,\bm{y}^*(\bm{x}^*)))
    \|^2
    \nonumber \\
    &\leq 3\sum_{i\in C_r}\tilde{p}_i\|
    \nabla\mathcal{F}_i(\bm{x}_r,\bm{y}_r,\bm{v}_r(\bm{x}_r,\bm{y}_r))
    -\nabla\mathcal{F}_i(\bm{x}_r,\bm{y}^*(\bm{x}_r),\bm{v}_r^*(\bm{x}_r,\bm{y}^*(\bm{x}_r)))\|^2\nonumber \\
    &\quad + 3\sum_{i\in C_r}\tilde{p}_i\|\nabla\mathcal{F}_i(\bm{x}_r,\bm{y}^*(\bm{x}_r),\bm{v}_r^*(\bm{x}_r,\bm{y}^*(\bm{x}_r)))
    -\nabla\mathcal{F}_i(\bm{x}_r,\bm{y}^*(\bm{x}_r,\bm{x}^*),\bm{v}_r^*(\bm{x}_r,\bm{y}^*(\bm{x}_r,\bm{x}^*)))
    \|^2\nonumber \\
    &\quad +3\sum_{i\in C_r}\tilde{p}_i\|
    \nabla\mathcal{F}_i(\bm{x}_r,\bm{y}^*(\bm{x}_r,\bm{x}^*),\bm{v}_r^*(\bm{x}_r,\bm{y}^*(\bm{x}_r,\bm{x}^*)))-\nabla\mathcal{F}_i(\bm{x}^*,\bm{y}^*(\bm{x}^*),\bm{v}^*(\bm{x}^*,\bm{y}^*(\bm{x}^*)))
    \|^2\nonumber \\
    &\overset{(a)}{\leq} 3(L_{Fc}^y)^2\|\bm{y}_r-\bm{y}^*(\bm{x}_r)\|^2+3(\tilde{L}_{Fc}^x)^2\|\bm{x}_r-\bm{x}^*\|^2\nonumber \\
    &\quad +3\sum_{i\in C_r}\tilde{p}_i\|
    \nabla\mathcal{F}_i(\bm{x}_r,\bm{y}^*(\bm{x}_r,\bm{x}^*),\bm{v}_r^*(\bm{x}_r,\bm{y}^*(\bm{x}_r,\bm{x}^*)))-\nabla\mathcal{F}_i(\bm{x}^*,\bm{y}^*(\bm{x}^*),\bm{v}^*(\bm{x}^*,\bm{y}^*(\bm{x}^*)))
    \|^2,
%
    %
\end{align}
where $(a)$ is based on Lemma \ref{lemma_5} and Lemma \ref{lemma_7}.

Similarly, we have
\begin{align}
            &\sum_{i\in C_r}\tilde{p}_i\|\nabla_vl_i(\bm{x}_r,\bm{y}_r,\bm{v}_r)\|^2\nonumber \\
            &\leq 3(L_{l_c}^y)^2\|\bm{y}_r-\bm{y}^*(\bm{x}_r)\|^2+3(\tilde{L}_{l_c}^x)^2\|\bm{x}_r-\bm{x}^*\|^2\nonumber \\
    &\quad +3\sum_{i\in C_r}\tilde{p}_i\|
    \nabla l_i(\bm{x}_r,\bm{y}^*(\bm{x}_r,\bm{x}^*),\bm{v}_r^*(\bm{x}_r,\bm{y}^*(\bm{x}_r,\bm{x}^*)))-\nabla l_i(\bm{x}^*,\bm{y}^*(\bm{x}^*),\bm{v}^*(\bm{x}^*,\bm{y}^*(\bm{x}^*)))
    \|^2.
        \end{align}

\end{proof}

\begin{lemma}\label{lemma_9}
    Under Assumption \ref{assump_1}, \ref{assump_3}, and \ref{assump_4}, we have
    \begin{align}
        &\|\nabla \mathcal{F}_i(\bm{x}_{r},\bm{y}_{r},\bm{v}_{r})-\nabla \mathcal{F}_i(\bm{x}_{i,r,k},\bm{y}_{i,r,k},\bm{v}_{i,r,k})\|^2\leq \Delta_{f,i,r,k},\\
        &\|\nabla_y g_i(\bm{x}_r,\bm{y}_r)-\nabla_y g_i(\bm{x}_{i,r,k},\bm{y}_{i,r,k})\|^2\leq \Delta_{g,i,r,k},\\
        &\|\nabla l_i(\bm{x}_{r},\bm{y}_{r},\bm{v}_{r})-\nabla l_i(\bm{x}_{i,r,k},\bm{y}_{i,r,k},\bm{v}_{i,r,k})\|^2\leq\Delta_{l,i,r,k},
    \end{align}
    where 
    \begin{align}
        &\Delta_{f,i,r,k}=\left(
        6(L_f^x)^2+24\iota^2(L_{gxy}^y\varepsilon_d+L_{gxy}^x)^2
        \right)\|\bm{x}_{i,r,k}-\bm{x}_r\|^2
        +\left(
        6(L_f^\xi\varepsilon_c+\bar{L}_f^x)^2+24\iota^2(L_{gxy}^y)^2
        \right)\|\bm{y}_{i,r,k}-\bm{y}_r\|^2\nonumber \\
        &\quad\quad\quad\quad
        +3(C_g^{xy})^2\|\bm{v}_{i,r,k}-\bm{v}_r\|^2,\nonumber \\
        &\Delta_{g,i,r,k}=2(L^y_g\varepsilon_d+L_g^x)^2\|\bm{x}_{i,r,k}-\bm{x}_r\|^2+2(L_g^y)^2\|\bm{y}_{i,r,k}-\bm{y}_r\|^2,\nonumber \\
        &\Delta_{l,i,r,k}=\left(
        6(L_f^y)^2+24\iota^2(L_{gyy}^y\varepsilon_d+L_{gyy}^x)^2
        \right)\|\bm{x}_{i,r,k}-\bm{x}_r\|^2
        +\left(
        6(L_f^\xi\varepsilon_c+\bar{L}_f^y)^2+24\iota^2(L_{gyy}^y)^2
        \right)\|\bm{y}_{i,r,k}-\bm{y}_r\|^2\nonumber \\
        &\quad\quad\quad\quad
        +3(C_g^{yy})^2\|\bm{v}_{i,r,k}-\bm{v}_r\|^2.
    \end{align}
\end{lemma}
\begin{proof}
From the definitions of $\nabla \mathcal{F}_i(\bm{x}_{r},\bm{y}_{r},\bm{v}_{r})$ and $\nabla \mathcal{F}_i(\bm{x}_{i,r,k},\bm{y}_{i,r,k},\bm{v}_{i,r,k})$, it's easy to obtain that
    \begin{align}
        &\|\nabla \mathcal{F}_i(\bm{x}_{r},\bm{y}_{r},\bm{v}_{r})-\nabla \mathcal{F}_i(\bm{x}_{i,r,k},\bm{y}_{i,r,k},\bm{v}_{i,r,k})\|^2\nonumber \\
        &\leq \|
        \nabla_x f_i(\bm{x}_{r},\bm{y}_{r})-\nabla_xf_i(\bm{x}_{i,r,k},\bm{y}_{i,r,k})
        -\nabla_{xy}^2g_i(\bm{x}_{r},\bm{y}_{r})\bm{v}_r+\nabla_{xy}^2g_i(\bm{x}_{i,r,k},\bm{y}_{i,r,k})\bm{v}_{i,r,k}
        \|^2\nonumber \\
        &\leq 3\|
        \nabla_x f_i(\bm{x}_{r},\bm{y}_{r})-\nabla_xf_i(\bm{x}_{i,r,k},\bm{y}_{i,r,k})
        \|^2+
        3\|
        (
        \nabla_{xy}^2g_i(\bm{x}_{r},\bm{y}_{r})-\nabla_{xy}^2g_i(\bm{x}_{i,r,k},\bm{y}_{i,r,k})
        )\bm{v}_{r}
        \|^2\nonumber \\
        &\quad+
        3\|
        \nabla_{xy}^2g_i(\bm{x}_{i,r,k},\bm{y}_{i,r,k})(\bm{v}_r-\bm{v}_{i,r,k})
        \|^2\nonumber \\
        &\leq 3\|\underset{\xi\sim\mathcal{C}_i(\bm{y}_r)}{\mathbb{E}}
        \nabla_x f_i(\bm{x}_{r},\bm{y}_{r};\xi)-\underset{\xi\sim\mathcal{C}_i(\bm{y}_{i,r,k})}{\mathbb{E}}\nabla_xf_i(\bm{x}_{i,r,k},\bm{y}_{i,r,k};\xi)
        \|^2\nonumber \\
        &\quad +12\iota^2\|
        \underset{\zeta\sim\mathcal{D}_i(\bm{x}_r)}{\mathbb{E}}\nabla_{xy}^2g_i(\bm{x}_{r},\bm{y}_{r};\zeta)-\underset{\zeta\sim\mathcal{D}_i(\bm{x}_{i,r,k})}{\mathbb{E}}\nabla_{xy}^2g_i(\bm{x}_{i,r,k},\bm{y}_{i,r,k};\zeta)
        \|^2 +3(C_g^{xy})^2\|\bm{v}_{i,r,k}-\bm{v}_r\|^2.
    \end{align}
Similar to the proof in Lemma \ref{lemma_6}, we further obtain that 
\begin{align}
    &\|\nabla \mathcal{F}_i(\bm{x}_{r},\bm{y}_{r},\bm{v}_{r})-\nabla \mathcal{F}_i(\bm{x}_{i,r,k},\bm{y}_{i,r,k},\bm{v}_{i,r,k})\|^2\nonumber \\
        &\leq \left(
        6(L_f^x)^2+24\iota^2(L_{gxy}^y\varepsilon_d+L_{gxy}^x)^2
        \right)\|\bm{x}_{i,r,k}-\bm{x}_r\|^2
        +\left(
        6(L_f^\xi\varepsilon_c+\bar{L}_f^x)^2+24\iota^2(L_{gxy}^y)^2
        \right)\|\bm{y}_{i,r,k}-\bm{y}_r\|^2\nonumber \\
        &\quad
        +3(C_g^{xy})^2\|\bm{v}_{i,r,k}-\bm{v}_r\|^2,
\end{align}
\begin{align}
    \|\nabla_y g_i(\bm{x}_r,\bm{y}_r)-\nabla_y g_i(\bm{x}_{i,r,k},\bm{y}_{i,r,k})\|^2
    &\leq \|
        \underset{\zeta\sim\mathcal{D}_i(\bm{x}_r)}{\mathbb{E}}\nabla_{y}g_i(\bm{x}_{r},\bm{y}_{r};\zeta)-\underset{\zeta\sim\mathcal{D}_i(\bm{x}_{i,r,k})}{\mathbb{E}}\nabla_{y}g_i(\bm{x}_{i,r,k},\bm{y}_{i,r,k};\zeta)
        \|^2\nonumber \\
    &\leq 2(L^y_g\varepsilon_d+L_g^x)^2\|\bm{x}_{i,r,k}-\bm{x}_r\|^2+2(L_g^y)^2\|\bm{y}_{i,r,k}-\bm{y}_r\|^2,
\end{align}
and
\begin{align}
    &\|\nabla_v l_i(\bm{x}_{r},\bm{y}_{r},\bm{v}_{r})-\nabla_v l_i(\bm{x}_{i,r,k},\bm{y}_{i,r,k},\bm{v}_{i,r,k})\|^2\nonumber \\
    &\leq \|
    \nabla_{yy}^2g_i(\bm{x}_r,\bm{y}_r)\bm{v}_r-\nabla_yf_i(\bm{x}_r,\bm{y}_r)-\nabla_{yy}^2g_i(\bm{x}_{i,r,k},\bm{y}_{i,r,k})\bm{v}_{i,r,k}+\nabla_yf_i(\bm{x}_{i,r,k},\bm{y}_{i,r,k})
    \|^2\nonumber \\
    &\leq 3\|
        \nabla_y f_i(\bm{x}_{r},\bm{y}_{r})-\nabla_yf_i(\bm{x}_{i,r,k},\bm{y}_{i,r,k})
        \|^2+
        3\|
        (
        \nabla_{yy}^2g_i(\bm{x}_{r},\bm{y}_{r})-\nabla_{yy}^2g_i(\bm{x}_{i,r,k},\bm{y}_{i,r,k})
        )\bm{v}_{r}
        \|^2\nonumber \\
        &\quad+
        3\|
        \nabla_{yy}^2g_i(\bm{x}_{i,r,k},\bm{y}_{i,r,k})(\bm{v}_r-\bm{v}_{i,r,k})
        \|^2\nonumber \\
        &\leq 3\|\underset{\xi\sim\mathcal{C}_i(\bm{y}_r)}{\mathbb{E}}
        \nabla_y f_i(\bm{x}_{r},\bm{y}_{r};\xi)-\underset{\xi\sim\mathcal{C}_i(\bm{y}_{i,r,k})}{\mathbb{E}}\nabla_yf_i(\bm{x}_{i,r,k},\bm{y}_{i,r,k};\xi)
        \|^2\nonumber \\
        &\quad +12\iota^2\|
        \underset{\zeta\sim\mathcal{D}_i(\bm{x}_r)}{\mathbb{E}}\nabla_{yy}^2g_i(\bm{x}_{r},\bm{y}_{r};\zeta)-\underset{\zeta\sim\mathcal{D}_i(\bm{x}_{i,r,k})}{\mathbb{E}}\nabla_{yy}^2g_i(\bm{x}_{i,r,k},\bm{y}_{i,r,k};\zeta)
        \|^2 +3(C_g^{yy})^2\|\bm{v}_{i,r,k}-\bm{v}_r\|^2\nonumber\\
        &\leq \left(
        6(L_f^y)^2+24\iota^2(L_{gyy}^y\varepsilon_d+L_{gyy}^x)^2
        \right)\|\bm{x}_{i,r,k}-\bm{x}_r\|^2
        +\left(
        6(L_f^\xi\varepsilon_c+\bar{L}_f^y)^2+24\iota^2(L_{gyy}^y)^2
        \right)\|\bm{y}_{i,r,k}-\bm{y}_r\|^2\nonumber \\
        &\quad
        +3(C_g^{yy})^2\|\bm{v}_{i,r,k}-\bm{v}_r\|^2.
\end{align}
\end{proof}

\clearpage
\section{Convergence of the FBi-SGD Method (Proof of Theorem 2)}
In this section, we first present some useful Lemmas and then prove the convergence of the FBi-SGD method.

\paragraph{Proof roadmap.}
The proof of Theorem~2 is organized around a Lyapunov analysis of the coupled stochastic iterates $(\bm{x}_r,\bm{y}_r,\bm{v}_r)$. First, Appendix~\ref{app:client_drift} controls the client drifts caused by local SGD steps, namely the deviations between local variables $(\bm{x}_{i,r,k},\bm{y}_{i,r,k},\bm{v}_{i,r,k})$ and the server variables $(\bm{x}_r,\bm{y}_r,\bm{v}_r)$. Second, Appendix~\ref{app:aggregated_estimation} bounds the bias and variance of the aggregated stochastic estimators for the UL, LL, and auxiliary LS updates. Third, Appendices~\ref{app:ll_descent}--\ref{app:ls_descent} derive one-step descent inequalities for the LL variable $\bm{y}_r$, the UL variable $\bm{x}_r$, and the auxiliary variable $\bm{v}_r$, respectively. These inequalities show how the three errors interact under stochastic local updates and decision-dependent distribution shifts. Finally, Appendix~\ref{sec.sgd_whole_convergence} combines the three descent inequalities into a single Lyapunov recursion. Under diminishing stepsizes and sufficiently small performative sensitivities, the contraction terms dominate the stochastic and drift errors, yielding the mean-square convergence rate and almost-sure convergence stated in Theorem~2.
\subsection{Bounds of Client Drifts}\label{app:client_drift}
\begin{lemma}\label{lemma_10}
    Under Assumptions \ref{assump_1}, \ref{assump_3}, \ref{assump_4}, \ref{assump_5}, the client drifts of $\bm{y}_{i,r,k}$, $\bm{x}_{i,r,k}$, and $\bm{v}_{i,r,k}$ at the local iteration are bounded as
    \begin{align}
        &\sum_{i\in C_r}\tilde{p}_i\sum_{k=0}^{K_r-1}\mathbb{E}\|\bm{x}_{i,r,k}-\bm{x}_r\|^2 \leq p_{x,7}+p_{x,8}\|\bm{y}_r-\bm{y}^*(\bm{x}_r)\|^2 +p_{x,9}\|\bm{x}_r-\bm{x}^*\|^2+p_{x,10}\mathcal{V}_F+p_{x,11}\mathcal{V}_l,\\
    &\sum_{i\in C_r}\tilde{p}_i\sum_{k=0}^{K_r-1}\mathbb{E}\|\bm{y}_{i,r,k}-\bm{y}_r\|^2 \leq p_{y,4}+p_{y,5}\|\bm{y}_r-\bm{y}^*(\bm{x}_r)\|^2 +p_{y,6}\|\bm{x}_r-\bm{x}^*\|^2+p_{y,7}\mathcal{V}_F+p_{y,8}\mathcal{V}_l,\\
    &\sum_{i\in C_r}\tilde{p}_i\sum_{k=0}^{K_r-1}\mathbb{E}\|\bm{v}_{i,r,k}-\bm{v}_r\|^2 \leq p_{v,7}+p_{v,8}\|\bm{y}_r-\bm{y}^*(\bm{x}_r)\|^2 +p_{v,9}\|\bm{x}_r-\bm{x}^*\|^2+p_{v,10}\mathcal{V}_F+p_{v,11}\mathcal{V}_l,
    \end{align}
    where we define

\begin{align}
    &p_{x,1} = \frac{8(\eta_x^{(c)})^2\bar{K}(\sigma_f^2+\sigma_r^2)
    +16(\eta_x^{(c)})^2\bar{K}\iota^2\sigma_{gg}^2}{1-4(\eta_x^{(c)})^2\bar{K}\left(
        6(L_f^x)^2+24\iota^2(L_{gxy}^y\varepsilon_d+L_{gxy}^x)^2
        \right)},
    p_{x,2} = \frac{4(\eta_x^{(c)})^2\bar{K}\left(
        6(L_f^\xi\varepsilon_c+\bar{L}_f^x)^2+24\iota^2(L_{gxy}^y)^2
        \right)}{1-4(\eta_x^{(c)})^2\bar{K}\left(
        6(L_f^x)^2+24\iota^2(L_{gxy}^y\varepsilon_d+L_{gxy}^x)^2
        \right)},\nonumber\\
    &p_{x,3} = \frac{12(\eta_x^{(c)})^2\bar{K} (C_g^{xy})^2   }{1-4(\eta_x^{(c)})^2\bar{K}\left(
        6(L_f^x)^2+24\iota^2(L_{gxy}^y\varepsilon_d+L_{gxy}^x)^2
        \right)},
    p_{x,4}=\frac{ 12(\eta_x^{(c)})^2\bar{K} (L_{Fc}^y)^2}{1-4(\eta_x^{(c)})^2\bar{K}\left(
        6(L_f^x)^2+24\iota^2(L_{gxy}^y\varepsilon_d+L_{gxy}^x)^2
        \right)},\nonumber\\
    &p_{x,5}=\frac{ 12(\eta_x^{(c)})^2\bar{K} (\tilde{L}_{Fc}^x)^2}{1-4(\eta_x^{(c)})^2\bar{K}\left(
        6(L_f^x)^2+24\iota^2(L_{gxy}^y\varepsilon_d+L_{gxy}^x)^2
        \right)},
    p_{x,6}=\frac{ 12(\eta_x^{(c)})^2\bar{K} }{1-4(\eta_x^{(c)})^2\bar{K}\left(
        6(L_f^x)^2+24\iota^2(L_{gxy}^y\varepsilon_d+L_{gxy}^x)^2
        \right)},\nonumber\\
    &p_{x,7} = \frac{p_{x,1}+(p_{x,2}+p_{x,3}p_{v,3})p_{y,1}+p_{x,3}p_{v,1}}{1-p_{x,2}p_{y,2}-p_{x,3}p_{v,2}-p_{x,3}p_{v,3}p_{y,2}},
    p_{x,8} = \frac{p_{x,4}+(p_{x,2}+p_{x,3}p_{v,3}) p_{y,3}  +p_{x,3} p_{v,4}  }{1-p_{x,2}p_{y,2}-p_{x,3}p_{v,2}-p_{x,3}p_{v,3}p_{y,2}},\nonumber\\
    &p_{x,9} = \frac{p_{x,5}+ p_{x,3} p_{v,5}   }{1-p_{x,2}p_{y,2}-p_{x,3}p_{v,2}-p_{x,3}p_{v,3}p_{y,2}},
    p_{x,10} = \frac{ p_{x,6}  }{1-p_{x,2}p_{y,2}-p_{x,3}p_{v,2}-p_{x,3}p_{v,3}p_{y,2}},\nonumber \\
    &p_{x,11} = \frac{p_{x,3}p_{v,6}   }{1-p_{x,2}p_{y,2}-p_{x,3}p_{v,2}-p_{x,3}p_{v,3}p_{y,2}},
\end{align}
\begin{align}
    &p_{y,1}=\frac{2\sigma_g^2(\eta_y^{(c)})^2 \bar{K}}{1-8(\eta_y^{(c)})^2\bar{K}(L_g^y)^2},
    p_{y,2}=\frac{8 (L^y_g\varepsilon_d+L_g^x)^2(\eta_y^{(c)})^2 \bar{K}}{1-8(\eta_y^{(c)})^2\bar{K}(L_g^y)^2},
    p_{y,3}=\frac{4(L_g^y)^2(\eta_y^{(c)})^2 \bar{K}}{1-8(\eta_y^{(c)})^2\bar{K}(L_g^y)^2},\nonumber\\
    &p_{y,4} = \frac{
    p_{y,2}p_{x,1}+(1-p_{v,2}p_{x,3})p_{y,1}+p_{y,2}p_{x,3}p_{v,1}
    }{1-p_{x,2}p_{y,2}-p_{x,3}p_{v,2}-p_{x,3}p_{v,3}p_{y,2}},
    p_{y,5}= \frac{
    p_{y,2}p_{x,4}+(1-p_{v,2}p_{x,3})p_{y,3}+p_{y,2}p_{x,3}p_{v,4}
    }{1-p_{x,2}p_{y,2}-p_{x,3}p_{v,2}-p_{x,3}p_{v,3}p_{y,2}},\nonumber\\
    &p_{y,6} = \frac{
    p_{y,2}p_{x,5}+p_{y,2}p_{x,3}p_{v,5}
    }{1-p_{x,2}p_{y,2}-p_{x,3}p_{v,2}-p_{x,3}p_{v,3}p_{y,2}},
    p_{y,7} = \frac{
    p_{y,2}p_{x,6}
    }{1-p_{x,2}p_{y,2}-p_{x,3}p_{v,2}-p_{x,3}p_{v,3}p_{y,2}},\nonumber\\
    &p_{y,8} = \frac{
    p_{y,2}p_{x,3}p_{v,6}
    }{1-p_{x,2}p_{y,2}-p_{x,3}p_{v,2}-p_{x,3}p_{v,3}p_{y,2}},
    p_{v,1} = \frac{8(\eta_v^{(c)})^2\bar{K}(\sigma_f^2+\sigma_r^2)+16(\eta_v^{(c)})^2\bar{K}\iota^2\sigma_{gg}^2}{1-12(\eta_v^{(c)})^2\bar{K}(C_g^{yy})^2},\nonumber\\
    &p_{v,2}=\frac{4(\eta_v^{(c)})^2\bar{K}\left(
        6(L_f^y)^2+24\iota^2(L_{gyy}^y\varepsilon_d+L_{gyy}^x)^2
        \right) }{1-12(\eta_v^{(c)})^2\bar{K}(C_g^{yy})^2},
    p_{v,3}=\frac{4(\eta_v^{(c)})^2\bar{K}  \left(
        6(L_f^\xi\varepsilon_c+\bar{L}_f^y)^2+24\iota^2(L_{gyy}^y)^2
        \right) }{1-12(\eta_v^{(c)})^2\bar{K}(C_g^{yy})^2},\nonumber\\
    &p_{v,4}=\frac{12(\eta_v^{(c)})^2\bar{K} (L_{l_c}^y)^2  }{1-12(\eta_v^{(c)})^2\bar{K}(C_g^{yy})^2},
    p_{v,5}=\frac{12(\eta_v^{(c)})^2\bar{K}  (\tilde{L}_{l_c}^x)^2 }{1-12(\eta_v^{(c)})^2\bar{K}(C_g^{yy})^2},
    p_{v,6}=\frac{12(\eta_v^{(c)})^2\bar{K} }{1-12(\eta_v^{(c)})^2\bar{K}(C_g^{yy})^2},\nonumber\\
    &p_{v,7}=\frac{
    (p_{v,2}+p_{v,3}p_{y,2})p_{x,1}+(p_{v,2}p_{x,2}+p_{v,3})p_{y,1}+(1-p_{x,2}p_{y,2})p_{v,1}
    }{1-p_{x,2}p_{y,2}-p_{x,3}p_{v,2}-p_{x,3}p_{v,3}p_{y,2}},\nonumber\\
    &p_{v,8} = \frac{
    (p_{v,2}+p_{v,3}p_{y,2})p_{x,4}+(p_{v,2}p_{x,2}+p_{v,3})p_{y,3}+(1-p_{x,2}p_{y,2})p_{v,4}
    }{1-p_{x,2}p_{y,2}-p_{x,3}p_{v,2}-p_{x,3}p_{v,3}p_{y,2}},\nonumber\\
    &p_{v,9} = \frac{
    (p_{v,2}+p_{v,3}p_{y,2})p_{x,5}+(1-p_{x,2}p_{y,2})p_{v,5}
    }{1-p_{x,2}p_{y,2}-p_{x,3}p_{v,2}-p_{x,3}p_{v,3}p_{y,2}},
    p_{v,10} = \frac{
    (p_{v,2}+p_{v,3}p_{y,2})p_{x,6}
    }{1-p_{x,2}p_{y,2}-p_{x,3}p_{v,2}-p_{x,3}p_{v,3}p_{y,2}},\nonumber\\
    &p_{v,11} = \frac{
    (1-p_{x,2}p_{y,2})p_{v,6}
    }{1-p_{x,2}p_{y,2}-p_{x,3}p_{v,2}-p_{x,3}p_{v,3}p_{y,2}}.
\end{align}

\end{lemma}
\begin{proof}
With the definitions of $\bm{y}_{i,r,k}$ and $\bm{y}_r$, we have
    \begin{align}
        & \sum_{i\in C_r}\tilde{p}_i\sum_{k=0}^{K_r-1}\mathbb{E}\|\bm{y}_{i,r,k}-\bm{y}_r\|^2\nonumber \\
        &=(\eta_y^{(c)})^2\sum_{i\in C_r}\tilde{p}_i\sum_{k=0}^{K_r-1}\mathbb{E}\left\|
        \sum_{j=0}^{k-1}\left(
        \widehat{\nabla}_yg_i(\bm{x}_{i,r,j},\bm{y}_{i,r,j};\zeta)-\nabla_yg_i(\bm{x}_{i,r,j},\bm{y}_{i,r,j})+\nabla_yg_i(\bm{x}_{i,r,j},\bm{y}_{i,r,j})
        \right)
        \right\|^2\nonumber \\
        &\leq 2(\eta_y^{(c)})^2\sum_{i\in C_r}\tilde{p}_i\sum_{k=0}^{K_r-1}\sum_{j=0}^{k-1}\mathbb{E}\left\|
        \widehat{\nabla}_yg_i(\bm{x}_{i,r,j},\bm{y}_{i,r,j};\zeta)-\nabla_yg_i(\bm{x}_{i,r,j},\bm{y}_{i,r,j})
        \right\|^2\nonumber \\
        &\quad + 2(\eta_y^{(c)})^2\sum_{i\in C_r}\tilde{p}_i\sum_{k=0}^{K_r-1}\sum_{j=0}^{k-1}\mathbb{E}\left\|
        \nabla_yg_i(\bm{x}_{i,r,j},\bm{y}_{i,r,j})
        \right\|^2\nonumber \\
        &\leq 2(\eta_y^{(c)})^2\sum_{i\in C_r}\tilde{p}_i\sum_{k=0}^{K_r-1}\sum_{j=0}^{k-1}\sigma_g^2
        + 4(\eta_y^{(c)})^2\sum_{i\in C_r}\tilde{p}_i\sum_{k=0}^{K_r-1}\sum_{j=0}^{k-1}\mathbb{E}\left\|
        \nabla_yg_i(\bm{x}_{r},\bm{y}_{r})
        \right\|^2\nonumber\\
        &\quad +4(\eta_y^{(c)})^2\sum_{i\in C_r}\tilde{p}_i\sum_{k=0}^{K_r-1}\sum_{j=0}^{k-1}\mathbb{E}\left\|
        \nabla_yg_i(\bm{x}_{i,r,j},\bm{y}_{i,r,j})-\nabla_yg_i(\bm{x}_{r},\bm{y}_{r})
        \right\|^2\nonumber \\
        &\leq 2(\eta_y^{(c)})^2\sum_{i\in C_r}\tilde{p}_i\sum_{k=0}^{K_r-1}\sum_{j=0}^{k-1}\sigma_g^2\nonumber \\
        &\quad +4(\eta_y^{(c)})^2\sum_{i\in C_r}\tilde{p}_i\sum_{k=0}^{K_r-1}\sum_{j=0}^{k-1}\mathbb{E}\left\|
        \nabla_yg_i(\bm{x}_{r},\bm{y}_{r})-\nabla_yg_i(\bm{x}_{r},\bm{y}^*(\bm{x}_r))
        \right\|^2\nonumber
    \end{align}
    \begin{align}
        &\quad + 4(\eta_y^{(c)})^2\sum_{i\in C_r}\tilde{p}_i\sum_{k=0}^{K_r-1}\sum_{j=0}^{k-1}\mathbb{E}\left\|
        \nabla_yg_i(\bm{x}_{i,r,j},\bm{y}_{i,r,j})-\nabla_yg_i(\bm{x}_{r},\bm{y}_{r})
        \right\|^2
        \nonumber\\
        &\overset{(a)}{\leq}2(\eta_y^{(c)})^2\sum_{i\in C_r}\tilde{p}_i\sum_{k=0}^{K_r-1}\sum_{j=0}^{k-1}\sigma_g^2
        + 4(\eta_y^{(c)})^2\sum_{i\in C_r}\tilde{p}_i\sum_{k=0}^{K_r-1}\sum_{j=0}^{k-1} \Delta_{g,i,r,j}\nonumber \\
        &\quad + 4(\eta_y^{(c)})^2\sum_{i\in C_r}\tilde{p}_i\sum_{k=0}^{K_r-1}\sum_{j=0}^{k-1}(L_g^y)^2\|\bm{y}_r-\bm{y}^*(\bm{x}_r)\|^2\nonumber \\
        &\overset{(b)}{\leq}
        2(\eta_y^{(c)})^2\bar{K}\sigma_g^2
        +4(\eta_y^{(c)})^2\bar{K} \sum_{i\in C_r}\tilde{p}_i\sum_{k=0}^{K_r-1}
        \left(2(L^y_g\varepsilon_d+L_g^x)^2\|\bm{x}_{i,r,k}-\bm{x}_r\|^2+2(L_g^y)^2\|\bm{y}_{i,r,k}-\bm{y}_r\|^2\right)\nonumber \\
        &\quad
        +4(\eta_y^{(c)})^2\bar{K}(L_g^y)^2\|\bm{y}_r-\bm{y}^*(\bm{x}_r)\|^2,\label{eq.lemma10_1}
    \end{align}
    where $(a)$ is based on Lemma \ref{lemma_8} and Lemma \ref{lemma_9}, and $(b)$ holds because $\bar{K}$ is defined as $\bar{K}=\sum_{i\in C_r}\tilde{p}_i\sum_{k=0}^{K_r-1}\sum_{j=0}^{k-1}1$.
    
    From Eq. (\ref{eq.lemma10_1}), we further obtain that
    \begin{align}
    &\sum_{i\in C_r}\tilde{p}_i\sum_{k=0}^{K_r-1}\mathbb{E}\|\bm{y}_{i,r,k}-\bm{y}_r\|^2\nonumber \\
        &\leq \frac{(\eta_y^{(c)})^2 \bar{K}}{1-8(\eta_y^{(c)})^2\bar{K}(L_g^y)^2}\Bigg(
        2\sigma_g^2
        +8 (L^y_g\varepsilon_d+L_g^x)^2\sum_{i\in C_r}\tilde{p}_i\sum_{k=0}^{K_r-1}
        \mathbb{E}\|\bm{x}_{i,r,k}-\bm{x}_r\|^2 +4(L_g^y)^2\mathbb{E}\|\bm{y}_r-\bm{y}^*(\bm{x}_r)\|^2
        \Bigg).\label{eq.lemma_10_2}
    \end{align}
For the sake of simplicity, let
\begin{align}
    &p_{y,1}=\frac{2\sigma_g^2(\eta_y^{(c)})^2 \bar{K}}{1-8(\eta_y^{(c)})^2\bar{K}(L_g^y)^2},
    p_{y,2}=\frac{8 (L^y_g\varepsilon_d+L_g^x)^2(\eta_y^{(c)})^2 \bar{K}}{1-8(\eta_y^{(c)})^2\bar{K}(L_g^y)^2},
    p_{y,3}=\frac{4(L_g^y)^2(\eta_y^{(c)})^2 \bar{K}}{1-8(\eta_y^{(c)})^2\bar{K}(L_g^y)^2}.
\end{align}
Obviously, we need to make the term $1-8(\eta_y^{(c)})^2\bar{K}(L_g^y)^2>0$, and further know that the condition is $\eta_y^{(c)} < \frac{1}{4\bar{K}L_g^y}$.


Therefore, Eq. (\ref{eq.lemma_10_2}) can be rewritten as
\begin{align}
    \sum_{i\in C_r}\tilde{p}_i\sum_{k=0}^{K_r-1}\mathbb{E}\|\bm{y}_{i,r,k}-\bm{y}_r\|^2
    \leq p_{y,1}+p_{y,2}\sum_{i\in C_r}\tilde{p}_i\sum_{k=0}^{K_r-1}
        \mathbb{E}\|\bm{x}_{i,r,k}-\bm{x}_r\|^2+p_{y,3}\mathbb{E}\|\bm{y}_r-\bm{y}^*(\bm{x}_r)\|^2.\label{eq.lemma10_re1}
\end{align}

Similarly, based on the definitions of $\bm{x}_{i,r,k}$ and $\bm{x}_r$, we know that 
    \begin{align}
    &\sum_{i\in C_r}\tilde{p}_i\sum_{k=0}^{K_r-1}\mathbb{E}\|\bm{x}_{i,r,k}-\bm{x}_r\|^2\nonumber \\
    &=(\eta_x^{(c)})^2\sum_{i\in C_r}\tilde{p}_i\sum_{k=0}^{K_r-1}\mathbb{E}\left\|
        \sum_{j=0}^{k-1}\left(
        \widehat{\nabla}_x\mathcal{F}_i(\bm{x}_{i,r,j},\bm{y}_{i,r,j};\psi)-\nabla_x\mathcal{F}_i(\bm{x}_{i,r,j},\bm{y}_{i,r,j})+\nabla_x\mathcal{F}_i(\bm{x}_{i,r,j},\bm{y}_{i,r,j})
        \right)
        \right\|^2\nonumber \\
    &\leq2(\eta_x^{(c)})^2\sum_{i\in C_r}\tilde{p}_i\sum_{k=0}^{K_r-1}\sum_{j=0}^{k-1}\mathbb{E}\left\|
        \left(
        \widehat{\nabla}_x\mathcal{F}_i(\bm{x}_{i,r,j},\bm{y}_{i,r,j};\psi)-\nabla_x\mathcal{F}_i(\bm{x}_{i,r,j},\bm{y}_{i,r,j})\right)\right\|^2\nonumber\\
    &\quad +4(\eta_x^{(c)})^2\sum_{i\in C_r}\tilde{p}_i\sum_{k=0}^{K_r-1}\sum_{j=0}^{k-1}\mathbb{E}\left\|\nabla_x\mathcal{F}_i(\bm{x}_{i,r,j},\bm{y}_{i,r,j})-\nabla_x\mathcal{F}_i(\bm{x}_r,\bm{y}_r)
        \right\|^2+4(\eta_x^{(c)})^2\sum_{i\in C_r}\tilde{p}_i\sum_{k=0}^{K_r-1}\sum_{j=0}^{k-1}\mathbb{E}\left\|\nabla_x\mathcal{F}_i(\bm{x}_r,\bm{y}_r)
        \right\|^2\nonumber \\
    &\overset{(a)}{\leq} 2(\eta_x^{(c)})^2\sum_{i\in C_r}\tilde{p}_i\sum_{k=0}^{K_r-1}\sum_{j=0}^{k-1}\mathbb{E}\left\|
        \left(
        \widehat{\nabla}_x\mathcal{F}_i(\bm{x}_{i,r,j},\bm{y}_{i,r,j};\psi)-\nabla_x\mathcal{F}_i(\bm{x}_{i,r,j},\bm{y}_{i,r,j})\right)\right\|^2\nonumber\\
    &\quad +4(\eta_x^{(c)})^2\sum_{i\in C_r}\tilde{p}_i\sum_{k=0}^{K_r-1}\sum_{j=0}^{k-1} \Delta_{f,i,r,j}
     +12(\eta_x^{(c)})^2\bar{K}\Big(
    (L_{Fc}^y)^2\|\bm{y}_r-\bm{y}^*(\bm{x}_r)\|^2+(\tilde{L}_{Fc}^x)^2\|\bm{x}_r-\bm{x}^*\|^2+\mathcal{V}_F
    \Big)\nonumber \\
    &\overset{(b)}{\leq}2(\eta_x^{(c)})^2\sum_{i\in C_r}\tilde{p}_i\sum_{k=0}^{K_r-1}\sum_{j=0}^{k-1}\mathbb{E}\Big\|
        \Big(
        \widehat{\nabla}_xf_i(\bm{x}_{i,r,j},\bm{y}_{i,r,j};\xi)-\widehat{\nabla}_{xy}^2g_i(\bm{x}_{i,r,j},\bm{y}_{i,r,j};\zeta)\bm{v}_{i,r,k}\nonumber \\
        &\quad -\nabla_xf_i(\bm{x}_{i,r,j},\bm{y}_{i,r,j};\xi)+\nabla_{xy}^2g_i(\bm{x}_{i,r,j},\bm{y}_{i,r,j};\zeta)\bm{v}_{i,r,j}
        \Big)\Big\|^2 \nonumber \\
    &\quad +4(\eta_x^{(c)})^2\bar{K}\sum_{i\in C_r}\tilde{p}_i\sum_{k=0}^{K_r-1}\mathbb{E}\Bigg(\left(
        6(L_f^x)^2+24\iota^2(L_{gxy}^y\varepsilon_d+L_{gxy}^x)^2
        \right)\|\bm{x}_{i,r,k}-\bm{x}_r\|^2
        +\left(
        6(L_f^\xi\varepsilon_c+\bar{L}_f^x)^2+24\iota^2(L_{gxy}^y)^2
        \right)\nonumber \\
        &\quad 
        \cdot\|\bm{y}_{i,r,k}-\bm{y}_r\|^2+3(C_g^{xy})^2\|\bm{v}_{i,r,k}-\bm{v}_r\|^2\Bigg)
         + 12(\eta_x^{(c)})^2\bar{K}\Big(
    (L_{Fc}^y)^2\|\bm{y}_r-\bm{y}^*(\bm{x}_r)\|^2+(\tilde{L}_{Fc}^x)^2\|\bm{x}_r-\bm{x}^*\|^2 +\mathcal{V}_F
    \Big)\nonumber \\
    &\leq 4(\eta_x^{(c)})^2\sum_{i\in C_r}\tilde{p}_i\sum_{k=0}^{K_r-1}\sum_{j=0}^{k-1}\mathbb{E}\Big\|
        \Big(
        \widehat{\nabla}_xf_i(\bm{x}_{i,r,j},\bm{y}_{i,r,j};\xi)-\nabla_xf_i(\bm{x}_{i,r,j},\bm{y}_{i,r,j};\xi)\Big\|^2\nonumber \\
        &\quad +4(\eta_x^{(c)})^2\sum_{i\in C_r}\tilde{p}_i\sum_{k=0}^{K_r-1}\sum_{j=0}^{k-1}\Big\|\bm{v}_{i,r,j}\Big\|^2\mathbb{E}\Big\|
        \widehat{\nabla}_{xy}^2g_i(\bm{x}_{i,r,j},\bm{y}_{i,r,j};\zeta)
        -\nabla_{xy}^2g_i(\bm{x}_{i,r,j},\bm{y}_{i,r,j};\zeta)
        \Big\|^2\nonumber \\
        &\quad 
        +4(\eta_x^{(c)})^2\bar{K}\sum_{i\in C_r}\tilde{p}_i\sum_{k=0}^{K_r-1}\mathbb{E}\Bigg(\left(
        6(L_f^x)^2+24\iota^2(L_{gxy}^y\varepsilon_d+L_{gxy}^x)^2
        \right)\|\bm{x}_{i,r,k}-\bm{x}_r\|^2
        +\left(
        6(L_f^\xi\varepsilon_c+\bar{L}_f^x)^2+24\iota^2(L_{gxy}^y)^2
        \right)\nonumber \\
        &\quad 
        \cdot \|\bm{y}_{i,r,k}-\bm{y}_r\|^2 +3(C_g^{xy})^2\|\bm{v}_{i,r,k}-\bm{v}_r\|^2\Bigg)
         + 12(\eta_x^{(c)})^2\bar{K}\Big(
    (L_{Fc}^y)^2\|\bm{y}_r-\bm{y}^*(\bm{x}_r)\|^2+(\tilde{L}_{Fc}^x)^2\|\bm{x}_r-\bm{x}^*\|^2 +\mathcal{V}_F
    \Big)
        \nonumber \\
    &\overset{(c)}{\leq} 8(\eta_x^{(c)})^2\bar{K}(\sigma_f^2+\sigma_r^2)
    +16(\eta_x^{(c)})^2\bar{K}\iota^2\sigma_{gg}^2\nonumber \\
    &\quad
    +4(\eta_x^{(c)})^2\bar{K}\sum_{i\in C_r}\tilde{p}_i\sum_{k=0}^{K_r-1}\mathbb{E}\Bigg(\left(
        6(L_f^x)^2+24\iota^2(L_{gxy}^y\varepsilon_d+L_{gxy}^x)^2
        \right)\|\bm{x}_{i,r,k}-\bm{x}_r\|^2
        +\left(
        6(L_f^\xi\varepsilon_c+\bar{L}_f^x)^2+24\iota^2(L_{gxy}^y)^2
        \right)\nonumber \\
        &\quad 
        \cdot \|\bm{y}_{i,r,k}-\bm{y}_r\|^2 +3(C_g^{xy})^2\|\bm{v}_{i,r,k}-\bm{v}_r\|^2\Bigg)
         + 12(\eta_x^{(c)})^2\bar{K}\Big(
    (L_{Fc}^y)^2\|\bm{y}_r-\bm{y}^*(\bm{x}_r)\|^2+(\tilde{L}_{Fc}^x)^2\|\bm{x}_r-\bm{x}^*\|^2 +\mathcal{V}_F
    \Big),\label{eq.lemma10_3}
    \end{align}
where $(a)$ is based on Lemma \ref{lemma_8}, $(b)$ is correct because of Lemma \ref{lemma_9}, $(c)$ holds because of Assumption \ref{assump_5}.

From Eq. (\ref{eq.lemma10_3}), we know
\begin{align}
    &\sum_{i\in C_r}\tilde{p}_i\sum_{k=0}^{K_r-1}\mathbb{E}\|\bm{x}_{i,r,k}-\bm{x}_r\|^2\nonumber \\
    &\leq \frac{8(\eta_x^{(c)})^2\bar{K}(\sigma_f^2+\sigma_r^2)
    +16(\eta_x^{(c)})^2\bar{K}\iota^2\sigma_{gg}^2}{1-4(\eta_x^{(c)})^2\bar{K}\left(
        6(L_f^x)^2+24\iota^2(L_{gxy}^y\varepsilon_d+L_{gxy}^x)^2
        \right)}\nonumber \\
    &\quad +\frac{4(\eta_x^{(c)})^2\bar{K}\left(
        6(L_f^\xi\varepsilon_c+\bar{L}_f^x)^2+24\iota^2(L_{gxy}^y)^2
        \right)}{1-4(\eta_x^{(c)})^2\bar{K}\left(
        6(L_f^x)^2+24\iota^2(L_{gxy}^y\varepsilon_d+L_{gxy}^x)^2
        \right)}\sum_{i\in C_r}\tilde{p}_i\sum_{k=0}^{K_r-1}\mathbb{E}\|\bm{y}_{i,r,k}-\bm{y}_r\|^2\nonumber \\
    &\quad + \frac{12(\eta_x^{(c)})^2\bar{K} (C_g^{xy})^2   }{1-4(\eta_x^{(c)})^2\bar{K}\left(
        6(L_f^x)^2+24\iota^2(L_{gxy}^y\varepsilon_d+L_{gxy}^x)^2
        \right)}\sum_{i\in C_r}\tilde{p}_i\sum_{k=0}^{K_r-1}\mathbb{E}\|\bm{v}_{i,r,k}-\bm{v}_r\|^2\nonumber \\
    &\quad + \frac{ 12(\eta_x^{(c)})^2\bar{K} }{1-4(\eta_x^{(c)})^2\bar{K}\left(
        6(L_f^x)^2+24\iota^2(L_{gxy}^y\varepsilon_d+L_{gxy}^x)^2
        \right)}\Big(
    (L_{Fc}^y)^2\|\bm{y}_r-\bm{y}^*(\bm{x}_r)\|^2+(\tilde{L}_{Fc}^x)^2\|\bm{x}_r-\bm{x}^*\|^2 +\mathcal{V}_F
    \Big).
\end{align}
We define the following constants to simplify the expressions:
\begin{align}
    &p_{x,1} = \frac{8(\eta_x^{(c)})^2\bar{K}(\sigma_f^2+\sigma_r^2)
    +16(\eta_x^{(c)})^2\bar{K}\iota^2\sigma_{gg}^2}{1-4(\eta_x^{(c)})^2\bar{K}\left(
        6(L_f^x)^2+24\iota^2(L_{gxy}^y\varepsilon_d+L_{gxy}^x)^2
        \right)},\\
    &p_{x,2} = \frac{4(\eta_x^{(c)})^2\bar{K}\left(
        6(L_f^\xi\varepsilon_c+\bar{L}_f^x)^2+24\iota^2(L_{gxy}^y)^2
        \right)}{1-4(\eta_x^{(c)})^2\bar{K}\left(
        6(L_f^x)^2+24\iota^2(L_{gxy}^y\varepsilon_d+L_{gxy}^x)^2
        \right)},\\
    &p_{x,3} = \frac{12(\eta_x^{(c)})^2\bar{K} (C_g^{xy})^2   }{1-4(\eta_x^{(c)})^2\bar{K}\left(
        6(L_f^x)^2+24\iota^2(L_{gxy}^y\varepsilon_d+L_{gxy}^x)^2
        \right)},\\
    &p_{x,4}=\frac{ 12(\eta_x^{(c)})^2\bar{K} (L_{Fc}^y)^2}{1-4(\eta_x^{(c)})^2\bar{K}\left(
        6(L_f^x)^2+24\iota^2(L_{gxy}^y\varepsilon_d+L_{gxy}^x)^2
        \right)},\\
    &p_{x,5}=\frac{ 12(\eta_x^{(c)})^2\bar{K} (\tilde{L}_{Fc}^x)^2}{1-4(\eta_x^{(c)})^2\bar{K}\left(
        6(L_f^x)^2+24\iota^2(L_{gxy}^y\varepsilon_d+L_{gxy}^x)^2
        \right)},\\
    &p_{x,6}=\frac{ 12(\eta_x^{(c)})^2\bar{K} }{1-4(\eta_x^{(c)})^2\bar{K}\left(
        6(L_f^x)^2+24\iota^2(L_{gxy}^y\varepsilon_d+L_{gxy}^x)^2
        \right)}.
\end{align}
From the constants defined above, we need $1-4(\eta_x^{(c)})^2\bar{K}\left(
        6(L_f^x)^2+24\iota^2(L_{gxy}^y\varepsilon_d+L_{gxy}^x)^2
        \right) > 0$, and we have another condition as $\eta_x^{(c)} < \frac{1}{4\bar{K}\left(3L_f^x + 5\iota(L_{gxy}^y\varepsilon_d+L_{gxy}^x)\right)}$.
        
Thus, we further obtain
\begin{align}
    &\sum_{i\in C_r}\tilde{p}_i\sum_{k=0}^{K_r-1}\mathbb{E}\|\bm{x}_{i,r,k}-\bm{x}_r\|^2\nonumber \\
    &\leq
    p_{x,1}+p_{x,2}\sum_{i\in C_r}\tilde{p}_i\sum_{k=0}^{K_r-1}\mathbb{E}\|\bm{y}_{i,r,k}-\bm{y}_r\|^2 + p_{x,3}\sum_{i\in C_r}\tilde{p}_i\sum_{k=0}^{K_r-1}\mathbb{E}\|\bm{v}_{i,r,k}-\bm{v}_r\|^2
    +p_{x,4}\|\bm{y}_r-\bm{y}^*(\bm{x}_r)\|^2\nonumber \\
    &\quad + p_{x,5}\|\bm{x}_r-\bm{x}^*\|^2 + p_{x,6}\mathcal{V}_F.\label{eq.lemma10_re2}
\end{align}


According to the definitions of $\bm{v}_{i,r,k}$ and $\bm{v}_r$, it's easy to know that
\begin{align}
    &\sum_{i\in C_r}\tilde{p}_i\sum_{k=0}^{K_r-1}\mathbb{E}\|\bm{v}_{i,r,k}-\bm{v}_r\|^2\nonumber \\
    &=(\eta_v^{(c)})^2\sum_{i\in C_r}\tilde{p}_i\sum_{k=0}^{K_r-1}\mathbb{E}\left\|
        \sum_{j=0}^{k-1}\left(
        \widehat{\nabla}_vl_i(\bm{x}_{i,r,j},\bm{y}_{i,r,j},\bm{v}_{i,r,j};\psi)-\nabla_vl_i(\bm{x}_{i,r,j},\bm{y}_{i,r,j},\bm{v}_{i,r,j})+\nabla_vl_i(\bm{x}_{i,r,j},\bm{y}_{i,r,j},\bm{v}_{i,r,j})
        \right)
        \right\|^2\nonumber \\
    &\leq2(\eta_v^{(c)})^2\sum_{i\in C_r}\tilde{p}_i\sum_{k=0}^{K_r-1}\sum_{j=0}^{k-1}\mathbb{E}\left\|
        \left(
        \widehat{\nabla}_vl_i(\bm{x}_{i,r,j},\bm{y}_{i,r,j},\bm{v}_{i,r,j};\psi)-\nabla_vl_i(\bm{x}_{i,r,j},\bm{y}_{i,r,j},\bm{v}_{i,r,j})\right)\right\|^2\nonumber\\
    &\quad +2(\eta_v^{(c)})^2\sum_{i\in C_r}\tilde{p}_i\sum_{k=0}^{K_r-1}\sum_{j=0}^{k-1}\mathbb{E}\left\|\nabla_vl_i(\bm{x}_{i,r,j},\bm{y}_{i,r,j},\bm{v}_{i,r,j})-\nabla_vl_i(\bm{x}_r,\bm{y}_r,\bm{v}_r)+\nabla_vl_i(\bm{x}_r,\bm{y}_r,\bm{v}_r)
        \right\|^2\nonumber \\
    &\overset{(a)}{\leq}2(\eta_v^{(c)})^2\sum_{i\in C_r}\tilde{p}_i\sum_{k=0}^{K_r-1}\sum_{j=0}^{k-1}\mathbb{E}\left\|
        \left(
        \widehat{\nabla}_vl_i(\bm{x}_{i,r,j},\bm{y}_{i,r,j},\bm{v}_{i,r,j};\psi)-\nabla_vl_i(\bm{x}_{i,r,j},\bm{y}_{i,r,j},\bm{v}_{i,r,j})\right)\right\|^2\nonumber\\
    &\quad +4(\eta_v^{(c)})^2\sum_{i\in C_r}\tilde{p}_i\sum_{k=0}^{K_r-1}\sum_{j=0}^{k-1}\Delta_{l,i,r,j} + 12(\eta_v^{(c)})^2\bar{K}\left(
    (L_{l_c}^y)^2\|\bm{y}_r-\bm{y}^*(\bm{x}_r)\|^2+(\tilde{L}_{l_c}^x)^2\|\bm{x}_r-\bm{x}^*\|^2+\mathcal{V}_l
    \right)\nonumber \\
    &\overset{(b)}{\leq}4(\eta_v^{(c)})^2\sum_{i\in C_r}\tilde{p}_i\sum_{k=0}^{K_r-1}\sum_{j=0}^{k-1}\|\bm{v}_{i,r,j}\|^2\mathbb{E}\Big\|
        \widehat{\nabla}_{yy}^2g_i(\bm{x}_{i,r,j},\bm{y}_{i,r,j};\zeta)
        -\nabla_{yy}^2g_i(\bm{x}_{i,r,j},\bm{y}_{i,r,j};\zeta)\Big\|^2\nonumber \\
        &\quad +4(\eta_v^{(c)})^2\sum_{i\in C_r}\tilde{p}_i\sum_{k=0}^{K_r-1}\sum_{j=0}^{k-1}\mathbb{E}\Big\|
        \widehat{\nabla}_xf_i(\bm{x}_{i,r,j},\bm{y}_{i,r,j};\xi)-\nabla_xf_i(\bm{x}_{i,r,j},\bm{y}_{i,r,j};\xi)
        \Big\|^2\nonumber\\
        &\quad +4(\eta_v^{(c)})^2\bar{K}\sum_{i\in C_r}\tilde{p}_i\sum_{k=0}^{K_r-1}
        \Bigg(\left(
        6(L_f^y)^2+24\iota^2(L_{gyy}^y\varepsilon_d+L_{gyy}^x)^2
        \right)\|\bm{x}_{i,r,k}-\bm{x}_r\|^2
        +\left(
        6(L_f^\xi\varepsilon_c+\bar{L}_f^y)^2+24\iota^2(L_{gyy}^y)^2
        \right)\|\bm{y}_{i,r,k}-\bm{y}_r\|^2\nonumber \\
        &\quad
        +3(C_g^{yy})^2\|\bm{v}_{i,r,k}-\bm{v}_r\|^2\Bigg)
        + 12(\eta_v^{(c)})^2\bar{K}\left(
    (L_{l_c}^y)^2\|\bm{y}_r-\bm{y}^*(\bm{x}_r)\|^2+(\tilde{L}_{l_c}^x)^2\|\bm{x}_r-\bm{x}^*\|^2+\mathcal{V}_l
    \right)\nonumber \\
    &\overset{(c)}{\leq}8(\eta_v^{(c)})^2\bar{K}(\sigma_f^2+\sigma_r^2)+16(\eta_v^{(c)})^2\bar{K}\iota^2\sigma_{gg}^2
    +4(\eta_v^{(c)})^2\bar{K}\sum_{i\in C_r}\tilde{p}_i\sum_{k=0}^{K_r-1}
        \Bigg(\left(
        6(L_f^y)^2+24\iota^2(L_{gyy}^y\varepsilon_d+L_{gyy}^x)^2
        \right)\|\bm{x}_{i,r,k}-\bm{x}_r\|^2\nonumber \\
    &\quad + \left(
        6(L_f^\xi\varepsilon_c+\bar{L}_f^y)^2+24\iota^2(L_{gyy}^y)^2
        \right)\|\bm{y}_{i,r,k}-\bm{y}_r\|^2
        +3(C_g^{yy})^2\|\bm{v}_{i,r,k}-\bm{v}_r\|^2\Bigg)\nonumber \\
    &\quad + 12(\eta_v^{(c)})^2\bar{K}\left(
    (L_{l_c}^y)^2\|\bm{y}_r-\bm{y}^*(\bm{x}_r)\|^2+(\tilde{L}_{l_c}^x)^2\|\bm{x}_r-\bm{x}^*\|^2+\mathcal{V}_l
    \right),
\end{align}
where $(a)$ is based on Lemma \ref{lemma_8}, $(a)$ is based on Lemma \ref{lemma_9}, and $(c)$ holds because of Assumption \ref{assump_5}.
Therefore, we further obtain
\begin{align}
    &\sum_{i\in C_r}\tilde{p}_i\sum_{k=0}^{K_r-1}\mathbb{E}\|\bm{v}_{i,r,k}-\bm{v}_r\|^2\nonumber \\
    &\leq \frac{8(\eta_v^{(c)})^2\bar{K}(\sigma_f^2+\sigma_r^2)+16(\eta_v^{(c)})^2\bar{K}\iota^2\sigma_{gg}^2}{1-12(\eta_v^{(c)})^2\bar{K}(C_g^{yy})^2}\nonumber \\
    &\quad +\frac{4(\eta_v^{(c)})^2\bar{K}\left(
        6(L_f^y)^2+24\iota^2(L_{gyy}^y\varepsilon_d+L_{gyy}^x)^2
        \right) }{1-12(\eta_v^{(c)})^2\bar{K}(C_g^{yy})^2}\sum_{i\in C_r}\tilde{p}_i\sum_{k=0}^{K_r-1}\mathbb{E}\|\bm{x}_{i,r,k}-\bm{x}_r\|^2\nonumber 
\end{align}
\begin{align}    
    &\quad +\frac{4(\eta_v^{(c)})^2\bar{K}  \left(
        6(L_f^\xi\varepsilon_c+\bar{L}_f^y)^2+24\iota^2(L_{gyy}^y)^2
        \right) }{1-12(\eta_v^{(c)})^2\bar{K}(C_g^{yy})^2}\sum_{i\in C_r}\tilde{p}_i\sum_{k=0}^{K_r-1}\mathbb{E}\|\bm{y}_{i,r,k}-\bm{y}_r\|^2\nonumber\\
    &\quad +\frac{12(\eta_v^{(c)})^2\bar{K}}{1-12(\eta_v^{(c)})^2\bar{K}(C_g^{yy})^2}\left(
    (L_{l_c}^y)^2\|\bm{y}_r-\bm{y}^*(\bm{x}_r)\|^2+(\tilde{L}_{l_c}^x)^2\|\bm{x}_r-\bm{x}^*\|^2+\mathcal{V}_l
    \right).\label{eq.lemma10_4}
\end{align}
For the sake of simplicity, we define
\begin{align}
    &p_{v,1} = \frac{8(\eta_v^{(c)})^2\bar{K}(\sigma_f^2+\sigma_r^2)+16(\eta_v^{(c)})^2\bar{K}\iota^2\sigma_{gg}^2}{1-12(\eta_v^{(c)})^2\bar{K}(C_g^{yy})^2},\\
    &p_{v,2}=\frac{4(\eta_v^{(c)})^2\bar{K}\left(
        6(L_f^y)^2+24\iota^2(L_{gyy}^y\varepsilon_d+L_{gyy}^x)^2
        \right) }{1-12(\eta_v^{(c)})^2\bar{K}(C_g^{yy})^2},\\
    &p_{v,3}=\frac{4(\eta_v^{(c)})^2\bar{K}  \left(
        6(L_f^\xi\varepsilon_c+\bar{L}_f^y)^2+24\iota^2(L_{gyy}^y)^2
        \right) }{1-12(\eta_v^{(c)})^2\bar{K}(C_g^{yy})^2},\\
    &p_{v,4}=\frac{12(\eta_v^{(c)})^2\bar{K} (L_{l_c}^y)^2  }{1-12(\eta_v^{(c)})^2\bar{K}(C_g^{yy})^2},
    p_{v,5}=\frac{12(\eta_v^{(c)})^2\bar{K}  (\tilde{L}_{l_c}^x)^2 }{1-12(\eta_v^{(c)})^2\bar{K}(C_g^{yy})^2},
    p_{v,6}=\frac{12(\eta_v^{(c)})^2\bar{K} }{1-12(\eta_v^{(c)})^2\bar{K}(C_g^{yy})^2}.
\end{align}
Obviously, we need $1-12(\eta_v^{(c)})^2\bar{K}(C_g^{yy})^2 > 0$. Thus, the condition is obtained as $\eta_v^{(c)} < \frac{1}{4\bar{K}C_g^{yy}}$.
Therefore, we rewrite Eq. (\ref{eq.lemma10_4}) as
\begin{align}
    &\sum_{i\in C_r}\tilde{p}_i\sum_{k=0}^{K_r-1}\mathbb{E}\|\bm{v}_{i,r,k}-\bm{v}_r\|^2\nonumber \\
    &\leq
    p_{v,1}+p_{v,2}\sum_{i\in C_r}\tilde{p}_i\sum_{k=0}^{K_r-1}\mathbb{E}\|\bm{x}_{i,r,k}-\bm{x}_r\|^2 + p_{v,3}\sum_{i\in C_r}\tilde{p}_i\sum_{k=0}^{K_r-1}\mathbb{E}\|\bm{y}_{i,r,k}-\bm{y}_r\|^2
    +p_{v,4}\|\bm{y}_r-\bm{y}^*(\bm{x}_r)\|^2\nonumber \\
    &\quad + p_{v,5}\|\bm{x}_r-\bm{x}^*\|^2 + p_{v,6}\mathcal{V}_l.\label{eq.lemma10_re3}
\end{align}
Combining Eqs. (\ref{eq.lemma10_re1}), (\ref{eq.lemma10_re2}), and (\ref{eq.lemma10_re3}), we have
\begin{align}
    &\sum_{i\in C_r}\tilde{p}_i\sum_{k=0}^{K_r-1}\mathbb{E}\|\bm{x}_{i,r,k}-\bm{x}_r\|^2\nonumber \\
    &\leq \frac{p_{x,1}+(p_{x,2}+p_{x,3}p_{v,3})p_{y,1}+p_{x,3}p_{v,1}}{1-p_{x,2}p_{y,2}-p_{x,3}p_{v,2}-p_{x,3}p_{v,3}p_{y,2}}
    + \frac{p_{x,4}+(p_{x,2}+p_{x,3}p_{v,3}) p_{y,3}  +p_{x,3} p_{v,4}  }{1-p_{x,2}p_{y,2}-p_{x,3}p_{v,2}-p_{x,3}p_{v,3}p_{y,2}}\|\bm{y}_r-\bm{y}^*(\bm{x}_r)\|^2 
    \nonumber \\
    &\quad+ \frac{p_{x,5}+ p_{x,3} p_{v,5}   }{1-p_{x,2}p_{y,2}-p_{x,3}p_{v,2}-p_{x,3}p_{v,3}p_{y,2}}\|\bm{x}_r-\bm{x}^*\|^2
    + \frac{ p_{x,6}  }{1-p_{x,2}p_{y,2}-p_{x,3}p_{v,2}-p_{x,3}p_{v,3}p_{y,2}} \mathcal{V}_F
    \nonumber \\
    &\quad+ \frac{p_{x,3}p_{v,6}   }{1-p_{x,2}p_{y,2}-p_{x,3}p_{v,2}-p_{x,3}p_{v,3}p_{y,2}}\mathcal{V}_l,\\
    &\sum_{i\in C_r}\tilde{p}_i\sum_{k=0}^{K_r-1}\mathbb{E}\|\bm{y}_{i,r,k}-\bm{y}_r\|^2\nonumber \\
    &\leq \frac{
    p_{y,2}p_{x,1}+(1-p_{v,2}p_{x,3})p_{y,1}+p_{y,2}p_{x,3}p_{v,1}
    }{1-p_{x,2}p_{y,2}-p_{x,3}p_{v,2}-p_{x,3}p_{v,3}p_{y,2}}
    +\frac{
    p_{y,2}p_{x,4}+(1-p_{v,2}p_{x,3})p_{y,3}+p_{y,2}p_{x,3}p_{v,4}
    }{1-p_{x,2}p_{y,2}-p_{x,3}p_{v,2}-p_{x,3}p_{v,3}p_{y,2}}\|\bm{y}_r-\bm{y}^*(\bm{x}_r)\|^2 
    \nonumber \\
    &\quad +\frac{
    p_{y,2}p_{x,5}+p_{y,2}p_{x,3}p_{v,5}
    }{1-p_{x,2}p_{y,2}-p_{x,3}p_{v,2}-p_{x,3}p_{v,3}p_{y,2}}\|\bm{x}_r-\bm{x}^*\|^2
    +\frac{
    p_{y,2}p_{x,6}
    }{1-p_{x,2}p_{y,2}-p_{x,3}p_{v,2}-p_{x,3}p_{v,3}p_{y,2}}\mathcal{V}_F
    \nonumber \\
    &\quad +\frac{
    p_{y,2}p_{x,3}p_{v,6}
    }{1-p_{x,2}p_{y,2}-p_{x,3}p_{v,2}-p_{x,3}p_{v,3}p_{y,2}}\mathcal{V}_l,
\end{align}


\begin{align}
    &\sum_{i\in C_r}\tilde{p}_i\sum_{k=0}^{K_r-1}\mathbb{E}\|\bm{v}_{i,r,k}-\bm{v}_r\|^2\nonumber \\
    &\leq \frac{
    (p_{v,2}+p_{v,3}p_{y,2})p_{x,1}+(p_{v,2}p_{x,2}+p_{v,3})p_{y,1}+(1-p_{x,2}p_{y,2})p_{v,1}
    }{1-p_{x,2}p_{y,2}-p_{x,3}p_{v,2}-p_{x,3}p_{v,3}p_{y,2}}
    \nonumber \\
    &\quad+\frac{
    (p_{v,2}+p_{v,3}p_{y,2})p_{x,4}+(p_{v,2}p_{x,2}+p_{v,3})p_{y,3}+(1-p_{x,2}p_{y,2})p_{v,4}
    }{1-p_{x,2}p_{y,2}-p_{x,3}p_{v,2}-p_{x,3}p_{v,3}p_{y,2}}\|\bm{y}_r-\bm{y}^*(\bm{x}_r)\|^2 
    \nonumber \\
    &\quad+\frac{
    (p_{v,2}+p_{v,3}p_{y,2})p_{x,5}+(1-p_{x,2}p_{y,2})p_{v,5}
    }{1-p_{x,2}p_{y,2}-p_{x,3}p_{v,2}-p_{x,3}p_{v,3}p_{y,2}}\|\bm{x}_r-\bm{x}^*\|^2
    +\frac{
    (p_{v,2}+p_{v,3}p_{y,2})p_{x,6}
    }{1-p_{x,2}p_{y,2}-p_{x,3}p_{v,2}-p_{x,3}p_{v,3}p_{y,2}}\mathcal{V}_F
    \nonumber \\
    &\quad+\frac{
    (1-p_{x,2}p_{y,2})p_{v,6}
    }{1-p_{x,2}p_{y,2}-p_{x,3}p_{v,2}-p_{x,3}p_{v,3}p_{y,2}}\mathcal{V}_l.
\end{align}
We need $1-p_{x,2}p_{y,2}-p_{x,3}p_{v,2}-p_{x,3}p_{v,3}p_{y,2} > 0$. Therefore, we should set 
\begin{align}
    \eta_x^{(c)} \leq \min \left\{
    \frac{1}{2\bar{K} \left(5(L_f^\xi\varepsilon_c+\bar{L}_f^x)+9\iota L_{gxy}^y+3L_f^x + 5\iota (L_{gxy}^y \varepsilon_d + L_{gxy}^x)\right)},
    \frac{1}{2\bar{K} \left(
    3C_g^{xy} + 3L_f^x + 5\iota(L_{gxy}^y \varepsilon_d + L_{gxy}^x)
    \right)}
    \right\}.
\end{align}
\begin{align}
    \eta_y^{(c)} \leq \frac{1}{3\bar{K}\left(
    2(L_g^y \varepsilon_d + L_g^x)+ L_g^y
    \right)}.
\end{align}
\begin{align}
    \eta_v^{(c)} \leq \min \left\{
    \frac{1}{2\bar{K} \left(
    5L_f^y + 9\iota(L_{gyy}^y\varepsilon_d+L_{gyy}^x)+2C_g^{yy}
    \right)},
    \frac{1}{2\bar{K} \left(
    5(L_f^\xi\varepsilon_c + \bar{L}_f^y)+9\iota L_{gyy}^y+2C_g^{yy}
    \right)}
    \right\}.
\end{align}

To simplify the notation, let
\begin{align}
    &p_{x,7} = \frac{p_{x,1}+(p_{x,2}+p_{x,3}p_{v,3})p_{y,1}+p_{x,3}p_{v,1}}{1-p_{x,2}p_{y,2}-p_{x,3}p_{v,2}-p_{x,3}p_{v,3}p_{y,2}},
    p_{x,8} = \frac{p_{x,4}+(p_{x,2}+p_{x,3}p_{v,3}) p_{y,3}  +p_{x,3} p_{v,4}  }{1-p_{x,2}p_{y,2}-p_{x,3}p_{v,2}-p_{x,3}p_{v,3}p_{y,2}},\\
    &p_{x,9} = \frac{p_{x,5}+ p_{x,3} p_{v,5}   }{1-p_{x,2}p_{y,2}-p_{x,3}p_{v,2}-p_{x,3}p_{v,3}p_{y,2}},
    p_{x,10} = \frac{ p_{x,6}  }{1-p_{x,2}p_{y,2}-p_{x,3}p_{v,2}-p_{x,3}p_{v,3}p_{y,2}},\\
    &p_{x,11} = \frac{p_{x,3}p_{v,6}   }{1-p_{x,2}p_{y,2}-p_{x,3}p_{v,2}-p_{x,3}p_{v,3}p_{y,2}},\\
    &p_{y,4} = \frac{
    p_{y,2}p_{x,1}+(1-p_{v,2}p_{x,3})p_{y,1}+p_{y,2}p_{x,3}p_{v,1}
    }{1-p_{x,2}p_{y,2}-p_{x,3}p_{v,2}-p_{x,3}p_{v,3}p_{y,2}},
    p_{y,5}= \frac{
    p_{y,2}p_{x,4}+(1-p_{v,2}p_{x,3})p_{y,3}+p_{y,2}p_{x,3}p_{v,4}
    }{1-p_{x,2}p_{y,2}-p_{x,3}p_{v,2}-p_{x,3}p_{v,3}p_{y,2}},\\
    &p_{y,6} = \frac{
    p_{y,2}p_{x,5}+p_{y,2}p_{x,3}p_{v,5}
    }{1-p_{x,2}p_{y,2}-p_{x,3}p_{v,2}-p_{x,3}p_{v,3}p_{y,2}},
    p_{y,7} = \frac{
    p_{y,2}p_{x,6}
    }{1-p_{x,2}p_{y,2}-p_{x,3}p_{v,2}-p_{x,3}p_{v,3}p_{y,2}},\\
    &p_{y,8} = \frac{
    p_{y,2}p_{x,3}p_{v,6}
    }{1-p_{x,2}p_{y,2}-p_{x,3}p_{v,2}-p_{x,3}p_{v,3}p_{y,2}},\\
    &p_{v,7}=\frac{
    (p_{v,2}+p_{v,3}p_{y,2})p_{x,1}+(p_{v,2}p_{x,2}+p_{v,3})p_{y,1}+(1-p_{x,2}p_{y,2})p_{v,1}
    }{1-p_{x,2}p_{y,2}-p_{x,3}p_{v,2}-p_{x,3}p_{v,3}p_{y,2}},\\
    &p_{v,8} = \frac{
    (p_{v,2}+p_{v,3}p_{y,2})p_{x,4}+(p_{v,2}p_{x,2}+p_{v,3})p_{y,3}+(1-p_{x,2}p_{y,2})p_{v,4}
    }{1-p_{x,2}p_{y,2}-p_{x,3}p_{v,2}-p_{x,3}p_{v,3}p_{y,2}},\\
    &p_{v,9} = \frac{
    (p_{v,2}+p_{v,3}p_{y,2})p_{x,5}+(1-p_{x,2}p_{y,2})p_{v,5}
    }{1-p_{x,2}p_{y,2}-p_{x,3}p_{v,2}-p_{x,3}p_{v,3}p_{y,2}},
    p_{v,10} = \frac{
    (p_{v,2}+p_{v,3}p_{y,2})p_{x,6}
    }{1-p_{x,2}p_{y,2}-p_{x,3}p_{v,2}-p_{x,3}p_{v,3}p_{y,2}},\\
    &p_{v,11} = \frac{
    (1-p_{x,2}p_{y,2})p_{v,6}
    }{1-p_{x,2}p_{y,2}-p_{x,3}p_{v,2}-p_{x,3}p_{v,3}p_{y,2}}.
\end{align}


The following inequalities are further obtained:
\begin{align}
    &\sum_{i\in C_r}\tilde{p}_i\sum_{k=0}^{K_r-1}\mathbb{E}\|\bm{x}_{i,r,k}-\bm{x}_r\|^2 \leq p_{x,7}+p_{x,8}\|\bm{y}_r-\bm{y}^*(\bm{x}_r)\|^2 +p_{x,9}\|\bm{x}_r-\bm{x}^*\|^2+p_{x,10}\mathcal{V}_F+p_{x,11}\mathcal{V}_l,\\
    &\sum_{i\in C_r}\tilde{p}_i\sum_{k=0}^{K_r-1}\mathbb{E}\|\bm{y}_{i,r,k}-\bm{y}_r\|^2 \leq p_{y,4}+p_{y,5}\|\bm{y}_r-\bm{y}^*(\bm{x}_r)\|^2 +p_{y,6}\|\bm{x}_r-\bm{x}^*\|^2+p_{y,7}\mathcal{V}_F+p_{y,8}\mathcal{V}_l,\\
    &\sum_{i\in C_r}\tilde{p}_i\sum_{k=0}^{K_r-1}\mathbb{E}\|\bm{v}_{i,r,k}-\bm{v}_r\|^2 \leq p_{v,7}+p_{v,8}\|\bm{y}_r-\bm{y}^*(\bm{x}_r)\|^2 +p_{v,9}\|\bm{x}_r-\bm{x}^*\|^2+p_{v,10}\mathcal{V}_F+p_{v,11}\mathcal{V}_l.
\end{align}

\end{proof}

Combining Lemma \ref{lemma_9} and Lemma \ref{lemma_10}, we can have
\begin{align}
     &\sum_{i\in C_r}\tilde{p}_i\sum_{k=0}^{K_r-1}\|\nabla \mathcal{F}_i(\bm{x}_{r},\bm{y}_{r},\bm{v}_{r})-\nabla \mathcal{F}_i(\bm{x}_{i,r,k},\bm{y}_{i,r,k},\bm{v}_{i,r,k})\|^2\nonumber \\
     &\leq 
     \left(
        6(L_f^x)^2+24\iota^2(L_{gxy}^y\varepsilon_d+L_{gxy}^x)^2
        \right)p_{x,7}
    + \left(
        6(L_f^\xi\varepsilon_c+\bar{L}_f^x)^2+24\iota^2(L_{gxy}^y)^2
        \right)p_{y,4}
    + 3(C_g^{xy})^2 p_{v,7}
     \nonumber \\
     &\quad +
     \left(
     \left(
        6(L_f^x)^2+24\iota^2(L_{gxy}^y\varepsilon_d+L_{gxy}^x)^2
        \right)p_{x,8}
    + \left(
        6(L_f^\xi\varepsilon_c+\bar{L}_f^x)^2+24\iota^2(L_{gxy}^y)^2
        \right)p_{y,5}
    + 3(C_g^{xy})^2 p_{v,8}
     \right)\|\bm{y}_r-\bm{y}^*(\bm{x}_r)\|^2
     \nonumber \\
     &\quad +
     \left(
     \left(
        6(L_f^x)^2+24\iota^2(L_{gxy}^y\varepsilon_d+L_{gxy}^x)^2
        \right)p_{x,9}
    + \left(
        6(L_f^\xi\varepsilon_c+\bar{L}_f^x)^2+24\iota^2(L_{gxy}^y)^2
        \right)p_{y,6}
    + 3(C_g^{xy})^2 p_{v,9}
     \right)\|\bm{x}_r-\bm{x}^*\|^2
     \nonumber \\
     &\quad +
     \left(
     \left(
        6(L_f^x)^2+24\iota^2(L_{gxy}^y\varepsilon_d+L_{gxy}^x)^2
        \right)p_{x,10}
    + \left(
        6(L_f^\xi\varepsilon_c+\bar{L}_f^x)^2+24\iota^2(L_{gxy}^y)^2
        \right)p_{y,7}
    + 3(C_g^{xy})^2 p_{v,10}
     \right)\mathcal{V}_F
     \nonumber \\
     &\quad +
     \left(
     \left(
        6(L_f^x)^2+24\iota^2(L_{gxy}^y\varepsilon_d+L_{gxy}^x)^2
        \right)p_{x,11}
    + \left(
        6(L_f^\xi\varepsilon_c+\bar{L}_f^x)^2+24\iota^2(L_{gxy}^y)^2
        \right)p_{y,8}
    + 3(C_g^{xy})^2 p_{v,11}
     \right)\mathcal{V}_l
\end{align}
\begin{align}
     &\sum_{i\in C_r}\tilde{p}_i\sum_{k=0}^{K_r-1}\|\nabla_y g_i(\bm{x}_r,\bm{y}_r)-\nabla_y g_i(\bm{x}_{i,r,k},\bm{y}_{i,r,k})\|^2\nonumber \\
     &\leq 2(L^y_g\varepsilon_d+L_g^x)^2 p_{x,7}+ 2(L_g^y)^2 p_{y,4}
     +\left(
    2(L^y_g\varepsilon_d+L_g^x)^2 p_{x,8}+ 2(L_g^y)^2 p_{y,5}
     \right)\|\bm{y}_r-\bm{y}^*(\bm{x}_r)\|^2
     \nonumber \\
     &\quad +\left(
    2(L^y_g\varepsilon_d+L_g^x)^2 p_{x,9}+ 2(L_g^y)^2 p_{y,6}
     \right)\|\bm{x}_r-\bm{x}^*\|^2
     +\left(
    2(L^y_g\varepsilon_d+L_g^x)^2 p_{x,10}+ 2(L_g^y)^2 p_{y,7}
     \right)\mathcal{V}_F
     \nonumber \\
     &\quad +\left(
    2(L^y_g\varepsilon_d+L_g^x)^2 p_{x,11}+ 2(L_g^y)^2 p_{y,8}
     \right)\mathcal{V}_l
\end{align}
\begin{align}
     &\sum_{i\in C_r}\tilde{p}_i\sum_{k=0}^{K_r-1}\|\nabla l_i(\bm{x}_{r},\bm{y}_{r},\bm{v}_{r})-\nabla l_i(\bm{x}_{i,r,k},\bm{y}_{i,r,k},\bm{v}_{i,r,k})\|^2\nonumber \\
     &\leq \left(
        6(L_f^y)^2+24\iota^2(L_{gyy}^y\varepsilon_d+L_{gyy}^x)^2
        \right)p_{x,7}
        + \left(
        6(L_f^\xi\varepsilon_c+\bar{L}_f^y)^2+24\iota^2(L_{gyy}^y)^2
        \right)p_{y,4}
        + 3(C_g^{yy})^2 p_{v,7}
        \nonumber \\
    &\quad +\left(
    \left(
        6(L_f^y)^2+24\iota^2(L_{gyy}^y\varepsilon_d+L_{gyy}^x)^2
        \right)p_{x,8}
        + \left(
        6(L_f^\xi\varepsilon_c+\bar{L}_f^y)^2+24\iota^2(L_{gyy}^y)^2
        \right)p_{y,5}
        + 3(C_g^{yy})^2 p_{v,8}
    \right)\|\bm{y}_r-\bm{y}^*(\bm{x}_r)\|^2
    \nonumber \\
    &\quad +\left(
    \left(
        6(L_f^y)^2+24\iota^2(L_{gyy}^y\varepsilon_d+L_{gyy}^x)^2
        \right)p_{x,9}
        + \left(
        6(L_f^\xi\varepsilon_c+\bar{L}_f^y)^2+24\iota^2(L_{gyy}^y)^2
        \right)p_{y,6}
        + 3(C_g^{yy})^2 p_{v,9}
    \right)\|\bm{x}_r-\bm{x}^*\|^2
    \nonumber \\
    &\quad +\left(
    \left(
        6(L_f^y)^2+24\iota^2(L_{gyy}^y\varepsilon_d+L_{gyy}^x)^2
        \right)p_{x,10}
        + \left(
        6(L_f^\xi\varepsilon_c+\bar{L}_f^y)^2+24\iota^2(L_{gyy}^y)^2
        \right)p_{y,7}
        + 3(C_g^{yy})^2 p_{v,10}
    \right)\mathcal{V}_F
    \nonumber \\
    &\quad +\left(
    \left(
        6(L_f^y)^2+24\iota^2(L_{gyy}^y\varepsilon_d+L_{gyy}^x)^2
        \right)p_{x,11}
        + \left(
        6(L_f^\xi\varepsilon_c+\bar{L}_f^y)^2+24\iota^2(L_{gyy}^y)^2
        \right)p_{y,8}
        + 3(C_g^{yy})^2 p_{v,11}
    \right)\mathcal{V}_l.
\end{align}

\subsection{Bounds of Aggregated Estimations}\label{app:aggregated_estimation}
\begin{lemma}\label{lemma_11}
    Let $|C_r|=C$, which is a constant number. Under Assumptions \ref{assump_1} to \ref{assump_5}, the aggregated estimations of $\bm{x}_r$, $\bm{y}_r$, and $\bm{v}_r$ satisfy
    \begin{align}
    \mathbb{E}\|\nabla_{x,r}^{(s)}\|^2
    &\leq\frac{8M}{C}\sum_{i=1}^M p_i^2 \sum_{k=0}^{K_r-1}
    \left(
    \sigma_f^2+\sigma_r^2+2\iota^2\sigma_{gg}^2
    \right)
    \nonumber \\
    &\quad +2p_{\text{max}}\left(
    \frac{M}{C}\left(\frac{C-1}{M-1}\right)+\frac{M}{C}\left(\frac{M-C}{M-1}\right)
    \right)
    \sum_{i=1}^{M}p_i \sum_{k=0}^{K_r-1}\Delta_{f,i,r,k}
    \nonumber \\
    &\quad +6p_{\text{max}}\left(
    \frac{M}{C}\left(\frac{C-1}{M-1}\right)+\frac{M}{C}\left(\frac{M-C}{M-1}\right)
    \right)
    \left(
    (L_{Fc}^y)^2\|\bm{y}_r-\bm{y}^*(\bm{x}_r)\|^2+(\tilde{L}_{Fc}^x)^2\|\bm{x}_r-\bm{x}^*\|^2+\mathcal{V}_F
    \right),
\end{align}
\begin{align}
    &\mathbb{E}\|\nabla_{y,r}^{(s)}\|^2\nonumber \\
    &\leq \frac{2M}{C}\sum_{i=1}^M p_i^2\sum_{k=0}^{K_r-1}\sigma_g^2
    +p_{\text{max}}\left(
    \frac{4M}{C}\left(\frac{M-C}{M-1}\right)
    +\frac{6M}{C}\left(\frac{C-1}{M-1}\right)
    \right)
    \sum_{i=1}^{M}p_i \sum_{k=0}^{K_r-1}
    \Delta_{g,i,r,k}
    \nonumber \\
    &\quad +\left(
    \frac{4M}{C}\left(\frac{M-C}{M-1}\right)p_{\text{max}} 
    (L_g^y)^2
    + \frac{6M}{C}\left(\frac{C-1}{M-1}\right)(L_g^y)^2
    \right)
    \mathbb{E}\|\bm{y}_r-\bm{y}^*(\bm{x}_r)\|^2,
\end{align}
and 
\begin{align}
    &\mathbb{E}\|\nabla_{v,r}^{(s)}\|^2\nonumber \\
    &\leq \frac{8M}{C}\sum_{i=1}^M p_i^2 \sum_{k=0}^{K_r-1}
    \left(
    \sigma_f^2+\sigma_r^2+2\iota^2\sigma_{gg}^2
    \right)
    +\frac{2M}{C}\left(\frac{C-1}{M-1}\right)\mathbb{E}\left\|\sum_{i=1}^{M}p_i\sum_{k=0}^{K_r-1}\nabla_v l_i(\bm{x}_{i,r,k},\bm{y}_{i,r,k},\bm{v}_{i,r,k};\psi)\right\|^2\nonumber \\
    &\quad +\frac{4M}{C}\left(\frac{M-C}{M-1}\right)\sum_{i=1}^{M}p_i^2 \sum_{k=0}^{K_r-1}\mathbb{E}\left\|\nabla_v l_i(\bm{x}_{i,r,k},\bm{y}_{i,r,k},\bm{v}_{i,r,k};\psi)-\nabla_v l_i(\bm{x}_r,\bm{y}_r,\bm{v}_r;\psi)
    \right\|^2\nonumber \\
    &\quad +\frac{12M}{C}\left(\frac{M-C}{M-1}\right)p_{\text{max}} 
    \left(
    (L_{l_c}^y)^2\|\bm{y}_r-\bm{y}^*(\bm{x}_r)\|^2+(\tilde{L}_{l_c}^x)^2\|\bm{x}_r-\bm{x}^*\|^2+\mathcal{V}_l
    \right).
\end{align}
\end{lemma}
\begin{proof}
Based on the definition of $\nabla_{x,r}^{(s)}$, it yields that
\begin{align}
    &\mathbb{E}\|\nabla_{x,r}^{(s)}\|^2\nonumber \\
    &=\mathbb{E}\left\|\sum_{i\in C_r}\tilde{p}_i\sum_{k=0}^{K_r-1}\widehat{\nabla}_x \mathcal{F}_i(\bm{x}_{i,r,k},\bm{y}_{i,r,k},\bm{v}_{i,r,k};\psi)\right\|^2\nonumber \\
    &=\mathbb{E}\left\|\sum_{i\in C_r}\tilde{p}_i\sum_{k=0}^{K_r-1}\widehat{\nabla}_x \mathcal{F}_i(\bm{x}_{i,r,k},\bm{y}_{i,r,k},\bm{v}_{i,r,k};\psi)
    -\nabla_x \mathcal{F}_i(\bm{x}_{i,r,k},\bm{y}_{i,r,k},\bm{v}_{i,r,k};\psi)
    +\nabla_x \mathcal{F}_i(\bm{x}_{i,r,k},\bm{y}_{i,r,k},\bm{v}_{i,r,k};\psi)
    \right\|^2\nonumber \\
    &\leq\frac{2M}{C}\sum_{i=1}^M p_i^2 \sum_{k=0}^{K_r-1}\mathbb{E}\left\|\widehat{\nabla}_x \mathcal{F}_i(\bm{x}_{i,r,k},\bm{y}_{i,r,k},\bm{v}_{i,r,k};\psi)
    -\nabla_x \mathcal{F}_i(\bm{x}_{i,r,k},\bm{y}_{i,r,k},\bm{v}_{i,r,k};\psi)\right\|^2\nonumber \\
    &\quad + 2\mathbb{E}\left\|\sum_{i\in C_r}\tilde{p}_i\sum_{k=0}^{K_r-1}\nabla_x \mathcal{F}_i(\bm{x}_{i,r,k},\bm{y}_{i,r,k},\bm{v}_{i,r,k};\psi)\right\|^2\nonumber \\
    &\leq \frac{8M}{C}\sum_{i=1}^M p_i^2 \sum_{k=0}^{K_r-1}
    \left(
    \sigma_f^2+\sigma_r^2+2\iota^2\sigma_{gg}^2
    \right)
    +2\mathbb{E}\left\|\sum_{i\in C_r}\tilde{p}_i\sum_{k=0}^{K_r-1}\nabla_x \mathcal{F}_i(\bm{x}_{i,r,k},\bm{y}_{i,r,k},\bm{v}_{i,r,k};\psi)\right\|^2.\label{eq.lemma11_1}
\end{align}
For the second term in Eq. (\ref{eq.lemma11_1}), we know 
\begin{align}
    &\mathbb{E}\left\|\sum_{i\in C_r}\tilde{p}_i\sum_{k=0}^{K_r-1}\nabla_x \mathcal{F}_i(\bm{x}_{i,r,k},\bm{y}_{i,r,k},\bm{v}_{i,r,k};\psi)\right\|^2\nonumber \\
    &\overset{(a)}{=}
    \left(
    \frac{M}{C}\left(\frac{C-1}{M-1}\right)+\frac{M}{C}\left(\frac{M-C}{M-1}\right)
    \right)
    \sum_{i=1}^{M}p_i^2 \mathbb{E}\left\|\sum_{k=0}^{K_r-1}\nabla_x \mathcal{F}_i(\bm{x}_{i,r,k},\bm{y}_{i,r,k},\bm{v}_{i,r,k};\psi)\right\|^2\nonumber \\
    &\overset{(b)}{\leq} 
    2p_{\text{max}}\left(
    \frac{M}{C}\left(\frac{C-1}{M-1}\right)+\frac{M}{C}\left(\frac{M-C}{M-1}\right)
    \right)
    \sum_{i=1}^{M}p_i \sum_{k=0}^{K_r-1}\Delta_{f,i,r,k}
    \nonumber \\
    &\quad +6p_{\text{max}}\left(
    \frac{M}{C}\left(\frac{C-1}{M-1}\right)+\frac{M}{C}\left(\frac{M-C}{M-1}\right)
    \right)
    \left(
    (L_{Fc}^y)^2\|\bm{y}_r-\bm{y}^*(\bm{x}_r)\|^2+(\tilde{L}_{Fc}^x)^2\|\bm{x}_r-\bm{x}^*\|^2+\mathcal{V}_F
    \right),\label{eq.lemma11_2}
\end{align}
where $(a)$ holds based on Eq. (24) in the Appendix of \cite{yang2023simfbo}, and $(b)$ is true because of Lemma \ref{lemma_8}.
By combining Eq. (\ref{eq.lemma11_1}) and Eq. (\ref{eq.lemma11_2}), it yields that
\begin{align}
    \mathbb{E}\|\nabla_{x,r}^{(s)}\|^2
    &\leq\frac{8M}{C}\sum_{i=1}^M p_i^2 \sum_{k=0}^{K_r-1}
    \left(
    \sigma_f^2+\sigma_r^2+2\iota^2\sigma_{gg}^2
    \right)
    \nonumber \\
    &\quad +2p_{\text{max}}\left(
    \frac{M}{C}\left(\frac{C-1}{M-1}\right)+\frac{M}{C}\left(\frac{M-C}{M-1}\right)
    \right)
    \sum_{i=1}^{M}p_i \sum_{k=0}^{K_r-1}\Delta_{f,i,r,k}
    \nonumber \\
    &\quad +6p_{\text{max}}\left(
    \frac{M}{C}\left(\frac{C-1}{M-1}\right)+\frac{M}{C}\left(\frac{M-C}{M-1}\right)
    \right)
    \left(
    (L_{Fc}^y)^2\|\bm{y}_r-\bm{y}^*(\bm{x}_r)\|^2+(\tilde{L}_{Fc}^x)^2\|\bm{x}_r-\bm{x}^*\|^2+\mathcal{V}_F
    \right),
\end{align}
Similarly, it's easy to obtain
\begin{align}
    \mathbb{E}\|\nabla_{y,r}^{(s)}\|^2&=\mathbb{E}\left\|\sum_{i\in C_r}\tilde{p}_i\sum_{k=0}^{K_r-1}\widehat{\nabla}_y g_i(\bm{x}_{i,r,k},\bm{y}_{i,r,k};\zeta)\right\|^2\nonumber \\
    &=\mathbb{E}\left\|\sum_{i\in C_r}\tilde{p}_i\sum_{k=0}^{K_r-1}\widehat{\nabla}_y g_i(\bm{x}_{i,r,k},\bm{y}_{i,r,k};\zeta)-\nabla_yg_i(\bm{x}_{i,r,k},\bm{y}_{i,r,k};\zeta)+\nabla_yg_i(\bm{x}_{i,r,k},\bm{y}_{i,r,k};\zeta)\right\|^2\nonumber \\
    &\leq \frac{2M}{C}\sum_{i=1}^M p_i^2\sum_{k=0}^{K_r-1}\mathbb{E}\left\|\widehat{\nabla}_y g_i(\bm{x}_{i,r,k},\bm{y}_{i,r,k};\zeta)-\nabla_yg_i(\bm{x}_{i,r,k},\bm{y}_{i,r,k};\zeta)\right\|^2\nonumber \\
    &\quad +2\mathbb{E}\left\|\sum_{i\in{C_r}} \tilde{p}_i\sum_{k=0}^{K_r-1}\nabla_yg_i(\bm{x}_{i,r,k},\bm{y}_{i,r,k};\zeta)\right\|^2\nonumber \\
    &\leq \frac{2M}{C}\sum_{i=1}^M p_i^2\sum_{k=0}^{K_r-1}\sigma_g^2
    +\frac{2M}{C}\left(\frac{C-1}{M-1}\right)\mathbb{E}\left\|\sum_{i=1}^{M}p_i\sum_{k=0}^{K_r-1}\nabla_y g_i(\bm{x}_{i,r,k},\bm{y}_{i,r,k};\zeta)\right\|^2\nonumber \\
    &\quad
    + \frac{2M}{C}\left(\frac{M-C}{M-1}\right)\sum_{i=1}^{M}p_i^2 \mathbb{E}\left\|\sum_{k=0}^{K_r-1}\nabla_y g_i(\bm{x}_{i,r,k},\bm{y}_{i,r,k};\zeta)\right\|^2\nonumber \\
    &\leq  \frac{2M}{C}\sum_{i=1}^M p_i^2\sum_{k=0}^{K_r-1}\sigma_g^2
    +\frac{2M}{C}\left(\frac{C-1}{M-1}\right)\mathbb{E}\left\|\sum_{i=1}^{M}p_i\sum_{k=0}^{K_r-1}\nabla_y g_i(\bm{x}_{i,r,k},\bm{y}_{i,r,k};\zeta)\right\|^2\nonumber \\
    &\quad
    +\frac{4M}{C}\left(\frac{M-C}{M-1}\right)\sum_{i=1}^{M}p_i^2 \sum_{k=0}^{K_r-1}\mathbb{E}\left\|\nabla_y g_i(\bm{x}_{i,r,k},\bm{y}_{i,r,k};\zeta)-\nabla_y g_i(\bm{x}_r,\bm{y}_r;\zeta)
    \right\|^2\nonumber \\
    &\quad +\frac{4M}{C}\left(\frac{M-C}{M-1}\right)p_{\text{max}} 
    (L_g^y)^2\|\bm{y}_r-\bm{y}^*(\bm{x}_r)\|^2.\label{eq.lemma11_3}
\end{align}
We then analyze the second term in Eq. (\ref{eq.lemma11_3}):
\begin{align}
    &\mathbb{E}\left\|\sum_{i=1}^{M}p_i\sum_{k=0}^{K_r-1}\nabla_y g_i(\bm{x}_{i,r,k},\bm{y}_{i,r,k};\zeta)\right\|^2\nonumber \\
    &\leq 3\mathbb{E}\left\|\sum_{i=1}^{M}p_i\sum_{k=0}^{K_r-1}\nabla_y g_i(\bm{x}_{i,r,k},\bm{y}_{i,r,k};\zeta)
    -\sum_{i=1}^{M}p_i\nabla_yg_i(\bm{x}_r,\bm{y}_r)
    \right\|^2\nonumber \\
    &\quad +3\mathbb{E}\left\|
    \sum_{i=1}^{M}p_i\nabla_yg_i(\bm{x}_r,\bm{y}_r)-\sum_{i=1}^{M}p_i\nabla_yg_i(\bm{x}_r,\bm{y}^*(\bm{x}_r))
    \right\|^2
    +3\mathbb{E}\left\|
    \sum_{i=1}^{M}p_i\nabla_yg_i(\bm{x}_r,\bm{y}^*(\bm{x}_r))
    \right\|^2
    \nonumber \\
    &\leq 3(L_g^y)^2\mathbb{E}\left\|
    \bm{y}_r-\bm{y}^*(\bm{x}_r)
    \right\|^2
    + 3\sum_{i=1}^{M}p_i^2\sum_{k=0}^{K_r-1}\left\|\nabla_y g_i(\bm{x}_{i,r,k},\bm{y}_{i,r,k};\zeta)
    -\nabla_yg_i(\bm{x}_r,\bm{y}_r)
    \right\|^2.\label{eq.lemma11_4}
\end{align}
Combining Eq. (\ref{eq.lemma11_3}) and Eq. (\ref{eq.lemma11_4}), we get
\begin{align}
    &\mathbb{E}\|\nabla_{y,r}^{(s)}\|^2\nonumber \\
    &\leq \frac{2M}{C}\sum_{i=1}^M p_i^2\sum_{k=0}^{K_r-1}\sigma_g^2
    +p_{\text{max}}\left(
    \frac{4M}{C}\left(\frac{M-C}{M-1}\right)
    +\frac{6M}{C}\left(\frac{C-1}{M-1}\right)
    \right)
    \sum_{i=1}^{M}p_i \sum_{k=0}^{K_r-1}
    \Delta_{g,i,r,k}
    \nonumber \\
    &\quad +\left(
    \frac{4M}{C}\left(\frac{M-C}{M-1}\right)p_{\text{max}} 
    (L_g^y)^2
    + \frac{6M}{C}\left(\frac{C-1}{M-1}\right)(L_g^y)^2
    \right)
    \mathbb{E}\|\bm{y}_r-\bm{y}^*(\bm{x}_r)\|^2.
\end{align}
For $\mathbb{E}\|\nabla_{v,r}^{(s)}\|^2$, we can similarly have 
\begin{align}
    &\mathbb{E}\|\nabla_{v,r}^{(s)}\|^2\nonumber \\
    &=\mathbb{E}\left\|\sum_{i\in C_r}\tilde{p}_i\sum_{k=0}^{K_r-1}\widehat{\nabla}_x l_i(\bm{x}_{i,r,k},\bm{y}_{i,r,k},\bm{v}_{i,r,k};\psi)\right\|^2\nonumber \\
    &=\mathbb{E}\left\|\sum_{i\in C_r}\tilde{p}_i\sum_{k=0}^{K_r-1}\widehat{\nabla}_x l_i(\bm{x}_{i,r,k},\bm{y}_{i,r,k},\bm{v}_{i,r,k};\psi)
    -\nabla_x l_i(\bm{x}_{i,r,k},\bm{y}_{i,r,k},\bm{v}_{i,r,k};\psi)
    +\nabla_x l_i(\bm{x}_{i,r,k},\bm{y}_{i,r,k},\bm{v}_{i,r,k};\psi)
    \right\|^2\nonumber \\
    &\leq\frac{2M}{C}\sum_{i=1}^M p_i^2 \sum_{k=0}^{K_r-1}\mathbb{E}\left\|\widehat{\nabla}_x l_i(\bm{x}_{i,r,k},\bm{y}_{i,r,k},\bm{v}_{i,r,k};\psi)
    -\nabla_x l_i(\bm{x}_{i,r,k},\bm{y}_{i,r,k},\bm{v}_{i,r,k};\psi)\right\|^2\nonumber \\
    &\quad + 2\mathbb{E}\left\|\sum_{i\in C_r}\tilde{p}_i\sum_{k=0}^{K_r-1}\nabla_x l_i(\bm{x}_{i,r,k},\bm{y}_{i,r,k},\bm{v}_{i,r,k};\psi)\right\|^2\nonumber \\
    &\leq \frac{8M}{C}\sum_{i=1}^M p_i^2 \sum_{k=0}^{K_r-1}
    \left(
    \sigma_f^2+\sigma_r^2+2\iota^2\sigma_{gg}^2
    \right)
    \nonumber \\
    &\quad +
    4p_{\text{max}}\left(
    \frac{M}{C}\left(\frac{M-C}{M-1}\right)
    +\frac{M}{C}\left(\frac{C-1}{M-1}\right)
    \right)
    \sum_{i=1}^{M}p_i \sum_{k=0}^{K_r-1}\Delta_{l,i,r,k}
    \nonumber \\
    &\quad +
    12p_{\text{max}}\left(
    \frac{M}{C}\left(\frac{M-C}{M-1}\right)
    +\frac{M}{C}\left(\frac{C-1}{M-1}\right)
    \right)
    \left(
    (L_{l_c}^y)^2\|\bm{y}_r-\bm{y}^*(\bm{x}_r)\|^2+(\tilde{L}_{l_c}^x)^2\|\bm{x}_r-\bm{x}^*\|^2+\mathcal{V}_l
    \right).
\end{align}

\end{proof}

\subsection{Descent Lemma of the LL Problem}\label{app:ll_descent}
\begin{lemma}
    Under Assumptions \ref{assump_1} to \ref{assump_5}, the iterates of the LL problem generated according to the FBi-SGD algorithm satisfy
    \begin{align}
    &\mathbb{E}\|\bm{y}_{r+1}-\bm{y}^*(\bm{x}_r)\|^2\nonumber \\
    &\leq \left(
    1-\gamma_g\eta_y^{(s)}\sum_{i\in C_r}\tilde{p}_i\sum_{k=0}^{K_r-1}
    \right)\mathbb{E}\|\bm{y}_{r}-\bm{y}^*(\bm{x}_r)\|^2 + (\eta_y^{(s)})^2\mathbb{E}\left\|
    \nabla^{(s)}_{y,r}
    \right\|^2
    \nonumber \\
    &\quad +\frac{8\eta_y^{(s)}}{\gamma_g}(L^y_g\varepsilon_d+L_g^x)^2\sum_{i\in C_r}\tilde{p}_i\sum_{k=0}^{K_r-1}\mathbb{E}\|\bm{x}_{i,r,k}-\bm{x}_r\|^2
    +\frac{8\eta_y^{(s)}}{\gamma_g}(L_g^y)^2\sum_{i\in C_r}\tilde{p}_i\sum_{k=0}^{K_r-1}\mathbb{E}\|\bm{y}_{i,r,k}-\bm{y}_r\|^2 + \frac{2\eta_y^{(s)}K_r}{\gamma_g}\sigma_g^2,
\end{align}
and 
\begin{align}
    &\mathbb{E}\|\bm{y}_{r+1}-
    \bm{y}^*(\bm{x}_{r+1})
    \|^2 \nonumber \\
    &\leq \left(
    1+\frac{C_g^{xy} \eta_x^{(s)}}{\gamma_g}
    \right) \mathbb{E}\|\bm{y}_{r+1}-
    \bm{y}^*(\bm{x}_r)
    \|^2 
    +\left(
    \bar{L}_y^2 \left(\eta_x^{(s)}\right)^2+\frac{C_g^{xy} \eta_x^{(s)}}{\gamma_g}
    \right)
    \mathbb{E}\|\nabla_{x,r}^{(s)}\|^2\nonumber \\
    &\leq \left(
    1+\frac{C_g^{xy} \eta_x^{(s)}}{\gamma_g}
    \right)\left(
    1-\gamma_g\eta_y^{(s)}K_r
    \right)\mathbb{E}\|\bm{y}_{r}-\bm{y}^*(\bm{x}_r)\|^2
    \nonumber \\
    &\quad +\left(
    1+\frac{C_g^{xy} \eta_x^{(s)}}{\gamma_g}
    \right)\left(
    \frac{8\eta_y^{(s)}}{\gamma_g}(L^y_g\varepsilon_d+L_g^x)^2\right)\sum_{i\in C_r}\tilde{p}_i\sum_{k=0}^{K_r-1}\mathbb{E}\|\bm{x}_{i,r,k}-\bm{x}_r\|^2
    \nonumber \\
    &\quad +
    \left(
    1+\frac{C_g^{xy} \eta_x^{(s)}}{\gamma_g}
    \right)\left(\frac{8\eta_y^{(s)}}{\gamma_g}(L_g^y)^2\right)\sum_{i\in C_r}\tilde{p}_i\sum_{k=0}^{K_r-1}\mathbb{E}\|\bm{y}_{i,r,k}-\bm{y}_r\|^2
    \nonumber \\
    &\quad +\left(
    1+\frac{C_g^{xy} \eta_x^{(s)}}{\gamma_g}
    \right)\left(\eta_y^{(s)}\right)^2 \mathbb{E}\|\nabla_{y,r}^{(s)}\|^2
    +\left(
    \bar{L}_y^2 \left(\eta_x^{(s)}\right)^2+\frac{C_g^{xy} \eta_x^{(s)}}{\gamma_g}
    \right)
    \mathbb{E}\|\nabla_{x,r}^{(s)}\|^2 + \left(
    1+\frac{C_g^{xy} \eta_x^{(s)}}{\gamma_g}
    \right)\frac{2\eta_y^{(s)}K_r}{\gamma_g}\sigma_g^2.
\end{align}
\end{lemma}
\begin{proof}
It's easy to know that
\begin{align}
    &\mathbb{E}\|\bm{y}_{r+1}-\bm{y}^*(\bm{x}_r)\|^2\nonumber \\
    &=\mathbb{E}\|\bm{y}_{r+1}-\bm{y}_r+\bm{y}_r-\bm{y}^*(\bm{x}_r)\|^2\nonumber \\
    &=\mathbb{E}\|\bm{y}_r-\bm{y}^*(\bm{x}_r)-\eta_y^{(s)}\sum_{i\in C_r}\tilde{p}_i\nabla^{(ac)}_{y,i,r}\|^2\nonumber \\
    &=\mathbb{E}\|\bm{y}_{r}-\bm{y}^*(\bm{x}_r)\|^2 + (\eta_y^{(s)})^2\mathbb{E}\left\|
    \sum_{i\in C_r}\tilde{p}_i\nabla^{(ac)}_{y,i,r}
    \right\|^2-2\eta_y^{(s)}\mathbb{E}\left\langle
    \bm{y}_{r}-\bm{y}^*(\bm{x}_r),\sum_{i\in C_r}\tilde{p}_i\nabla^{(ac)}_{y,i,r}
    \right\rangle.\label{eq.desll_eq_1}
\end{align}
For the last term in Eq. (\ref{eq.desll_eq_1}), we can obtain that
\begin{align}
    &-\mathbb{E}\left\langle
    \bm{y}_{r}-\bm{y}^*(\bm{x}_r),\sum_{i\in C_r}\tilde{p}_i\nabla^{(ac)}_{y,i,r}
    \right\rangle\nonumber \\
    &=-\mathbb{E}\left\langle
    \bm{y}_{r}-\bm{y}^*(\bm{x}_r),\sum_{i\in C_r}\tilde{p}_i\nabla^{(ac)}_{y,i,r}-\sum_{i\in C_r}\tilde{p}_i \nabla_yg_i(\bm{x}_r,\bm{y}_r)+\sum_{i\in C_r}\tilde{p}_i \nabla_yg_i(\bm{x}_r,\bm{y}_r)-\sum_{i\in C_r}\tilde{p}_i \nabla_yg_i(\bm{x}_r,\bm{y}^*(\bm{x}_r))
    \right\rangle\nonumber \\
    &=-\sum_{i\in C_r}\tilde{p}_i\sum_{k=0}^{K_r-1}\mathbb{E}\left\langle
    \bm{y}_{r}-\bm{y}^*(\bm{x}_r),
    \widehat{\nabla}_yg_i(\bm{x}_{i,r,k},\bm{y}_{i,r,k})-\nabla_yg_i(\bm{x}_r,\bm{y}_r)
    \right\rangle\nonumber \\
    &\quad -\sum_{i\in C_r}\tilde{p}_i\sum_{k=0}^{K_r-1}\mathbb{E}\left\langle
    \bm{y}_{r}-\bm{y}^*(\bm{x}_r),
    \nabla_yg_i(\bm{x}_{r},\bm{y}_{r})-\nabla_yg_i(\bm{x}_r,\bm{y}^*(\bm{x}_r))
    \right\rangle\nonumber \\
    &\overset{(a)}{\leq} \sum_{i\in C_r}\tilde{p}_i\sum_{k=0}^{K_r-1}\mathbb{E}\left[
    \frac{1}{\gamma_g}\|
    \widehat{\nabla}_yg_i(\bm{x}_{i,r,k},\bm{y}_{i,r,k})-\nabla_yg_i(\bm{x}_r,\bm{y}_r)
    \|^2+
    \frac{\gamma_g}{2}\|
    \bm{y}_{r}-\bm{y}^*(\bm{x}_r)
    \|^2
    \right]\nonumber \\
    &\quad-
    \sum_{i\in C_r}\tilde{p}_i\sum_{k=0}^{K_r-1}\gamma_g\mathbb{E}\|
    \bm{y}_{r}-\bm{y}^*(\bm{x}_r)
    \|^2\nonumber \\
    &\overset{(b)}{\leq}
    -
    \sum_{i\in C_r}\tilde{p}_i\sum_{k=0}^{K_r-1}\frac{\gamma_g}{2}\mathbb{E}\|
    \bm{y}_{r}-\bm{y}^*(\bm{x}_r)
    \|^2+
    \frac{4}{\gamma_g}\sum_{i\in C_r}\tilde{p}_i\sum_{k=0}^{K_r-1}(L^y_g\varepsilon_d+L_g^x)^2\mathbb{E}\|\bm{x}_{i,r,k}-\bm{x}_r\|^2\nonumber\\
    &\quad +\frac{4}{\gamma_g}\sum_{i\in C_r}\tilde{p}_i\sum_{k=0}^{K_r-1}(L_g^y)^2\mathbb{E}\|\bm{y}_{i,r,k}-\bm{y}_r\|^2 + \frac{2K_r}{\gamma_g}\sigma_g^2,
\end{align}
where $(a)$ is based on Young's inequality and the convexity of $g_i(\cdot)$, and $(b)$ is based on Lemma \ref{lemma_9}.
Therefore, we have
\begin{align}
    &\mathbb{E}\|\bm{y}_{r+1}-\bm{y}^*(\bm{x}_r)\|^2\nonumber \\
    &\leq \left(
    1-\gamma_g\eta_y^{(s)}\sum_{i\in C_r}\tilde{p}_i\sum_{k=0}^{K_r-1}
    \right)\mathbb{E}\|\bm{y}_{r}-\bm{y}^*(\bm{x}_r)\|^2 + (\eta_y^{(s)})^2\mathbb{E}\left\|
    \nabla^{(s)}_{y,r}
    \right\|^2
    \nonumber \\
    &\quad +\frac{8\eta_y^{(s)}}{\gamma_g}(L^y_g\varepsilon_d+L_g^x)^2\sum_{i\in C_r}\tilde{p}_i\sum_{k=0}^{K_r-1}\mathbb{E}\|\bm{x}_{i,r,k}-\bm{x}_r\|^2
    +\frac{8\eta_y^{(s)}}{\gamma_g}(L_g^y)^2\sum_{i\in C_r}\tilde{p}_i\sum_{k=0}^{K_r-1}\mathbb{E}\|\bm{y}_{i,r,k}-\bm{y}_r\|^2 + \frac{2\eta_y^{(s)}K_r}{\gamma_g}\sigma_g^2.
\end{align}
Next, we analyze $\|\bm{y}_{r+1}-\bm{y}^*(\bm{x}_{r+1})\|^2$ as follows:
\begin{align}
    &\mathbb{E}\|\bm{y}_{r+1}-\bm{y}^*(\bm{x}_{r+1})\|^2\nonumber \\
    &=\mathbb{E}\|\bm{y}_{r+1}-\bm{y}^*(\bm{x}_{r})\|^2
    +\underbrace{\mathbb{E}\|\bm{y}^*(\bm{x}_{r+1})-\bm{y}^*(\bm{x}_{r})\|^2}_{:=\Delta_1}+2
    \underbrace{\mathbb{E}\langle
    \bm{y}_{r+1}-\bm{y}^*(\bm{x}_{r}),\bm{y}^*(\bm{x}_{r})-\bm{y}^*(\bm{x}_{r+1})
    \rangle}_{:=\Delta_2}
    .
\end{align}
Then, we have
\begin{align}
    \Delta_1
    &=\mathbb{E}\|\bm{y}^*(\bm{x}_{r+1})-\bm{y}^*(\bm{x}_{r})\|^2\nonumber \\
    &\overset{(a)}{\leq}\bar{L}_y^2\mathbb{E}\|\bm{x}_{r+1}-\bm{x}_{r}\|^2\nonumber \\
    &\leq \bar{L}_y^2 (\eta_x^{(s)})^2\mathbb{E}\left\|\nabla_{x,r}^{(s)}
    \right\|^2
,
\end{align}
where in $(a)$ we use Lemma \ref{lemma_3}, and $(b)$ holds because of Lemma \ref{lemma_11}.
We also know
\begin{align}
    \Delta_2
    &=\mathbb{E}\langle \bm{y}^*(\bm{x}_r)-\bm{y}_{r+1},\bm{y}^*(\bm{x}_{r+1})-\bm{y}^*(\bm{x}_r)
    \rangle\nonumber \\
    &\overset{(a)}{=}\mathbb{E}\langle \bm{y}^*(\bm{x}_r)-\bm{y}_{r+1},\nabla \bm{y}^*(\hat{\bm{x}}_{r+1})^\intercal(\bm{x}_{r+1}-\bm{x}_r)\rangle\nonumber \\
    &=\mathbb{E}\langle \bm{y}^*(\bm{x}_r)-\bm{y}_{r+1}, \eta_x^{(s)} \nabla \bm{y}^*(\hat{\bm{x}}_{r+1})^\intercal
    \nabla_{x,r}^{(s)}
    \rangle \nonumber \\
    &\leq \frac{C_g^{xy}}{\gamma_g}\eta_x^{(s)}\mathbb{E}\|\bm{y}^*(\bm{x}_r)-\bm{y}_{r+1}\| \|\nabla_{x,r}^{(s)}\|\nonumber \\
    &\leq \frac{C_g^{xy} \eta_x^{(s)}}{2\gamma_g}\mathbb{E}\|\bm{y}^*(\bm{x}_r)-\bm{y}_{r+1}\|^2+\frac{C_g^{xy} \eta_x^{(s)}}{2\gamma_g}\mathbb{E}\|\nabla_{x,r}^{(s)}\|^2,
\end{align}
where $(a)$ holds with the mean-value theorem, and let $\hat{x}_{r+1} := \vartheta\bm{x}_r + (1-\vartheta)\bm{x}_{r+1}$, $\vartheta\in[0,1]$.

Therefore, we can obtain
\begin{align}
    &\mathbb{E}\|\bm{y}_{r+1}-
    \bm{y}^*(\bm{x}_{r+1})
    \|^2 \nonumber \\
    &\leq \left(
    1+\frac{C_g^{xy} \eta_x^{(s)}}{\gamma_g}
    \right) \mathbb{E}\|\bm{y}_{r+1}-
    \bm{y}^*(\bm{x}_r)
    \|^2 
    +\left(
    \bar{L}_y^2 \left(\eta_x^{(s)}\right)^2+\frac{C_g^{xy} \eta_x^{(s)}}{\gamma_g}
    \right)
    \mathbb{E}\|\nabla_{x,r}^{(s)}\|^2\nonumber \\
    &\leq \left(
    1+\frac{C_g^{xy} \eta_x^{(s)}}{\gamma_g}
    \right)\left(
    1-\gamma_g\eta_y^{(s)}K_r
    \right)\mathbb{E}\|\bm{y}_{r}-\bm{y}^*(\bm{x}_r)\|^2
    \nonumber \\
    &\quad +\left(
    1+\frac{C_g^{xy} \eta_x^{(s)}}{\gamma_g}
    \right)\left(
    \frac{8\eta_y^{(s)}}{\gamma_g}(L^y_g\varepsilon_d+L_g^x)^2\right)\sum_{i\in C_r}\tilde{p}_i\sum_{k=0}^{K_r-1}\mathbb{E}\|\bm{x}_{i,r,k}-\bm{x}_r\|^2
    \nonumber \\
    &\quad +
    \left(
    1+\frac{C_g^{xy} \eta_x^{(s)}}{\gamma_g}
    \right)\left(\frac{8\eta_y^{(s)}}{\gamma_g}(L_g^y)^2\right)\sum_{i\in C_r}\tilde{p}_i\sum_{k=0}^{K_r-1}\mathbb{E}\|\bm{y}_{i,r,k}-\bm{y}_r\|^2
    \nonumber \\
    &\quad +\left(
    1+\frac{C_g^{xy} \eta_x^{(s)}}{\gamma_g}
    \right)\left(\eta_y^{(s)}\right)^2 \mathbb{E}\|\nabla_{y,r}^{(s)}\|^2
    +\left(
    \bar{L}_y^2 \left(\eta_x^{(s)}\right)^2+\frac{C_g^{xy} \eta_x^{(s)}}{\gamma_g}
    \right)
    \mathbb{E}\|\nabla_{x,r}^{(s)}\|^2 + \left(
    1+\frac{C_g^{xy} \eta_x^{(s)}}{\gamma_g}
    \right)\frac{2\eta_y^{(s)}K_r}{\gamma_g}\sigma_g^2.
\end{align}
\end{proof}

\subsection{Descent Lemma of the UL Problem}\label{app:ul_descent}
\begin{lemma}
    Under Assumptions \ref{assump_1} to \ref{assump_5}, the iterates of the UL problem generated according to the FBi-SGD algorithm satisfy
    \begin{align}
    &\mathbb{E} \|\bm{x}_{r+1}-\bm{x}^*\|^2\nonumber \\
    &\leq \left(
    1-\frac{3\eta_x^{(s)}K_rL_{Fc}^x\gamma_f}{2(L_{Fc}^x+\gamma_f)}
    + 2\eta_x^{(s)}K_r\tilde{L}_{Fc}^x 
    \right)\|\bm{x}_r-\bm{x}^*\|^2+\frac{4\eta_x^{(s)}K_r\hat{L}_f^2(L_{Fc}^x+\gamma_f)}{L_{Fc}^x\gamma_f}\|\bm{y}_r-\bm{y}^*(\bm{x}_r)\|^2 - \frac{2\eta_x^{(s)}K_r}{L_{Fc}^x+\gamma_f} \mathcal{V}_F
    \nonumber \\
    &\quad + \left(\eta_x^{(s)}\right)^2\|\nabla_{x,r}^{(s)}\|^2 
    + \frac{4\eta_x^{(s)}K_r(L_{Fc}^x+\gamma_f)}{L_{Fc}^x\gamma_f}(\sigma_f^2 + \sigma_r^2 + 2\iota^2\sigma_{gg}^2).
    %
    %
    %
\end{align}

\end{lemma}

\begin{proof}
Let $\bm{x}^*$ be the $\bm{x}_s$. It's obvious that
\begin{align}
    &\mathbb{E} \|\bm{x}_{r+1}-\bm{x}^*\|^2\nonumber \\
    &=\mathbb{E} \|\bm{x}_{r+1}-\bm{x}_r+\bm{x}_r-\bm{x}^*\|^2\nonumber \\
    &=\mathbb{E}\|\bm{x}_{r+1}-\bm{x}_r\|^2+\mathbb{E}\|\bm{x}_r-\bm{x}^*\|^2+2\mathbb{E}
    \langle
    \bm{x}_{r+1}-\bm{x}_r,\bm{x}_r-\bm{x}^*
    \rangle
    \nonumber \\
    &=\mathbb{E}\|\bm{x}_{r+1}-\bm{x}_r\|^2+\mathbb{E}\|\bm{x}_r-\bm{x}^*\|^2
    -2\eta_x^{(s)}\mathbb{E}\langle
    \nabla_{x,r}^{(s)},\bm{x}_r-\bm{x}^*
    \rangle\nonumber \\
    &=\mathbb{E}\|\bm{x}_{r+1}-\bm{x}_r\|^2+\mathbb{E}\|\bm{x}_r-\bm{x}^*\|^2 
    \nonumber \\
    &\quad - 2\eta_x^{(s)}\mathbb{E}\left\langle
    \sum_{i\in C_r}\tilde{p}_i\sum_{k=0}^{K_r-1}\widehat{\nabla}_x \mathcal{F}_i(\bm{x}_{i,r,k},\bm{y}_{i,r,k},\bm{v}_{i,r,k};\psi) - \nabla_x \mathcal{F}_i(\bm{x}_{i,r,k},\bm{y}_{i,r,k},\bm{v}_{i,r,k};\psi),\bm{x}_r-\bm{x}^*
    \right\rangle
    \nonumber \\
    &\quad - 2\eta_x^{(s)}\mathbb{E}\left\langle
    \sum_{i\in C_r}\tilde{p}_i\sum_{k=0}^{K_r-1} \nabla_x \mathcal{F}_i(\bm{x}_{i,r,k},\bm{y}_{i,r,k},\bm{v}_{i,r,k};\psi) - \nabla_x \mathcal{F}_i(\bm{x}_r,\bm{y}_r,\bm{v}_r;\psi),\bm{x}_r-\bm{x}^*
    \right\rangle
    \nonumber \\
    &\quad - 2\eta_x^{(s)}\mathbb{E}\left\langle
    \sum_{i\in C_r}\tilde{p}_i\sum_{k=0}^{K_r-1} \nabla_x \mathcal{F}_i(\bm{x}_r,\bm{y}_r,\bm{v}_r;\psi),\bm{x}_r-\bm{x}^*
    \right\rangle. \label{lemma_13_eq1}
\end{align}
First, we analyze the expectation in the fifth term in Eq. (\ref{lemma_13_eq1}): 
\begin{align}
    &\mathbb{E}\left\langle
    \sum_{i\in C_r}\tilde{p}_i\sum_{k=0}^{K_r-1} \nabla_x \mathcal{F}_i(\bm{x}_r,\bm{y}_r,\bm{v}_r;\psi),\bm{x}_r-\bm{x}^*
    \right\rangle
    \nonumber \\
    &=\mathbb{E}\left\langle
    \sum_{i\in C_r}\tilde{p}_i\sum_{k=0}^{K_r-1} \nabla_x \mathcal{F}_i(\bm{x}_r,\bm{y}_r,\bm{v}_r;\psi)-\nabla_x\mathcal{F}_i(\bm{x}_r,\bm{y}^*(\bm{x}_r),\bm{v}_r^*(\bm{x}_r,\bm{y}^*(\bm{x}_r)))
    ,\bm{x}_r-\bm{x}^*
    \right\rangle
    \nonumber \\
    &\quad +\mathbb{E}\left\langle
    \sum_{i\in C_r}\tilde{p}_i\sum_{k=0}^{K_r-1}
    \nabla_x\mathcal{F}_i(\bm{x}_r,\bm{y}^*(\bm{x}_r),\bm{v}_r^*(\bm{x}_r,\bm{y}^*(\bm{x}_r)))-\nabla_x\mathcal{F}_i(\bm{x}_r,\bm{y}^*(\bm{x}_r,\bm{x}^*),\bm{v}_r^*(\bm{x}_r,\bm{y}^*(\bm{x}_r,\bm{x}^*)))
    ,\bm{x}_r-\bm{x}^*
    \right\rangle
    \nonumber \\
    &\quad +\mathbb{E}\left\langle
    \sum_{i\in C_r}\tilde{p}_i\sum_{k=0}^{K_r-1}
    \nabla_x\mathcal{F}_i(\bm{x}_r,\bm{y}^*(\bm{x}_r,\bm{x}^*),\bm{v}_r^*(\bm{x}_r,\bm{y}^*(\bm{x}_r,\bm{x}^*)))-\nabla_x\mathcal{F}_i(\bm{x}^*,\bm{y}^*(\bm{x}^*),\bm{v}^*(\bm{x}^*,\bm{y}^*(\bm{x}^*)))
    ,\bm{x}_r-\bm{x}^*
    \right\rangle. \label{lemma_13_eq2}
\end{align}
The last two terms in Eq. (\ref{lemma_13_eq2}) have the lower bound as
\begin{align}
    &\mathbb{E}\left\langle
    \sum_{i\in C_r}\tilde{p}_i\sum_{k=0}^{K_r-1}
    \nabla_x\mathcal{F}_i(\bm{x}_r,\bm{y}^*(\bm{x}_r),\bm{v}_r^*(\bm{x}_r,\bm{y}^*(\bm{x}_r)))-\nabla_x\mathcal{F}_i(\bm{x}_r,\bm{y}^*(\bm{x}_r,\bm{x}^*),\bm{v}_r^*(\bm{x}_r,\bm{y}^*(\bm{x}_r,\bm{x}^*)))
    ,\bm{x}_r-\bm{x}^*
    \right\rangle
    \nonumber \\
    &+\mathbb{E}\left\langle
    \sum_{i\in C_r}\tilde{p}_i\sum_{k=0}^{K_r-1}
    \nabla_x\mathcal{F}_i(\bm{x}_r,\bm{y}^*(\bm{x}_r,\bm{x}^*),\bm{v}_r^*(\bm{x}_r,\bm{y}^*(\bm{x}_r,\bm{x}^*)))-\nabla_x\mathcal{F}_i(\bm{x}^*,\bm{y}^*(\bm{x}^*),\bm{v}^*(\bm{x}^*,\bm{y}^*(\bm{x}^*)))
    ,\bm{x}_r-\bm{x}^*
    \right\rangle
    \nonumber \\
    &\overset{(a)}{\geq} - K_r\tilde{L}_{Fc}^x\|\bm{x}_r-\bm{x}^*\|^2 \nonumber \\
    &\quad +\mathbb{E}\left\langle
    \sum_{i\in C_r}\tilde{p}_i\sum_{k=0}^{K_r-1}
    \nabla_x\mathcal{F}_i(\bm{x}_r,\bm{y}^*(\bm{x}_r,\bm{x}^*),\bm{v}_r^*(\bm{x}_r,\bm{y}^*(\bm{x}_r,\bm{x}^*)))-\nabla_x\mathcal{F}_i(\bm{x}^*,\bm{y}^*(\bm{x}^*),\bm{v}^*(\bm{x}^*,\bm{y}^*(\bm{x}^*)))
    ,\bm{x}_r-\bm{x}^*
    \right\rangle
    \nonumber \\
    &\overset{(b)}{\geq}
    \left(
    - K_r\tilde{L}_{Fc}^x+\frac{K_rL_{Fc}^x\gamma_f}{L_{Fc}^x+\gamma_f}
    \right)
    \|\bm{x}_r-\bm{x}^*\|^2
     + \frac{K_r}{L_{Fc}^x+\gamma_f} \mathcal{V}_F,
\end{align}
where $(a)$ is based on Lemma \ref{lemma_7}, and $(b)$ is true because of the $L_{Fc}^x$-smooth and $\gamma_f$-strongly convexity of $\mathcal{F}_i(\bm{x},\bm{y}^*(\bm{x})$ with Lemma 3.11 in \cite{bubeck2015convex}.

We then analyze the expectations in the third and fourth terms in Eq. (\ref{lemma_13_eq1}):
\begin{align}
    &\mathbb{E}\left\langle
    \sum_{i\in C_r}\tilde{p}_i\sum_{k=0}^{K_r-1}\widehat{\nabla}_x \mathcal{F}_i(\bm{x}_{i,r,k},\bm{y}_{i,r,k},\bm{v}_{i,r,k};\psi) - \nabla_x \mathcal{F}_i(\bm{x}_{i,r,k},\bm{y}_{i,r,k},\bm{v}_{i,r,k};\psi),\bm{x}_r-\bm{x}^*
    \right\rangle
    \nonumber \\
    &+\mathbb{E}\left\langle
    \sum_{i\in C_r}\tilde{p}_i\sum_{k=0}^{K_r-1} \nabla_x \mathcal{F}_i(\bm{x}_r,\bm{y}_r,\bm{v}_r;\psi)-\nabla_x\mathcal{F}_i(\bm{x}_r,\bm{y}^*(\bm{x}_r),\bm{v}_r^*(\bm{x}_r,\bm{y}^*(\bm{x}_r)))
    ,\bm{x}_r-\bm{x}^*
    \right\rangle
    \nonumber \\
    &\geq -\frac{K_rL_{Fc}^x\gamma_f}{4(L_{Fc}^x+\gamma_f)}\|\bm{x}_r-\bm{x}^*\|^2
    - \frac{2K_r\hat{L}_f^2(L_{Fc}^x+\gamma_f)}{L_{Fc}^x\gamma_f}\|\bm{y}_r-\bm{y}^*(\bm{x}_r)\|^2 - \frac{2K_r(L_{Fc}^x+\gamma_f)}{L_{Fc}^x\gamma_f}(\sigma_f^2 + \sigma_r^2 + 2\iota^2\sigma_{gg}^2).
\end{align}
Thus, we finally get
\begin{align}
    &\mathbb{E} \|\bm{x}_{r+1}-\bm{x}^*\|^2\nonumber \\
    &\leq \left(
    1-\frac{3\eta_x^{(s)}K_rL_{Fc}^x\gamma_f}{2(L_{Fc}^x+\gamma_f)}
    + 2\eta_x^{(s)}K_r\tilde{L}_{Fc}^x 
    \right)\|\bm{x}_r-\bm{x}^*\|^2+\frac{4\eta_x^{(s)}K_r\hat{L}_f^2(L_{Fc}^x+\gamma_f)}{L_{Fc}^x\gamma_f}\|\bm{y}_r-\bm{y}^*(\bm{x}_r)\|^2 - \frac{2\eta_x^{(s)}K_r}{L_{Fc}^x+\gamma_f} \mathcal{V}_F
    \nonumber \\
    &\quad + \left(\eta_x^{(s)}\right)^2\|\nabla_{x,r}^{(s)}\|^2 
    + \frac{4\eta_x^{(s)}K_r(L_{Fc}^x+\gamma_f)}{L_{Fc}^x\gamma_f}(\sigma_f^2 + \sigma_r^2 + 2\iota^2\sigma_{gg}^2),
\end{align}
where we use the Young's inequality.

\end{proof}

\subsection{Descent Lemma of the LS Problem}\label{app:ls_descent}
\begin{lemma}
    Under Assumptions \ref{assump_1} to \ref{assump_5}, the iterates of the LS problem generated according to the FBi-SGD method satisfy
    \begin{align}
    &\mathbb{E} \|\bm{v}_{r+1}-\bm{v}^*(\bm{x}_r)\|^2\nonumber \\
    &\leq \left(1+
    \eta_v^{(s)}(K_r \tilde{L}_{lc}^x - K_r \gamma_g) - \frac{2\eta_v^{(s)}K_rC_g^{yy}\gamma_g}{C_g^{yy}+\gamma_g}
    \right)\mathbb{E}\|\bm{v}_r-\bm{v}^*(\bm{x}_r)\|^2
    + \eta_v^{(s)}K_r (\tilde{L}_{lc}^x)^2\|\bm{x}_r-\bm{x}^*\|^2
    +(\eta_v^{(s)})^2\mathbb{E}\|\nabla_{v,r}^{(s)}\|^2 
    \nonumber \\
    &\quad + \frac{2\eta_v^{(s)}}{\gamma_g}\sum_{i\in C_r}\tilde{p}_i\sum_{k=0}^{K_r-1}\Delta_{l,i,r,k}
    - \frac{2\eta_v^{(s)}K_r}{C_g^{yy}+\gamma_g}\mathcal{V}_l
    ,
\end{align}
and
\begin{align}
    &\mathbb{E} \|\bm{v}_{r+1}-\bm{v}^*(\bm{x}_{r+1})\|^2\nonumber \\
    &\leq (1 + \bar{L}_v\eta_x^{(s)})\left(1+
    \eta_v^{(s)}(K_r \tilde{L}_{lc}^x - K_r \gamma_g) - \frac{2\eta_v^{(s)}K_rC_g^{yy}\gamma_g}{C_g^{yy}+\gamma_g}
    \right)\mathbb{E}\|\bm{v}_r-\bm{v}^*(\bm{x}_r)\|^2
    \nonumber \\
    &\quad 
    + (1 + \bar{L}_v\eta_x^{(s)})\eta_v^{(s)}K_r (\tilde{L}_{lc}^x)^2\|\bm{x}_r-\bm{x}^*\|^2
    +(1 + \bar{L}_v\eta_x^{(s)})(\eta_v^{(s)})^2\mathbb{E}\|\nabla_{v,r}^{(s)}\|^2 
    +\left(\bar{L}_v\eta_x^{(s)}+\bar{L}_v^2 (\eta_x^{(s)})^2\right)\mathbb{E} \|\nabla_{x,r}^{(s)}\|^2
    \nonumber \\
    &\quad + (1 + \bar{L}_v\eta_x^{(s)})\frac{2\eta_v^{(s)}}{\gamma_g}\sum_{i\in C_r}\tilde{p}_i\sum_{k=0}^{K_r-1}\Delta_{l,i,r,k}
    - (1 + \bar{L}_v\eta_x^{(s)})\frac{2\eta_v^{(s)}K_r}{C_g^{yy}+\gamma_g}\mathcal{V}_l
    .
\end{align}

\end{lemma}
\begin{proof}
It's obvious that
\begin{align}
    &\mathbb{E} \|\bm{v}_{r+1}-\bm{v}^*(\bm{x}_r)\|^2\nonumber \\
    &=\mathbb{E} \|\bm{v}_{r+1}-\bm{v}_r+\bm{v}_r-\bm{v}^*(\bm{x}_r)\|^2\nonumber \\
    &=\mathbb{E}\|\bm{v}_{r+1}-\bm{v}_r\|^2+\mathbb{E}\|\bm{v}_r-\bm{v}^*(\bm{x}_r)\|^2+2\mathbb{E}
    \langle
    \bm{v}_{r+1}-\bm{v}_r,\bm{v}_r-\bm{v}^*(\bm{x}_r)
    \rangle
    \nonumber \\
    &\leq(\eta_v^{(s)})^2\mathbb{E}\|\nabla_{v,r}^{(s)}\|^2+\mathbb{E}\|\bm{v}_r-\bm{v}^*(\bm{x}_r)\|^2
    -2\eta_v^{(s)}\mathbb{E}\langle
    \nabla_{v,r}^{(s)},\bm{v}_r-\bm{v}^*(\bm{x}_r)
    \rangle.\label{lemma_14_eq1}
\end{align}

For the last term on the right side in inequality (\ref{lemma_14_eq1}), we have
\begin{align}
    &-\mathbb{E}\langle
    \nabla_{v,r}^{(s)},\bm{v}_r-\bm{v}^*(\bm{x}_r)
    \rangle
    \nonumber \\
    &=- \sum_{i\in C_r}\tilde{p}_i\sum_{k=0}^{K_r-1}\mathbb{E}\left\langle
    \widehat{\nabla}_vl_i(\bm{x}_{i,r,k},\bm{y}_{i,r,k},\bm{v}_{i,r,k}) - \nabla_v l_i(\bm{x}_r,\bm{y}_r,\bm{v}_r),\bm{v}_r-\bm{v}^*(\bm{x}_r)
    \right\rangle
    \nonumber \\
    &\quad - \sum_{i\in C_r}\tilde{p}_i\sum_{k=0}^{K_r-1}\mathbb{E}\left\langle
    \nabla_v l_i(\bm{x}_r,\bm{y}_r,\bm{v}_r) - \nabla_v l_i(\bm{x}_r,\bm{y}^*(\bm{x}_r),\bm{v}^*(\bm{x}_r,\bm{y}^*(\bm{x}_r)))
    ,\bm{v}_r-\bm{v}^*(\bm{x}_r)
    \right\rangle
    \nonumber \\
    &\quad - \sum_{i\in C_r}\tilde{p}_i\sum_{k=0}^{K_r-1}\mathbb{E}\left\langle
    \nabla_v l_i(\bm{x}_r,\bm{y}^*(\bm{x}_r),\bm{v}^*(\bm{x}_r,\bm{y}^*(\bm{x}_r))) 
    - \nabla_v l_i(\bm{x}_r,\bm{y}^*(\bm{x}_r,\bm{x}^*),\bm{v}^*(\bm{x}_r,\bm{y}^*(\bm{x}_r,\bm{x}^*)))
    ,\bm{v}_r-\bm{v}^*(\bm{x}_r)
    \right\rangle
    \nonumber \\
    &\quad - \sum_{i\in C_r}\tilde{p}_i\sum_{k=0}^{K_r-1}\mathbb{E}\left\langle
    \nabla_v l_i(\bm{x}_r,\bm{y}^*(\bm{x}_r,\bm{x}^*),\bm{v}^*(\bm{x}_r,\bm{y}^*(\bm{x}_r,\bm{x}^*)))
    - \nabla_v l_i(\bm{x}_r,\bm{y}^*(\bm{x}^*),\bm{v}^*(\bm{x}_r,\bm{y}^*(\bm{x}^*)))
    ,\bm{v}_r-\bm{v}^*(\bm{x}_r)
    \right\rangle
    \nonumber \\
    &\overset{(a)}{\leq}
    \sum_{i\in C_r}\tilde{p}_i\sum_{k=0}^{K_r-1}\mathbb{E}\left[
    \frac{1}{\gamma_g}\|\widehat{\nabla}_vl_i(\bm{x}_{i,r,k},\bm{y}_{i,r,k},\bm{v}_{i,r,k}) - \nabla_v l_i(\bm{x}_r,\bm{y}_r,\bm{v}_r) \|^2 + \frac{\gamma_g}{2}\|\bm{v}_r-\bm{v}^*(\bm{x}_r)\|^2
    \right]
    \nonumber \\
    &\quad -\sum_{i\in C_r}\tilde{p}_i\sum_{k=0}^{K_r-1}\gamma_g\mathbb{E}\|\bm{v}_r-\bm{v}^*(\bm{x}_r)\|^2
    \nonumber \\
    &\quad - \sum_{i\in C_r}\tilde{p}_i\sum_{k=0}^{K_r-1}\mathbb{E}\left\langle
    \nabla_v l_i(\bm{x}_r,\bm{y}^*(\bm{x}_r),\bm{v}^*(\bm{x}_r,\bm{y}^*(\bm{x}_r))) 
    - \nabla_v l_i(\bm{x}_r,\bm{y}^*(\bm{x}_r,\bm{x}^*),\bm{v}^*(\bm{x}_r,\bm{y}^*(\bm{x}_r,\bm{x}^*)))
    ,\bm{v}_r-\bm{v}^*(\bm{x}_r)
    \right\rangle
    \nonumber \\
    &\quad - \sum_{i\in C_r}\tilde{p}_i\sum_{k=0}^{K_r-1}\mathbb{E}\left\langle
    \nabla_v l_i(\bm{x}_r,\bm{y}^*(\bm{x}_r,\bm{x}^*),\bm{v}^*(\bm{x}_r,\bm{y}^*(\bm{x}_r,\bm{x}^*)))
    - \nabla_v l_i(\bm{x}_r,\bm{y}^*(\bm{x}^*),\bm{v}^*(\bm{x}_r,\bm{y}^*(\bm{x}^*)))
    ,\bm{v}_r-\bm{v}^*(\bm{x}_r)
    \right\rangle
    \nonumber \\
    &\overset{(b)}{\leq} -\frac{K_r\gamma_g}{2}\mathbb{E}\|\bm{v}_r-\bm{v}^*(\bm{x}_r)\|^2
    + \frac{1}{\gamma_g}\sum_{i\in C_r}\tilde{p}_i\sum_{k=0}^{K_r-1}\Delta_{l,i,r,k}
    + K_r \tilde{L}_{lc}^x \|\bm{x}_r-\bm{x}^*\|\|\bm{v}_r-\bm{v}^*(\bm{x}_r)\|
    \nonumber \\
    &\quad - \frac{K_rC_g^{yy}\gamma_g}{C_g^{yy}+\gamma_g}\|\bm{v}_r - \bm{v}^*(\bm{x}_r)\|^2
    - \frac{K_r}{C_g^{yy}+\gamma_g}\mathcal{V}_l
    \nonumber \\
    &\overset{(c)}{\leq}
    \frac{K_r \tilde{L}_{lc}^x}{2}\|\bm{x}_r-\bm{x}^*\|^2 + 
    \left(
    \frac{K_r\tilde{L}_{lc}^x - K_r \gamma_g}{2} - \frac{K_rC_g^{yy}\gamma_g}{C_g^{yy}+\gamma_g}
    \right)\|\bm{v}_r-\bm{v}^*(\bm{x}_r)\|^2 + \frac{1}{\gamma_g}\sum_{i\in C_r}\tilde{p}_i\sum_{k=0}^{K_r-1}\Delta_{l,i,r,k}- \frac{K_r}{C_g^{yy}+\gamma_g}\mathcal{V}_l,
\end{align}
where $(a)$ is based on Young's inequality and the convexity of $l_i(\cdot)$, $(b)$ holds with Lemma \ref{lemma_7} and the $C_g^{yy}$-smooth and $\gamma_g$-strongly convexity of $l_i(\cdot)$ by applying Lemma 3.11 in \cite{bubeck2015convex}, and in $(c)$ we use Young's inequality.
Therefore, we get
\begin{align}
    &\mathbb{E} \|\bm{v}_{r+1}-\bm{v}^*(\bm{x}_r)\|^2\nonumber \\
    &\leq \left(1+
    \eta_v^{(s)}(K_r\tilde{L}_{lc}^x - K_r \gamma_g) - \frac{2\eta_v^{(s)}K_rC_g^{yy}\gamma_g}{C_g^{yy}+\gamma_g}
    \right)\mathbb{E}\|\bm{v}_r-\bm{v}^*(\bm{x}_r)\|^2
    + \eta_v^{(s)}K_r \tilde{L}_{lc}^x\|\bm{x}_r-\bm{x}^*\|^2
    +(\eta_v^{(s)})^2\mathbb{E}\|\nabla_{v,r}^{(s)}\|^2 
    \nonumber \\
    &\quad + \frac{2\eta_v^{(s)}}{\gamma_g}\sum_{i\in C_r}\tilde{p}_i\sum_{k=0}^{K_r-1}\Delta_{l,i,r,k}
    - \frac{2\eta_v^{(s)}K_r}{C_g^{yy}+\gamma_g}\mathcal{V}_l
    ,\label{lemma_14_eq2}
\end{align}
and 
\begin{align}
    &\mathbb{E} \|\bm{v}_{r+1}-\bm{v}^*(\bm{x}_{r+1})\|^2\nonumber \\
    &\leq \mathbb{E} \|\bm{v}_{r+1}-\bm{v}^*(\bm{x}_r)\|^2 + \mathbb{E} \|\bm{v}^*(\bm{x}_{r+1}) - \bm{v}^*(\bm{x}_r)\|^2 + 2\mathbb{E}\langle\bm{v}_{r+1}-\bm{v}^*(\bm{x}_r),\bm{v}^*(\bm{x}_r) - \bm{v}^*(\bm{x}_{r+1})\rangle
    \nonumber \\
    &\overset{(a)}{\leq} (1 + \bar{L}_v\eta_x^{(s)})\mathbb{E} \|\bm{v}_{r+1}-\bm{v}^*(\bm{x}_r)\|^2 + \left(\bar{L}_v\eta_x^{(s)}+\bar{L}_v^2 (\eta_x^{(s)})^2\right)\mathbb{E} \|\nabla_{x,r}^{(s)}\|^2,\label{lemma_14_eq3}
\end{align}
where $(a)$ holds with the mean-value theorem and Young's inequality. Combining inequalities (\ref{lemma_14_eq2}) and (\ref{lemma_14_eq3}), we further get
\begin{align}
    &\mathbb{E} \|\bm{v}_{r+1}-\bm{v}^*(\bm{x}_{r+1})\|^2\nonumber \\
    &\leq (1 + \bar{L}_v\eta_x^{(s)})\left(1+
    \eta_v^{(s)}(K_r\tilde{L}_{lc}^x - K_r \gamma_g) - \frac{2\eta_v^{(s)}K_rC_g^{yy}\gamma_g}{C_g^{yy}+\gamma_g}
    \right)\mathbb{E}\|\bm{v}_r-\bm{v}^*(\bm{x}_r)\|^2
    \nonumber \\
    &\quad 
    + (1 + \bar{L}_v\eta_x^{(s)})\eta_v^{(s)}K_r \tilde{L}_{lc}^x\|\bm{x}_r-\bm{x}^*\|^2
    +(1 + \bar{L}_v\eta_x^{(s)})(\eta_v^{(s)})^2\mathbb{E}\|\nabla_{v,r}^{(s)}\|^2 
    +\left(\bar{L}_v\eta_x^{(s)}+\bar{L}_v^2 (\eta_x^{(s)})^2\right)\mathbb{E} \|\nabla_{x,r}^{(s)}\|^2
    \nonumber \\
    &\quad + (1 + \bar{L}_v\eta_x^{(s)})\frac{2\eta_v^{(s)}}{\gamma_g}\sum_{i\in C_r}\tilde{p}_i\sum_{k=0}^{K_r-1}\Delta_{l,i,r,k}
    - (1 + \bar{L}_v\eta_x^{(s)})\frac{2\eta_v^{(s)}K_r}{C_g^{yy}+\gamma_g}\mathcal{V}_l
    .
\end{align}

\end{proof}

\subsection{Convergence of the Whole Sequence}\label{sec.sgd_whole_convergence}
\begin{proof}
With Lemma 12, Lemma 13, and Lemma 14, we have
\begin{align}
    &\mathbb{E}\|\bm{x}_{r+1} - \bm{x}^*\|^2 - \mathbb{E}\|\bm{x}_{r} - \bm{x}^*\|^2 \nonumber \\
    &\leq\left(
    -\frac{3\eta_x^{(s)}K_rL_{Fc}^x\gamma_f}{2(L_{Fc}^x+\gamma_f)}
    + 2\eta_x^{(s)}K_r\tilde{L}_{Fc}^x 
    \right)\|\bm{x}_r-\bm{x}^*\|^2+\frac{4\eta_x^{(s)}K_r\hat{L}_f^2(L_{Fc}^x+\gamma_f)}{L_{Fc}^x\gamma_f}\|\bm{y}_r-\bm{y}^*(\bm{x}_r)\|^2 - \frac{2\eta_x^{(s)}K_r}{L_{Fc}^x+\gamma_f} \mathcal{V}_F
    \nonumber \\
    &\quad + \left(\eta_x^{(s)}\right)^2\|\nabla_{x,r}^{(s)}\|^2 
    + \frac{4\eta_x^{(s)}K_r(L_{Fc}^x+\gamma_f)}{L_{Fc}^x\gamma_f}(\sigma_f^2 + \sigma_r^2 + 2\iota^2\sigma_{gg}^2),
\end{align}

\begin{align}
    &\mathbb{E}\|\bm{y}_{r+1} - \bm{y}^*(\bm{x}_{r+1})\|^2 - \mathbb{E}\|\bm{y}_{r} - \bm{y}^*(\bm{x}_{r})\|^2\nonumber \\
    &\leq 
    \left(
    \frac{C_g^{xy} \eta_x^{(s)}}{\gamma_g}- \gamma_g\eta_y^{(s)}K_r-C_g^{xy}\eta_x^{(s)}\eta_y^{(s)}K_r
    \right)
    \mathbb{E}\|\bm{y}_{r}-\bm{y}^*(\bm{x}_r)\|^2
    \nonumber \\
    &\quad +\left(
    1+\frac{C_g^{xy} \eta_x^{(s)}}{\gamma_g}
    \right)\left(
    \frac{8\eta_y^{(s)}}{\gamma_g}(L^y_g\varepsilon_d+L_g^x)^2\right)\sum_{i\in C_r}\tilde{p}_i\sum_{k=0}^{K_r-1}\mathbb{E}\|\bm{x}_{i,r,k}-\bm{x}_r\|^2
    \nonumber \\
    &\quad +
    \left(
    1+\frac{C_g^{xy} \eta_x^{(s)}}{\gamma_g}
    \right)\left(\frac{8\eta_y^{(s)}}{\gamma_g}(L_g^y)^2\right)\sum_{i\in C_r}\tilde{p}_i\sum_{k=0}^{K_r-1}\mathbb{E}\|\bm{y}_{i,r,k}-\bm{y}_r\|^2
    \nonumber \\
    &\quad +\left(
    1+\frac{C_g^{xy} \eta_x^{(s)}}{\gamma_g}
    \right)\left(\eta_y^{(s)}\right)^2 \mathbb{E}\|\nabla_{y,r}^{(s)}\|^2
    +\left(
    \bar{L}_y^2 \left(\eta_x^{(s)}\right)^2+\frac{C_g^{xy} \eta_x^{(s)}}{\gamma_g}
    \right)
    \mathbb{E}\|\nabla_{x,r}^{(s)}\|^2+\left(
    1+\frac{C_g^{xy} \eta_x^{(s)}}{\gamma_g}
    \right)\frac{2\eta_y^{(s)}K_r}{\gamma_g}\sigma_g^2,
\end{align}
and
\begin{align}
    &\mathbb{E}\|\bm{v}_{r+1}-\bm{v}^*(\bm{x}_{r+1})\|^2 - \mathbb{E}\|\bm{v}_r-\bm{v}^*(\bm{x}_r)\|^2\nonumber \\
    &\leq \left(\bar{L}_v\eta_x^{(s)}+
    (1 + \bar{L}_v\eta_x^{(s)})\left(
    \eta_v^{(s)}(K_r \tilde{L}_{lc}^x- K_r \gamma_g) - \frac{2\eta_v^{(s)}K_rC_g^{yy}\gamma_g}{C_g^{yy}+\gamma_g}
    \right)\right)\mathbb{E}\|\bm{v}_r-\bm{v}^*(\bm{x}_r)\|^2
    \nonumber \\
    &\quad 
    + (1 + \bar{L}_v\eta_x^{(s)})\eta_v^{(s)}K_r \tilde{L}_{lc}^x\|\bm{x}_r-\bm{x}^*\|^2
    +(1 + \bar{L}_v\eta_x^{(s)})(\eta_v^{(s)})^2\mathbb{E}\|\nabla_{v,r}^{(s)}\|^2 
    +\left(\bar{L}_v\eta_x^{(s)}+\bar{L}_v^2 (\eta_x^{(s)})^2\right)\mathbb{E} \|\nabla_{x,r}^{(s)}\|^2
    \nonumber \\
    &\quad + (1 + \bar{L}_v\eta_x^{(s)})\frac{2\eta_v^{(s)}}{\gamma_g}\sum_{i\in C_r}\tilde{p}_i\sum_{k=0}^{K_r-1}\Delta_{l,i,r,k}
    - (1 + \bar{L}_v\eta_x^{(s)})\frac{2\eta_v^{(s)}K_r}{C_g^{yy}+\gamma_g}\mathcal{V}_l
    .
\end{align}

Combining the above inequalities, we get

\begin{align}
&\mathbb{E}\|\bm{x}_{r+1} - \bm{x}^*\|^2 + \mathbb{E}\|\bm{y}_{r+1} - \bm{y}^*(\bm{x}_{r+1})\|^2
    + \mathbb{E}\|\bm{v}_{r+1} - \bm{v}^*(\bm{x}_{r+1})\|^2\nonumber \\
    &\quad 
    -\underbrace{\left(
    \mathbb{E}\|\bm{x}_{r} - \bm{x}^*\|^2 + \mathbb{E}\|\bm{y}_{r} - \bm{y}^*(\bm{x}_{r})\|^2
    + \mathbb{E}\|\bm{v}_{r} - \bm{v}^*(\bm{x}_{r})\|^2
    \right)}_{:=\mathcal{P}_r}\nonumber \\
    &\leq X_1 \mathbb{E}\|\nabla_{x,r}^{(s)}\|^2
    + Y_1 \mathbb{E}\|\nabla_{y,r}^{(s)}\|^2
    + V_1 \mathbb{E}\|\nabla_{v,r}^{(s)}\|^2
    + X_2 \sum_{i\in C_r}\tilde{p}_i\sum_{k=0}^{K_r-1}\mathbb{E}\|\bm{x}_{i,r,k}-\bm{x}_r\|^2\nonumber \\
    &\quad + Y_2 \sum_{i\in C_r}\tilde{p}_i\sum_{k=0}^{K_r-1}\mathbb{E}\|\bm{y}_{i,r,k}-\bm{y}_r\|^2\nonumber \\
    &\quad +\left(
    -\frac{3\eta_x^{(s)}K_rL_{Fc}^x\gamma_f}{2(L_{Fc}^x+\gamma_f)}
    + 2\eta_x^{(s)}K_r\tilde{L}_{Fc}^x +(1 + \bar{L}_v\eta_x^{(s)})\eta_v^{(s)}K_r \tilde{L}_{lc}^x
    \right)\mathbb{E}\|\bm{x}_r-\bm{x}^*\|^2\nonumber \\
    &\quad + \left(
    \frac{C_g^{xy} \eta_x^{(s)}}{\gamma_g}- \gamma_g\eta_y^{(s)}K_r-C_g^{xy}\eta_x^{(s)}\eta_y^{(s)}K_r
    \right)
    \mathbb{E}\|\bm{y}_{r}-\bm{y}^*(\bm{x}_r)\|^2
    \nonumber \\
    &\quad + \left(\bar{L}_v\eta_x^{(s)}+
    (1 + \bar{L}_v\eta_x^{(s)})\left(
    \eta_v^{(s)}(K_r \tilde{L}_{lc}^x- K_r \gamma_g) - \frac{2\eta_v^{(s)}K_rC_g^{yy}\gamma_g}{C_g^{yy}+\gamma_g}
    \right)\right)
    \mathbb{E}\|\bm{v}_r-\bm{v}^*(\bm{x}_{r})\|^2\nonumber \\
    &\quad +(1 + \bar{L}_v\eta_x^{(s)})\frac{2\eta_v^{(s)}}{\gamma_g}\sum_{i\in C_r}\tilde{p}_i\sum_{k=0}^{K_r-1}\Delta_{l,i,r,k}
    \nonumber \\
    &\quad 
    - \frac{2\eta_x^{(s)}K_r}{L_{Fc}^x+\gamma_f} \mathcal{V}_F - (1 + \bar{L}_v\eta_x^{(s)})\frac{2\eta_v^{(s)}K_r}{C_g^{yy}+\gamma_g}\mathcal{V}_l
    \nonumber \\
    &\quad 
    + \frac{4\eta_x^{(s)}K_r(L_{Fc}^x+\gamma_f)}{L_{Fc}^x\gamma_f}(\sigma_f^2 + \sigma_r^2 + 2\iota^2\sigma_{gg}^2)+ \left(
    1+\frac{C_g^{xy} \eta_x^{(s)}}{\gamma_g}
    \right)\frac{2\eta_y^{(s)}K_r}{\gamma_g}\sigma_g^2,\label{sgd_conv_eq}
\end{align}
where we define
\begin{align}
    &X_1 = (\eta_x^{(s)})^2
    + \bar{L}_y^2 \left(\eta_x^{(s)}\right)^2+\frac{C_g^{xy} \eta_x^{(s)}}{\gamma_g}
    + \bar{L}_v\eta_x^{(s)}+\bar{L}_v^2 (\eta_x^{(s)})^2,
    X_2 = \left(
    1+\frac{C_g^{xy} \eta_x^{(s)}}{\gamma_g}
    \right)\left(
    \frac{8\eta_y^{(s)}}{\gamma_g}(L^y_g\varepsilon_d+L_g^x)^2\right),\\
    &Y_1 = \left(
    1+\frac{C_g^{xy} \eta_x^{(s)}}{\gamma_g}
    \right)\left(\eta_y^{(s)}\right)^2,
    Y_2 = \left(
    1+\frac{C_g^{xy} \eta_x^{(s)}}{\gamma_g}
    \right)\left(\frac{8\eta_y^{(s)}}{\gamma_g}(L_g^y)^2\right),
    V_1 = (1 + \bar{L}_v\eta_x^{(s)})(\eta_v^{(s)})^2,
\end{align}

We further analyze the term $X_1 \mathbb{E}\|\nabla_{x,r}^{(s)}\|^2$ based on Lemma 11:

\begin{align}
    &X_1 \mathbb{E}\|\nabla_{x,r}^{(s)}\|^2\nonumber \\
    &\leq \frac{8M}{C}p_{\text{max}} K_rX_1
    \left(
    \sigma_f^2+\sigma_r^2+2\iota^2\sigma_{gg}^2
    \right)
    \nonumber \\
    &\quad +2X_1X_3p_{\text{max}}\left(
    \frac{M}{C}\left(\frac{C-1}{M-1}\right)+\frac{M}{C}\left(\frac{M-C}{M-1}\right)
    \right)
    \nonumber \\
    &\quad +\left(
    6X_1p_{\text{max}}(L_{Fc}^y)^2+2X_1X_3p_{\text{max}}
    \right)
    \left(
    \frac{M}{C}\left(\frac{C-1}{M-1}\right)+\frac{M}{C}\left(\frac{M-C}{M-1}\right)
    \right)
    \|\bm{y}_r-\bm{y}^*(\bm{x}_r)\|^2
    \nonumber \\
    &\quad +
    \left(
    6X_1p_{\text{max}}(\tilde{L}_{Fc}^x)^2+2X_1X_3p_{\text{max}}
    \right)
    \left(
    \frac{M}{C}\left(\frac{C-1}{M-1}\right)+\frac{M}{C}\left(\frac{M-C}{M-1}\right)
    \right)
    \|\bm{x}_r-\bm{x}^*\|^2\nonumber \\
    &\quad +
    (
    6X_1p_{\text{max}} + 2X_1X_3p_{\text{max}}
    )
    \left(
    \frac{M}{C}\left(\frac{C-1}{M-1}\right)+\frac{M}{C}\left(\frac{M-C}{M-1}\right)
    \right)
    \mathcal{V}_F\nonumber \\
    &\quad + 2X_1X_3p_{\text{max}}\left(
    \frac{M}{C}\left(\frac{C-1}{M-1}\right)+\frac{M}{C}\left(\frac{M-C}{M-1}\right)
    \right) \mathcal{V}_l,
\end{align}
where
\begin{align}
    &X_3 = \left(
        6(L_f^x)^2+24\iota^2(L_{gxy}^y\varepsilon_d+L_{gxy}^x)^2
        \right)p_{x,7}
    + \left(
        6(L_f^\xi\varepsilon_c+\bar{L}_f^x)^2+24\iota^2(L_{gxy}^y)^2
        \right)p_{y,4}
    + 3(C_g^{xy})^2 p_{v,7}.
\end{align}

We also present the upper bound of the $Y_1 \mathbb{E}\|\nabla_{y,r}^{(s)}\|^2$ by Lemma 11:
\begin{align}
    &Y_1 \mathbb{E}\|\nabla_{y,r}^{(s)}\|^2\nonumber \\
    &\leq \frac{2M}{C}Y_1p_{\text{max}}K_r\sigma_g^2
    +Y_1p_{\text{max}}\left(
    \frac{4M}{C}\left(\frac{M-C}{M-1}\right)
    +\frac{6M}{C}\left(\frac{C-1}{M-1}\right)
    \right)
    \sum_{i=1}^{M}p_i \sum_{k=0}^{K_r-1}
    \Delta_{g,i,r,k}
    \nonumber \\
    &\quad +Y_1\left(
    \frac{4M}{C}\left(\frac{M-C}{M-1}\right)p_{\text{max}} 
    (L_g^y)^2
    + \frac{6M}{C}\left(\frac{C-1}{M-1}\right)(L_g^y)^2
    \right)
    \mathbb{E}\|\bm{y}_r-\bm{y}^*(\bm{x}_r)\|^2.\label{sgd_conv_eq1}
\end{align}

For the sake of simplicity, we define 
\begin{align}
    &Y_3 = 2(L^y_g\varepsilon_d+L_g^x)^2 p_{x,7}+ 2(L_g^y)^2 p_{y,4},
    Y_4 = 
    2(L^y_g\varepsilon_d+L_g^x)^2 p_{x,8}+ 2(L_g^y)^2 p_{y,5},\\
    &Y_5 = 2(L^y_g\varepsilon_d+L_g^x)^2 p_{x,9}+ 2(L_g^y)^2 p_{y,6},
    Y_6 = 2(L^y_g\varepsilon_d+L_g^x)^2 p_{x,10}+ 2(L_g^y)^2 p_{y,7},\\
    &Y_7 = 2(L^y_g\varepsilon_d+L_g^x)^2 p_{x,11}+ 2(L_g^y)^2 p_{y,8}.
\end{align}

Therefore, inequality (\ref{sgd_conv_eq1}) can be rewritten as
\begin{align}
    &Y_1 \mathbb{E}\|\nabla_{y,r}^{(s)}\|^2\nonumber \\
    &\leq \frac{2M}{C}Y_1p_{\text{max}}K_r\sigma_g^2
    +Y_1Y_3p_{\text{max}}\left(
    \frac{4M}{C}\left(\frac{M-C}{M-1}\right)
    +\frac{6M}{C}\left(\frac{C-1}{M-1}\right)
    \right)
    \nonumber \\
    &\quad + \Bigg(
    Y_1Y_4p_{\text{max}}\left(
    \frac{4M}{C}\left(\frac{M-C}{M-1}\right)
    +\frac{6M}{C}\left(\frac{C-1}{M-1}\right)
    \right)
    \nonumber \\
    &\quad +Y_1\left(
    \frac{4M}{C}\left(\frac{M-C}{M-1}\right)p_{\text{max}} 
    (L_g^y)^2
    + \frac{6M}{C}\left(\frac{C-1}{M-1}\right)(L_g^y)^2
    \right)
    \Bigg)
    \mathbb{E}\|\bm{y}_r-\bm{y}^*(\bm{x}_r)\|^2 \nonumber\\
    &\quad + Y_1Y_5p_{\text{max}}\left(
    \frac{4M}{C}\left(\frac{M-C}{M-1}\right)
    +\frac{6M}{C}\left(\frac{C-1}{M-1}\right)
    \right)  \mathbb{E}\|\bm{x}_r-\bm{x}^*\|^2 \nonumber\\
    &\quad + Y_1Y_6p_{\text{max}}\left(
    \frac{4M}{C}\left(\frac{M-C}{M-1}\right)
    +\frac{6M}{C}\left(\frac{C-1}{M-1}\right)
    \right)\mathcal{V}_F\nonumber \\
    &\quad + Y_1Y_7p_{\text{max}}\left(
    \frac{4M}{C}\left(\frac{M-C}{M-1}\right)
    +\frac{6M}{C}\left(\frac{C-1}{M-1}\right)
    \right)\mathcal{V}_l.
\end{align}
Similarly, we have
\begin{align}
    &V_1 \mathbb{E}\|\nabla_{v,r}^{(s)}\|^2\nonumber \\
    &\leq \frac{8M}{C}p_{\text{max}} K_rV_1
    \left(
    \sigma_f^2+\sigma_r^2+2\iota^2\sigma_{gg}^2
    \right)
    \nonumber \\
    &\quad +
    4p_{\text{max}}V_1V_2\left(
    \frac{M}{C}\left(\frac{M-C}{M-1}\right)
    +\frac{M}{C}\left(\frac{C-1}{M-1}\right)
    \right)
    \nonumber \\
    &\quad +
    (12p_{\text{max}}V_1(L_{l_c}^y)^2 + 4p_{\text{max}}V_1V_3
    )
    \left(
    \frac{M}{C}\left(\frac{M-C}{M-1}\right)
    +\frac{M}{C}\left(\frac{C-1}{M-1}\right)
    \right)
    \|\bm{y}_r-\bm{y}^*(\bm{x}_r)\|^2
    \nonumber \\
    &\quad +
    (
    12p_{\text{max}}V_1(\tilde{L}_{l_c}^x)^2+4p_{\text{max}}V_1V_4
    )
    \left(
    \frac{M}{C}\left(\frac{M-C}{M-1}\right)
    +\frac{M}{C}\left(\frac{C-1}{M-1}\right)
    \right)
    \|\bm{x}_r-\bm{x}^*\|^2
    \nonumber \\
    &\quad +4p_{\text{max}}V_1V_5\left(
    \frac{M}{C}\left(\frac{M-C}{M-1}\right)
    +\frac{M}{C}\left(\frac{C-1}{M-1}\right)
    \right)\mathcal{V}_F\nonumber \\
    &\quad +
    (
    12p_{\text{max}}V_1 + 4p_{\text{max}}V_1V_6
    )
    \left(
    \frac{M}{C}\left(\frac{M-C}{M-1}\right)
    +\frac{M}{C}\left(\frac{C-1}{M-1}\right)
    \right)
    \mathcal{V}_l,
\end{align}
where we define
\begin{align}
    &V_2 = \left(
        6(L_f^y)^2+24\iota^2(L_{gyy}^y\varepsilon_d+L_{gyy}^x)^2
        \right)p_{x,7}
        + \left(
        6(L_f^\xi\varepsilon_c+\bar{L}_f^y)^2+24\iota^2(L_{gyy}^y)^2
        \right)p_{y,4}
        + 3(C_g^{yy})^2 p_{v,7},\\
    &V_3 = \left(
        6(L_f^y)^2+24\iota^2(L_{gyy}^y\varepsilon_d+L_{gyy}^x)^2
        \right)p_{x,8}
        + \left(
        6(L_f^\xi\varepsilon_c+\bar{L}_f^y)^2+24\iota^2(L_{gyy}^y)^2
        \right)p_{y,5}
        + 3(C_g^{yy})^2 p_{v,8},\\
    &V_4 = \left(
        6(L_f^y)^2+24\iota^2(L_{gyy}^y\varepsilon_d+L_{gyy}^x)^2
        \right)p_{x,9}
        + \left(
        6(L_f^\xi\varepsilon_c+\bar{L}_f^y)^2+24\iota^2(L_{gyy}^y)^2
        \right)p_{y,6}
        + 3(C_g^{yy})^2 p_{v,9},\\
    &V_5 = \left(
        6(L_f^y)^2+24\iota^2(L_{gyy}^y\varepsilon_d+L_{gyy}^x)^2
        \right)p_{x,10}
        + \left(
        6(L_f^\xi\varepsilon_c+\bar{L}_f^y)^2+24\iota^2(L_{gyy}^y)^2
        \right)p_{y,7}
        + 3(C_g^{yy})^2 p_{v,10},\\
    &V_6 = \left(
        6(L_f^y)^2+24\iota^2(L_{gyy}^y\varepsilon_d+L_{gyy}^x)^2
        \right)p_{x,11}
        + \left(
        6(L_f^\xi\varepsilon_c+\bar{L}_f^y)^2+24\iota^2(L_{gyy}^y)^2
        \right)p_{y,8}
        + 3(C_g^{yy})^2 p_{v,11}.
\end{align}
With the constants we define, we then obtain the upper bounds of other terms in inequality (\ref{sgd_conv_eq}) as
\begin{align}
    &X_2 \sum_{i\in C_r}\tilde{p}_i\sum_{k=0}^{K_r-1}\mathbb{E}\|\bm{x}_{i,r,k}-\bm{x}_r\|^2\nonumber \\
    &\leq X_2p_{x,7}+X_2p_{x,8}\|\bm{y}_r-\bm{y}^*(\bm{x}_r)\|^2 +X_2p_{x,9}\|\bm{x}_r-\bm{x}^*\|^2+X_2p_{x,10}\mathcal{V}_F+X_2p_{x,11}\mathcal{V}_l,
\end{align}
\begin{align}
    &Y_2 \sum_{i\in C_r}\tilde{p}_i\sum_{k=0}^{K_r-1}\mathbb{E}\|\bm{y}_{i,r,k}-\bm{y}_r\|^2\nonumber\\
    &\leq Y_2p_{y,4}+Y_2p_{y,5}\|\bm{y}_r-\bm{y}^*(\bm{x}_r)\|^2 +Y_2p_{y,6}\|\bm{x}_r-\bm{x}^*\|^2+Y_2p_{y,7}\mathcal{V}_F+Y_2p_{y,8}\mathcal{V}_l,
\end{align}

and 
\begin{align}
    &\left(
    1+\bar{L}_v\eta_x^{(s)}
    \right)\frac{2\eta_v^{(s)}}{\gamma_g}\sum_{i\in C_r}\tilde{p}_i\sum_{k=0}^{K_r-1}\Delta_{l,i,r,k}
    \nonumber \\
    &\leq V_2\left(
    1+\bar{L}_v\eta_x^{(s)}
    \right)\frac{2\eta_v^{(s)}}{\gamma_g} + V_3\left(
    1+\bar{L}_v\eta_x^{(s)}
    \right)\frac{2\eta_v^{(s)}}{\gamma_g} \|\bm{y}_r-\bm{y}^*(\bm{x}_r)\|^2 + V_4\left(
    1+\bar{L}_v\eta_x^{(s)}
    \right)\frac{2\eta_v^{(s)}}{\gamma_g} \|\bm{x}_r-\bm{x}^*\|^2 
    \nonumber \\
    &\quad + V_5\left(
    1+\bar{L}_v\eta_x^{(s)}
    \right)\frac{2\eta_v^{(s)}}{\gamma_g}\mathcal{V}_F 
    + V_6\left(
    1+\bar{L}_v\eta_x^{(s)}
    \right)\frac{2\eta_v^{(s)}}{\gamma_g} \mathcal{V}_l.
\end{align}

Due to the extremely complex expression of inequality (\ref{sgd_conv_eq}), we will separately analyze the coefficients of each term. First, here are the constant terms in inequality (\ref{sgd_conv_eq}):
\begin{align}
    &\frac{8M}{C}p_{\text{max}} K_rX_1
    \left(
    \sigma_f^2+\sigma_r^2+2\iota^2\sigma_{gg}^2
    \right)
    +2X_1X_3p_{\text{max}}\left(
    \frac{M}{C}\left(\frac{C-1}{M-1}\right)+\frac{M}{C}\left(\frac{M-C}{M-1}\right)
    \right)
    \nonumber \\
    &+\frac{2M}{C}Y_1p_{\text{max}}K_r\sigma_g^2
    +Y_1Y_3p_{\text{max}}\left(
    \frac{4M}{C}\left(\frac{M-C}{M-1}\right)
    +\frac{6M}{C}\left(\frac{C-1}{M-1}\right)
    \right)
    \nonumber \\
    &+\frac{8M}{C}p_{\text{max}} K_rV_1
    \left(
    \sigma_f^2+\sigma_r^2+2\iota^2\sigma_{gg}^2
    \right)
    + 4p_{\text{max}}V_1V_2\left(
    \frac{M}{C}\left(\frac{M-C}{M-1}\right)
    +\frac{M}{C}\left(\frac{C-1}{M-1}\right)
    \right)
    \nonumber \\
    &+X_2p_{x,7} + Y_2p_{y,4} +  V_2\left(
    1+\bar{L}_v\eta_x^{(s)}
    \right)\frac{2\eta_v^{(s)}}{\gamma_g}+ \frac{4\eta_x^{(s)}K_r(L_{Fc}^x+\gamma_f)}{L_{Fc}^x\gamma_f}(\sigma_f^2 + \sigma_r^2 + 2\iota^2\sigma_{gg}^2)
    \nonumber \\
    &\quad + \left(
    1+\frac{C_g^{xy} \eta_x^{(s)}}{\gamma_g}
    \right)\frac{2\eta_y^{(s)}K_r}{\gamma_g}\sigma_g^2.\nonumber 
\end{align}

The coefficients of $\mathbb{E}\|\bm{x}_r - \bm{x}^*\|^2$ are:
\begin{align}
    &\left(
    -\frac{3\eta_x^{(s)}K_rL_{Fc}^x\gamma_f}{2(L_{Fc}^x+\gamma_f)}
    + 2\eta_x^{(s)}K_r\tilde{L}_{Fc}^x +(1 + \bar{L}_v\eta_x^{(s)})\eta_v^{(s)}K_r \tilde{L}_{lc}^x
    \right)
    \nonumber \\
    &+ \left(
    6X_1p_{\text{max}}(\tilde{L}_{Fc}^x)^2+2X_1X_3p_{\text{max}}
    \right)
    \left(
    \frac{M}{C}\left(\frac{C-1}{M-1}\right)+\frac{M}{C}\left(\frac{M-C}{M-1}\right)
    \right)
    \nonumber \\
    &+Y_1Y_5p_{\text{max}}\left(
    \frac{4M}{C}\left(\frac{M-C}{M-1}\right)
    +\frac{6M}{C}\left(\frac{C-1}{M-1}\right)
    \right)
    \nonumber \\
    &
    + (
    12p_{\text{max}}V_1(\tilde{L}_{l_c}^x)^2+4p_{\text{max}}V_1V_4
    )
    \left(
    \frac{M}{C}\left(\frac{M-C}{M-1}\right)
    +\frac{M}{C}\left(\frac{C-1}{M-1}\right)
    \right)
    \nonumber \\
    &+X_2p_{x,9} + Y_2p_{y,6} +  V_4\left(
    1+\bar{L}_v\eta_x^{(s)}
    \right)\frac{2\eta_v^{(s)}}{\gamma_g}.\nonumber 
\end{align}
To make them negative, we first define $S = \frac{M}{C}$, 
\begin{equation}
A_x  =  K_r\!\left(\frac{3L_{Fc}^x\gamma_f}{2\,(L_{Fc}^x+\gamma_f)}-2\,\tilde L_{Fc}^x\right), 
\qquad
b_{x,1}  =  \frac{C_g^{xy}}{\gamma_g}+\bar L_v,
\qquad
b_{x,2}  =  1+\bar L_y^2+\bar L_v^2.
\end{equation}

\begin{equation}
L_{\mathrm{lin}}^x  =  S\,p_{\max}\,b_{x,1}\!\left(6(\tilde L_{Fc}^x)^2+2X_3\right), 
\qquad
C_{x,2}  =  S\,p_{\max}\,b_{x,2}\!\left(6(\tilde L_{Fc}^x)^2+2X_3\right).
\end{equation}

\begin{equation}
B_x^y  =  \left(1+\frac{C_g^{xy}}{\gamma_g}\right)\frac{8}{\gamma_g}\!\left((L_g^y\varepsilon_d+L_g^x)^2\,p_{x,9}+(L_g^y)^2\,p_{y,6}\right),
\qquad
B_x^v  =  (1+\bar L_v)\!\left(K_r\tilde L_{l_c}^x+\frac{2}{\gamma_g}V_4\right).
\end{equation}

\begin{equation}
T  =  S\cdot \frac{4M+2C-6}{M-1},
\qquad
K_x^{yy}  =  p_{\max}\,Y_5\,T\left(1+\frac{C_g^{xy}}{\gamma_g}\right),
\qquad
K_x^{vv}  =  S\,p_{\max}\,(1+\bar L_v)\left(12(\tilde L_{l_c}^x)^2+4V_4\right).
\end{equation}
We need to choose learning steps and $\varepsilon_c$ and $\varepsilon_d$ to satisfy that
\begin{equation}
 \mu_x  =  A_x - L_{\mathrm{lin}}^x  =  
K_r\!\left(\frac{3L_{Fc}^x\gamma_f}{2\,(L_{Fc}^x+\gamma_f)}-2\,\tilde L_{Fc}^x\right)
 - 
\frac{M}{C}\,p_{\max}\!\left(\frac{C_g^{xy}}{\gamma_g}+\bar L_v\right)\!\left(6(\tilde L_{Fc}^x)^2+2X_3\right) > 0 .
\end{equation}
Therefore, by choosing
\begin{equation}
 \rho_{x,y}  =  \frac{\mu_x}{8\,B_x^y} ,\qquad \rho_{x,v}  =  \frac{\mu_x}{8\,B_x^v} ,
 \qquad  \eta_y^{(s)}  \le  \rho_{x,y}\eta_x^{(s)} , 
\qquad
 \eta_v^{(s)}  \le  \rho_{x,v}\eta_x^{(s)} ,
\end{equation}
and
\begin{equation}
 \eta_x^{(s)}  \le 
\frac{\mu_x}{
4\,C_{x,2} + 2\,K_x^{yy}\,\rho_{x,y}^2 + 2\,K_x^{vv}\,\rho_{x,v}^2
} ,
\end{equation}
we can make
\begin{equation}
\text{coefficients of}\,\, \mathbb{E}\|\bm{x}_r - \bm{x}^*\|^2 \le  -\,\frac{\mu_x^2}{
8\,C_{x,2} + 4\,K_x^{yy}\,\rho_{x,y}^2 + 4\,K_x^{vv}\,\rho_{x,v}^2
}  <  0,
\end{equation}
with the following conditions of $\varepsilon_c$ and $\varepsilon_d$:
\begin{align}
    X_3(\varepsilon_c,\varepsilon_d)
    <
    \frac{CK_r}{2Mp_{\text{max}}\left(\frac{C_g^{xy}}{\gamma_g}+\bar{L}_v\right)}\left(\frac{3L_{Fc}^x\gamma_f}{2\,(L_{Fc}^x+\gamma_f)}-2\,\tilde L_{Fc}^x\right)-3\left(\tilde{L}^x_{F_c}\right)^2,
\end{align}
and
\begin{align}
    \tilde{L}^x_{F_c}<\frac{3L^x_{F_c}\gamma_f}{4(L^x_{F_c}+\gamma_f)}.\label{sgd_conv_eq_cond1}
\end{align}
The second condition can be implied by the first condition, therefore we just need to choose sufficiently small learning steps and ($\varepsilon_c$ and $\varepsilon_d$) to satisfy the first condition.


The coefficients of $\mathbb{E}\|\bm{y}_r - \bm{y}^*(\bm{x}_r)\|^2$ are:
\begin{align}
    &\left(
    \frac{C_g^{xy} \eta_x^{(s)}}{\gamma_g}- \gamma_g\eta_y^{(s)}K_r-C_g^{xy}\eta_x^{(s)}\eta_y^{(s)}K_r
    \right) + \left(
    6X_1p_{\text{max}}(L_{Fc}^y)^2+2X_1X_3p_{\text{max}}
    \right)
    \left(
    \frac{M}{C}\left(\frac{C-1}{M-1}\right)+\frac{M}{C}\left(\frac{M-C}{M-1}\right)
    \right)
    \nonumber \\
    &+\Bigg(
    Y_1Y_4p_{\text{max}}\left(
    \frac{4M}{C}\left(\frac{M-C}{M-1}\right)
    +\frac{6M}{C}\left(\frac{C-1}{M-1}\right)
    \right)
    + Y_1\left(
    \frac{4M}{C}\left(\frac{M-C}{M-1}\right)p_{\text{max}} 
    (L_g^y)^2
    + \frac{6M}{C}\left(\frac{C-1}{M-1}\right)(L_g^y)^2
    \right)
    \Bigg) 
    \nonumber \\
    &+ (12p_{\text{max}}V_1(L_{l_c}^y)^2 + 4p_{\text{max}}V_1V_3
    )
    \left(
    \frac{M}{C}\left(\frac{M-C}{M-1}\right)
    +\frac{M}{C}\left(\frac{C-1}{M-1}\right)
    \right)
    \nonumber \\
    &+X_2p_{x,8} + Y_2p_{y,5} + V_3\left(
    1+\bar{L}_v\eta_x^{(s)}
    \right)\frac{2\eta_v^{(s)}}{\gamma_g}.\nonumber
\end{align}
For the sake of simplicity, we define
\begin{align}
    L_{\text{lin}}^y = Sp_{\text{max}}b_{x,1}\left(
    6(L_{F_c}^y)^2+2X_3
    \right),
    C_{y,2} = Sp_{\text{max}}b_{x,2}\left(
    6(L_{F_c}^y)^2+2X_3
    \right),
\end{align}
\begin{align}
    K_y^{yy} = \left(\
    1+\frac{C_g^{xy}}{\gamma_g}
    \right)\left[
    p_{\text{max}}Y_4T+S\left(
    \frac{4(M-C)}{M-1}p_{\text{max}}(L_g^y)^2+\frac{6(C-1)}{M-1}(L_g^y)^2
    \right)
    \right],
\end{align}
\begin{align}
    B_{y,\text{lin}}^y = 
    \frac{8}{\gamma_g}\left(\
    1+\frac{C_g^{xy}}{\gamma_g}
    \right)
    \left(
    (L_g^y\varepsilon_d + L_g^x)^2p_{x,8}+(L_g^y)^2 p_{y,5}
    \right),
    B_{y,\text{lin}}^v = (1+\bar{L}_v)\frac{2V_3}{\gamma_g},
\end{align}
\begin{align}
    K_y^{vv} = Sp_{\text{max}}(1+\bar{L}_v) (12(L_{lc}^y)^2+4V_3),
    \mu_y = (\gamma_gK_r-B_{y,\text{lin}}^y)\rho_y - \frac{C_g^{xy}}{\gamma_g} - L_{\text{lin}}^y - B_{y,\text{lin}}^v\rho_v>0,
\end{align}
\begin{align}
    \rho_{y,y} = \frac{\frac{C_g^{xy}}{\gamma_g}+L_{\text{lin}}^y}{\gamma_gK_r - B_{y,\text{lin}}^y} + \frac{1}{2(\gamma_gK_r-B_{y,\text{lin}}^y)},
    \rho_{y,v} = \frac{1}{4B_{y,\text{lin}}^v}.
\end{align}
To make the coefficients of $\mathbb{E}\|\bm{y}_r - \bm{y}^*(\bm{x}_r)\|^2$ negative, we choose learning steps as
\begin{align}
    \eta_y^{(s)}\leq \rho_{y,y}\eta_x^{(s)},
    \eta_v^{(s)}\leq \rho_{y,v}\eta_x^{(s)},
    \eta_x^{(s)}\leq \frac{1}{16C_{y,2}+8K_y^{yy}\rho_{y,y}^2+8K_y^{vv}\rho_{y,v}^2},
\end{align}
and then
\begin{align}
    \text{coefficients of} \,\,\mathbb{E}\|\bm{y}_r - \bm{y}^*(\bm{x}_r)\|^2\leq -\frac{1}{64C_{y,2}+32K_y^{yy}\rho_{y,y}^2+32K_y^{vv}\rho_{y,v}^2}<0,
\end{align}
with the conditions of $\varepsilon_c$ and $\varepsilon_d$:
\begin{align}
    (L_g^y\varepsilon_d + L_g^x)^2p_{x,8}+(L_g^y)^2 p_{y,5}<\frac{\gamma_g^2 K_r}{8\left(\
    1+\frac{C_g^{xy}}{\gamma_g}
    \right)},\label{sgd_conv_eq_cond2}
\end{align}
where $p_{x,8},p_{y,5}$ are functions of $\varepsilon_c,\varepsilon_d$.

The coefficients of $\mathbb{E}\|\bm{v}_r - \bm{v}^*(\bm{x}_r)\|^2$ are:
\begin{align}
    \bar{L}_v\eta_x^{(s)}+
    (1 + \bar{L}_v\eta_x^{(s)})\left(
    \eta_v^{(s)}(K_r \tilde{L}_{lc}^x- K_r \gamma_g) - \frac{2\eta_v^{(s)}K_rC_g^{yy}\gamma_g}{C_g^{yy}+\gamma_g}
    \right),\nonumber
\end{align}
and it's easy to know that
\begin{align}
    &\bar{L}_v\eta_x^{(s)}+
    (1 + \bar{L}_v\eta_x^{(s)})\left(
    \eta_v^{(s)}(K_r \tilde{L}_{lc}^x- K_r \gamma_g) - \frac{2\eta_v^{(s)}K_rC_g^{yy}\gamma_g}{C_g^{yy}+\gamma_g}
    \right)
    \nonumber \\
    &\leq 1 + \bar{L}_v\eta_x^{(s)}+
    (1 + \bar{L}_v\eta_x^{(s)})\left(
    \eta_v^{(s)}(K_r \tilde{L}_{lc}^x- K_r \gamma_g) - \frac{2\eta_v^{(s)}K_rC_g^{yy}\gamma_g}{C_g^{yy}+\gamma_g}
    \right)
    \nonumber \\
    &\leq (1 + \bar{L}_v\eta_x^{(s)})\left(
    1 + \eta_v^{(s)}(K_r \tilde{L}_{lc}^x- K_r \gamma_g) - \frac{2\eta_v^{(s)}K_rC_g^{yy}\gamma_g}{C_g^{yy}+\gamma_g}
    \right).
\end{align}
If we want 
\begin{align}
    &\bar{L}_v\eta_x^{(s)}+
    (1 + \bar{L}_v\eta_x^{(s)})\left(
    \eta_v^{(s)}(K_r \tilde{L}_{lc}^x- K_r \gamma_g) - \frac{2\eta_v^{(s)}K_rC_g^{yy}\gamma_g}{C_g^{yy}+\gamma_g}
    \right)< -(1 + \bar{L}_v\eta_x^{(s)})\frac{C_g^{yy}\gamma_g}{2(C_g^{yy}+\gamma_g)}<0,
\end{align}
the step size $\eta_v^{(s)}$ should be chosen as
\begin{align}
    &\eta_v^{(s)}>\frac{\bar{L}_v\eta_x^{(s)} + (1 + \bar{L}_v\eta_x^{(s)})\frac{C_g^{yy}\gamma_g}{2(C_g^{yy}+\gamma_g)}}{K_r (1+\bar{L}_v\eta_x^{(s)})\left(\gamma_g + \frac{2C_g^{yy}\gamma_g}{C_g^{yy}+\gamma_g}-\tilde{L}_{lc}^x\right)},
\end{align}
where we know that if $\eta_x^{(s)}$ is small enough, $\eta_v^{(s)} = \mathcal{O}\left(\eta_x^{(s)}\right)$.
This yields that 
\begin{align}
    \text{coefficients of}\,\, \mathbb{E}\|\bm{v}_r - \bm{v}^*(\bm{x}_r)\|^2&\leq -(1 + \bar{L}_v\eta_x^{(s)})\frac{C_g^{yy}\gamma_g}{2(C_g^{yy}+\gamma_g)}\leq -\frac{C_g^{yy}\gamma_g}{2(C_g^{yy}+\gamma_g)},
\end{align}
with the condition that
\begin{align}
    &\tilde{L}_{lc}^x<\gamma_g + \frac{2C_g^{yy}\gamma_g}{C_g^{yy}+\gamma_g}.\label{sgd_conv_eq_cond3}
\end{align}

The coefficients of $\mathcal{V}_F$ are:
\begin{align}
    &(
    6X_1p_{\text{max}} + 2X_1X_3p_{\text{max}}
    )
    \left(
    \frac{M}{C}\left(\frac{C-1}{M-1}\right)+\frac{M}{C}\left(\frac{M-C}{M-1}\right)
    \right)
    + Y_1Y_6p_{\text{max}}\left(
    \frac{4M}{C}\left(\frac{M-C}{M-1}\right)
    +\frac{6M}{C}\left(\frac{C-1}{M-1}\right)
    \right)
    \nonumber \\
    &
    + 4p_{\text{max}}V_1V_5\left(
    \frac{M}{C}\left(\frac{M-C}{M-1}\right)
    +\frac{M}{C}\left(\frac{C-1}{M-1}\right)
    \right) + X_2p_{x,10} + Y_2p_{y,7} + V_5\left(
    1+\bar{L}_v\eta_x^{(s)}
    \right)\frac{2\eta_v^{(s)}}{\gamma_g}
    - \frac{2\eta_x^{(s)}K_r}{L_{Fc}^x+\gamma_f}
    .\nonumber 
\end{align}
We define the following constants to make the expressions more concise:

\begin{align}
    &L_{\text{lin}}^F 
    = S p_{\text{max}} b_{x,1} (6 + 2X_3),
    C_{F,2} 
    = S p_{\text{max}} b_{x,2} (6 + 2X_3),
    K_F^{yy} 
    = \left(
        1+\frac{C_g^{xy}}{\gamma_g}
    \right)
    \big(
        p_{\text{max}} Y_6 T
    \big),\\
    &
    K_F^{vv} 
    = S p_{\text{max}} (1+\bar{L}_v)(4 V_5),
    B_{F,\text{lin}}^y 
    = 
    \frac{8}{\gamma_g}
    \left(
        1+\frac{C_g^{xy}}{\gamma_g}
    \right)
    \Big(
        (L_g^y\varepsilon_d + L_g^x)^2 p_{x,10}
        + (L_g^y)^2 p_{y,7}
    \Big),\\
    &
    B_{F,\text{lin}}^v 
    = (1+\bar{L}_v)\frac{2 V_5}{\gamma_g},
    D_F 
    = \frac{2K_r}{L_{F_c}^x+\gamma_f},
    \rho_{vf,v} 
    = \frac{1}{4 B_{F,\text{lin}}^v}
    = \frac{ \gamma_g}{8(1+\bar{L}_v)V_5},
\quad
    \rho_{vf,y} 
    = \frac{D_F - L_{\text{lin}}^F - \frac{1}{2}}{B_{F,\text{lin}}^y},
\end{align}
To make the coefficients of $\mathcal{V}_F$ negative, we set
\begin{align}
    &\eta_y^{(s)} \le \rho_{vf,y} \eta_x^{(s)},
\qquad
    \eta_v^{(s)} \le \rho_{vf,v} \eta_x^{(s)},
    \eta_x^{(s)}
    \le \frac{\mu_F}{4  C_{F,2} + 4K_F^{yy}\rho_{vf,y}^2 + 4K_F^{vv}\rho_{vf,v}^2}.
\end{align}
This yields that
\begin{align}
    \text{coefficients of} \,\, \mathcal{V}_F
    \le
    - \frac{\mu_F}{16 C_{F,2} + 16K_F^{yy}\rho_y^2 + 16K_F^{vv}\rho_v^2}
    < 0.
\end{align}
where 
\begin{align}
    \mu_F = D_F - L_{\text{lin}}^F - B_{F,\text{lin}}^y \rho_{vf,y} - B_{F,\text{lin}}^v \rho_{vf,v}
    ,
\end{align}
with the condition:
\begin{align}
    &X_3(\varepsilon_c,\varepsilon_d)<\frac{C}{Mp_{\text{max}}\left(
    \frac{C_g^{xy}}{\gamma_g}+\bar{L}_v
    \right)}\left(
    \frac{K_r}{L_{F_c}^x+\gamma_f} - \frac{1}{4}
    \right)-3.\label{sgd_conv_eq_cond4}
\end{align}

The coefficients of $\mathcal{V}_l$ are:
\begin{align}
    &2X_1X_3p_{\text{max}}\left(
    \frac{M}{C}\left(\frac{C-1}{M-1}\right)+\frac{M}{C}\left(\frac{M-C}{M-1}\right)
    \right)
    + Y_1Y_7p_{\text{max}}\left(
    \frac{4M}{C}\left(\frac{M-C}{M-1}\right)
    +\frac{6M}{C}\left(\frac{C-1}{M-1}\right)
    \right)
    \nonumber \\
    &+(
    12p_{\text{max}}V_1 + 4p_{\text{max}}V_1V_6
    )
    \left(
    \frac{M}{C}\left(\frac{M-C}{M-1}\right)
    +\frac{M}{C}\left(\frac{C-1}{M-1}\right)
    \right) + X_2p_{x,11} + Y_2p_{y,8} +  V_6\left(
    1+\bar{L}_v\eta_x^{(s)}
    \right)\frac{2\eta_v^{(s)}}{\gamma_g}
    \nonumber \\
    &-(1 + \bar{L}_v\eta_x^{(s)})\frac{2\eta_v^{(s)}K_r}{C_g^{yy}+\gamma_g}.\nonumber 
\end{align}
Similarly, we define

\begin{align}
    L_{\text{lin}}^{l}  =  S\,p_{\text{max}}\,(2X_3)\,b_{x,1},
    \qquad
    C_{l,2}  =  S\,p_{\text{max}}\,(2X_3)\,b_{x,2},
    \qquad
    K_{l}^{yy}  =  
    \Big(1+\frac{C_g^{xy}}{\gamma_g}\Big)\,
    p_{\text{max}}\,Y_7\,T ,
\end{align}


\begin{align}
    B_{l,\text{lin}}^{y}  = 
    \Big(1+\frac{C_g^{xy}}{\gamma_g}\Big)\frac{8}{\gamma_g}\!
    \left((L_g^y\varepsilon_d + L_g^x)^2\,p_{x,11} + (L_g^y)^2\,p_{y,8}\right),
    \qquad
    B_{l,\text{lin}}^{v}  =  (1+\bar{L}_v)\,\frac{2V_6}{\gamma_g}, 
    \qquad
    \mu_l = \frac{1}{4},
\end{align}

\begin{align}
    K_{l}^{vv}  =  S\,p_{\text{max}}\,(1+\bar{L}_v)\,\Big(12 + 4V_6\Big),
    \qquad
    D_v  =  \frac{2K_r}{C_g^{yy}+\gamma_g},
    \qquad
    \rho_{vl,v}  =  \frac{1 + 2L_{\text{lin}}^{l}}{\,2\big(D_v - B_{l,\text{lin}}^{v}\big)} ,
    \qquad
    \rho_{vl,y}  =  \frac{1}{\,4\,B_{l,\text{lin}}^{y}} .
\end{align}

To make the coefficients of $\mathcal{V}_l$ negative, we set
\begin{align}
    \eta_y^{(s)}  \le  \rho_{vl,y}\,\eta_x^{(s)}, 
    \qquad
    \eta_v^{(s)}  \le  \rho_{vl,v}\,\eta_x^{(s)},
    \qquad
    \eta_x^{(s)}  \le   \frac{\mu_l}{4C_{l,2} + 4K_{l}^{yy}\rho_{vl,y}^2 + 4K_{l}^{vv}\rho_{vl,v}^2},
\end{align}



and it yields that
\begin{align}
    \text{coefficients of }\,\, \mathcal{V}_l  \le  -\,\frac{\mu_l^2}{16C_{l,2} + 16K_{l}^{yy}\rho_{vl,y}^2 + 16K_{l}^{vv}\rho_{vl,v}^2} ,
\end{align}

with the condition that
\begin{align}
    &V_6(\varepsilon_c,\varepsilon_d)<\frac{K_r\gamma_g}{(1+\bar{L}_v)(C_g^{yy}+\gamma_g)}.\label{sgd_conv_eq_cond5}
\end{align}

Combining all conditions (\ref{sgd_conv_eq_cond1}), (\ref{sgd_conv_eq_cond2}), (\ref{sgd_conv_eq_cond3}), (\ref{sgd_conv_eq_cond4}), and (\ref{sgd_conv_eq_cond5}), we know that the complexity of them is in the order of $\mathcal{O}(\varepsilon_c + \varepsilon_c\varepsilon_d + \varepsilon_d)$. Under the case of heterogeneous sensitivities, we also can get the order of sufficient conditions' complexity as $\mathcal{O}(\bar{\varepsilon}_c + \bar{\varepsilon}_c\bar{\varepsilon}_d + \bar{\varepsilon}_d)$.

We further define
\begin{align}
    &M_x = \frac{\mu_x^2}{
8\,C_{x,2} + 4\,K_x^{yy}\,\rho_{x,y}^2 + 4\,K_x^{vv}\,\rho_{x,v}^2
},
M_y = \frac{1}{64C_{y,2}+32K_y^{yy}\rho_{y,y}^2+32K_y^{vv}\rho_{y,v}^2},
M_v = \frac{C_g^{yy}\gamma_g}{2(C_g^{yy}+\gamma_g)},
 \\
&M_F = \frac{\mu_F}{16 C_{F,2} + 16K_F^{yy}\rho_{vf,y}^2 + 16K_F^{vv}\rho_{vf,v}^2},
M_l = \frac{\mu_l^2}{16C_{l,2} + 16K_{l}^{yy}\rho_{vl,y}^2 + 16K_{l}^{vv}\rho_{vl,v}^2},
\sigma^2 = \sigma_f^2 + \sigma_r^2 + 2\iota^2\sigma_{gg}^2.
\end{align}

Based on the above analysis, we obtain that
\begin{align}
    \mathbb{E}[\mathcal{P}_{r+1}]&\leq \left(
    1- \min\left(
    M_x,
    M_y,
    M_v,
    M_F,
    M_l
    \right)
    \right)\mathbb{E}[\mathcal{P}_r]
    \nonumber \\
    &\quad 
    +\mathcal{O}\left(
    \frac{M+C}{C}
    \sqrt{\bar{K}}\left[
    \sigma^2\left(\eta_x^{(s)} + \left(\eta_x^{(s)}\right)^2 + \left(\eta_v^{(s)}\right)^2\right) + \sigma_g^2 \left(\eta_y^{(s)}+\left(\eta_y^{(s)}\right)^2\right)
    \right]
    \right)
    \nonumber \\
    &\quad 
    +\mathcal{O}\left(
    \frac{M}{C}
    \left(\eta_x^{(s)}+\left(\eta_x^{(s)}\right)^2\right)\left(
    \left(\eta_x^{(c)}\right)^2 + \left(\eta_y^{(c)}\right)^2 + \left(\eta_v^{(c)}\right)^2
    \right)
    \right)
    +\mathcal{O}\left(
    \frac{M+C}{C}
    \left(\eta_y^{(s)}\right)^2\left(
    \left(\eta_x^{(c)}\right)^2 + \left(\eta_y^{(c)}\right)^2
    \right)
    \right)
    \nonumber \\
    &\quad 
    + \mathcal{O}\left(
    \frac{M}{C}
    \left(\eta_v^{(s)}\right)^2\left(
    \left(\eta_x^{(c)}\right)^2 + \left(\eta_y^{(c)}\right)^2 + \left(\eta_v^{(c)}\right)^2
    \right)
    \right)
    + \mathcal{O}\left(
    \eta_y^{(s)}\left(
    \left(\eta_x^{(c)}\right)^2 + \left(\eta_y^{(c)}\right)^2
    \right)
    \right)
    \nonumber \\
    &\quad + \mathcal{O}\left(
    \eta_v^{(s)}\left(
    \left(\eta_x^{(c)}\right)^2 + \left(\eta_y^{(c)}\right)^2 + \left(\eta_v^{(c)}\right)^2
    \right)
    \right).
\end{align}
If we choose $\eta_x^{(s)}\sim\eta_y^{(s)}\sim\eta_v^{(s)}$, $\eta_x^{(c)}\sim\eta_y^{(c)}\sim\eta_v^{(c)}$, and consider a number of data samples on the order of $\mathcal{O}\left(\log\left(\left(\eta_x^{(s)}\right)^{-1}\right)\right)$, we obtain that
\begin{align}
    \mathbb{E}[\mathcal{P}_{r+1}]&\leq \left(
    1- \mathcal{O}\left(
    \eta_x^{(s)}
    \right)
    \right)\mathbb{E}[\mathcal{P}_r] 
    +\mathcal{O}\left(\left(\frac{M+C}{C}\right)\sqrt{\bar{K}}\left(\eta_x^{(s)}\right)^2 \right)
    +\mathcal{O}\left(\frac{M}{C}\left(\eta_x^{(c)}\right)^2\eta_x^{(s)}\right).
\end{align}
If we choose 
\begin{align}
    &\eta_x^{(s)} = \mathcal{O}\left(
    \frac{C}{\bar{K}r}
    \right),\eta_y^{(s)} = \rho_y\eta_x^{(s)}, \eta_v^{(s)} = \rho_v\eta_x^{(s)},
    \nonumber \\
    &\eta_x^{(c)} = \mathcal{O}\left(
    \frac{1}{\bar{K}r}
    \right), \eta_y^{(c)} = \mathcal{O}\left(
    \frac{1}{\bar{K}r}
    \right), \eta_v^{(c)} = \mathcal{O}\left(
    \frac{1}{\bar{K}r}
    \right),
\end{align}
based on the Robbins-Siegmund theorem \cite{robbins1971convergence}, we know
\begin{align}
    \lim_{r\rightarrow \infty}\|\bm{x}_r - \bm{x}_s\|\rightarrow 0, 
    \lim_{r\rightarrow \infty}\|\bm{y}_r - \bm{y}^*({\bm{x}_s})\|\rightarrow 0,
    \lim_{r\rightarrow \infty}\|\bm{v}_r - \bm{v}^*({\bm{x}_s})\|\rightarrow 0 
    \quad \text{almost surely},
\end{align}
and 
\begin{align}
    \mathbb{E}[\mathcal{P}_r] = \mathcal{O}\left(\frac{M+C}{\bar{K}r}
    \right)
    =\mathcal{O}\left(\frac{1}{r}
    \right).
\end{align}
\end{proof}

\section{Distance between FBPO and FBPS Points}
\begin{proof}
From the strong convexity of $g_i(\bm{x},\bm{y})$, we know that
\begin{align}
    \sum_{i\in C_r}\tilde{p}_i \underset{{\zeta\sim D_i(\bm{x}_s)}}{\mathbb{E}}\left\langle 
    \nabla g_i(\bm{x}_o,\bm{y}^*(\bm{x}_o)) - \nabla g_i(\bm{x}_o,\bm{y}^*(\bm{x}_o,\bm{x}_s)), \bm{y}^*(\bm{x}_o) - \bm{y}^*(\bm{x}_o,\bm{x}_s)
    \right\rangle
    \geq \gamma_g\|\bm{y}^*(\bm{x}_o) - \bm{y}^*(\bm{x}_o,\bm{x}_s)\|^2.
\end{align}
Based on the fact that
\begin{align}
    \sum_{i\in C_r}\tilde{p}_i \underset{{\zeta\sim D_i(\bm{x}_s)}}{\mathbb{E}}\nabla g_i(\bm{x}_o,\bm{y}^*(\bm{x}_o),\bm{x}_s);\zeta)=0,
\end{align}
and
\begin{align}
    \sum_{i\in C_r}\tilde{p}_i \underset{{\zeta\sim D_i(\bm{x}_o)}}{\mathbb{E}}\nabla g_i(\bm{x}_o,\bm{y}^*(\bm{x}_o);\zeta)=0,
\end{align}
it yields that
\begin{align}
    \sum_{i\in C_r}\tilde{p}_i \left\langle 
    \underset{{\zeta\sim D_i(\bm{x}_s)}}{\mathbb{E}}\nabla g_i(\bm{x}_o,\bm{y}^*(\bm{x}_o)) 
    - \underset{{\zeta\sim D_i(\bm{x}_o)}}{\mathbb{E}}\nabla g_i(\bm{x}_o,\bm{y}^*(\bm{x}_o)), \bm{y}^*(\bm{x}_o) - \bm{y}^*(\bm{x}_o,\bm{x}_s)
    \right\rangle
    \geq \gamma_g\|\bm{y}^*(\bm{x}_o) - \bm{y}^*(\bm{x}_o,\bm{x}_s)\|^2.\label{appendix_relation_eq1}
\end{align}
It's easy to know that 
\begin{align}
    &\sum_{i\in C_r}\tilde{p}_i \left\langle 
    \underset{{\zeta\sim D_i(\bm{x}_s)}}{\mathbb{E}}\nabla g_i(\bm{x}_o,\bm{y}^*(\bm{x}_o)) 
    - \underset{{\zeta\sim D_i(\bm{x}_o)}}{\mathbb{E}}\nabla g_i(\bm{x}_o,\bm{y}^*(\bm{x}_o)), \bm{y}^*(\bm{x}_o) - \bm{y}^*(\bm{x}_o,\bm{x}_s)
    \right\rangle\nonumber \\
    &\leq L_g^\zeta\varepsilon_d\|\bm{x}_s-\bm{x}_o\|\|\bm{y}^*(\bm{x}_o) - \bm{y}^*(\bm{x}_o,\bm{x}_s)\|.\label{appendix_relation_eq2}
\end{align}
Combining Eq. (\ref{appendix_relation_eq1}) and Eq. (\ref{appendix_relation_eq2}), we know that
\begin{align}
    \|\bm{y}^*(\bm{x}_o) - \bm{y}^*(\bm{x}_o,\bm{x}_s)\|\leq \frac{L_g^\zeta\varepsilon_d}{\gamma_g}\|\bm{x}_s-\bm{x}_o\|.\label{appendix_relation_eq3}
\end{align}
For the UL function, we have
\begin{align}
    \sum_{i\in C_r}\tilde{p}_i \underset{{\xi\sim C_i(\bm{y}^*(\bm{x}_s))}}{\mathbb{E}}
    \left(
    f_i(\bm{x}_o,\bm{y}^*(\bm{x}_o);\xi) - f_i(\bm{x}_s,\bm{y}^*(\bm{x}_s);\xi)
    \right)
    \geq \frac{\gamma_f}{2}\|\bm{x}_o - \bm{x}_s\|^2,\label{appendix_relation_eq4}
\end{align}
which is obtained by the strong convexity and optimality condition of the UL function.
Based on Assumption 3, we have
\begin{align}
    &\left\|
    \sum_{i\in C_r}\tilde{p}_i
    \underset{{\xi\sim C_i(\bm{y}^*(\bm{x}_s))}}{\mathbb{E}} f_i(\bm{x}_o,\bm{y}^*(\bm{x}_o);\xi) 
    - 
    \sum_{i\in C_r}\tilde{p}_i
    \underset{{\xi\sim C_i(\bm{y}^*(\bm{x}_o))}}{\mathbb{E}} f_i(\bm{x}_o,\bm{y}^*(\bm{x}_o);\xi) 
    \right\|\nonumber \\
    &\leq L_f^\xi \varepsilon_c \|\bm{y}^*(\bm{x}_s)-\bm{y}^*(\bm{x}_o)\|\nonumber \\
    &\leq L_f^\xi \varepsilon_c \|\bm{y}^*(\bm{x}_s)-\bm{y}^*(\bm{x}_o,\bm{x}_s)+\bm{y}^*(\bm{x}_o,\bm{x}_s)-\bm{y}^*(\bm{x}_o)\|\nonumber \\
    &\overset{(a)}{\leq}L_f^\xi \varepsilon_c 
    \frac{C_g^{xy} + L_g^\zeta \varepsilon_d}{\gamma_g}
    \|\bm{x}_s - \bm{x}_o\|,\label{appendix_relation_eq5}
\end{align}
where $(a)$ holds by Eq. (\ref{appendix_relation_eq3}).
Note that 
\begin{align}
    &\sum_{i\in C_r}\tilde{p}_i\underset{{\xi\sim C_i(\bm{y}^*(\bm{x}_s))}}{\mathbb{E}}
    \left(
    f_i(\bm{x}_o,\bm{y}^*(\bm{x}_o);\xi) - f_i(\bm{x}_o,\bm{y}^*(\bm{x}_s);\xi)
    \right)\nonumber \\
    &\leq 
    \sum_{i\in C_r}\tilde{p}_i
    \underset{{\xi\sim C_i(\bm{y}^*(\bm{x}_s))}}{\mathbb{E}}
    f_i (\bm{x}_o,\bm{y}^*(\bm{x}_o);\xi)
    - \sum_{i\in C_r}\tilde{p}_i
    \underset{{\xi\sim C_i(\bm{y}^*(\bm{x}_o))}}{\mathbb{E}}
    f_i (\bm{x}_o,\bm{y}^*(\bm{x}_o);\xi),\label{appendix_relation_eq6}
\end{align}
which is based on the fact that 
\begin{align}
    \sum_{i\in C_r}\tilde{p}_i\underset{{\xi\sim C_i(\bm{y}^*(\bm{x}_o))}}{\mathbb{E}}f_i (\bm{x}_o,\bm{y}^*(\bm{x}_o);\xi)
    \leq 
    \sum_{i\in C_r}\tilde{p}_i\underset{{\xi\sim C_i(\bm{y}^*(\bm{x}_s))}}{\mathbb{E}}f_i (\bm{x}_o,\bm{y}^*(\bm{x}_s);\xi).
\end{align}
Combining Eq. (\ref{appendix_relation_eq4}), Eq. (\ref{appendix_relation_eq5}), and Eq. (\ref{appendix_relation_eq6}), we obtain
\begin{align}
    \|\bm{x}_o - \bm{x}_s\| \leq \frac{2L_f^\xi \varepsilon_c \left(C_g^{xy} + L_g^\zeta \varepsilon_d\right)}{\gamma_f \gamma_g}.
\end{align}



\end{proof}

\section{Distance between FBPO and FBPS Points under Heterogeneous Sensitivities}
\begin{proof}
We follow the same proof process as the homogeneous case, and only track the changes induced by the heterogeneous sensitivities $\{\varepsilon_{i,c},\varepsilon_{i,d}\}_{i\in C_r}$. Define the weighted averages
\begin{equation}\label{eq:bar_eps_cd_def}
\bar{\varepsilon}_c:=\sum_{i\in C_r}\tilde{p}_i\,\varepsilon_{i,c},
\qquad
\bar{\varepsilon}_d:=\sum_{i\in C_r}\tilde{p}_i\,\varepsilon_{i,d}.
\end{equation}

From the strong convexity of $g_i(\bm{x},\bm{y})$ in $\bm{y}$, we have
\begin{equation}\label{eq:het_relation_eq1}
\sum_{i\in C_r}\tilde{p}_i \mathbb{E}_{\zeta\sim D_i(\bm{x}_s)}
\left\langle 
\nabla g_i(\bm{x}_o,\bm{y}^*(\bm{x}_o)) - \nabla g_i(\bm{x}_o,\bm{y}^*(\bm{x}_o,\bm{x}_s)),\,
\bm{y}^*(\bm{x}_o) - \bm{y}^*(\bm{x}_o,\bm{x}_s)
\right\rangle
\geq \gamma_g\|\bm{y}^*(\bm{x}_o) - \bm{y}^*(\bm{x}_o,\bm{x}_s)\|^2.
\end{equation}
Using the two LL optimality conditions
\begin{equation}\label{eq:het_opt1}
\sum_{i\in C_r}\tilde{p}_i \mathbb{E}_{\zeta\sim D_i(\bm{x}_s)}\nabla g_i(\bm{x}_o,\bm{y}^*(\bm{x}_o,\bm{x}_s);\zeta)=0,
\end{equation}
and
\begin{equation}\label{eq:het_opt2}
\sum_{i\in C_r}\tilde{p}_i \mathbb{E}_{\zeta\sim D_i(\bm{x}_o)}\nabla g_i(\bm{x}_o,\bm{y}^*(\bm{x}_o);\zeta)=0,
\end{equation}
Eq.~\eqref{eq:het_relation_eq1} can be rewritten as
\begin{align}\label{eq:het_relation_eq1b}
&\sum_{i\in C_r}\tilde{p}_i \left\langle 
\mathbb{E}_{\zeta\sim D_i(\bm{x}_s)}\nabla g_i(\bm{x}_o,\bm{y}^*(\bm{x}_o);\zeta) 
- \mathbb{E}_{\zeta\sim D_i(\bm{x}_o)}\nabla g_i(\bm{x}_o,\bm{y}^*(\bm{x}_o);\zeta),\,
\bm{y}^*(\bm{x}_o) - \bm{y}^*(\bm{x}_o,\bm{x}_s)
\right\rangle\nonumber \\
&\geq \gamma_g\|\bm{y}^*(\bm{x}_o) - \bm{y}^*(\bm{x}_o,\bm{x}_s)\|^2.
\end{align}

Under the heterogeneous LL sensitivity assumption, for each client $i$,
\begin{equation}\label{eq:het_LL_shift_i}
\left\|
\mathbb{E}_{\zeta\sim D_i(\bm{x}_s)}\nabla g_i(\bm{x}_o,\bm{y}^*(\bm{x}_o);\zeta) 
- \mathbb{E}_{\zeta\sim D_i(\bm{x}_o)}\nabla g_i(\bm{x}_o,\bm{y}^*(\bm{x}_o);\zeta)
\right\|
\leq L_g^\zeta\,\varepsilon_{i,d}\,\|\bm{x}_s-\bm{x}_o\|.
\end{equation}
Therefore, by Cauchy--Schwarz inequality and \eqref{eq:bar_eps_cd_def},
\begin{align}\label{eq:het_relation_eq2}
&\sum_{i\in C_r}\tilde{p}_i \left\langle 
\mathbb{E}_{\zeta\sim D_i(\bm{x}_s)}\nabla g_i(\bm{x}_o,\bm{y}^*(\bm{x}_o);\zeta) 
- \mathbb{E}_{\zeta\sim D_i(\bm{x}_o)}\nabla g_i(\bm{x}_o,\bm{y}^*(\bm{x}_o);\zeta),\,
\bm{y}^*(\bm{x}_o) - \bm{y}^*(\bm{x}_o,\bm{x}_s)
\right\rangle\nonumber \\
&\leq L_g^\zeta\,\bar{\varepsilon}_d\,\|\bm{x}_s-\bm{x}_o\|\,\|\bm{y}^*(\bm{x}_o) - \bm{y}^*(\bm{x}_o,\bm{x}_s)\|.
\end{align}
Combining \eqref{eq:het_relation_eq1b} and \eqref{eq:het_relation_eq2} yields
\begin{equation}\label{eq:het_relation_eq3}
\|\bm{y}^*(\bm{x}_o) - \bm{y}^*(\bm{x}_o,\bm{x}_s)\|
\leq \frac{L_g^\zeta\,\bar{\varepsilon}_d}{\gamma_g}\|\bm{x}_s-\bm{x}_o\|.
\end{equation}

For the UL function, strong convexity and the optimality condition imply
\begin{equation}\label{eq:het_relation_eq4}
\sum_{i\in C_r}\tilde{p}_i \mathbb{E}_{\xi\sim C_i(\bm{y}^*(\bm{x}_s))}
\left(
f_i(\bm{x}_o,\bm{y}^*(\bm{x}_o);\xi) - f_i(\bm{x}_s,\bm{y}^*(\bm{x}_s);\xi)
\right)
\geq \frac{\gamma_f}{2}\|\bm{x}_o - \bm{x}_s\|^2.
\end{equation}

Under the heterogeneous UL sensitivity assumption, for each client $i$,
\begin{equation}\label{eq:het_UL_shift_i}
\left\|
\mathbb{E}_{\xi\sim C_i(\bm{y}^*(\bm{x}_s))} f_i(\bm{x}_o,\bm{y}^*(\bm{x}_o);\xi) 
- \mathbb{E}_{\xi\sim C_i(\bm{y}^*(\bm{x}_o))} f_i(\bm{x}_o,\bm{y}^*(\bm{x}_o);\xi)
\right\|
\leq L_f^\xi\,\varepsilon_{i,c}\,\|\bm{y}^*(\bm{x}_s)-\bm{y}^*(\bm{x}_o)\|.
\end{equation}
Summing \eqref{eq:het_UL_shift_i} over $i\in C_r$ with weights $\tilde p_i$ and using \eqref{eq:bar_eps_cd_def} gives
\begin{equation}\label{eq:het_UL_shift_sum}
\left\|
\sum_{i\in C_r}\tilde{p}_i
\mathbb{E}_{\xi\sim C_i(\bm{y}^*(\bm{x}_s))} f_i(\bm{x}_o,\bm{y}^*(\bm{x}_o);\xi) 
-
\sum_{i\in C_r}\tilde{p}_i
\mathbb{E}_{\xi\sim C_i(\bm{y}^*(\bm{x}_o))} f_i(\bm{x}_o,\bm{y}^*(\bm{x}_o);\xi) 
\right\|
\leq L_f^\xi\,\bar{\varepsilon}_c\,\|\bm{y}^*(\bm{x}_s)-\bm{y}^*(\bm{x}_o)\|.
\end{equation}

Moreover,
\begin{equation}\label{eq:het_y_triangle}
\|\bm{y}^*(\bm{x}_s)-\bm{y}^*(\bm{x}_o)\|
\leq
\|\bm{y}^*(\bm{x}_s)-\bm{y}^*(\bm{x}_o,\bm{x}_s)\|
+
\|\bm{y}^*(\bm{x}_o,\bm{x}_s)-\bm{y}^*(\bm{x}_o)\|.
\end{equation}
Using the LL sensitivity w.r.t.\ $\bm{x}$ (as in the homogeneous proof),
\begin{equation}\label{eq:het_y_xsens}
\|\bm{y}^*(\bm{x}_s)-\bm{y}^*(\bm{x}_o,\bm{x}_s)\|
\leq \frac{C_g^{xy}}{\gamma_g}\|\bm{x}_s-\bm{x}_o\|,
\end{equation}
and applying \eqref{eq:het_relation_eq3} to the second term in \eqref{eq:het_y_triangle}, we obtain
\begin{equation}\label{eq:het_y_total}
\|\bm{y}^*(\bm{x}_s)-\bm{y}^*(\bm{x}_o)\|
\leq
\frac{C_g^{xy}+L_g^\zeta\,\bar{\varepsilon}_d}{\gamma_g}\|\bm{x}_s-\bm{x}_o\|.
\end{equation}

Plugging \eqref{eq:het_y_total} into \eqref{eq:het_UL_shift_sum} yields
\begin{equation}\label{eq:het_relation_eq5}
\left\|
\sum_{i\in C_r}\tilde{p}_i
\mathbb{E}_{\xi\sim C_i(\bm{y}^*(\bm{x}_s))} f_i(\bm{x}_o,\bm{y}^*(\bm{x}_o);\xi) 
-
\sum_{i\in C_r}\tilde{p}_i
\mathbb{E}_{\xi\sim C_i(\bm{y}^*(\bm{x}_o))} f_i(\bm{x}_o,\bm{y}^*(\bm{x}_o);\xi) 
\right\|
\leq
L_f^\xi\,\bar{\varepsilon}_c\,
\frac{C_g^{xy}+L_g^\zeta\,\bar{\varepsilon}_d}{\gamma_g}\,
\|\bm{x}_s-\bm{x}_o\|.
\end{equation}

Finally, using the same monotonicity argument as Eq.~(\ref{appendix_relation_eq6}) in the homogeneous proof, and combining \eqref{eq:het_relation_eq4} with \eqref{eq:het_relation_eq5}, we obtain
\begin{equation}
\frac{\gamma_f}{2}\|\bm{x}_o - \bm{x}_s\|^2
\leq
L_f^\xi\,\bar{\varepsilon}_c\,
\frac{C_g^{xy}+L_g^\zeta\,\bar{\varepsilon}_d}{\gamma_g}\,
\|\bm{x}_s-\bm{x}_o\|.
\end{equation}
Cancelling $\|\bm{x}_s-\bm{x}_o\|$ on both sides gives
\begin{equation}
\|\bm{x}_o - \bm{x}_s\|
\leq
\frac{2L_f^\xi\,\bar{\varepsilon}_c\left(C_g^{xy}+L_g^\zeta\,\bar{\varepsilon}_d\right)}{\gamma_f\gamma_g}.
\end{equation}
This completes the proof.
\end{proof}

\section{Limitation and Future Work}
Our analysis adopts strong convexity assumptions on the UL/LL objectives and requires sufficiently small performative sensitivities to establish the existence and uniqueness of the FBPS point and derive explicit contraction-based convergence guarantees. These assumptions are standard in the theoretical study of performative prediction and bilevel optimization, as they provide a tractable foundation for characterizing stability under decision-dependent distribution shifts. Nevertheless, they may not fully capture practical federated deep learning systems, where neural network objectives are typically nonconvex, client data are highly heterogeneous, and model deployment may induce stronger or more complex feedback effects. In such settings, our theoretical results should be interpreted as sufficient stability conditions rather than necessary guarantees. Empirically, our CNN-based simulation provides preliminary evidence that the proposed methods can remain effective beyond the convex regime, but a complete theory for nonconvex federated bilevel performative learning remains an important direction for future work.

In addition, performative prediction models can influence user behavior, data collection, and downstream system dynamics after deployment. In real-world federated systems, such feedback may introduce behavioral or system-level risks, including distributional amplification, unequal effects across clients, instability under rapid user adaptation, and potential misalignment between optimized objectives and user welfare. Therefore, practical deployment of federated bilevel performative methods should consider monitoring mechanisms, sensitivity control, fairness across clients, and safeguards against harmful feedback loops. Future work will extend the proposed framework to fully distributed bilevel performative prediction and investigate stability, robustness, and broader impact considerations in more realistic deployment environments. It is also of interest to study robustness to misspecified sensitivities and to incorporate privacy mechanisms (e.g., differential privacy) without sacrificing performative stability.


\end{document}